%% file: main.tex
\documentclass{article}[12pt]

\usepackage[margin=1.1in]{geometry}

\usepackage[utf8]{inputenc} %
\usepackage[T1]{fontenc}    %
\usepackage[hidelinks,colorlinks,linkcolor=black,citecolor=orange]{hyperref}  
 \hypersetup{
     colorlinks=true,
     linkcolor=blue,
     filecolor=blue,    
     urlcolor=cyan,
     }
\usepackage{placeins}
\usepackage{url}            %

\usepackage{booktabs}       %
\usepackage{amsfonts}       %
\usepackage{nicefrac}       %
\usepackage{microtype}      %
\usepackage{xcolor}         %
\usepackage[sortcites=true,citestyle=numeric,natbib=true,bibstyle=numeric,maxbibnames=99]{biblatex}

\usepackage{tikz}
\usetikzlibrary{decorations.pathreplacing}
\usetikzlibrary{positioning}
\usepackage{makecell}
\usepackage{graphicx}
\usepackage{caption}
\usepackage{subcaption}
\usepackage{amsmath,amssymb,amsthm}
\usepackage{bbm}
\usepackage{enumitem}   
\usepackage{capt-of}
\usepackage{wrapfig}
\usepackage{algorithm}
\usepackage{algpseudocode}
\newcommand{\upperR}[1]{\uppercase\expandafter{\romannumeral#1}}
\usepackage{xcolor}                %
\usepackage[most]{tcolorbox}             %
\usepackage{titletoc}  
\usepackage{relsize}
\usepackage{soul}

\tcbuselibrary{theorems}

\newtcbtheorem{defn}{Definition}%
{	colframe=blue!30!gray, theorem name, fonttitle=\bfseries,
	colback=blue!5, top=1mm,bottom=1mm, boxrule=0.5pt}{Example}

\newtcbtheorem[use counter=theorem, number within=section]{THM}{Theorem}%
{	colframe=blue!30!gray, fonttitle=\bfseries,
	colback=blue!5,  fontupper=\itshape, top=1mm,bottom=1mm, boxrule=0.5pt}{Theorem}
\usepackage{varwidth}
\newtcbox{\mybox}{colback=blue!5,
	colframe=blue!30!black, center, enhanced, varwidth upper}
\newtcbox{\mymath}[1][]{%
	nobeforeafter, math upper, tcbox raise base,
	enhanced, colframe=blue!30!black,
	colback=blue!5, boxrule=0.5pt, top=1mm,bottom=1mm,
	#1}

\input{math_commands.tex}
\usepackage[normalem]{ulem}

\makeatletter
\@removefromreset{theorem}{section}
\makeatother

\title{Iteratively reweighted kernel machines \\efficiently learn sparse functions%
}
\author{
Libin Zhu\thanks{Department of Mathematics, University of Washington, Seattle, WA 98195; \texttt{https://libinzhu.github.io/}. Research of Zhu was supported by NSF DMS-2023166.}
\and
Damek Davis\thanks{Wharton Department of Statistics and Data Science, University of Pennsylvania,
		Philadelphia, PA 19104, USA;
		\texttt{www.damekdavis.com}. Research of Davis supported by an Alfred P. Sloan research fellowship and NSF DMS award 2047637}
\and
Dmitriy Drusvyatskiy\thanks{Department of Mathematics, U. Washington, Seattle, WA 98195; \texttt{www.math.washington.edu/$\sim$ddrusv}.
		Research of Drusvyatskiy was supported by NSF DMS-2306322, NSF DMS-2023166, and AFOSR FA9550-24-1-0092 awards.}
\and 
Maryam Fazel\thanks{Department of Electrical \& Computer Engineering, University of Washington, Seattle, WA 98195; \texttt{https://people.ece.uw.edu/fazel\_maryam/}. Research supported by NSF TRIPODS II DMS-2023166, CCF 2007036, CCF 2212261, and CCF 2312775.}
}

\date{May 2025}

\begin{document}
\maketitle
\begin{abstract}
The impressive practical performance of neural networks is often attributed to their ability to learn low-dimensional data representations and hierarchical structure directly from data. In this work, we argue that these two phenomena are not unique to neural networks, and can be elicited from classical kernel methods. Namely, we show that the derivative of the kernel predictor can detect the influential coordinates with low sample complexity. Moreover, by iteratively using the derivatives to reweight the data and retrain kernel machines, one is able to efficiently learn hierarchical polynomials with finite leap complexity. Numerical experiments illustrate the developed theory.
\end{abstract}

\input{sections/introduction}

\input{sections/prelim_ortho_poly}
\input{sections/main_results}

\input{sections/proof_sketch}

\input{sections/numerical}

\input{sections/conclusion}

\printbibliography
\newpage

\appendix
\section*{Appendix}
\startcontents[apx] 
\printcontents[apx]{l}{1}{\section*{Contents}}
\input{sections/appendix_prelim}
\input{sections/appendix_A}
\input{sections/appendix_B}

\input{sections/appendix_C}
\newpage

\end{document}

%% file: math_commands.tex
\usepackage{amsmath,amsfonts,bm}

\def\ceil#1{\lceil #1 \rceil}
\def\floor#1{\lfloor #1 \rfloor}
\def\1{\bm{1}}

\def\rva{{\mathbf{a}}}

\def\rvc{{\mathbf{c}}}

\def\vm{{\bm{m}}}

\def\vs{{\bm{s}}}

\def\vy{{\bm{y}}}

\def\mT{{\bm{T}}}

\DeclareMathAlphabet{\mathsfit}{\encodingdefault}{\sfdefault}{m}{sl}
\SetMathAlphabet{\mathsfit}{bold}{\encodingdefault}{\sfdefault}{bx}{n}

\def\gD{{\mathcal{D}}}
\def\gE{{\mathcal{E}}}

\newcommand{\E}{\mathbb{E}}

\renewcommand{\H}{\mathbb{H}}
\newcommand{\R}{\mathbb{R}}
\newcommand{\emp}{{\widetilde{p}}}

\DeclareMathOperator*{\argmax}{arg\,max}

\DeclareMathOperator{\Tr}{Tr}

\renewcommand\S{{\mathcal{S}}}

\newcommand\A{{\mathcal{A}}}
\newcommand\C{{\mathcal{C}}}

\newcommand\D{{\mathcal{D}}}
\newcommand\U{{\mathcal{U}}}

\renewcommand\R{{\mathbb{R}}}
\renewcommand\P{{\mathbb{P}}}

\newcommand\lm{{{w}}}

\newcommand\Diag{{\mathsf{Diag}}}
\newcommand\offd{{\mathsf{off}}}

\newcommand\diag{{\mathsf{diag}}}

\usepackage{physics}

\newcommand\hd{{\widehat{D}}}

\newcommand\I{{\mathcal{I}}}

\newcommand\J{{\mathcal{J}}}

\newcommand\Q{\mathcal{Q}}
\newcommand\T{{\mathcal{T}}}

\newcommand{\inner}[1]{\left<#1\right>}
\renewcommand{\norm}[1]{\left\|#1\right\|}
\newcommand{\snorm}[1]{\left\|#1\right\|_{\mathrm{op}}}

\newcommand{\round}[1]{\left(#1\right)}
\newcommand{\brac}[1]{\left[#1\right]}
\renewcommand{\abs}[1]{\left|#1\right|}
\newcommand{\biground}[1]{\left\{#1\right\}}

\newtheorem{theorem}{Theorem}
\newtheorem{lemma}[theorem]{Lemma}
\newtheorem{corollary}[theorem]{Corollary}
\newtheorem{claim}{Claim}
\newtheorem{proposition}[theorem]{Proposition}

\newtheorem{definition}{Definition}

\newtheorem{assumption}{Assumption}

\newcommand{\tran}{^\top}
\newcommand{\inv}{^{-1}}
\newcommand{\half}{^{-\frac{1}{2}}}
\newcommand{\halfe}{^{-\frac{1}{2}-\epsilon}}
\newcommand{\alp}{{\A_{\leq p_\delta}}}
\newcommand{\aslp}{{\A_{< p_\delta}}}
\newcommand{\aep}{{\A_{= p_\delta}}}

\newcommand{\agp}{{\A_{> p_\delta}}}

\newcommand{\vepsilon}{{\bm{\varepsilon}}}
\newcommand{\vlambda}{{\bm{\lambda}}}
\newcommand{\lmlp}{{\vlambda\tran\vm \leq p+\delta}}
\newcommand{\lmgp}{{\vlambda\tran\vm > p+\delta}}

\newcommand{\mgamma}{{\gamma_{\max}}}
\newcommand{\vgamma}{{\bm{\gamma}}}
\newcommand{\epsilonS}{{\epsilon_{\rm s}}}
\newcommand{\midcup}{%
  \mathop{\mathlarger{\mathlarger{\cup}}}\limits}
\newcommand{\aux}{{\rm{aux}}}

%% file: sections/introduction.tex
\section{Introduction}

Over the past decade, deep learning has revolutionized the field of artificial intelligence, driving remarkable advances in areas such as computer vision, natural language processing,
reinforcement learning and robotics. 
While impressive empirical achievements are apparent, understanding {\em why} deep learning methods are so successful remains an active area of research. In particular, it is now believed that two closely related phenomena---called {\em feature learning} and {\em hierarchical learning}---largely underpin deep learning's impressive performance.

 Heuristically, feature learning refers to the fact that gradient methods for training neural networks automatically filter out much of the irrelevant information in the data and provide a concise summary of only the most salient aspects for the learning task at hand. As a rudimentary example, detecting whether a person in an image is wearing a necktie only requires inspecting the neck area. Gradient methods for this classification task are able to automatically filter out the irrelevant pixels in the image by learning appropriate weights in the network. Importantly, feature learning implies that the trained neural network automatically performs a kind of nonlinear dimension reduction, which then leads to better generalization. Feature learning is a surprising phenomenon because it suggests that learning a data-dependent compression mechanism is an easier problem than making predictions directly.
 See, for example, papers \cite{damian2022neural,ba2022high} which explore feature learning in neural networks. 
 
 The second phenomenon is rooted in the observation that typical learning tasks are hierarchical, with features appearing in a natural-order of significance. For example, an airplane in an image may be detected by first %
 learning simpler concepts like `doors', `wheels', `wings', and so forth. There is now ample evidence that gradient methods for training neural networks are able to gradually learn these intermediate concepts, which significantly improves the overall sample efficiency. We refer the reader to \cite{allen2023backward,bengio2009learning,Goodfellow-et-al-2016,zeiler2014visualizing} for details.

Although recent work in AI has focused almost exclusively on neural networks, these models do have downsides. Neural networks are trained by running gradient-based methods on highly nonlinear and nonconvex functions. As a result, the training process can be brittle, requiring significant computational resources for tuning algorithmic parameters such as step-size, batch-size, initialization scale, momentum terms, etc. Consequently, it is worthwhile to ask whether simpler and more interpretable models can be as effective at learning features and hierarchical structure. One apparent candidate is kernel ridge regression---the preeminent model in machine learning prior to deep learning's ascension. %

\mybox{In this work, we argue that both phenomena of feature learning and hierarchical learning are not unique to neural networks, and can be elicited by iteratively retraining classical kernel machines.}

\noindent Consequently, our paper suggests that kernel methods, appropriately instantiated, are much more powerful than previously thought and are worth revisiting in contemporary settings. The algorithms we consider are largely inspired by the classical iteratively reweighted least squares (IRLS) for sparse recovery \cite{daubechies2010iteratively} and the recently introduced recursive feature machines (RFM) algorithm \cite{rfm_science} for learning multi-index models. Indeed, the retraining process we introduce is closely related to the ``diagonal'' variant of RFM for learning sparse non-linear models.

Importantly, the kernel-based algorithms we consider make no distributional assumptions on the data and can be applied on any supervised learning task. That being said, the theoretical guarantees we develop rely on strong distributional assumptions such as Gaussianity 
or uniform sampling on the hypercube. The reader should regard these assumptions as a ``simplification of reality'' that allows us to theoretically justify the use of an algorithm in more practical settings. This is in contrast for example to Fourier-based techniques for learning on the hypercube \cite{o2014analysis,mansour1994learning}, which are tailor-made for this specific task. In this sense, our work is in line with the vast literature on %
compressed sensing \cite{candes2006robust,donoho2006compressed,donoho2006most,candes2006stable}, structured signal recovery \cite{recht2010guaranteed,wright2022high,chandrasekaran2012convex,amelunxen2014living}, and gradient dynamics for neural networks \cite{damian2022neural,xu2023over}, 
where generative data assumptions are imposed  in order to support the use of a general purpose algorithm in practice.

\subsection{Summary of results}
Setting the stage, consider the problem of learning a degree $\ell$ polynomial $f$ on $\R^d$ from $n$ samples $(x^{(i)},y_{i})$ with noisy labels, 
$$y_i=f(x^{(i)})+\varepsilon_i \qquad \mbox{for} \quad i=1,\ldots,n,$$ 
where the feature vectors $x^{(i)}$ are drawn independently from some distribution $\mu$ on $\R^d$ and the noise terms are independent Gaussian $\varepsilon_i\sim N(0,\sigma_\varepsilon^2)$. 
Since we aim to learn a degree $\ell$ polynomial, we focus on the high-dimensional regime where $n$ scales polynomially in $d$, that is $n=d^{p+\delta}$ for some small constant $\delta\in (0,1)$ %
and an integer $p$.
Note that the interesting regime is when the power $p$ is smaller, and in some cases {\em much smaller}, than the degree $\ell$ of the target polynomial.

Now fix a positive semidefinite inner product kernel and let $\hat f$ be the Kernel Ridge Regression (KRR) estimator fit to the data set. 
A central question is to understand the generalization error $\|\hat{f}-f\|^2_{L_2(\mu)}$ of the prediction function. 
Since $p$ can be smaller than the degree of the target function, the generalization error will typically not tend to zero as $d$ increases. Nonetheless, in a number of interesting cases, the error can be controlled based on an expansion of $f$ in a distinguished orthonormal polynomial basis of $L_2(\mu)$. For example, for Gaussian, spherical, and hypercube data, the relevant polynomial bases are comprised of Hermite, Gaugenbauer, and Fourier-Walsh polynomials, respectively. In particular, the seminal work \cite{ghorbani2021linearized} showed that for the spherical data, the generalization error of $\hat f$ satisfies
\begin{equation}\label{eqn:montanari}
    \|\hat{f}-f\|^2_{L_2(\mu)}=\|f_{>p}\|^2_{L_2(\mu)}+ o_{d,\P}(1),
\end{equation}
where $f_{>p}$ is the function obtained by keeping the higher-order terms (the polynomial bases of degree greater than $p$) in the expansion of $f$. %
Further extensions of the estimate \eqref{eqn:montanari} appear in \cite{mei2022generalization,}. 

 Looking at equation~\eqref{eqn:montanari}, we see that $\hat f$ asymptotically recovers $f$ only in the regime $p\geq \ell$. 
In particular, it rules out favorable performance of kernel methods even if $f$ has ``low intrinsic complexity'', e.g., dependence on only a few coordinates. 
In this work, we show how to overcome this limitation in the case of such ``sparse'' functions. The method we propose alternates between two stages: (1) ``active'' coordinate detection and (2) kernel ridge regression on reweighted data. %

\subsubsection{Sparse feature learning.}
As a motivating example, Figure~\ref{fig:gen_error_cube} (blue curve) depicts the generalization error of the KRR predictor $\hat f$ when learning the degree three polynomial with a Laplace kernel and uniform sampling on the hypercube:
\begin{equation}\label{eqn:first_exa}
f(x)=x_1 + x_2 + x_3 +x_1x_2x_3\qquad \textrm{for}~x\in\R^{d}.
\end{equation}
 Even though this polynomial depends only on the first three coordinates and therefore has low intrinsic complexity, the generalization error remains large even in the regime $n=d^2$ as implied by \eqref{eqn:montanari}.  %
Thus, kernel methods are not adaptive to the low complexity of the function that is being learned. This observation often serves as the reason to discard kernel methods and instead focus on neural networks. Interestingly, when learning the cubic polynomial  \eqref{eqn:first_exa}, the KRR predictor $\hat f$ performs comparably to neural networks (orange curve) trained by Adam in the low sample regime $d^{0.3}\leq n\leq d^{1.0}$ and significantly worse in the high sample regime $d^{1.2}\leq n\leq d^{1.3}$ (Figure~\ref{fig:gen_error_cube}). %

\begin{figure}[h!]
    \centering
    \begin{subfigure}[b]{0.45\linewidth}
        \centering
        \includegraphics[width=\linewidth]{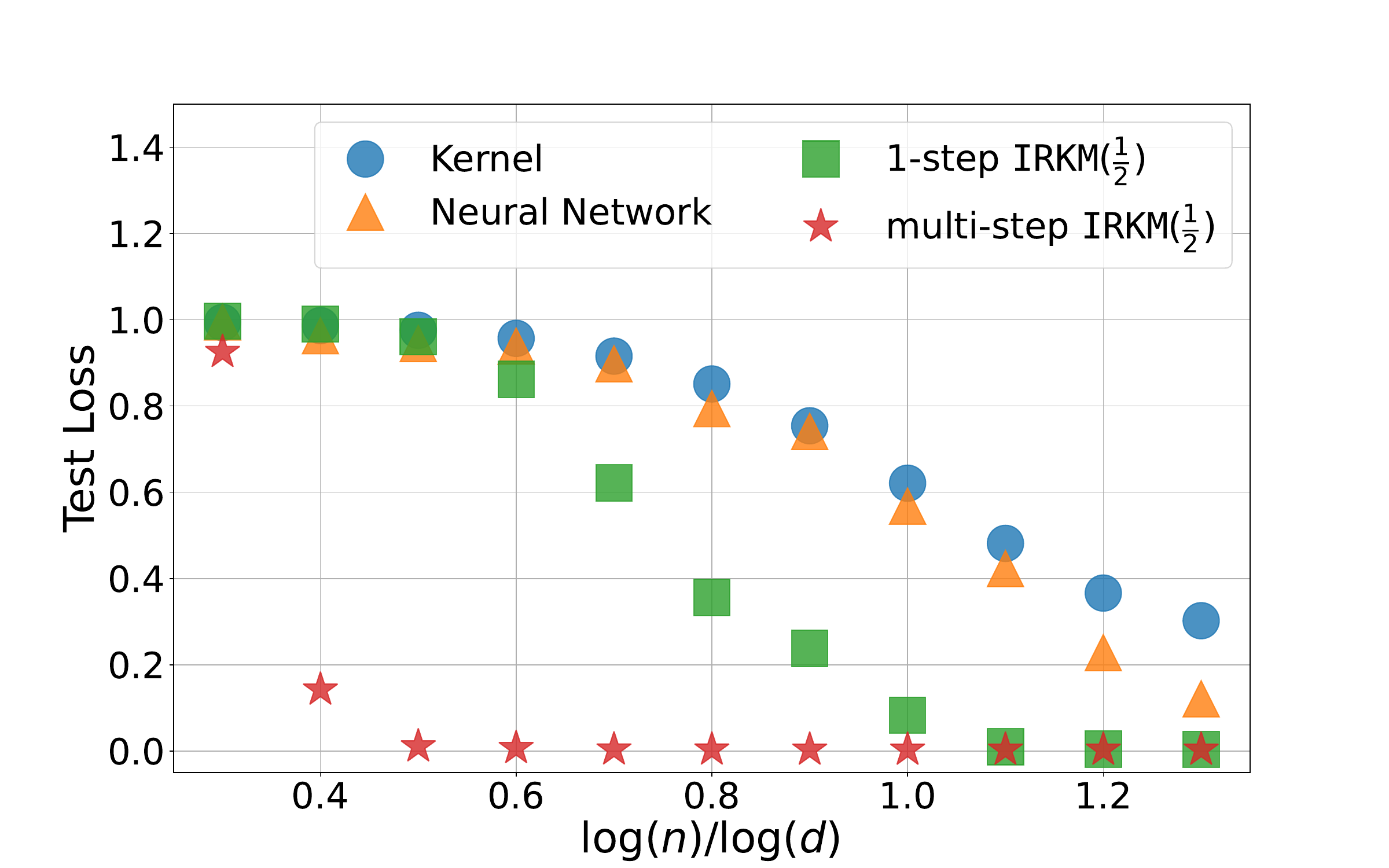}
        \caption{Generalization error\label{fig:gen_error_cube}}
    \end{subfigure}
            \begin{subfigure}[b]{0.45\linewidth}
        \centering
        \includegraphics[width=\linewidth]{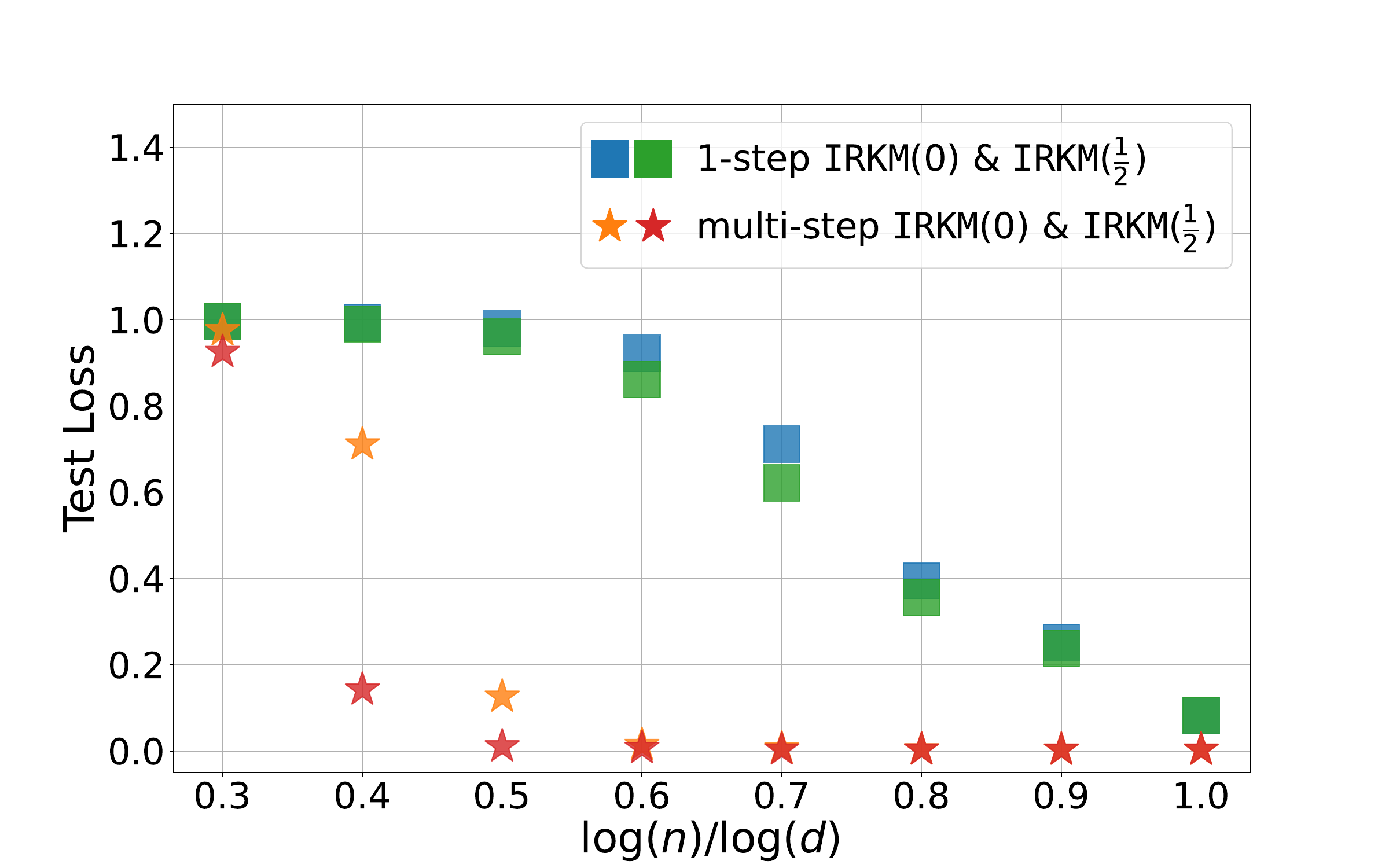}
        \caption{$\alpha = 0$ vs.  $\alpha = \tfrac{1}{2}$\label{fig:compare_task1}}
    \end{subfigure}\\
        \begin{subfigure}[b]{0.23\linewidth}
        \centering
        \includegraphics[width=\linewidth]{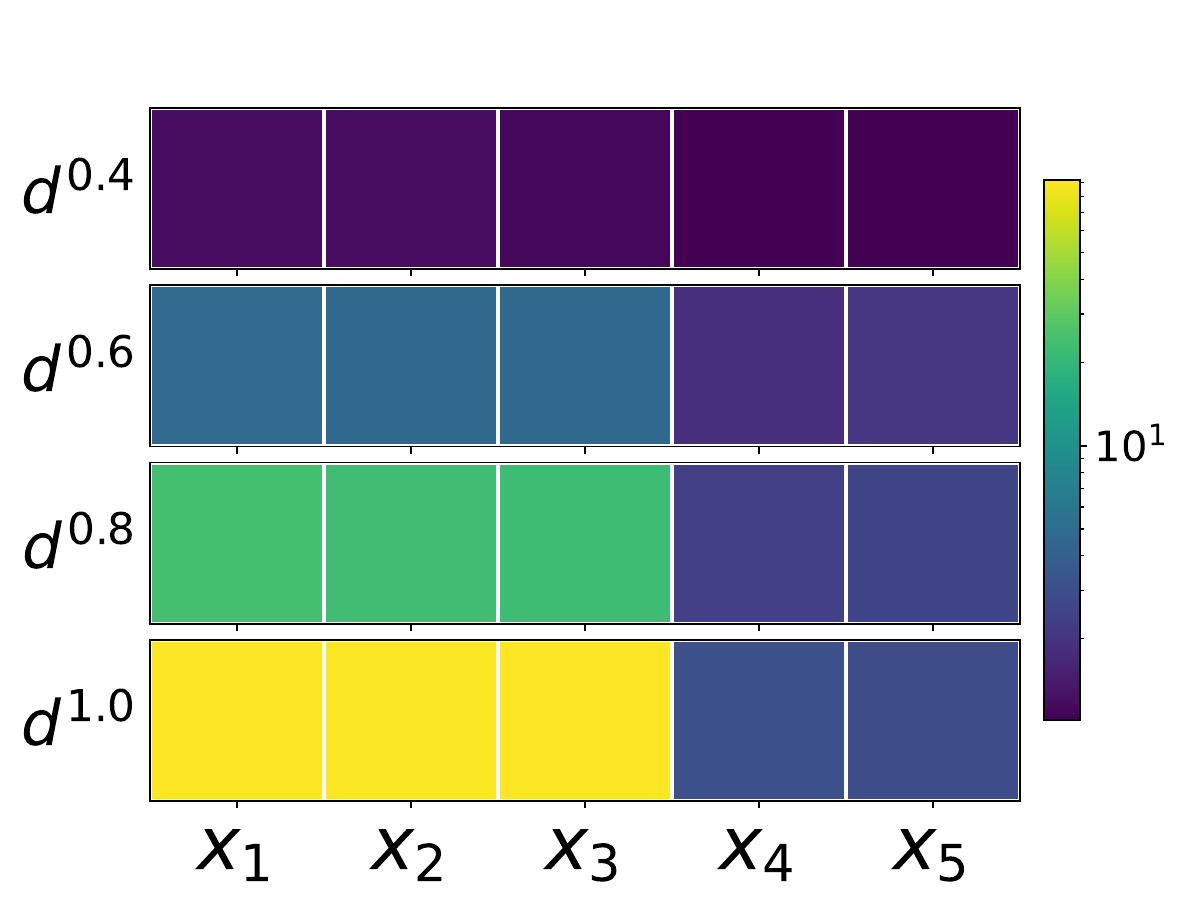}
    \caption{One step  $(\alpha=0)$\label{One step:coordinate_identification0}}
    \end{subfigure}
    \begin{subfigure}[b]{0.23\linewidth}
        \centering
        \includegraphics[width=\linewidth]{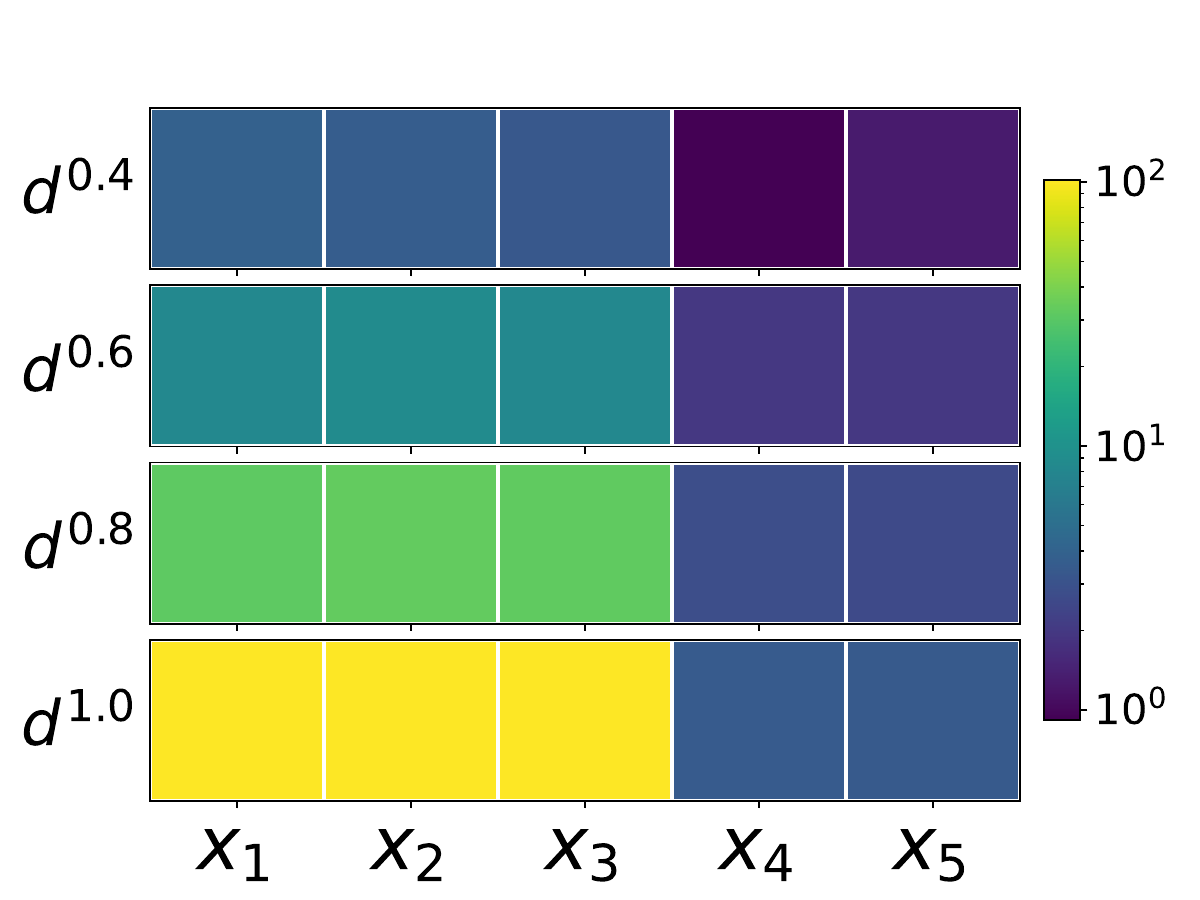}
    \caption{One step $(\alpha=\tfrac{1}{2})$\label{One step:coordinate_identification}}
    \end{subfigure}
        \begin{subfigure}[b]{0.23\linewidth}
        \centering
        \includegraphics[width=\linewidth]{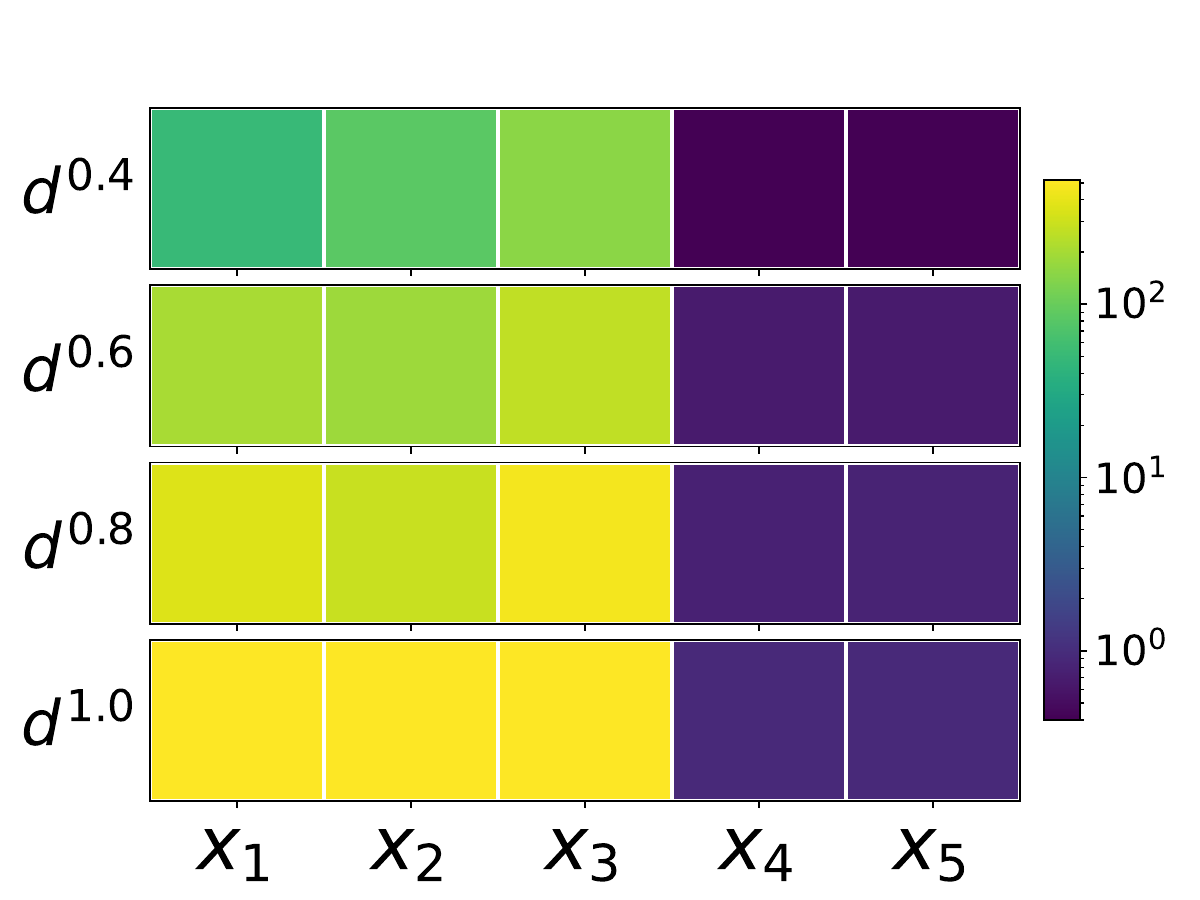}
    \caption{Multi-step  $(\alpha=0)$\label{One step:coordinate_identification_krr}}
    \end{subfigure}
            \begin{subfigure}[b]{0.23\linewidth}
        \centering
        \includegraphics[width=\linewidth]{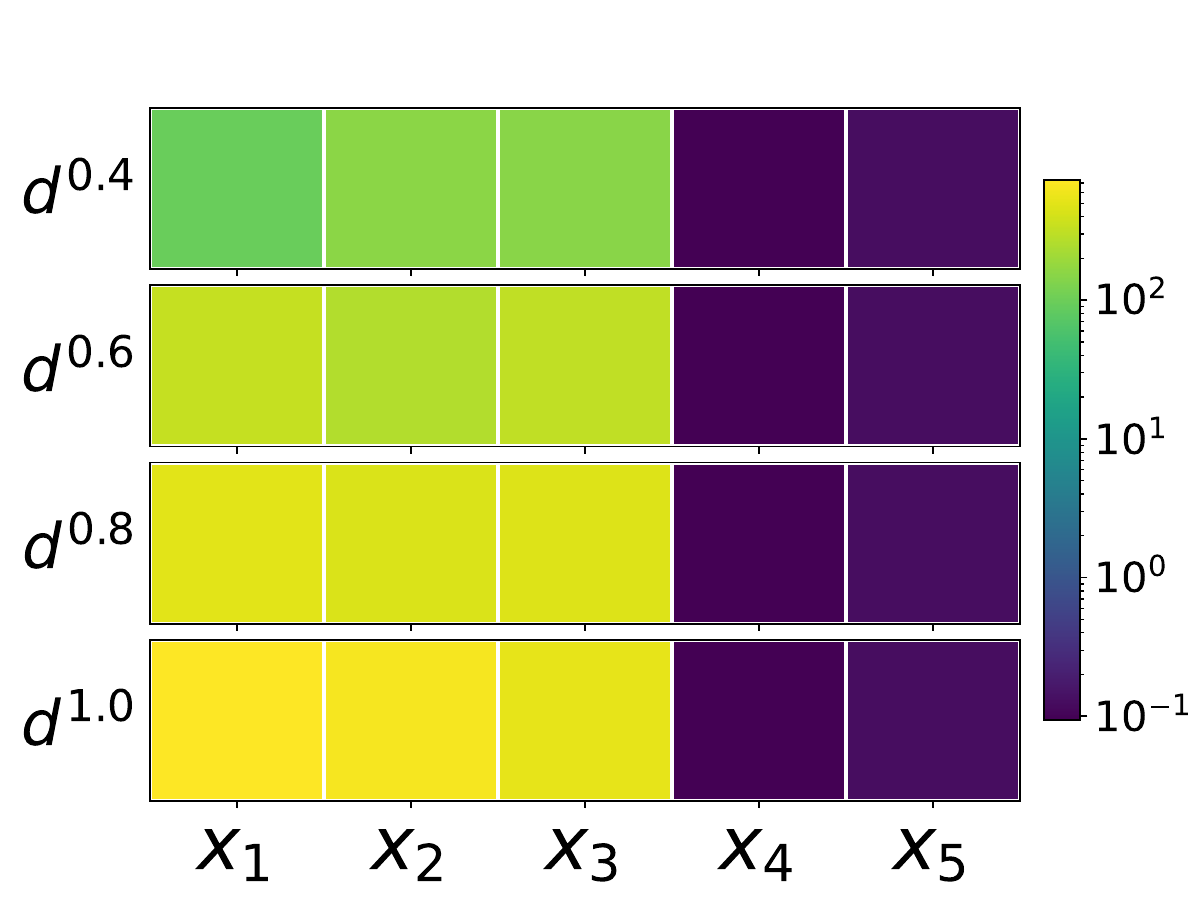}
    \caption{Multi-step $(\alpha=\tfrac{1}{2})$\label{One step:coordinate_identification_krr}}
    \end{subfigure}
\caption{{\bf %
(a) Test loss for kernel machines, neural networks and $\mathtt{IRKM}({\tfrac{1}{2}})$. (b) Comparison of $\mathtt{IRKM}(\alpha)$ performance for $\alpha = 0$ versus $\alpha = \tfrac{1}{2}$. (c,d) The empirical coordinate weights at the first step of $\mathtt{IRKM}(\alpha)$ with $\alpha=0$ and $\tfrac{1}{2}$ respectively. (e,f) The empirical coordinate weights after $T$ steps of $\mathtt{IRKM}(\alpha)$ with $\alpha=0$ and $\tfrac{1}{2}$ respectively.} The samples are i.i.d. uniformly drawn from the hypercube $\{\pm 1\}^d$ with $d= 500$ and the label is $y = x_1 + x_2 +x_3 + x_1x_2x_3 + \varepsilon$ with $\varepsilon\sim \mathcal{N}(0,0.1^2)$. We use the Laplacian kernel for both kernel machine and $\mathtt{IRKM}(\alpha)$. The $\mathtt{IRKM}(\alpha)$ algorithm is run for at most $T=20$ steps.
We train a fully-connected network by Adam,  varying the widths and batch sizes as $\{128, 256, 512\}$, depth as $\{2,3,4,5,6\}$ and learning rate as $\{10^{-3}, 10^{-4}\}$ and report the smallest test loss. For all experiments, we report an average of $10$ independent runs. }
    \label{fig:task_1}
\end{figure}

\paragraph{Detecting coordinates with empirical weights.}
On the other hand, wishful thinking suggests that it may be possible to extract the relevant coordinates (features) from $\hat f$ directly, even when $\hat f$ does not generalize well. One simple approach  is to use an empirical plug-in estimator 
$\E_n(\partial_j\hat f)^2:=\frac{1}{n}\sum_{i=1}^n (\partial_{j} \hat f(x^{(i)}))^2$ to measure the sensitivity of the prediction function $\hat f$ to the $j$'th coordinate. Returning to our running example \eqref{eqn:first_exa}, Figure~\ref{One step:coordinate_identification0} depicts the values $\E_n(\partial_j\hat f)^2$ for different sample sizes $n$. The first three coordinates are clearly identified as being the most significant. Our first contribution is to show that this is no accident: the empirical coordinate weights $\E_n(\partial_j\hat f)^2$ 
allow us to asymptotically estimate the population coordinate weights for the truncation of $f$.  In all theorem statements appearing in the introduction, we will assume that $\mu$ is the uniform measure on the hypercube $\{\pm 1\}^d$.

\begin{theorem}[\textbf{Informal}]\label{infthm:main_est}
Consider the regime $n=d^{p+\delta}$ for any $\delta\in (0,1)$ and $p\in \mathbb{N}$. Then there exists a constant $c>0$ such that the following estimate holds uniformly over all coordinates:
\begin{equation}\label{eqn:main_est}
\frac{\E_n(\partial_j\hat f)^2}{1+o_{d,\P}(1)}= \underbrace{\E(\partial_{j} f_{\leq p})^2}_{{\textrm order}~ p} ~+ ~c\cdot d^{2\delta-2}\cdot\underbrace{\E(\partial_{j} f_{p+1})^2}_{\textrm{order}~ p+1} +\underbrace{O_{d,\mathbb{P}}(d\inv)}_{noise}.
\end{equation}

\end{theorem}

 Theorem~\ref{infthm:main_est} is conceptually interesting because it shows that kernel ridge regression can detect relevant coordinates despite exhibiting a high generalization error. One unsatisfying aspect of \eqref{eqn:main_est} is that the term 
$d^{2\delta-2}\E(\partial_{j} f_{p+1})^2$  is only meaningful in the high-sampling regime $\delta\in (\frac{1}{2},1)$, since in the complementary regime $\delta\in (0,\frac{1}{2})$ this term is dominated by the noise $O_{d,\mathbb{P}}(d\inv)$ and therefore the ``signal'' $\E(\partial_{j} f_{p+1})^2$ is drowned out by the noise. %
With this in mind, we will now show that the coefficient $d^{2\delta-2}$ in \eqref{eqn:main_est} can be reduced by a square root factor to 
$d^{\delta-1}$ by carefully modifying the estimator $\E_n(\partial_j\hat f)^2$. This improved dependence on dimension will ultimately improve the sample efficiency of learning by a significant factor of $\sqrt{d}$. In particular, returning to our  running example \eqref{eqn:first_exa}, the modified estimator will be able to learn the first three coordinates with $n=d^{\delta}$ samples {\em for any} $\delta\in (0,1)$. This is in contrast to the pure empirical weights $\E_n(\partial_j\hat f)^2$ that require at least $n=\sqrt{d}$ samples to learn the first three coordinates; see Figure~\ref{One step:coordinate_identification0}.

\paragraph{Improved coordinate detection with the DN-estimator.}
The key idea underlying the modified estimator is to introduce a weighted kernel function 
$$K_w(x,z) = K(\sqrt{w}\odot x, \sqrt{w}\odot z),$$ 
parametrized by a weight vector $w\in\R^d$.\footnote{The symbol $v\odot w:=(v_iw_i)_{i=1}^d$ denotes the Hadamard product of two vectors $v$ and $w$.}  Intuitively, the weight vector $w$ encodes the importance of each coordinate for prediction.  Now let $\hat f_w$ be the KRR estimator with respect to the kernel $K_w$ and let $\|\cdot\|_w$ denote the RKHS norm induced by $K_w$. 
Intuitively, seeking to find the relevant coordinates of $f$, we should focus on those weights $w$ for which the norm of the KRR predictor $\|\hat f_w\|^2_w$ is as small as possible. Since the function $w\mapsto \|\hat f_w\|^2_w$ is difficult to optimize directly, we may instead follow its gradient. A quick computation shows that up to a sign flip, the partial derivatives of this function can be approximated\footnote{by dropping the term that scales linearly in the ridge parameter} by
$$\mathcal{D}_{j}(w):=\beta_w\tran \tfrac{\partial K_w(X,X)}{\partial w_j} \beta_w,$$
where $\beta_w\in \R^n$ are the dual coefficients 
of $f_w$.\footnote{This means that we can write $\hat f_w(x)=K_w(x,X)\beta_w$.}
We call $\mathcal{D}_{j}(w)$ the {\em derivative norm (DN) estimator} .\footnote{We note that since $K$ is a PSD inner-product kernel, the matrix $\tfrac{\partial K_w(X,X)}{\partial w_j}$ is easily seen to be positive semidefinite and therefore the DN estimator $\mathcal{D}_{j}(w)$ is nonnegative.} 
The following theorem shows that the average of the two estimators $\E_n(\partial_j\hat f)^2$ and $\mathcal{D}_{j}(w)$ with $w={\bf 1}_d$ enjoys an improved guarantee over \eqref{eqn:main_est}, wherein the coefficient $d^{2\delta-2}$ is replaced by $d^{\delta-1}$.

\begin{theorem}[\textbf{Informal}]\label{infthm:main_est_combined}
Consider the regime $n=d^{p+\delta}$ for any $\delta\in (0,1)$ and $p\in \mathbb{N}$. Then there exists a constant $c>0$ such that the following estimate holds uniformly over all coordinates:
\begin{equation}\label{eqn:main_est_combined}
\frac{\E_n(\partial_j\hat f)^2 + \tfrac{1}{n}\gD_j({\bf 1}_d)}{1+o_{d,\P}(1)}= \underbrace{\E(\partial_{j} f_{\leq p})^2}_{{\textrm order}~ p} ~+ ~c\cdot d^{\delta-1}\cdot\underbrace{\E(\partial_{j} f_{p+1})^2}_{\textrm{order}~ p+1} +\underbrace{O_{d,\mathbb{P}}(d\inv)}_{noise}.
\end{equation}
\end{theorem}

Returning to our running example \eqref{eqn:first_exa}, Figure~\ref{One step:coordinate_identification} shows that the first three coordinates are already detected by the averaged estimator in the setting $n=d^{0.4}$, but not by the pure empirical weights $\E_n(\partial_j\hat f)^2$. The reason is that these are the important coordinates for the first-order truncation $f_{\leq 1}(x)=x_1+x_2+x_3$. Moreover, note that in this example the higher-order terms of $f$ depend only on $x_1$, $x_2$, $x_3$. In other words, all influential coordinates of the third-order polynomial $f$, namely those coordinates $r$ that satisfy $\E(\partial _r f)^2 >0$, are fully captured by its first-order truncation. How can we use the averaged estimator to learn the function $f$ in this example? A natural idea is to introduce a reweighting scheme. Namely, one may draw a new set of samples, rescale the coordinates of the data using averaged estimator, and solve a new kernel regression problem which yields a new kernel estimator $\hat f_2(x)$, etc. We summarize this iterative process in the following algorithm, which we call Iteratively Reweighted Kernel Machines ($\mathtt{IRKM}$). Throughout, the symbol $v^{\odot 2}:=v\odot v$ denotes the entrywise square.

\begin{algorithm}[h]
\caption{$\mathtt{IRKM}(\alpha)$ \hfill{\bf I}teratively {\bf R}eweighted {\bf K}ernel {\bf M}achines}
\begin{algorithmic}[1] %
\State \textbf{Input:} averaging parameter $\alpha\in [0,1]$.
\State \textbf{Initialize:} $w={\bf 1}_d$ and fix a safeguard parameter $\epsilonS>0$.
\For{$t=1,\ldots, T$}:
\State Sample fresh data  $\{(x^{(i)},y_i)\}_{i=1}^n$ and define the kernel $K_w(x,z)=K(\sqrt{w}\odot x,\sqrt{w}\odot z)$.\label{line:resampleintro}
\State Obtain predictor $\hat f_t$ from Kernel Ridge Regression on the data $\{(x^{(i)},y_i)\}_{i=1}^n$ with the kernel $K_w$.\label{line:predictorintro}

\State Compute the first weight vector: $w^{[1]}=\epsilonS {\bf 1}_d+ \frac{1}{n}\sum_{i=1}^n [\nabla_x\hat f_t(x^{(i)})]^{\odot 2}  $.\label{line:safeintro_1}
\State Compute the second weight vector: $w^{[2]}=\epsilonS {\bf 1}_d+ \frac{1}{n} \mathcal{D}(w) \odot w $.\label{line:safeintro_2}
\State Normalize the weights $w^{[k]}=\frac{d}{\|w^{[k]}\|_1} \cdot w^{[k]}$,~~$k = 1,2$.\label{line:normalizeintro}
\State Update the weight vector $w = (1-\alpha)\cdot w^{[1]} + \alpha\cdot w^{[2]}$.\label{line:mixing}
\EndFor
\end{algorithmic}\label{alg:rfm}
\end{algorithm}

$\mathtt{IRKM}{(\alpha)}$ is closely related to and inspired by existing algorithms, in particular, iteratively reweighted least squares for sparse recovery \cite{daubechies2010iteratively} and low-rank matrix recovery \cite{mohan2012iterative} and recursive feature machines \cite{rfm_science} for learning multi-index models. We will comment on these connections at the end of the introduction. In the meantime, the main takeaway is that operationally $\mathtt{IRKM}{(\alpha)}$ proceeds by solving a sequence of kernel ridge regression problems on freshly sampled data that is reweighted by the average of the empirical coordinate weights and the DN estimator. 
More precisely, step \eqref{line:resampleintro} samples fresh data and forms a reweighted kernel $K_w$ based on the current weight vector $w$. Step \eqref{line:predictorintro} obtains a KRR predictor based on the kernel $K_w$. Steps \eqref{line:safeintro_1} and \eqref{line:safeintro_2} compute the empirical coordinate weights $\E(\partial f_{\leq p})^2$ and the DN-estimator $\mathcal{D}(w)$. Note that we safeguard the two intermediate weight vectors by adding a small multiple of the all-ones vector and the second weight vector is obtained as a Hadamard product of the DN estimator with the current weight vector $w$ analogously to a multiplicative weight update algorithm. Step \eqref{line:normalizeintro} normalizes the two intermediate weight vectors and \eqref{line:mixing} averages them to obtain the weight vector $w$.

Returning to our example, the green curve in Figure~\ref{fig:gen_error_cube} depicts the generalization error of $\hat f_2$ for various values of $n$ with the mixing parameter $\alpha=\frac{1}{2}$, which is clearly superior to both the standard kernel estimator and neural networks trained by Adam. Interestingly, iterating $\mathtt{IRKM}(\frac{1}{2})$ for a few extra steps leads to an even further decrease in generalization error (red) and better identification of relevant coordinates in Figure~\ref{One step:coordinate_identification_krr}. The following theorem identifies the reason for this improved performance: by running $\mathtt{IRKM}{(\frac{1}{2})}$ for a constant number of iterations, all features with order not exceeding $p+1$ will be learned. Consequently, as long as all relevant coordinates of $f$ are relevant for $f_{\leq p+1}$, after finitely many steps the predictor returned by $\mathtt{IRKM}{(\frac{1}{2})}$ is an asymptotically consistent estimator for $f$. This is the content of the following theorem.

\begin{theorem}[\textbf{Informal}]\label{thm:gener_features_ite}
Suppose that the target function $f(x)=g(x_1,\ldots,x_r)$ is a degree $\ell$ polynomial on $\R^d$.  Suppose that 
the weights  $\E [\partial_{j} f]^2$ are uniformly bounded away from zero for all $j=1,\ldots, r$ as $d \rightarrow \infty$. Consider the regime $n=d^{p+\delta}$ for any $\delta\in (0,1)$ and choose any $\alpha\in (0,1)$. Suppose moreover that the nondegeneracy condition holds:
\begin{itemize}
    \item $\E [\partial_{j} f_{\leq {p+1}}]^2$ is bounded away from zero uniformly in $j=1,\ldots, r$ as $d\to\infty$,
\end{itemize}
Then there exists a horizon $T = O_d(1)$ such that, under an appropriate choice of the safeguard parameter $\epsilon_{\mathrm S}$, the predictor $\hat f_T$ returned by $\mathtt{IRKM}{(\alpha)}$ 
satisfies:
\begin{align*}
    \norm{\hat{f}_T - f}^2_{L_2(\mu)} = o_{d,\P}(1).
\end{align*}
\end{theorem}

Interestingly, returning to our running example, Figure~\ref{fig:compare_task1} shows that $\mathtt{IRKM}{(\frac{1}{2})}$ significantly outperforms $\mathtt{IRKM}{(0)}$ thereby highlighting the importance of the DN-estimator.

\subsubsection{Hierarchical learning.} Theorem~\ref{thm:gener_features_ite} crucially relies on the assumption that all influential coordinates for $f$ are influential for its truncation $f_{\leq {p+1}}$; this is the case in the running example \eqref{eqn:first_exa}. As a concrete example where this fails, consider learning the function 
\begin{equation}\label{eqn:first_exa2}
f(x)=x_1 + x_2 + x_1x_2x_3+x_1x_2x_3x_4\qquad \textrm{for}~x\in\R^{d},
\end{equation}
with a Laplacian kernel and uniform sampling on a hypercube $\{\pm 1\}^d$. Such staircase-like functions have recently been used as a model for hierarchical structures in statistical and machine learning \cite{abbe2022merged,abbe2023sgd}. Figure~\ref{fig:gen_error_cube_huiarch} depicts the generalization error of kernel regression (blue), one-step $\mathtt{IRKM}(\frac{1}{2})$ (green), multi-step $\mathtt{IRKM}{(\frac{1}{2})}$ (red), and neural networks (orange) trained by Adam when learning the function~(\ref{eqn:first_exa2}). The superior performance of $\mathtt{IRKM}{(\frac{1}{2})}$ is striking; how can we understand this phenomenon? First, observe that in the sub-linear sampling regime $n=d^{\delta}$ with $\delta\in (0,1)$, ridge regression can only identify the important coordinates for the first-order truncation $f_{\leq 1}$, namely $x_1$ and $x_2$. This is supported by Figure~\ref{One step:coordinate_identification_hiarch}, which depicts the empirical coordinate weights using the KRR estimator. On the other hand, looking at Figures~\ref{fig:grad_learning} and~\ref{fig:multi_step_hiarch}, we see that $\mathtt{IRKM}{(\tfrac{1}{2})}$ quickly identifies the planted coordinates $x_3$ and $x_4$ as being significant, all still within the sub-linear sampling regime. 

\begin{figure}[h!]
    \centering
    \begin{subfigure}[b]{0.45\linewidth}
        \centering
        \includegraphics[width=\linewidth]{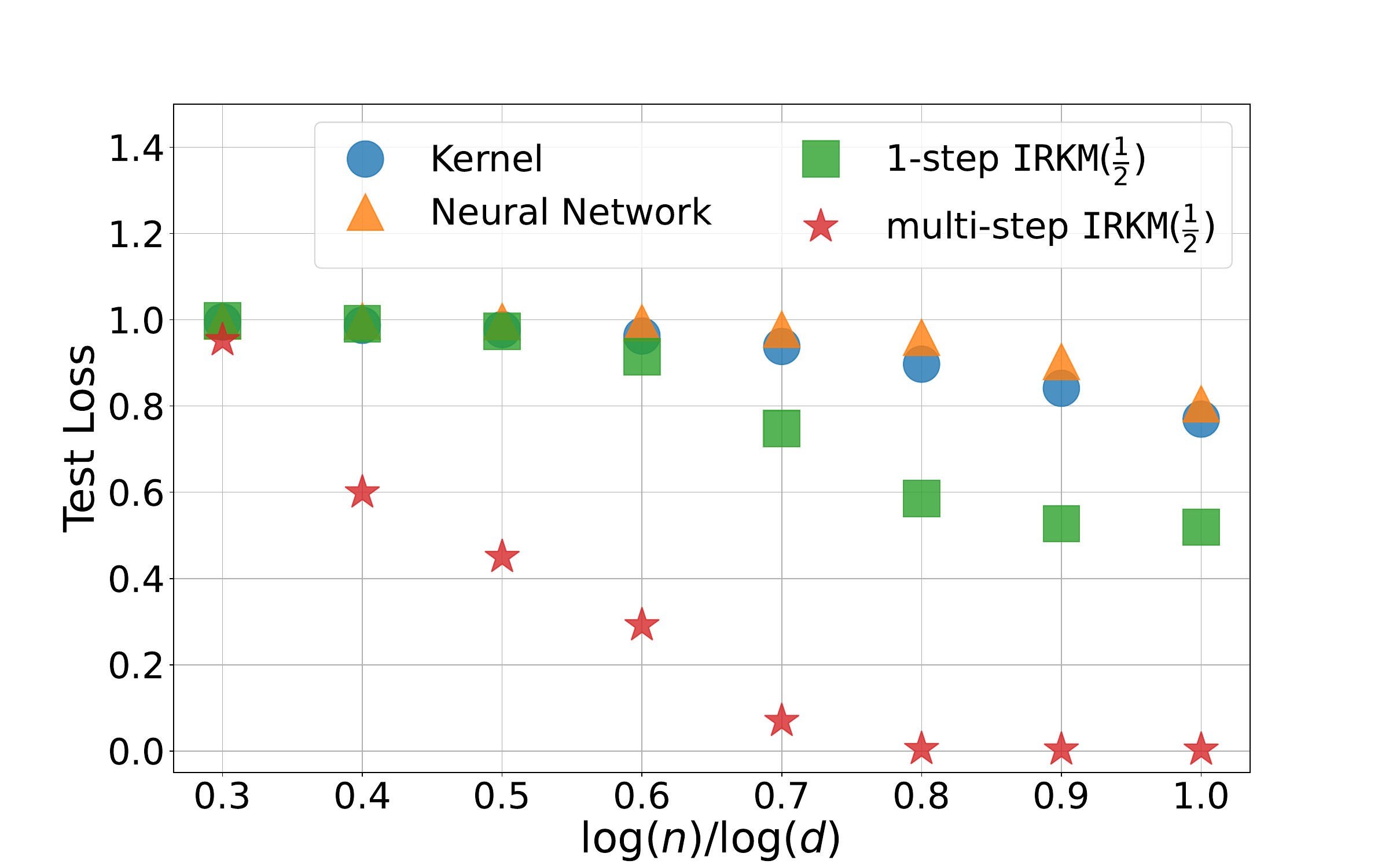}
        \caption{Generalization error\label{fig:gen_error_cube_huiarch}}
    \end{subfigure}
        \begin{subfigure}[b]{0.45\linewidth}
        \centering
        \includegraphics[width=\linewidth]{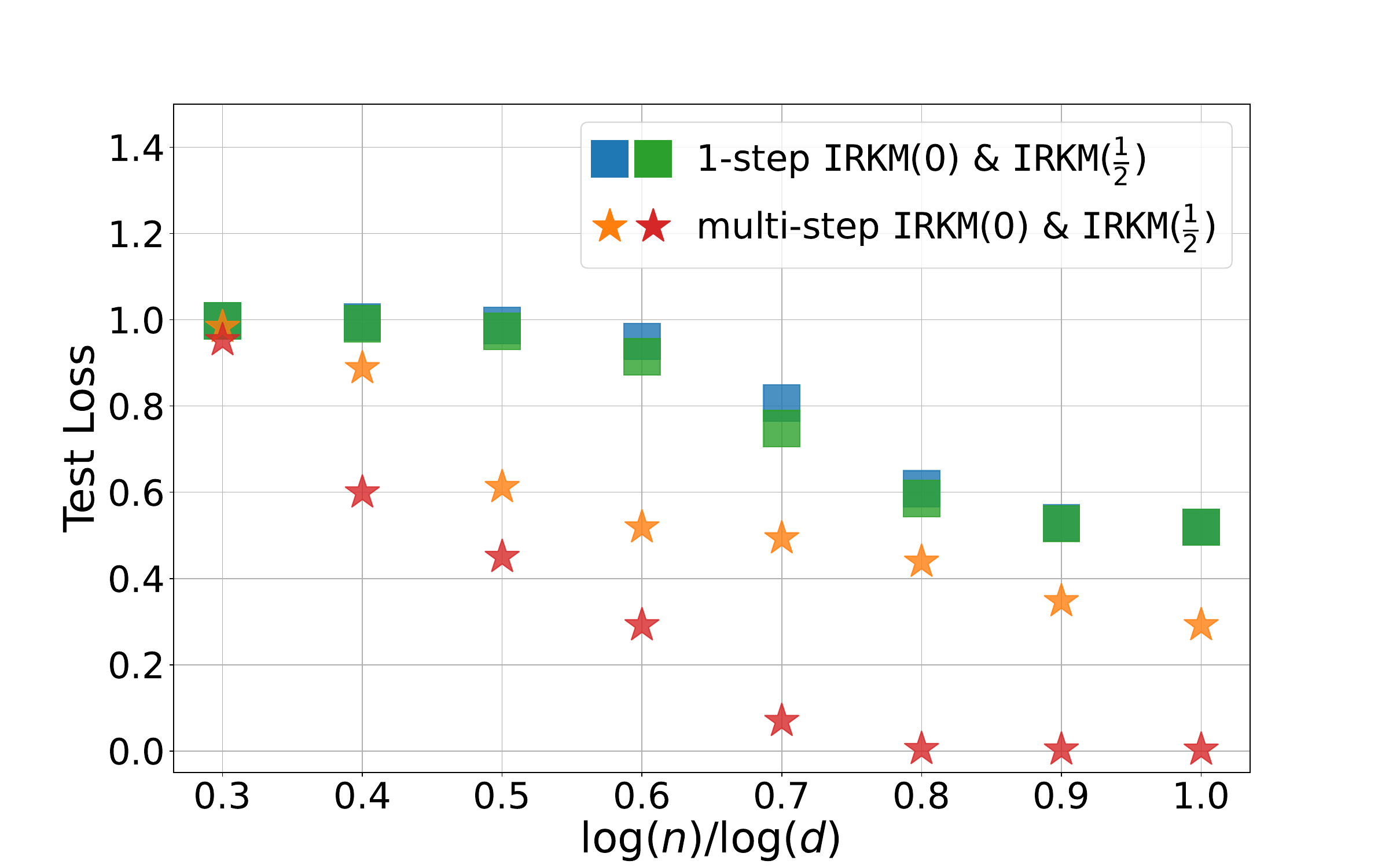}
    \caption{$\alpha = 0$ vs.  $\alpha = \tfrac{1}{2}$ \label{fig:compare_task2}}
    \end{subfigure}\\
            \begin{subfigure}[b]{0.4\linewidth}
        \centering
        \includegraphics[width=\linewidth]{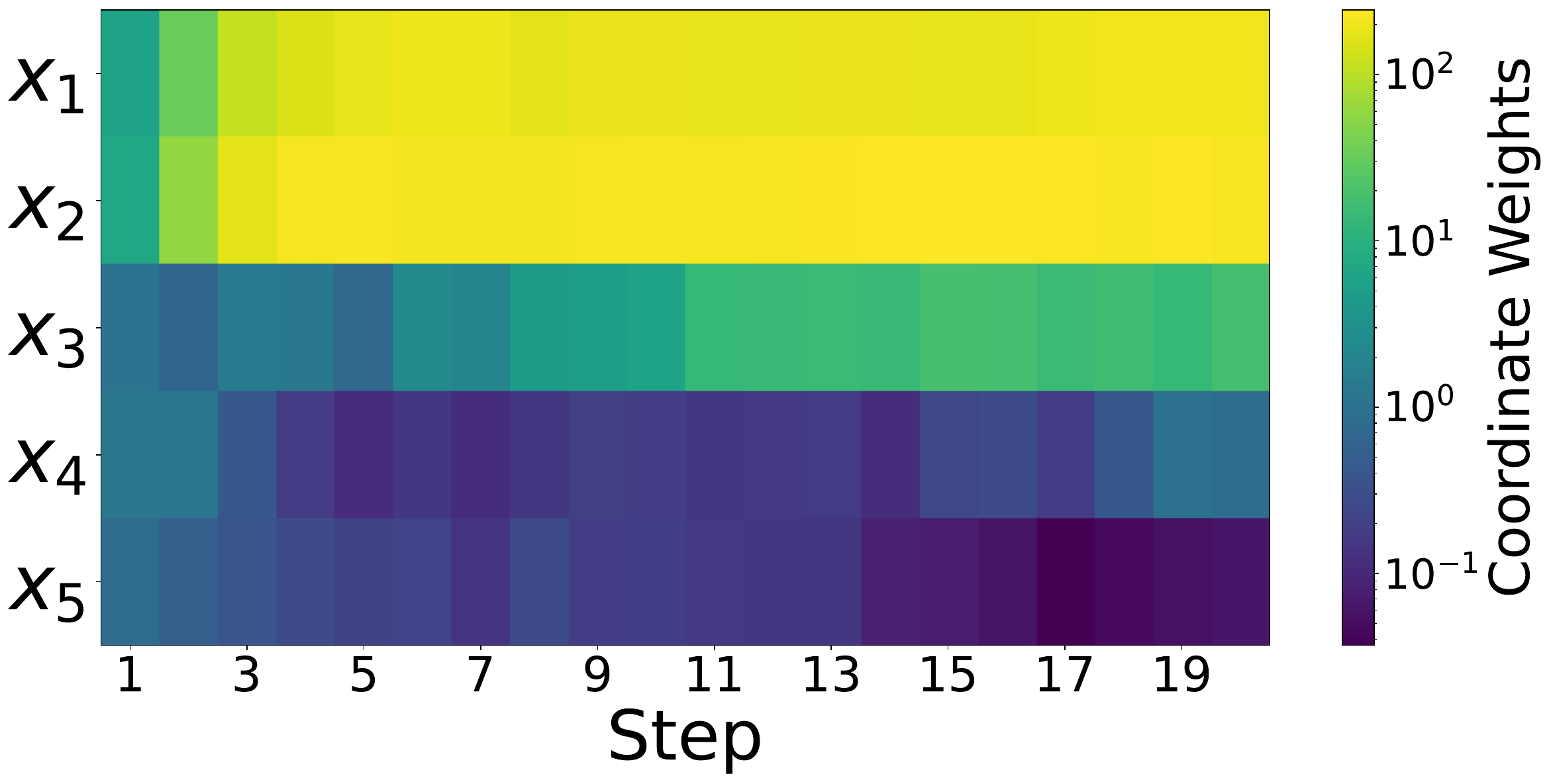}\vspace{-5pt}
        \caption{Multi-step w.r.t. step $(\alpha = \tfrac{1}{2})$\label{fig:grad_learning}}
         \end{subfigure}
    \begin{subfigure}[b]{0.28\linewidth}
        \centering
        \includegraphics[width=\linewidth]{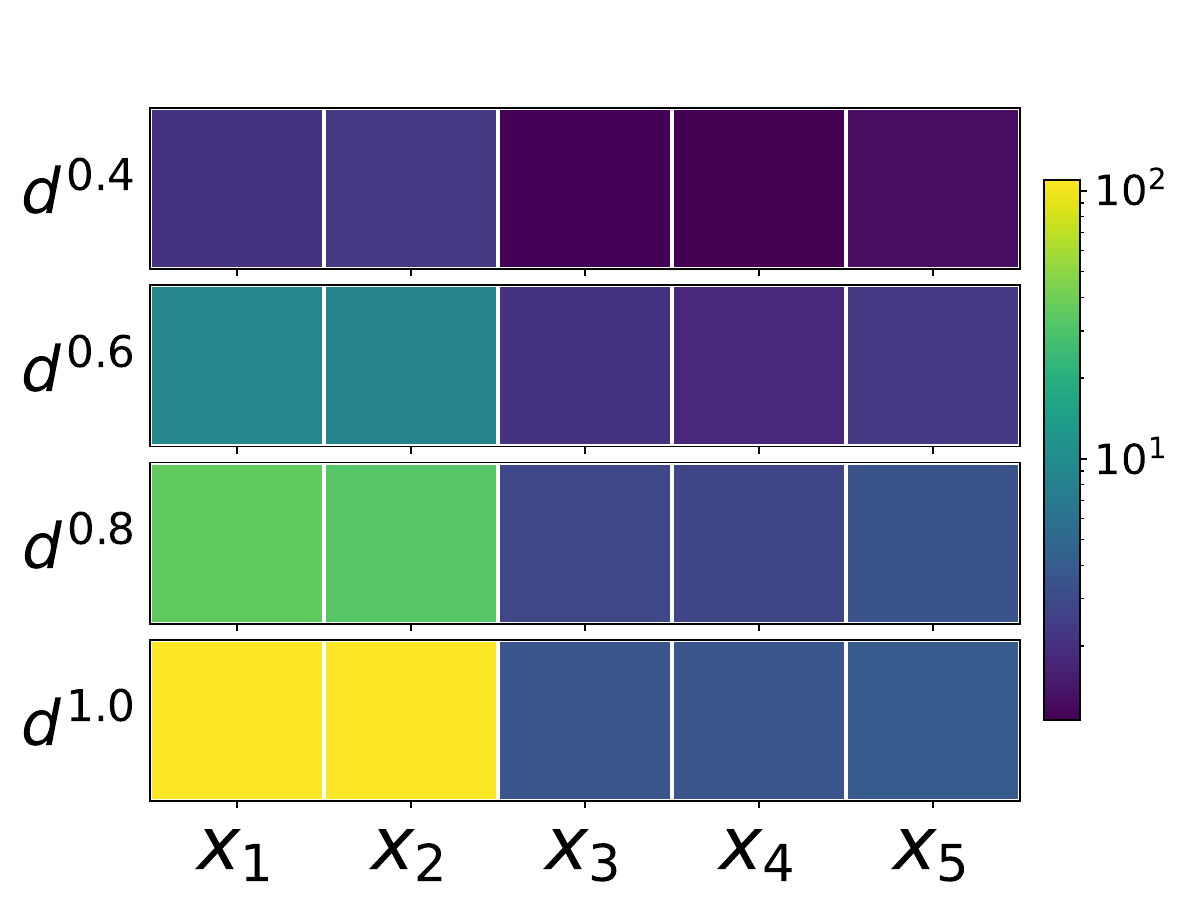}
    \caption{One step $(\alpha = \tfrac{1}{2})$ \label{One step:coordinate_identification_hiarch}}
    \end{subfigure}
        \begin{subfigure}[b]{0.28\linewidth}
        \centering
        \includegraphics[width=\linewidth]{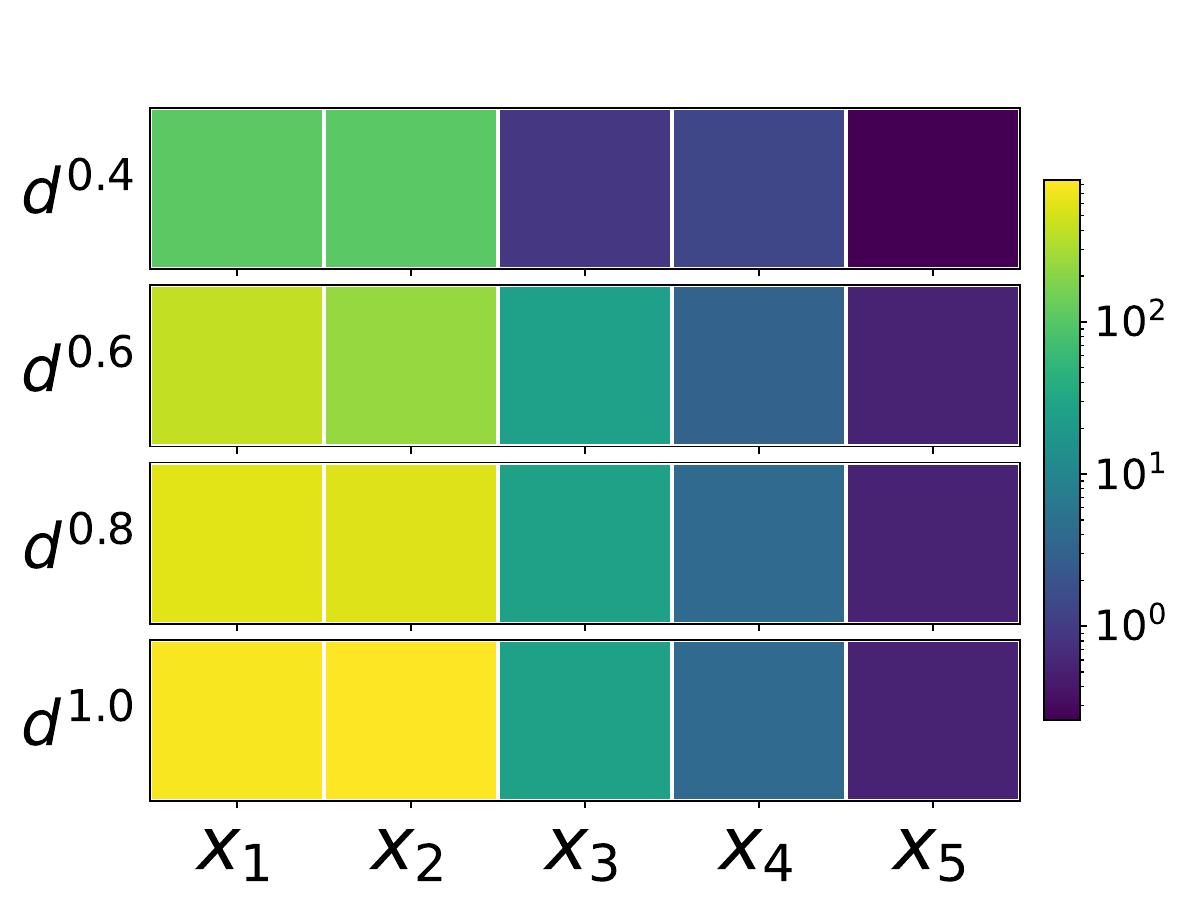}
        \caption{Multi-step $(\alpha = \tfrac{1}{2})$ \label{fig:multi_step_hiarch}}
    \end{subfigure}
    \caption{{\bf %
    (a) Test loss for kernel machines, neural networks and $\mathtt{IRKM}(\tfrac{1}{2})$. %
(b) Comparison of $\mathtt{IRKM}(\alpha)$ performance for $\alpha = 0$ versus $\alpha = \tfrac{1}{2}$.
    (c)The empirical coordinate weights of $\mathtt{IRKM}(\tfrac{1}{2})$ at step $1-20$ in the regime $n = d^{0.6}$. %
    (d,e) The empirical coordinate weights of $\mathtt{IRKM}(\alpha)$ at step $T$ for $\alpha = 0$ and $\tfrac{1}{2}$ respectively. } The samples are i.i.d.\ uniformly drawn from the hypercube ${\{- 1,1\}^d}$ with $d= 500$ and the label is $y = x_1 + x_2 + x_1x_2x_3 + x_1x_2x_3x_4 + \varepsilon$ with $\varepsilon\sim \mathcal{N}(0,0.1^2)$. We use Laplacian kernel for both kernel machines and $\mathtt{IRKM}(\alpha)$. The $\mathtt{IRKM}(\alpha)$ algorithm is iterated for at most $T=20$ steps. We train a fully-connected network by Adam,  varying the widths and batch sizes as $\{128, 256, 512\}$, depth as $\{2,3,4,5,6\}$ and learning rate as $\{10^{-3}, 10^{-4}\}$ and report the smallest test loss.  For all experiments, we report an average of $10$ runs. }
    \label{fig:task_2}
\end{figure}

We now explain this phenomenon. To this end, recall that an orthonormal basis for $L_2(\mu)$ is given by Fourier-Walsh monomials:
$$\phi_S(x)=\prod_{i\in S} x_i\qquad \forall S\subset[d].$$
Following \cite{abbe2023sgd}, we say that a function $f=\sum_{S\subset[d]}b_S\phi_S$ has {\em leap complexity $k$} if the nonzero coefficients $\{b_S\}$ can be ordered in such a way that the corresponding sets $\{S_{i}\}_{i=1}^m$ satisfy 
$$|S_1|\leq k \qquad \textrm{and}\qquad \left|S_i\setminus \midcup_{1\leq j<i} S_j\right|\leq k,$$
and $k$ is the smallest integer for which this is possible.
In other words, $f$ has leap complexity $k$ if $k$ is the smallest integer such that its nonzero coefficients can be ordered in such a way that the corresponding sets grow by at most $k$ indices.  In particular, the function \eqref{eqn:first_exa2} has leap complexity $k=1$. 
We use ${\rm Leap} (f)$ to denote the leap complexity of $f$. The following theorem shows that $n=d^{{\rm Leap}(f)-1+\delta}$ samples suffices for $\mathtt{IRKM}(\alpha)$  with $\alpha\in (0,1)$ to learn any function $f$, where $\delta\in (0,1)$ is arbitrary.

\begin{theorem}[\textbf{Informal}]\label{thm:staricase_learning}
 Consider a target function $f(x)=g(x_1,\ldots, x_r)$ that is a degree $\ell$ polynomial on $\R^d$ with $r=O_d(d^s)$ influential coordinates for some $s\in [0,1)$, meaning that the coordinate weights $\E [\partial_{j} f]^2$ are uniformly bounded away from zero for all $j=1,\ldots, r$ as $d \rightarrow \infty$. Suppose ${\rm Leap}(f) \geq 1$ and consider the regime $n = d^{{\rm Leap}(f)-1+\delta}$ with $\delta \in (0,1)$. Suppose, moreover, that the sparsity level satisfies $s\leq \delta/(\ell+1)$. 
 Then for any mixing parameter $\alpha\in(0,1)$ in step \eqref{line:mixing} and under an appropriate setting of the safeguard $\epsilonS$ in steps \eqref{line:safeintro_1} and \eqref{line:safeintro_2},  there exists a horizon $T = O_d(1)$ such that the predictor $\hat f_T$ returned by $\mathtt{IRKM}{(\alpha)}$ 
satisfies:
\begin{align*}
    \norm{\hat{f}_T - f}^2_{L_2(\mu)} = o_{d,\P}(1).
\end{align*}\label{thm:main_intro}
\end{theorem}
Applying Theorem~\ref{thm:staricase_learning} to the example~\eqref{eqn:first_exa2} with $s=0$, we see that  $\mathtt{IRKM}{(\frac{1}{2})}$ in a constant number of steps produces a predictor $\hat f_T$ that is asymptotically consistent for estimating $f$ as $d$ tends to infinity.  Theorem~\ref{thm:staricase_learning} is broader still, as it allows the sparsity level $r=d^s$ of the target function to grow with $d$.

It is important to compare Theorem~\ref{thm:main_intro} to the existing results in the literature \cite{abbe2022merged,abbe2023sgd}, which  analyzed the performance of neural networks trained using (projected) SGD when learning polynomials with bounded leap complexity ${\rm Leap}(f^*)$ from Gaussian or hypercube data.
In particular, the paper \cite{abbe2023sgd} showed that $\tilde{\Theta}_d(d^{\max({\rm Leap}(f^*)-1,1)})$
projected SGD steps suffice for a two-layer neural network  to learn a sparse polynomial $f^*$.
In this work, we consider learning with kernel machines and show that $\Theta_d(d^{{\rm Leap}(f^*)-1})$ samples suffices to learn $f^*$ after solving finitely many $n\times n$ kernel linear systems. Note that the sample complexity of $\mathtt{IRKM}{(\frac{1}{2})}$ is better than neural networks by a factor of $d$ when learning leap-one functions and matches  neural networks for functions with leap complexity ${\rm Leap}(f^*)>1$. In particular, Figure~\ref{fig:task_1} indeed shows that when trained with ADAM on the leap one function, we see that ADAM performs comparably to kernels until $n = \Omega(d).$ 
We note that in all of our experiments, $\mathtt{IRKM}{(\frac{1}{2})}$ typically outperforms neural networks and at worst performs comparably.

\paragraph{Iteratively reweighted least squares and recursive feature machines.}
$\mathtt{IRKM}{(\alpha)}$ is closely related to existing algorithms in the literature when the mixing coefficient is $\alpha = 0$. First, if the kernel is linear and the resampling \eqref{line:resampleintro} and normalization \eqref{line:normalizeintro} steps are omitted, then $\mathtt{IRKM}{(0)}$  reduces exactly to the celebrated Iteratively Re-weighted Least-Squares (IRLS) algorithm \cite{daubechies2010iteratively,candes2008enhancing} for recovering a sparse vector $z_{\sharp}\in\R^d$ from finitely many linear measurements $A z_{\sharp}=b$. 
In this way, $\mathtt{IRKM}{(\frac{1}{2})}$ can be understood as a generalization of %
IRLS for learning sparse nonlinear functions. We also note %
that extensions of %
IRLS for low-rank matrix recovery, where the weights being updated are themselves matrices rather than vectors, have been studied in \cite{mohan2012iterative,fornasier2011low}.

\begin{algorithm}[t]
\caption{$\mathtt{RFM}(\alpha)$\hfill {\bf R}ecursive {\bf F}eature {\bf M}achines}
\begin{algorithmic} %
\State \textbf{Initialize:} $M=I_d$ and fix a safeguard parameter $\epsilonS>0$ and mixing parameter $\alpha\in[0,1]$.
\For{$t=1,\ldots, T$}:
\State Sample fresh data  $\{(x^{(i)},y_i)\}_{i=1}^n$ and define the reweighted kernel $K_M(x,z)=K(\sqrt{M}x,\sqrt{M}z)$.
\State Obtain predictor $\hat f_t(x)$ from Kernel Ridge Regression on the data $\{(x^{(i)},y_i)\}_{i=1}^n$ with the kernel $K_M$.
    \State Update the first weight matrix 
    $M^{[1]}=\epsilonS I_d+\frac{1}{n}\sum_{i=1}^n \nabla_x \hat f_t(x^{(i)})\nabla_x \hat f_t(x^{(i)})^\top$.
        \State Update the second weight matrix 
    $M^{[2]}=\epsilonS I_d+ \frac{1}{n}M^{\frac{1}{2}}\gD(M)M^{\frac{1}{2}} $
\State Normalize the eigenvalues   $M^{[k]}= \frac{d}{\tr(M^{[k]})} \cdot M^{[k]},~~~k=1,2$. 
\State Update the weight matrix $M = (1-\alpha)\cdot M^{[1]} + \alpha\cdot M^{[2]}$.
\EndFor
\end{algorithmic}\label{alg:fullrfm}
\end{algorithm}

\begin{figure}[H]
    \centering
    \begin{subfigure}[b]{0.45\linewidth}
        \centering
        \includegraphics[width=\linewidth]{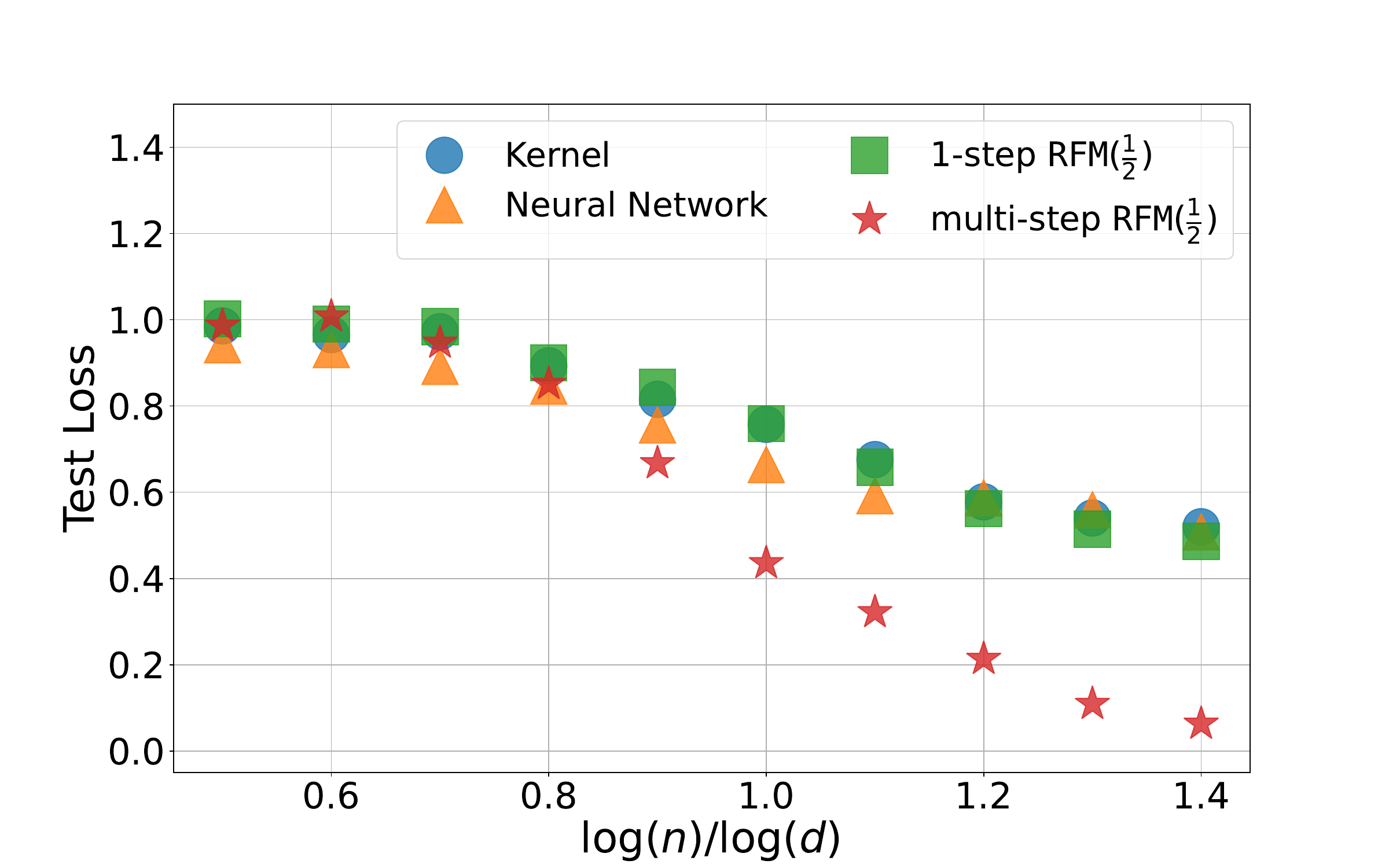}
        \caption{Generalization error\label{fig:gen_error_cube_huiarch2}}
    \end{subfigure}
    \begin{subfigure}[b]{0.45\linewidth}
        \centering
        \includegraphics[width=\linewidth]{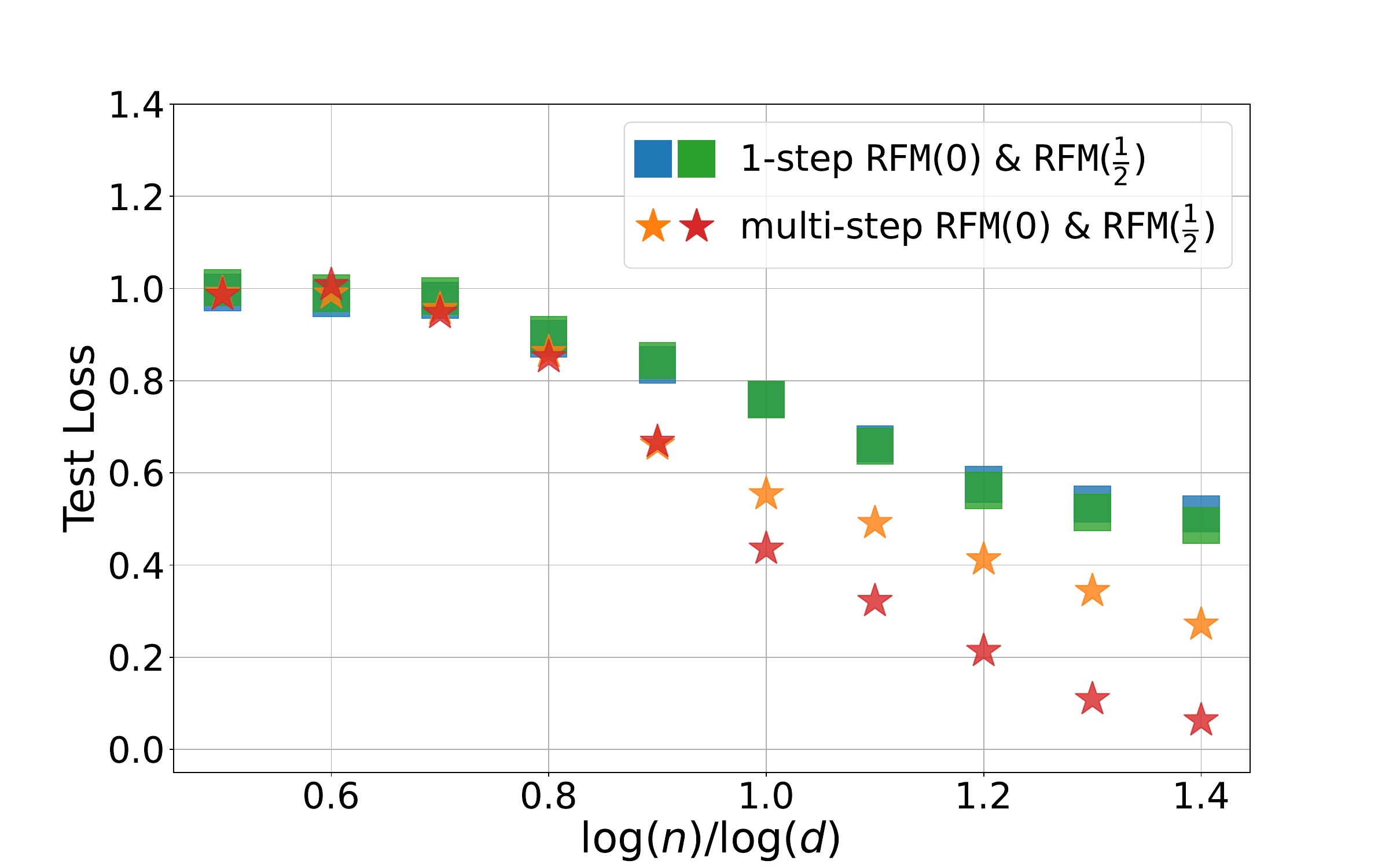}
    \caption{$\alpha = 0$ vs.  $\alpha = \tfrac{1}{2}$\label{compare:task_full}}
    \end{subfigure}\\
            \begin{subfigure}[b]{0.45\linewidth}
        \centering
        \includegraphics[width=\linewidth]{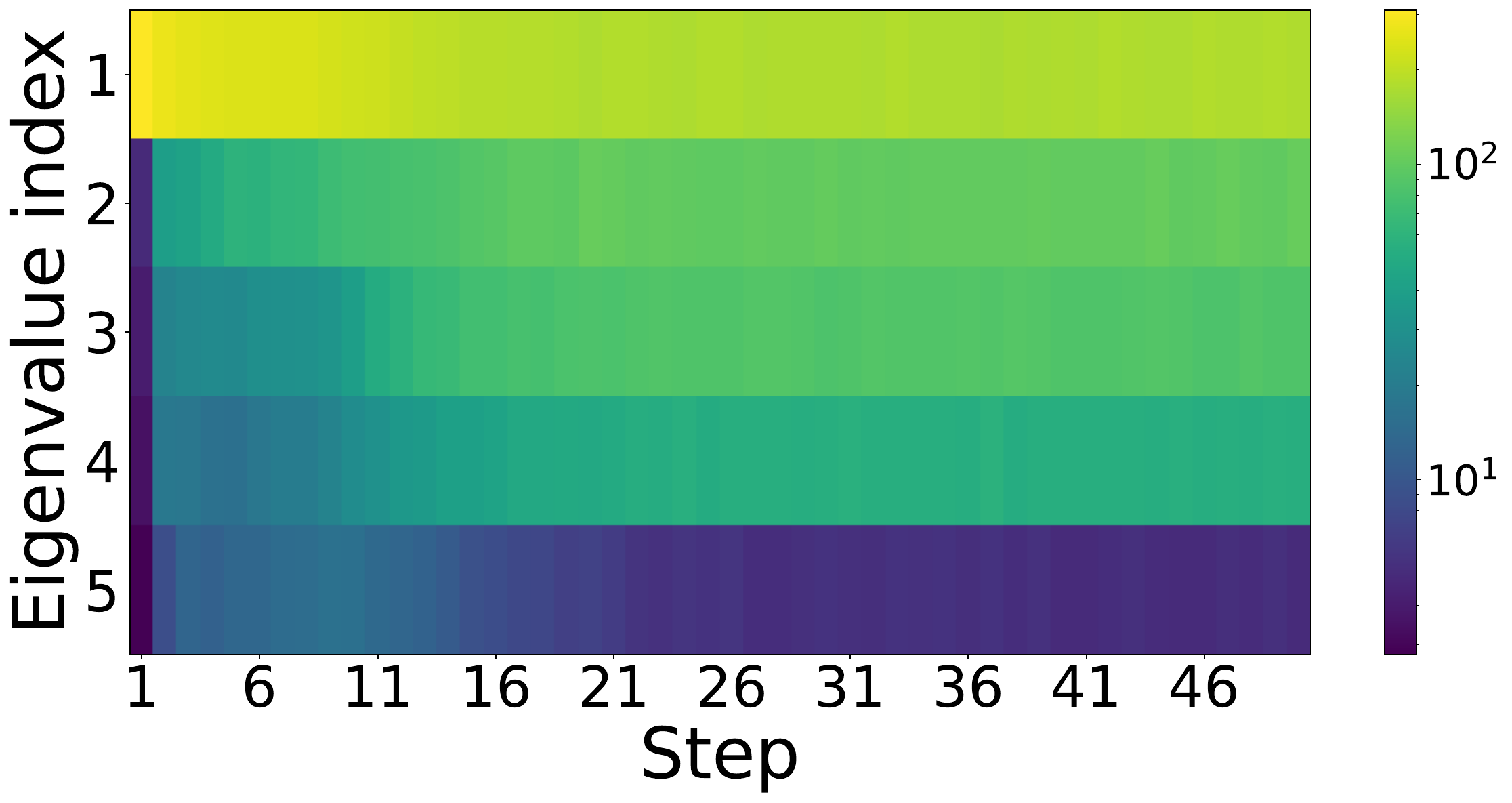}
        \caption{Multi-step identification w.r.t. step\label{fig:grad_learning2}}
         \end{subfigure}
        \begin{subfigure}[b]{0.45\linewidth}
        \centering
        \includegraphics[width=\linewidth]{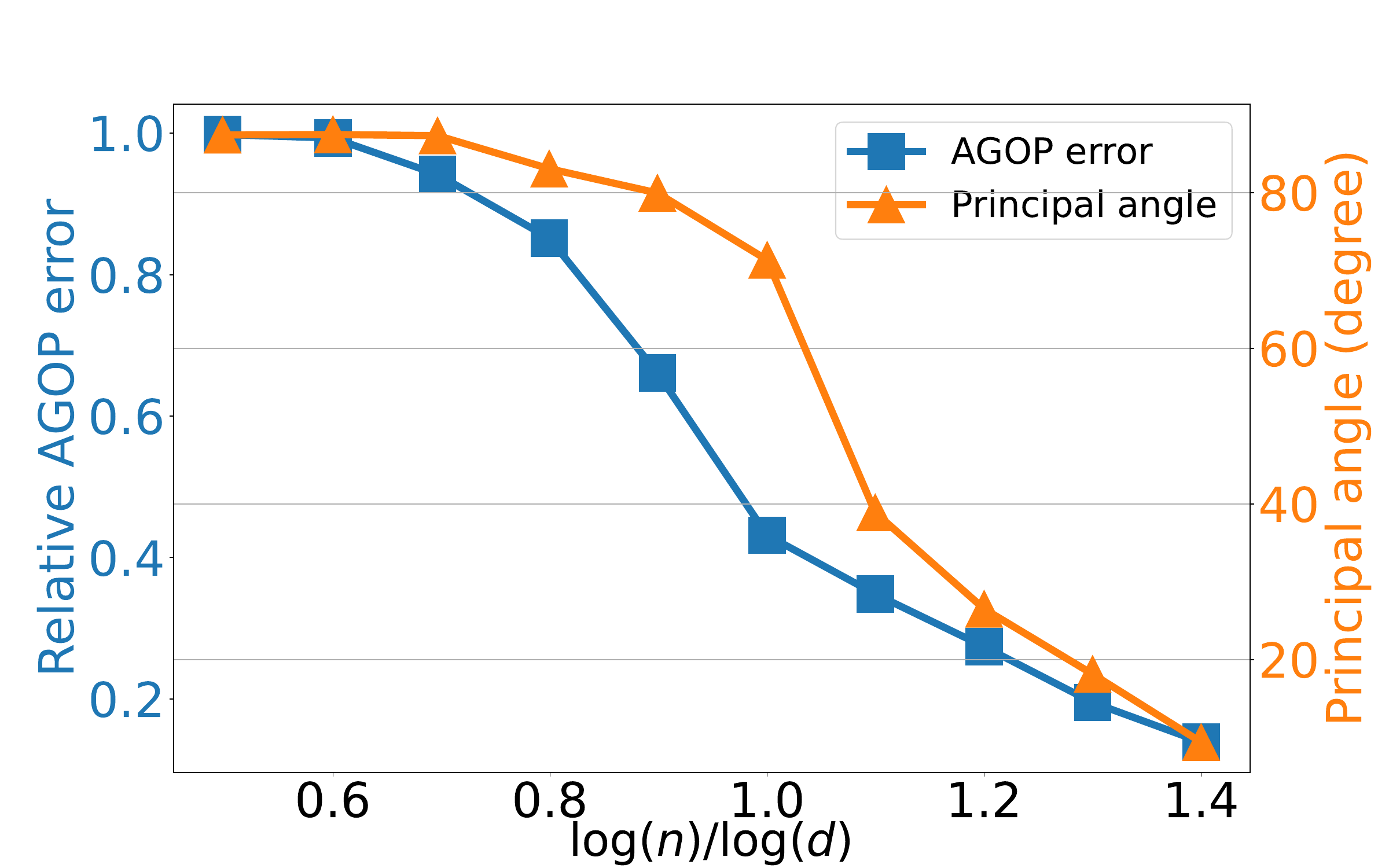}
    \caption{AGOP error \label{One step:coordinate_identification_hiarch2}}
    \end{subfigure}
    \caption{{\bf %
    (a) Test loss for kernel machines, neural networks and $\mathtt{RFM}(\tfrac{1}{2})$ (Algo.~\ref{alg:fullrfm}). 
    (b) Comparison of $\mathtt{RFM}(\alpha)$  performance for $\alpha = 0$ versus $\alpha = \tfrac{1}{2}$.
    (c) Eigenvalues of AGOP at step $1-50$ in the regime $n = d^{1.4}$. %
    (d) Relative AGOP error and principal angle between the top-4 eigenspace of AGOPs.  } 
    The samples are i.i.d.\ uniformly drawn from ${\{-1, 1\}^d}$ with $d= 500$ and label is $y = z_1 + z_2 + z_1z_2z_3 + z_1z_2z_3z_4+ \varepsilon$ with $\varepsilon\sim \mathcal{N}(0,0.1^2)$ and  $z = Ux$ where $U$ is a random rotation matrix. We use the Laplacian kernel for both kernel machines and $\mathtt{RFM}(\alpha)$. We train a fully-connected network by Adam,  varying the widths and batch sizes as $\{128, 256, 512\}$, depth as $\{2,3,4,5,6\}$ and learning rate as $\{10^{-3}, 10^{-4}\}$ and report the smallest test loss. For all experiments, we report an average of $10$ runs.}
    \label{fig:task_3}
\end{figure}

$\mathtt{IRKM}{(0)}$ is also closely related to the Recursive Feature Machines (RFM) algorithm introduced in \cite{rfm_science} for learning multi-index models. Indeed, if the resampling \eqref{line:resampleintro} and normalization \eqref{line:normalizeintro} steps  are omitted and one sets $\epsilonS=0$,
then $\mathtt{IRKM(0)}$ is a variant of RFM~\cite{rfm_science} where the Average Gradient Outer Product (AGOP) $M=\frac{1}{n}\sum_{i}\nabla \hat f(x^{(i)})\nabla \hat f(x^{(i)})^\top$ is replaced by its diagonal. We stress, however, that the fresh resampling \eqref{line:resampleintro} and normalization steps \eqref{line:normalizeintro}, and safeguarding in \eqref{line:safeintro_1} with $\epsilonS>0$ are essential both for our theoretical results to be valid and in our experiments when learning hierarchical functions. 
Moreover, $\mathtt{IRKM(\alpha)}$ suggests a modification of RFM for learning multi-index models (i.e., sparse functions in an unknown basis) online, which draws fresh samples and reweights the data by an average of the AGOP and the DN-estimator, which is naturally defined as
\begin{equation}
\mathcal{D}(M)_{i,j}:=\beta_M\tran \tfrac{\partial K_M(X,X)}{\partial M_{i,j}} \beta_M,~~~~\forall i,j\in[d].
\end{equation}
We record this modified algorithm as Algorithm~\ref{alg:fullrfm}, and call it $\mathtt{RFM(\alpha)}$.  
 We conjecture that many of our theoretical results should extend from learning sparse functions with $\mathtt{IRKM(\alpha)}$ to learning multi-index models with $\mathtt{RFM}(\alpha)$. %
 As an illustration, we apply $\mathtt{RFM}(\alpha)$ for learning the function \eqref{eqn:first_exa2}, precomposed by a random rotation. The results are recorded in Figure~\ref{fig:task_3}, where we clearly see that the unknown basis is gradually learned by the algorithm, albeit with higher sample complexity and a larger number of steps ($20$ vs.\ $50$). We measure the learning of the basis using the relative AGOP error and the principal angle against the ground truth AGOP, namely $\E_x\brac{\nabla f(x)\nabla  f(x)^\top}.$ As in the vector case, we observe in Figure~\ref{compare:task_full} the superior performance of $\mathtt{RFM}(\frac{1}{2})$ over $\mathtt{RFM}(0)$.

\subsection{Further related works}\label{sec:related_literature}
\paragraph{Kernel learning in the polynomial scaling regime.}
Kernel Ridge Regression is a classical technique in machine learning. Generalization guarantees for kernel ridge regression are by now classical in the regime where the dimension $d$ of the feature vectors is fixed and the number of samples $n$ can grow. See for example the monographs \cite{bach2024learning,cucker2007learning,scholkopf2002learning} on the subject. More recently, the high-dimensional regime, where $d$ and $n$ can grow simultaneously, has gained much attention. We refer the reader to the representative articles on the subject \cite{elkaroui2010spectrum,liang2020just,belkin2019reconciling,mei2022generalization,bordelon2020spectrum,cao2021generalization,pedregosa2021convergence}. The most relevant setting for us is the polynomial scaling regime $n= d^{p+\delta}$ for some $p\in \mathbb{N}$ and $\delta\in(0,1)$. In this regime, it is known that kernel methods can learn (at most) a degree-$p$ polynomial when the data is uniformly distributed on the sphere~\cite{ghorbani2021linearized} and the hypercube~\cite{mei2022generalization}. Moreover, the subsequent work~\cite{donhauser2021rotational} established analogous upper and lower bounds on the generalization error for a broader class of data distributions and kernels, including radial basis function (RBF) kernels.

\paragraph{Feature learning with neural networks and kernel machines.} The learning capacity of infinitely wide neural networks under specific parameterizations is equivalent to that of kernel machines with the Neural Tangent Kernel (NTK)~\cite{jacot2018neural}. However, neural networks empirically exhibit superior generalization performance compared to kernel machines, a phenomenon attributed to their ability to implicitly learn low-dimensional features which suffice for accurate prediction. For instance, a series of works has demonstrated a clear separation in generalization performance between neural networks and NTKs~\cite{allen2019can,li2020learning,chen2020towards}.   To better understand the feature learning phenomenon, a significant line of research focuses on single- and multi-index models, where the target function takes the form $f^*(x) = g(Ux)$ with $U:\R^d \rightarrow \R^r$ and $d \gg r$.
It has been shown that neural networks benefit from a small intrinsic dimension $r$, achieving improved sample complexity, whereas kernel methods do not exhibit similar gains~\cite{bietti2022learning,mousavineural,ba2022high,damian2022neural,barak2022hidden,daniely2020learning}.

Recently, kernel-based algorithms that explicitly learn features have gained some attention, with prominent examples including \cite{rfm_science,li2025kernel,xia2002adaptive,fukumizu2009kernel}. 
The most relevant work for our paper is the Recursive Feature Machines algorithm, introduced in~\cite{rfm_science}, which was experimentally shown to efficiently learn multi-index models and achieve impressive performance on tabular datasets and matrix completion tasks~\cite{radhakrishnan2024linear}, among others.  RFM learns the features through iteratively learning the so-called Average Gradient Outer Product, a matrix that captures the relevant low-dimensional subspace of the true function. The (implicit) AGOP learning mechanism has also been observed in deep ReLU Neural Networks with linear activation layers~\cite{parkinson2023relu}.
The paper \cite{chen2023kernel} provided theoretical guarantees demonstrating that AGOP learning is inherent in training two-layer models, where the dimensionality reduction happens in the first linear layer weights. In a parallel line of work, \cite{follain2024enhanced} introduces a model combining kernel methods and neural networks to promote feature learning. We note also that the paper \cite{aristoff2024fast} proposed a similar normalization scheme for AGOP as ours in the implementation of RFM for better bandwidth selection. 

Another interesting line of work studies kernel-based feature selection in nonparametric statistics~\cite{jordan2021self,ruan2021taming}. In particular, \cite{ruan2021taming} shows that by explicitly penalizing the $L_1$ norm of the coordinate weights,  kernel machines can recover influential coordinates and the hierarchical interactions up to leap complexity one.

\paragraph{Learning with orthogonal polynomials.}  A recent line of work focuses on learning sparse functions that depend only on a small number of basis elements in the $L^2$-space induced by the data distribution and analyzes the learning power of neural networks. Note that sparse functions can be viewed as a special case of multi-index models.  Neural networks, unlike classical methods, do not require prior knowledge of the functional basis, and were shown to be more sample-efficient than linear methods in a variety of settings; some representative examples of this line of work are \cite{bach2017breaking,Schmidt2020nonparametric,celentano2021minimum,damian2022neural,barak2022hidden,kou2024matching,ge2017learning,lee2024neural,joshi2024complexity}. Additionally,  the recent work \cite{tan2024statistical} showed 
 that Classification and Regression Trees (CART) can learn leap-$1$  functions with $O_d(\log(d))$ samples.

\subsection{Outline}
The outline of the paper is as follows. Section~\ref{sec:notation} sets forth basic notation that we will use. Section~\ref{sec:prelin_ortho} describes the framework of orthogonal polynomials, with a particular focus on Hermite and Fourier-Walsh polynomials. Section~\ref{sec:main}  develops the main results of the paper. The final Section~\ref{sec:numerics} provides further experimental evidence that illustrates our theoretical results.

\section{Notation}\label{sec:notation}
We denote the set of natural numbers by $\mathbb{N}$ and let $\mathbb{N}^+$ denote the set $\mathbb{N}\slash\{0\}$. For a positive integer $n$, the symbol $[n]$ will stand for the set $\{1,2,\ldots,n\}$. A vector $\lambda\in\mathbb{N}^d$ with nonnegative entries is called a multi-index and we denote its modulus by $|\lambda| = \lambda_1+\lambda_2+\ldots+\lambda_d$. For any vector $x\in \R^d$ and a multi-index $\lambda\in\mathbb{N}^d$, we define the Fourier-Walsh polynomial $x^\lambda = x_1^{\lambda_1}x_2^{\lambda_2}\ldots x_d^{\lambda_d}$. Throughout the paper, given a sequence of vectors $v_i^{(k)}\in \R^d$, the subscript $i$ denotes the coordinate index of a single vector while the superscript $(k)$ denotes the index of the vector in the sequence.

 For any vectors $u,v \in \R^d$, we use the symbol $\inner{u,v} = \sum_{i=1}^d u_iv_i$ to denote their inner product and  we denote the $\ell_p$-norm of $v$ by $\norm{v}_p = (\sum_{i=1}^d |v_i|^p)^{1/p}$.
For any matrix $M$, we denote its $ij$'th entry by $M_{i,j}$, its $i$'th row by $M_{[i,:]}$, and its $j$'th column by $M_{[:,j]}$. We write $\diag(M)$ for the vector of diagonal entries of $M$, and for a vector $v$, the symbol $\Diag(v)$ denotes the diagonal matrix whose diagonal entries are the elements of $v$. The symbols $\snorm{M}$ and $\norm{M}_F$ will stand for the spectral and Frobenius norms of a matrix $M$, respectively. We use the Hadamard product, denoted by $\odot$, to represent element-wise multiplication between vectors or matrices, meaning $(A\odot B)_{i,j} = A_{i,j} B_{i,j}$.

For any function $f,g\colon\R^d\to \R$, the symbol $f \lesssim g$ means that there exists a constant $C>0$ satisfying $f(x) \leq C g(x)$ for all $x$. The symbol $f \gtrsim g$ is defined similarly.  We write $f \asymp g$ to indicate that $f$ and $g$ are of the same order, meaning both $f \lesssim g$ and $f \gtrsim g$ hold. For any functions $f,g\colon\R^d\times \mathbb{N}\to \R$, the asymptotic notation $f=O_d(g)$ means that there exists a constant $C>0$ and $\bar d\in \mathbb{N}$ such that the inequality $|f(x,d)|\leq C g(x,d)$ holds for all $x$ and all $d\geq \bar d$. Similarly, notation $f=o_d(g)$ means that for any $\epsilon>0$ there exists $d_{\epsilon}\in \mathbb{N}$ such that the estimate $|f(x,d)|\leq \epsilon g(x,d)$ holds for all $x$ and all $d\geq d_{\epsilon}$.

Finally, we use $O_{d,\P}(\cdot)$ to denote the big-O in probability relation and similarly $o_{d,\P}(\cdot)$ for the little-o in probability relation. Namely, for two sequences of random variables $f_d$ and $g_d$, the equation $f_d = O_{d,\P}(g_d)$ means that for any $\epsilon>0$, there exists a constant $C_\epsilon>0$ satisfying $\limsup_{d\to \infty} \mathbb{P}(|f_d|> C_{\epsilon} g_d)\leq \epsilon$. Similarly, the notation $f_d = o_{d,\P}(g_d)$ means that for every $\epsilon>0$ we have $\lim_{d\to \infty} \mathbb{P}(|f_d|> \epsilon g_d)=0$.

%% file: sections/prelim_ortho_poly.tex
\section{Preliminaries on orthogonal polynomials}\label{sec:prelin_ortho}
In this section, we record notation and preliminaries on orthogonal polynomials, following the standard monographs on the subject \cite{szeg1939orthogonal,
chihara2011introduction,
andrews2006classical}. Consider a
measure space $(D,\mathcal{A},\mu)$, where $D$ is a set, $\mathcal{A}$ is a $\sigma$-algebra, and $\mu$ is a measure on $(D,\mathcal{A})$. We let $L^2(D,\mu)$ denote the Hilbert space of $\mu$ square-integrable functions on $D$ equipped with the usual inner product and the induced norm
$$\langle f,g \rangle=\int fg\,d\mu\qquad \|f\|=\sqrt{\langle f,f\rangle }.$$
We will suppress the symbol $D$ from $L^2(D,\mu)$ when it is clear from context. Every Hilbert space $L^2(\mu)$ admits an orthonormal basis $\{\phi_j\}_{j\in J}$, meaning  $\phi_j$ are unit norm, pairwise orthogonal, and their linear span is dense in $L^2(\mu)$. A favorable situation occurs when $J$ is countable, in which case $L^2(\mu)$ is called separable. Any function $f\in L^2(\mu)$ in a separable Hilbert space can be expanded in the orthogonal basis:
$$\left\|f-\sum_{i=0}^{m} \langle f,\phi_i \rangle \phi_i\right\|\to 0\quad\textrm{as}\quad m\nearrow |J|.$$
Here, we identify $J$ with the contiguous subset of natural numbers $\mathbb{N}$ starting at zero. We will primarily focus on the following two examples of separable Hilbert spaces along with orthonormal bases.

\subsection{Fourier expansion on the boolean Hypercube}\label{sec:fourier_exp}
As the first example, let $\mathbb{H}^d=\{-1,1\}^d$ be the hypercube equipped with the uniform measure $\tau_d$. Then the inner product between any two functions $f,g\colon \mathbb{H}^d\to\R$ is simply 
$$\langle f,g\rangle=\frac{1}{2^d}\sum_{x\in \mathbb{H}^d} f(x)g(x).$$
An orthonormal basis on $L_2(\mathbb{H}^d)$ is furnished by the multi-linear monomials (called Fourier-Walsh)
$$x^{\alpha}:=\prod_{i=1}^d x_i^{\alpha_i}\qquad \textrm{for each}~ \alpha\in \{0,1\}^d.$$
We will denote the degree of the monomial  $x^{\alpha}$ by the symbol $|\alpha|:=\sum_{i=1}^d\alpha_i$. Notice that there is a one-to-one correspondence between binary vectors $\alpha\in\{0,1\}^d$ and subsets $S\subset [d]$ (their support). We will therefore often abuse notation and treat $\alpha$ as both a vector and a set whenever convenient. %

Thus any polynomial $f$ on the hypercube $\mathbb{H}^d$ can be expanded in the monomial basis:
$$f(x)=\sum_{\alpha\in \{0,1\}^d} b_{\alpha} x^{\alpha}\qquad \textrm{with}\qquad b_{\alpha}=\langle f,x^{\alpha}\rangle,$$
where $b_{\alpha}$ is called the Fourier coefficient of $f$ indexed by $\alpha$.
The $p$-truncation of $f$ is then defined to be the truncated series
\begin{equation}\label{eqn:trunc_fourier}
f_{\leq p}(x)=\sum_{\alpha\in \{0,1\}^d:\,|\alpha|\leq p} b_{\alpha} x^{\alpha}.
\end{equation}
Equivalently, $f_{\leq p}$ is the projection of $f$ in $L_2(\mathbb{H}^d)$ onto the span of all Fourier-Walsh monomials $x^{\alpha}$ with $|\alpha|\leq p$. The functions $f_{p}$ and $f_{>p}$ are defined in an obvious way.

The Fourier coefficients $b_{\alpha}$ have a convenient interpretation in terms of discrete derivative of $f$. Namely,  the discrete derivative of any function $g\colon\mathbb{H}^d\to\R$ in direction $x_i$ is defined to be 
$$\mathtt{D}_i g(x)=\frac{g(x)-g(x_1,\ldots, x_{i-1}, -x_i, x_{i+1},\ldots, x_d)}{2x_i}.$$
The discrete derivative with respect to a set $S=\{i_1,\ldots, i_k\}\subset [d]$ is then defined by iterating:
$$\mathtt{D}_S f(x)=\mathtt{D}_{i_1}\mathtt{D}_{i_2}\ldots\mathtt{D}_{i_k} f(x).$$
Interestingly, Fourier coefficients correspond precisely to expectations of discrete derivatives:
\begin{equation}\label{eqn:herm_iterated_der}
b_S=\E[\mathtt{D}_S f].
\end{equation}
See for example \cite[Section 2.2]{o2014analysis} for details. Therefore,  Fourier coefficients measure the sensitivity of $f$ to coordinate perturbations.

\subsection{Hermite expansion in Gauss space}\label{sec:hermite}
Let $\gamma$ be the standard Gaussian measure on $\R$. Then the Hilbert space $L_2(\R,\gamma)$ is separable. One convenient orthonormal basis $\{h_k\}_{k\geq 0}$ is obtained by performing the Gram-Schmidt process on the standard monomial basis $\{1, x, x^2,\ldots\}$. The elements of this orthonormal basis are called Hermite polynomials. They can be explicitly described in many different ways. For example, they arise as the higher-order derivatives of the Gaussian density\footnote{Note that compared to the ``probabilist's Hermite polynomials'', we have an extra scaling factor $1/\sqrt{k!}$ to ensure the unit length of the bases elements.}:
\begin{equation}\label{eqn:defn_hermite}
h_k(x)=\tfrac{(-1)^k}{\sqrt{k!}}\cdot e^{\frac{x^2}{2}} \cdot p^{(k)}(x)\qquad \textrm{where}~~p(x):=e^{-\tfrac{x^2}{2}}.
\end{equation}
The first few Hermite polynomials are
\begin{align*}
    h_0(x) = 1,~~h_1(x) = x,~~ h_2(x) =\tfrac{1}{\sqrt{2!}}(x^2 - 1), ~~ h_3(x) = \tfrac{1}{\sqrt{3!}}(x^3 - 3x).
\end{align*}
Multivariate Hermite polynomials are similarly defined as products of univariate Hermite polynomials along coordinates. Namely, the Hermite polynomial $h_{\alpha}$ with respect to a multi-index $\alpha \in \mathbb{N}^d$ is defined by 
\begin{align*}
    h_{\alpha}(x)  = \prod_{i=1}^d h_{\alpha_i}(x_i),
\end{align*}
where $h_{\alpha_i}$ is the $\alpha_i$-th univariate Hermite polynomial. The degree of $h_{\alpha}$ as a polynomial on $\R^d$ is simply the order of the multi-index $\|\alpha\|_1=\sum_{i=1}^d \alpha_i$. Hermite polynomials $\{h_{\alpha}\}_{\alpha\in \mathbb{N}^d}$ form an orthonormal basis for $L_2(\R^d,\gamma_d)$ where $\gamma_d$ is the standard Gaussian measure on $\R^d$. Thus any square-integrable function $f\in L_2(\R^d,\gamma_d)$ can be expanded in the Hermite basis:
$$f(x)=\sum_{\alpha\in \mathbb{N}^d} b_{\lambda} h_{\alpha}(x)\qquad \textrm{with}\qquad b_{\alpha}=\langle f,x^{\alpha}\rangle,$$
where the quantity $b_{\alpha}$ is called the Hermite coefficient of $f$ indexed by $\alpha$.
The $p$-truncation of $f\in L_2(\R^d,\gamma_d)$ is then defined to be the truncated series
\begin{equation}\label{eqn:gauss_exp}
    f_{\leq p}=\sum_{\alpha\in \mathbb{N}^d: \|\alpha\|_1\leq p} b_{\alpha} h_{\alpha}.
\end{equation}
Equivalently, $f_{\leq p}$ is the projection of $f$ in $L_2(\R^d,\gamma_d)$ onto the span of all Hermite polynomials $h_{\alpha}$ of degree at most $p$.  The functions $f_{p}$ and $f_{>p}$ are defined in an obvious way.

Similarly to \eqref{eqn:herm_iterated_der}, the Hermite coefficients are closely related to partial derivatives by the formula: 
\begin{equation}\label{eqn:gradient formula}
b_{\alpha}= \frac{1}{\sqrt{\alpha!}}\cdot \E\,\partial^{(\alpha)}f,
\end{equation}
where we use the shorthand notation $\alpha!=\prod_{i=1}^d \alpha_i !$ and $\partial^{(\alpha)}$ is the iterated partial derivative.

\bigskip

Consider now $\R^d$ equipped with some probability measure $\mu_d$. We will often have to control the $L_q$-norm $\|f\|_{L_q(\mu_d)}=[\E |f|^q]^{1/q}$ of a polynomial $f$ on $\R^d$.
In general, one would expect that the ratio between the $L_q(\mu_d)$ and $L_2(\mu_d)$ norms depends strongly on the dimension $d$. Interestingly, for a broad class of measures $\mu$ and functions $f$ this is not the case. Indeed, the so-called hypercontractivity inequality ensures that any polynomial $f\colon\R^d\to\R$ of degree at most $\ell$ satisfies the inequality %
\begin{equation}\label{eqn:hypercontrac} 
\norm{f}_{L_q(\mu_d)} \leq (q-1)^{\ell/2}\norm{f}_{L_2(\mu_d)}\qquad \forall q\geq 2,
\end{equation}
where $\mu_d$ can be the standard Gaussian measure~\cite[Chapter 3.2]{ledoux2013probability} on $\R^d$ or the uniform  measure on the hypercube $\mathbb{H}^d$ \cite[Chapter 9]{o2014analysis}. The %
inequality \eqref{eqn:hypercontrac} is widely used in probability and theoretical computer science, and we will use it heavily here as well.

%% file: sections/main_results.tex
\section{Main results}\label{sec:main}
In this section, we formally state the main results of the paper. To this end, fix a measure space $(D,\mathcal{A},\mu)$, where $D\subset \R^d$ is a set, $\mathcal{A}$ is a $\sigma$-algebra, and $\mu$ is a measure on $(D,\mathcal{A})$. The two main examples for us will be the uniform measure $\tau_d$ on the Hypercube $\mathbb{H}^d=\{-1,1\}^d$ and the Gaussian measure $\gamma_d$ on $\R^d$. In both of these cases, the symbol $f_{\leq p}$ will refer to the truncated functions \eqref{eqn:trunc_fourier} and \eqref{eqn:gauss_exp}, respectively. Fix now  a function $f\colon D\to\R$ with finite second moment $\E_{\mu} f^2<\infty $ and a data set $\{(x^{(i)},y_i)\}_{i=1}^n$, where the points $x^{(1)},\ldots, x^{(n)}$ are sampled independently from $\mu$ and are labeled by the function values 
$$y_i=f^*(x^{(i)}) + \varepsilon_i,$$ with noise $\varepsilon_i \overset{\text{i.i.d.}}{\sim}\mathcal{N}(0,\sigma_\varepsilon^2)$ that is independent of $x_i$. Note that since $f^*$ is defined on $D\subset\R^d$ and we will be interested in asymptotics as $d$ tends to infinity, it would be more accurate to use the notation $f^*_d$ for the sequence of ground truth functions. Although we omit the subscript $d$ to simplify notation, the reader should nonetheless keep in mind the implicit dependence of $f^*$ on the dimension $d$ throughout.

Our goal is to construct a function $\hat f$ (the predictor) from the sampled data that approximates the target function $f^*$ in the $L_2$ sense. %
We construct a predictor $\hat f$ based on Kernel Ridge Regression, henceforth abbreviated as KRR. Throughout the paper we focus on inner product kernels $K\colon\R^d\times\R^d\to \R$ having the form 
\begin{align}\label{eq:inner_product_kernel}
    K (x,y) = g\left(\frac{\langle x,y\rangle}{d}\right), 
\end{align}
for all $x$ and $y$ in the support of $\mu$ and where $g\colon\R\to \R$ is some function. We then define the following kernel induced by the derivative of $g$: 
$$K'(x,z):=g'\left(\frac{\langle x, z\rangle}{d}\right).$$ Our main results will invoke a combination of the following two regularity assumptions on $g$.

\begin{assumption}[Analytic]\label{assump:g_0}
There exists  $\varepsilon\in (0,1)$ such that the function $g$ in \eqref{eq:inner_product_kernel} is analytic on $(-\varepsilon,\varepsilon)$ and its derivatives satisfy $g^{(k)}(0)\geq 0$ for all $k\geq 0$.
\end{assumption}
\begin{assumption}[Lipschitz]\label{assump:g_1}
There exists $\varepsilon\in (0,1)$ such that the function $g$ in \eqref{eq:inner_product_kernel} has Lipschitz continuous derivative $g'$ on $(1-\varepsilon,1+\varepsilon)$. 
\end{assumption}

Classical results \cite[Theorem 13.4]{scholkopf2002learning} show that  non-negativity of derivatives in Assumption~\ref{assump:g_0} implies that the kernels $K$ and $K'$ are positive semidefinite. Inner-product kernels satisfying Assumptions~\ref{assump:g_0} and \ref{assump:g_1} appear prominently in applications and include for example the polynomial kernel $K(x,y)=(b+\langle x,y\rangle)^m$ and the exponential kernel $K(x,y)=\exp(\alpha\langle x,y \rangle)$ for any $\alpha>0$. 
 When the vectors in the support of $\mu$ have constant norm, as is the case for spherical or hypercube data, the Gaussian kernel $K(x,y)=\exp(-\|x-y\|^2/2\sigma^2)$ satisfies Assumptions~\ref{assump:g_0} and~\ref{assump:g_1}, while the the Laplace kernel $K(x,y)=\exp(-\|x-y\|/\sigma)$ satisfies Assumption~\ref{assump:g_0} but not~\ref{assump:g_1}.

Given a set of data points $x^{(1)},\ldots,x^{(n)}\in \R^d$ we form the data matrix $X\in \R^{n\times d}$ by stacking $x^{(i)}$ as its rows and define the $n\times n$ similarity matrix
$K(X,X)=\{K(x^{(i)},x^{(j)})\}_{i,j=1}^n.$ %
The kernel ridge regression (KRR) prediction function with regularization parameter $\lambda>0$ is then given by the expression
\begin{align}\label{eq:krr}
    \hat f(z) = K(z,X)\beta\qquad \textrm{where}\qquad \beta:=\round{K(X,X) + \lambda I_n}^{-1} y,  
\end{align}
where we define the vector $K(z,X):=[K(z,x^{(1)}),\ldots,K(z,x^{(n)}))]\in \R^{1\times n}$. Throughout, we assume that the regularization parameter satisfies  $\lambda=O_d(1)$.
Much of our paper focuses on analyzing the partial derivative of the prediction function $\partial_r \hat f(x)$, which takes the simple form   
\begin{align}
   \partial_{r} \hat{f} (z)  = \frac{1}{d}  K'(z, X)X_r\beta,\label{eq:derivative}
\end{align}
Here, $\partial_r$ denotes the partial derivative with respect to $x_r$ and we define the matrix $X_r:=\Diag[x_r^{(1)},\ldots,x_r^{(n)}]$.

We require one last notational ingredient before stating our results. Namely, all theorems in this section impose a mild nondegeneracy condition of the form:
    \begin{equation}\label{eqn:nondegegn}
     \displaystyle\min_{k=0,\ldots,q}g^{(k)}(0) \textrm{ and } \lambda+\displaystyle\sum_{k=q+1}^\infty g^{(k)}(0) \textrm{ are strictly positive.}
    \end{equation}
  Here, the constant $q\in \mathbb{N}$ will be specified each time that \eqref{eqn:nondegegn} is invoked in theorem statements. In particular, note that if all the derivatives $g^{(k)}(0)$ are strictly positive, then \eqref{eqn:nondegegn} holds even with $\lambda=0$.

\subsection{Results for Gaussian data}

Our first main theorem shows that for Gaussian data in the regime $n=d^{p+\delta}$ with $\delta\in (0,1/2)$ and $p\in \{1,2\}$, the empirical weights  
$\E_n (\partial_r \hat f)^{2}:=\frac{1}{n}\sum_{i=1}^n (\partial_r \hat f(x^{(i)}))^{2}$ are asymptotically consistent for estimating the weight $\E(\partial_r f^*_{\leq p})^{2}$ of the truncated function $f_{\leq p}$. This theorem can be understood as a higher-order analogue of the bound on generalization error established in \cite{mei2022generalization,ghorbani2021linearized}. Although we only prove Theorem~\ref{thm:main_gauss} in the setting $p\in \{1,2\}$, we conjecture that the analogous result holds for all $p\in \mathbb{N}$. The proof of the theorem appears in Appendix~\ref{proof:gauss}.

\begin{theorem}[Gradient consistency with Gaussian data]\label{thm:main_gauss}
Suppose that Assumptions~\ref{assump:g_0} and \ref{assump:g_1} hold. Consider the regime $n = d^{p+\delta}$ where $\delta \in(0,1/2)$ is a constant and $p\in \{1,2\}$. Suppose moreover that condition \eqref{eqn:nondegegn} holds with $q=p$.  Then for any $\epsilon>0$ and $\eta>0$, the KRR predictor $\hat f$ satisfies:
\begin{align}\label{eqn:basic_gauss} 
   \max_{r=1,\ldots,d}~\abs{\E_n (\partial_r \hat f)^{2} - \E{\round{\partial_{r} f^*_{\leq p} }^2}} = O_{d,\P}\round{d^{\delta - 1/2+\epsilon}+ d^{-\delta/2+\epsilon}}\cdot \round{\norm{f^*}_{L_2}^2+  \norm{f^*_{>2}}_{L_{2+\eta}}^2+ \sigma_\varepsilon^2}.
\end{align}
\end{theorem}

A few comments are in order. First, the bound \eqref{eqn:basic_gauss}  depends on the target function $f^*$ only through the moments, $\norm{f^*}_{L_2}^2$ and $\norm{f_{>2}^*}_{L_{2+\eta}}^{2}$, where $\eta>0$ is arbitrary. In the particular case when $f^*$ is a finite degree polynomial with a bounded second moment $\norm{f^*}_{L_2}^2<\kappa$ uniformly in dimension, then the higher-order moment $\norm{f_{>2}^*}_{L_{2+\eta}}^{2}$ is automatically uniformly bounded due to the Hypercontractivity inequality \eqref{eqn:hypercontrac}. Consequently, in this case the right side of \eqref{eqn:basic_gauss}  tends to zero in probability as $d$ tends to infinity. The assumptions on the derivatives of $g$ are mild and hold for example if $K$ is an exponential kernel or a polynomial kernel with degree $m\geq 3$. In both of these examples, one can even set $\lambda=0$.

It is worthwhile to note that the quantities $\E{\round{\partial_{r} f^*_{\leq p} }^2}$ have an intuitive variational interpretation in terms of the derivatives of $f^*$ itself. Namely, suppose that we may write $f^*=\sum_{\alpha\in \mathbb{N}^d} b_{\alpha}h_{\alpha}$, where $h_{\alpha}$ are the Hermite polynomials. Then for $p=1$, we have the expression 
$$\E{\round{\partial_{r} f^*_{\leq 1} }^2}=b_{e_r}^2=(\E f^*(x)x_r)^2=(\E \partial_r f^*)^2,$$
where the last equality follows from \eqref{eqn:gradient formula}. Similarly, for the case $p=2$ simple arithmetic shows the expression
\begin{align*}
\E{\round{\partial_{r} f^*_{\leq 2} }^2}&=\E(b_{e_r}+\sqrt{2}b_{2e_r}x_r+\sum_{i\neq r}b_{e_i+e_r}x_i)^2\\
&=b_{e_r}^2+2\alpha_{2e_r}^2+\sum_{j\neq r}b_{e_j+e_r}^2,\\
&=(\E\partial_r f^*)^2+\sum_{j\in [d]}(\E\partial^{(rj)} f^*)^2,
\end{align*}%
 where again the last equality follows from \eqref{eqn:gradient formula}. Thus in both cases, the quantity $\E{\round{\partial_{r} f^*_{\leq p} }^2}$  measures the total magnitude of the expected derivatives of the target function $f^*$ with respect to $x_r$.

\subsection{Results for hypercube data}
For the rest of the section, we focus entirely on learning polynomials from data that is sampled uniformly over the hypercube. We formalize this setting in the following assumption. %
\begin{assumption}[Hypercube data]\label{assump:finite_basis_hypercube}
The measure $\mu$ is the uniform measure on the hypercube $\{-1,1\}^d$ and the underlying true function $f^*=\sum_{S\subset[d]}b_S\phi_S$ is a polynomial of constant degree $\ell$.
\end{assumption}

We will be interested in learning functions on $\R^d$ that only depend on a small portion of all coordinates. With this in mind, we introduce the following definition of ``influential coordinates''. Conceptually, $x_r$ is an influential coordinate of $g$ if variations in $x_r$ lead to nontrivial variations in $g$.

\begin{definition}[Influential coordinates]
{\rm
For a function $g\in L_2(\mu)$, a coordinate $x_r$ is said to be {\em influential} if the inequality $\E(\partial_{r} g)^2 >0$ holds.}
\end{definition}

The following is the analogue of Theorem~\ref{thm:main_gauss} for hypercube data and which holds for all values of $p\in \mathbb{N}$. Indeed, we prove an even stronger estimate. In the regime $n=d^{p+\delta}$, the error $\E_n (\partial_r \hat f)^{2} - \E{\round{\partial_{r} f^*_{\leq p} }^2}$ decays at slow rate $\approx d^{-\delta/2}$ when $r$ is an influential coordinate and at a fast rate $d^{-1}$ if $r$ is not influential! We will see that the improved rate for non-influential coordinates is the key reason why reweighted kernel machines can efficiently learn hierarchical polynomials. The proof of the theorem appears in Appendix~\ref{proof:main_hypercube}.

\begin{theorem}[Gradient consistency on hypercube]\label{thm:main_hypercube}
Suppose that Assumptions~\ref{assump:g_0} and \ref{assump:finite_basis_hypercube} hold. Consider the regime $n = d^{p+\delta}$ where $\delta \in(0,1)$ is a constant and suppose that condition \eqref{eqn:nondegegn} holds with $q=p$. 
 Then for any $\epsilon>0$, the KRR predictor $\hat{f}$ satisfies:
\begin{align*}
    &~~~\abs{\E_n (\partial_r \hat f)^{2} -\E{\round{\partial_{r} f^*_{\leq p} }^2} - c\cdot (d^{2\delta-2})\cdot \E(\partial_r f^*_{p+1})^2}\\
    &\leq O_{d,\P}(1)\cdot\left( \tfrac{\norm{f^*}_{L_2}^2+\sigma_\varepsilon^2}{d}+ \tfrac{\norm{\partial_r f^*_{\leq p}}_{L_2}^2 + d^{2\delta -2}\norm{\partial_r f^*_{ p+1}}_{L_2}^2}{d^{\max(\delta,1-\delta)/2-\epsilon}} +\tfrac{\norm{f^*}_{L_2}\norm{\partial_{r} f^*_{\leq p} }_{L_2}}{d^{1/2}}+ \tfrac{(\norm{f^*}_{L_2}  +\sigma_\varepsilon)\norm{\partial_{r} f^*_{p+1} }_{L_2}}{d^{3/2-\delta}}\right).
\end{align*}
simultaneously for all $r=1,\ldots,d$ where $c>0$ is a constant.

\end{theorem}

 Unpacking Theorem~\ref{thm:main_hypercube} we see that there are two interesting regimes corresponding to whether $\delta$ is greater or less than $1/2$. 
Namely, the following estimate holds 
\[ 
\E_n(\partial_j\hat f)^2 \approx
\begin{cases}
\E(\partial_{j} f^*_{\leq p})^2  
& \quad \text{if } \delta \in (0, \tfrac{1}{2}) \\
\E(\partial_{j} f^*_{\leq p})^2 + c \cdot d^{-2(1-\delta)} \E(\partial_{j} f^*_{p+1})^2 
& \quad \text{if } \delta \in (\tfrac{1}{2}, 1)
\end{cases},
\]
uniformly over all the coordinates up to a small order error.
 In words, this means that in the small oversampling regime $\delta \in (0, \tfrac{1}{2})$, the KRR estimator learns the coordinate weights of the $p$'th truncation $\E(\partial_{j} f^*_{\leq p})^2$. Interestingly, in the larger oversampling regime  $(\tfrac{1}{2}, 1)$, the KRR estimator can also detect the coordinate weights of the $(p+1)$-order term $\E(\partial_{j} f^*_{p+1})^2$ when the lower-order weight $\E(\partial_{j} f^*_{\leq p})^2$ is zero. 
 
 We note that Theorem~\ref{thm:main_hypercube} requires $f^*$ to be a polynomial of constant degree per Assumption~\ref{assump:finite_basis_hypercube}. This requirement was only needed to establish the fast rate $d^{-1}$ for learning non-influential coordinates,  namely the norm $\norm{\partial_r f^*_{\leq p}}_{L_2} = \norm{\partial_r f^*_{p+1}}_{L_2} = 0$. The slow rate $O_{d,\P}(d^{-\delta/2+\epsilon}+d^{(\delta-1)/2+\epsilon})$ is valid without this requirement, thereby directly paralleling Theorem~\ref{thm:main_gauss}.
We also note that as in the Gaussian case, the quantities $\E{\round{\partial_{r} f^*_{\leq p} }^2}$ have an intuitive variational interpretation in terms of the derivatives of $f^*$ itself. Namely, we have 
\begin{equation}\label{eqn:grad_interp_cube}
\E{\round{\partial_{r} f^*_{\leq p} }^2}=\E\left(\sum_{S: \,r\in S,\,|S|\leq p} b_S x^{S\setminus\{r\}}\right)^2=\sum_{S: \,r\in S,\,|S|\leq p} b_S^2=\sum_{S: \,r\in S,\,|S|\leq p}(\E[\mathtt{D}_S f^* ])^2,
\end{equation}
where the last inequality follows from \eqref{eqn:herm_iterated_der}.

Notice that the noise/error term in the theorem scales as $d\inv$ and therefore the empirical weights estimator $\E_n (\partial_r \hat f)^{2}$ fails to distinguish the coordinate weight of the $(p+1)$-order term $\E(\partial_{j} f^*_{p+1})^2$ from the noise when $\delta \in (0, \tfrac{1}{2})$. To address this limitation, we introduce the Derivative Norm (DN) estimator $\D(w)$, defined as
\begin{align}
   \mathcal{D}_{j}(w):=\beta_w\tran \tfrac{\partial K_w(X,X)}{\partial w_j} \beta_w, 
\end{align}
where $\beta_w\in \R^n$ denotes the dual coefficients $\beta_w\in \R^n$ of the KRR predictor $f_w$ induced by the reweighted kernel $K_w$.  Analogous to Theorem~\ref{thm:main_hypercube}, the following result shows that $\D({\bf 1}_d)$ provides an estimate of the population coordinate weights. The proof, which also extends to general weights, appears in Appendix~\ref{proof:learning_hypercube_gradient_rkhs}.

\begin{theorem}[Derivative norm estimator on hypercube]\label{thm:main_hypercube_rkhs}
Suppose that Assumptions~\ref{assump:g_0} and \ref{assump:finite_basis_hypercube} hold. Consider the regime $n = d^{p+\delta}$ where $\delta \in(0,1)$ is a constant and suppose that condition \eqref{eqn:nondegegn} holds with $q=p+1$. 
 Then for any $\epsilon>0$, the following holds:
\begin{align*}
    &~~~\abs{~\tfrac{1}{n}\D_r({\bf 1}_d) -\sum_{k=0}^{p+1}\Theta_d(d^{-|k-p-\delta|})\cdot \E(\partial_r f^*_{k})^2}\\
    &\leq O_{d,\P}(1)\cdot\left( \tfrac{\norm{f^*}_{L_2}^2 + \sigma_\varepsilon^2}{d^{1-\epsilon}} + \tfrac{\sum_{k=0}^{p+1}d^{-|k-p-\delta|}\norm{\partial_r f^*_{k}}_{L_2}^2 }{d^{\max(\delta,1-\delta)/2-\epsilon}} + \tfrac{\sum_{k=0}^{p+1}d^{-|k-p-\delta|}\norm{\partial_r f^*_{k}}_{L_2}(\norm{f^*}_{L_2} +\sigma_\varepsilon)}{d^{1/2-\epsilon}}\right).
\end{align*}
simultaneously for all $r=1,\ldots,d$.
\end{theorem}

The estimator $\mathcal{D}_{j}(w)$  addresses the limitation of $\E_n(\partial_j\hat f)^2$ in estimating the coordinate weights of the $(p+1)$-order term when $\delta\in(0,\tfrac{1}{2})$. 
Indeed, in this regime,   the last term in the sum is $\Theta_d(d^{1-\delta}) \E(\partial_{j} f_{p+1})^2$ and the leading coefficient $d^{1-\delta}$ is large compared to the error term $O_{d,\mathbb{P}}(d\inv)$. In contrast, for the coordinate weights of the $p$'th truncation $\E(\partial_{j} f_{\leq p})^2$, this estimator systematically underestimates the true coordinate weights compared to $\E_n(\partial_j\hat f)^2$.
A simple average of the empirical weight estimator and the DN estimator enjoys the best properties of both, as the following corollary shows.

\begin{corollary}\label{cor:combined_estimator}
Suppose that Assumptions~\ref{assump:g_0} and \ref{assump:finite_basis_hypercube} hold. Consider the regime $n = d^{p+\delta}$ where $\delta \in(0,1)$ is a constant and suppose that condition \eqref{eqn:nondegegn} holds with $q=p+1$. 
 Then the following holds for any average parameter $\alpha\in[0,1]$:
\begin{equation}\label{eqn:main_est_combined_copy}
\left|(1-\alpha)\cdot \E_n(\partial_j\hat f)^2 + \alpha\cdot \tfrac{1}{n}\gD_j({\bf 1}_d)-  (1+o_{d,\P}(1)) \round{(1-\alpha)\E(\partial_{j} f_{\leq p})^2 + c\cdot\alpha  d^{\delta-1}\E(\partial_{j} f_{p+1})^2 }\right|=O_{d,\mathbb{P}}(d\inv).
\end{equation}simultaneously for all $r=1,\ldots,d$ where $c>0$ is a constant.
\end{corollary}

In light of Corollary~\ref{cor:combined_estimator}, an appealing strategy to learn the function $f^*$ is to iteratively reweight the data using a combination of the empirical weights $\E_n (\partial_r\hat f)^2$ and DN-estimator $\D_r(w)$ corresponding to each coordinate $r=1,\ldots, d$, then refit the data with kernel ridge regression, and so forth. This retraining procedure was informally summarized in Algorithm~\ref{alg:rfm} and we formalize it here as Algorithm~\ref{alg:rfm2}.

\begin{algorithm}[h!]
\caption{$\mathtt{IRKM}(\alpha)$ \hfill{\bf I}teratively {\bf R}eweighted {\bf K}ernel {\bf M}achines}
\begin{algorithmic}[1] %
\State $\textbf{Require:}$ Iteration count $T\in \mathbb{N}$, kernel function $K$ on $\R^d$, ridge parameter $\lambda\geq 0$, safeguard $\epsilonS>0$, average parameter $\alpha\in[0,1]$.
\State \textbf{Initialize:} $w={\bf 1}_d$.
\For{$t=1,\ldots, T$}:
\State Sample fresh data $\{(x^{(i)},y_i)\}_{i=1}^n$ and define the kernel $K_w(x,z)=K(\sqrt{w}\odot x,\sqrt{w}\odot z)$.\label{line:resample}
\State Obtain predictor $\hat f_t$ from Kernel Ridge Regression:
$$\hat f_t(z):=K_w(z,X)\beta\qquad \textrm{where}\qquad \beta:=(K_{w}(X,X)+\lambda I)^{-1}y.$$
\State Compute the first weight vector: $w^{[1]}=\epsilonS {\bf 1}_d+\cdot \frac{1}{n}\sum_{i=1}^n [\nabla_x\hat f_t(x^{(i)})]^{\odot 2}  $.
\State Compute the second weight vector: $w^{[2]}=\epsilonS {\bf 1}_d+ \frac{1}{n}\D(w) \odot w $.  
\State Normalize the weights $w^{[k]}=\frac{d}{\|w^{[k]}\|_1} \cdot w^{[k]}$,~~$k = 1,2$.
\State Update the weight vector $w = (1-\alpha)\cdot w^{[1]} + \alpha\cdot w^{[2]}$.
\EndFor
\end{algorithmic}\label{alg:rfm2}
\end{algorithm}

Let us make a few comments. The algorithm maintains a sequence of prediction functions $\hat f$ and weights $w$. Each iteration $t$ of the algorithm begins by drawing fresh data $\{(x^{(i)},y_i)\}_{i=1}^n$ with  $x^{(1)},\ldots, x^{(n)} \overset{\text{iid}}{\sim}\mu$. The algorithm then forms a reweighted kernel $K_w(x,z)$, which in essence amounts to reweighing the feature vectors. The next prediction function $\hat f_t$ is then defined to be the KRR predictor obtained by fitting the data $\{(x^{(i)},y_i)\}_{i=1}^n$ with the kernel $K_w$. The weight vector $w$ is then updated using the average of the empirical coordinate weights and the DN estimator,  safeguarded by adding a small positive constant $\epsilonS>0$. This safeguard turns out to be important both in theory and in numerical experiments. The reason is intuitively clear: in the extreme case, a weight $w_r$ on coordinate $x_r$ that is set to zero in iteration $t$ will remain zero for all future iterations. In other words, without the safeguard it would be impossible for the algorithm to recover from incorrectly classifying a coordinate as noninfluential. In the update, we obtain the second weight vector $w^{[2]}$ by taking the Hadamard product of the DN estimator with the current weight vector $w$ analogously to a multiplicative weight update algorithm. Finally, we normalize the two intermediate weight vectors and combine them using an averaging parameter $\alpha$ to obtain the new weight vector $w$.

The analysis of Algorithm~\ref{alg:rfm2} proceeds in two stages. First, we describe which coordinates can be learned in a single step of $\mathtt{IRKM}(\alpha)$ for a fixed weight vector $w$ (regardless of how $w$ originated). Secondly, we monitor the evolution of the weights and inductively apply the one-step guarantees in order to describe all the coordinate that are learned along with the generalization error of the final predictor.

The first stage, which states which coordinates can be learned in a single step of $\mathtt{IRKM}(\alpha)$, is summarized in Theorem~\ref{thm:learning_hypercube}. We defer this theorem to Appendix~\ref{sec:one_step_learning} since its statement is quite technical.
For the second stage, we introduce the following definition from \cite{abbe2023sgd} to describe the class of functions that $\mathtt{IRKM}(\alpha)$ is capable of learning.

\begin{definition}[Leap complexity]{\rm 
    A function $f=\sum_{S\subset[d]}b_S\phi_S$ has {\em leap complexity $k$}, denoted ${\rm Leap}(f)=k$, if the nonzero coefficients $\{b_S\}$ can be ordered in such a way that the corresponding sets $\{S_{i}\}_{i=1}^m$ satisfy 
$$|S_1|\leq k \qquad \textrm{and}\qquad \left|S_i\setminus \midcup_{1\leq j<i} S_j\right|\leq k,$$
and $k$ is the smallest integer for which this is possible.}
\end{definition}

In other words, $f$ has leap complexity $k$ if $k$ is the smallest integer such that its nonzero coefficients can be ordered in a way that the corresponding sets grow by at most $k$ indices. As a concrete example, the polynomial 
\begin{equation}\label{eqn:repeat_hiarch}
f^*(x)=x_1 + x_2 + x_1x_2x_3+x_1x_2x_3x_4,
\end{equation}
which figured prominently in the introduction \eqref{eqn:first_exa2}, has leap complexity ${\rm Leap}(f^*)=1$ since we may define $S_1=\{1,2\}$, $S_2=\{1,2,3\}$, and $S_3=\{1,2,3,4\}$. We note that in the earlier work \cite{abbe2022merged}, functions with leap complexity one were said to satisfy the ``merged staircase property''.

In the same way that we may isolate a low-degree component of a function, we may define a low-leap complexity component. This is the content of the following definition.

\begin{definition}[Maximum $k$-leap complexity component]\label{def:leap_complexity}
{\rm  For any function $f = \sum_{\I\subset [d]} b_\I \phi_\I$,  its {\em maximum $k$-leap complexity component}, denoted by $\mathsf{L}_k f$, is defined as
 \begin{align}
     \mathsf{L}_k f := \sum_{\I \in \S} b_\I \phi_\I,~~~{\rm where}~~\S =\argmax_{\S' \subseteq [d]}\left\{|\S'|:~{\rm Leap}\left(\sum_{\I \in \S'} b_\I \phi_\I\right)\leq k\right\}.
 \end{align}}
\end{definition}

For example, the maximum $1$-leap component of the function $g(x)=x_1+x_2+x_1x_2x_3+x_3x_4x_6$ is $\mathsf{L}_1 f=x_1+x_2+x_1x_2x_3$. The following theorem shows that in the sampling regime $n=d^{p+\delta}$, a constant number of iterations of $\mathtt{IRKM}(\tfrac{1}{2})$ learns the maximum $(p+1)$-leap  component of a sparse polynomial. The proof appears in Appendix~\ref{proof:multi_step}.

\begin{theorem}[Learning the maximal leap component]\label{cor:multi_step}
Suppose that Assumptions~\ref{assump:g_0} and \ref{assump:finite_basis_hypercube} hold and that the number of influential coordinates of $f^*$ scales as $\Theta(d^s)$ for some constant $s\in [0,1)$. Consider the regime $n = d^{p+\delta}$ where $\delta\in(0,1)$ is a constant and suppose that condition \eqref{eqn:nondegegn} holds with $q=\ell$. Suppose the  sparsity level satisfies $s < \frac{\delta}{\ell+1}$ , choose any  $\eta\in\round{1 - \frac{\delta}{\ell+1},1 -s}$ and mixing parameter $\alpha\in(0,1)$.
Then there exists a horizon $T = O_d(1)$ such that the predictor $\hat f_T$ returned by $\mathtt{IRKM}(\alpha)$ (Algorithm~\ref{alg:rfm}) with safeguard $\epsilonS = d^{-\eta}$
satisfies:
\begin{align*}
\norm{\hat{f}_{T} - \mathsf{L}_{p+1} f^*}^2_{L_2} = o_{d,\P}(1)\cdot (\norm{f^*}_{L_2}^2+\sigma_\varepsilon^2).
\end{align*}
\end{theorem}

Let us illustrate Theorem~\ref{cor:multi_step} by applying it to the function \eqref{eqn:repeat_hiarch}. Recall that this function has leap complexity $p=1$. Then applying Theorem~\ref{cor:multi_step} with $s=0$ and $n=d^{\delta}$ for any $\delta\in (0,1)$, we deduce that after a constant number of steps $T = O_d(1)$ the  $\mathtt{IRKM}(\tfrac{1}{2})$ algorithm returns a predictor $\hat f_T$  that is a consistent estimator of $f^*$ in the sense that $\|\hat f_T - f^*\|^2_{L_2} = o_{d,\P}(1)\cdot (\norm{f^*}_{L_2}^2+\sigma_\varepsilon^2)$.

An important special case occurs when $f^*_{\leq p+1 }$ and $f^*$ share the same set of influential coordinates meaning that $\E(\partial_r f^*)^2$ is zero if and only if $\E(\partial_r f^*_{\leq p+1 })^2$ is zero. In this situation, we have $\mathsf{L}_{p+1}f^* = f^*$ and therefore Theorem~\ref{cor:multi_step} reduces to Theorem~
\ref{thm:gener_features_ite} stated in the introduction. As a concrete example, consider the function $f(x) = x_1 + x_2 + x_3 + x_1x_2x_3$ (Eq.~\eqref{eqn:first_exa}) discussed in the introduction. Setting $p=0$ gives $\mathsf{L}_{1}f^* = f^*_{\leq 1} = x_1 + x_2 + x_3$. Note that Corollary~\ref{cor:combined_estimator} shows that the linear component  $f_{\leq 1}(x) = x_1+x_2 +x_3$ can be weakly learned (with coefficient $d^{\delta-1}$) for any average parameter $\alpha\in(0,1)$ in the regime $n = d^{\delta}$. Thus we may instantiate Theorem~\ref{cor:multi_step} by setting $s=0$, $p=0$, and choosing any $\eta\in (1-\tfrac{\delta}{4},1)$.
Theorem~\ref{cor:multi_step} then implies that in a constant number of steps $T$ the $\mathtt{IRKM}(\tfrac{1}{2})$ algorithm return a predictor $\hat{f}_T$ that learns the third-order component $x_1x_2x_3$ and becomes an asymptotically consistent estimator for $f^*$ in the sense that $\|\hat f_T - f^*\|^2_{L_2} = o_{d,\P}(1)\cdot (\norm{f^*}_{L_2}^2+\sigma_\varepsilon^2).$

%% file: sections/proof_sketch.tex
\section{Proof sketches of Theorems~\ref{thm:main_hypercube} and \ref{cor:multi_step}}

In this section, we outline the key ingredients underpinning our results. We will focus on the case of hypercube data as this is the setting when our strongest results are applicable.

\subsection{Basic notation}
For each $\lambda\in\{0,1\}^d$, let $\phi_\lambda(x)=x^\lambda$ be the Fourier orthonormal basis element.  
For a subset $\S\subset\{0,1\}^d$ define the feature vector
$\phi_{\S}(x):=\{\phi_\lambda(x)\}_{\lambda\in\S}\in\mathbb{R}^{|\S|}$,
and the feature matrix $\Phi(X)\in\mathbb{R}^{n\times|\S|}$ whose $i$-th row is $\phi_{\S}(x^{(i)})$.
We write $\S_k:=\{\lambda:|\lambda|=k\}$ for the degree-$k$ index set, and use
\(\Phi_{\S_k}(X)\) and \(\Phi_{\le k}(X)\) for the corresponding submatrices.

An important observation is that multiplication of $x^{S}$ by a coordinate adds or removes that coordinate from the support:
\[
x_r x^{S}=
\begin{cases}
x^{S\cup\{r\}}, & r\notin S,\\
x^{S\setminus\{r\}}, & r\in S.
\end{cases}
\]
We let $\S_j$ consist of all subsets of $[d]$ of cardinality $j$.
Giving a coordinate $r$, we decompose each $\S_j$ as
\begin{align}\label{main:eq:s_j_0_1}
    \S_j^0=\{S\in \S_j: r\notin S\}\qquad \textrm{and}\qquad \S_j^1=\{S\in \S_j: r\in S\}.
\end{align}

\subsection{Kernel and feature decompositions}
Now define the matrices $K_{ij} = g\round{\nicefrac{\inner{x^{(i)},x^{(j)}}}{d}}$, $K'_{ij} = g'\round{\nicefrac{\inner{x^{(i)},x^{(j)}}}{d}}$,  and $K_\lambda:=K+\lambda I_n$. Direct algebraic manipulations show the convenient expression for the coordinate weights of the KRR predictor:
\begin{align}\label{main:eqn:deriv_est}
     &~~~\E_n(\partial_r\hat f)^2 =\frac{1}{nd^2} \|K'X_r \beta\|^2_2,
\end{align}
where $X_r \in \R^{d\times d}$ is a diagonal matrix with entry $(X_r)_{i,i} = x_r^{(i)}$ for $i\in[d]$ and $\beta=K_{\lambda}^{-1}y$ are the dual coefficients. Now a natural idea is to rewrite $K$ and $K'$ in the Fourier basis in order to identify the dominant terms in this expression. Interestingly, in this basis the two matrices are nearly diagonal up to an identity perturbation:
\begin{align}
K &= \Phi_{\le p}D\Phi_{\le p}^\top + \rho I_n + \Delta_1,\label{main:eqn:K_decomp_cube}\\
K'&= \Phi_{\le p}D'\Phi_{\le p}^\top + \rho' I_n + \Delta,\label{main:eqn:K'_decomp_cube}
\end{align}
where $\rho,\rho'>0$ are some constants,  
$\| \Delta\|_{\rm op},\|\Delta_1\|_{\rm op} = O_{d,\mathbb{P}}(d^{(\delta-1)/2+\varepsilon})$ are error terms,  
and $D$ and $D'$ are diagonal matrices satisfying
\[
D_{\S_k},D'_{\S_k} = \Theta_d(d^{-k}) I_{|\S_k|}.
\]

\subsection{Empirical weights estimator}
Inserting \eqref{main:eqn:K'_decomp_cube} into \eqref{main:eqn:deriv_est} and carefully bounding the error terms gives the intuitive expression 
\begin{align*}
    \E_n(\partial_r\hat f)^2 = \frac{1+o_{d,\P}(1)}{nd^2} \norm{ \beta^\top X_r\Phi_{\leq p} D'\Phi_{\leq p}^\top }^2_2+O_{d,\P}(d\inv \sigma_\varepsilon^2).
\end{align*}
Now one can show that $\tfrac{1}{n}\Phi_{\leq p}\tran \Phi_{\leq p}$ is almost an isometry in the sense that
$\snorm{  \tfrac{1}{n}\Phi_{\leq p}\tran \Phi_{\leq p} - I_n} ={O}_{d,\P}(d^{-\delta/2+\epsilon})$. Therefore, after controlling the error terms, we can further simplify to
\begin{align}
\E_n(\partial_r\hat f)^2=\tfrac{1+o_{d,\P}(1)}{d^2}\cdot \|\beta \tran X_r\Phi_{\leq p} D' \|_2^2.\label{main:eqn:basic:needed}
\end{align}
Next, let us focus on rewriting the product $X_r\Phi_{\leq p}$ in the Fourier basis. To this end, notice that multiplication by $x_r$ shifts degrees. Simple algebra and orthogonality considerations show that the right-side of \eqref{main:eq:agop} is proportional to
\begin{align}\label{main:eq:decompose_norm}
     \sum_{j=1}^{p+1}\|\beta \tran   \Phi_{\S_j^1} (\tfrac{1}{d} D_{\S_{j-1}^0}') \|_2^2 + \sum_{j=0}^{p-1}\|\beta\tran    \Phi_{\S_j^0} (\tfrac{1}{d}D_{\S_{j+1}^1}') \|_2^2.
\end{align}
Writing $y=f^*(x)+\bm{\varepsilon}=\sum_{S}b_S x^S +\bm{\varepsilon}$ in the Fourier basis, we decompose $\beta=K_{\lambda}^{-1}y$ as:
\begin{align}\label{main:eq:decompose_beta}
    \beta =  \sum_{j=1}^{p+1}K_\lambda\inv\Phi_{\S_{j}^1} b_{\S_{j}^1} + \sum_{j=0}^{p-1}K_\lambda\inv\Phi_{\S_{j}^0} b_{\S_{j}^0} + K_\lambda\inv (\Phi_\C b_{\C}+\bm{\varepsilon})
\end{align}
where $\C := 2^{[d]} \slash \round{\left(\bigsqcup_{j=0}^{p-1} S_j^0\right)\cup \left(\bigsqcup_{j=1}^{p+1} S_j^{1}\right)}$ and $\bm{\varepsilon}\in\R^n$ denotes the noise vector with $\varepsilon_i \sim \mathcal{N}(0,\sigma_\varepsilon^2)$.

Now the key observation is  $K_\lambda\inv \Phi_{\S_k}$ is essentially a multiple of $\Phi_{\S_k}$, where the multiple depends on whether k is smaller or greater than $p$. This is the content of the following proposition, which follows from the arguments in the proof of Theorem~\ref{thm:main_hypercube}.
\begin{proposition}\label{main:prop:property_k}
In the regime $n = d^{p+\delta}$, the following estimate holds
\begin{align}
    K_\lambda\inv \Phi_{\S_k} b_{\S_k} \approx\begin{cases}
        \Phi_{\S_k} (nD_{\S_k})\inv b_{\S_k},~~~&{\rm if}~k\leq p,\\
        \Phi_{\S_k} b_{\S_k} ,~~~&{\rm if}~k> p,\label{main:eq:property_k}
    \end{cases} 
\end{align}
up to an $o_{d,\mathbb{P}}(\|\cdot\|_2)$ error.
\end{proposition}

Note that $\tfrac{n}{d}D_{\S_{k}} = \Theta_d(d^{\delta-1}) I_{|\S_k|}$ for $k\leq p$ and therefore $K_\lambda\inv$ roughly acts on $\Phi_{\S_k}$ as a contraction. Combining \eqref{main:eq:decompose_norm}, \eqref{main:eq:decompose_beta}, and Proposition~\ref{main:prop:property_k} and using orthogonality of Fourier blocks, we can obtain
\begin{align}\label{main:eq:norm_decomp}
     \|\beta \tran X_r\Phi_{\leq p} (\tfrac{1}{d}D') \|_2^2 = A+B + O_{d,\P}(d\inv),
\end{align}
where $A$ collects the $\S_j^1$ terms and $B$ the $\S_j^0$ terms.  A detailed calculation yields
\begin{align*}
    A\approx  \norm{b_{\S_{p+1}^1}\tran \underbrace{\round{\tfrac{1}{n}\Phi_{\S_{p+1}^1}\tran  \Phi_{\S_{p+1}^1}}}_{\approx I_{|\S_{p+1}^1|}} (\tfrac{n}{d}D_{\S_{p}^0}') }_2^2 +\sum_{j=1}^{p}\norm{b_{\S_j^1}\tran D_{\S_{j}^1} \underbrace{\round{\tfrac{1}{n}\Phi_{\S_j^1}\tran  \Phi_{\S_j^1}}}_{\approx I_{|\S_j^1|}} (\tfrac{1}{d}D_{S_{j-1}^0}') }_2^2.
\end{align*}
Noting that $(\tfrac{n}{d}D_{\S_{p}^0}') = \Theta_d(d^{\delta-1}) I_{|\S_p^0|}$ in the first term and that $D_{\S_j^1}$ cancels out with $(\tfrac{1}{d}D_{S_{j-1}^0}') $ in the second term, the above equation reduces to
\[
A\approx \Theta_d(d^{2\delta-2})\|b_{\S_{p+1}^1}\|_2^2+\sum_{j=1}^p \|b_{\S_j^1}\|_2^2.
\]
A similar argument can be applied to show $B=O_{d,\mathbb{P}}(d^{-2})$.
Combining all the ingredients yields the desired expression
\begin{equation}\label{main:eq:agop}
\E_n(\partial_r\hat f)^2
\approx \Theta_d(d^{2\delta-2})\|b_{\S_{p+1}^1}\|_2^2+\sum_{j=1}^p \|b_{\S_j^1}\|_2^2 + O_{d,\P}(d\inv \sigma_\varepsilon^2).
\end{equation}

\subsection{DN estimator for $w=\mathbf 1_d$} 
The computations for the DN estimator with $w = {\bf 1}_d$ proceed similarly. Namely, elementary algebra shows:
\begin{align*}
    \D_r({\bf 1}_d) = \frac{1}{nd} \|(K')^{\frac{1}{2}}X_r \beta\|^2_2.
\end{align*}
Analogously to the argument that led to \eqref{main:eqn:basic:needed}, the estimator $\D_r$ can be approximated by 
\begin{align*}
  \D_r({\bf 1}_d)\approx   \norm{ \beta\tran \Phi_{\leq p} X_r \round{ \frac{1}{nd}D'}^{\frac{1}{2}}}_2^2.
\end{align*}
Compared to Eq.~\eqref{main:eqn:basic:needed}, we observe that the only difference is that $(d\inv D')$ is replaced by $\round{ \frac{1}{nd}D'}^{\frac{1}{2}}$. Continuing the same argument, but with this different scaling yields the approximation
\begin{align}\label{main:eq:nd}
    \D_r({\bm 1}_d) \approx\sum_{j=1}^{p+1}\Theta_{d,\P}(d^{-|j-p-\delta|}) \norm{b_{\S_j^1}}_2^2 + O_{d,\P}(d^{-1+\epsilon}\sigma_\varepsilon^2).
\end{align}

\subsection{Extension to general weights}
Now we extend the proof to the case where weights are present and highlight how the iterative process $\mathtt{IRKM}(\alpha)$ can learn hierarchical functions. We consider a weight $w\in\R^d_+$ that satisfies $\norm{w}_1 = d$. Along the iterations of $\mathtt{IRKM}(\alpha)$, the coordinates of $w$ may exhibit different scalings in $d$. It will be important to monitor these dimension dependencies. To this end, fix a partition $\pi_N=\{\T_1,\ldots,\T_N\}$ of $[d]$ with
\begin{align}\label{main:eq:t_k_w_k}
    |\T_k| = \Theta_d(d^{s_k}),~~~w_r = \Theta_d(d^{1-\lambda_k}),~~\forall r\in\T_k.
\end{align}
For a tuple $\mT = (T_1, T_2,\ldots,T_N)$ with $T_k\in 2^{\T_k}$ for $k\in[N]$, we define its corresponding Fourier feature
\begin{align}
    \phi_{\mT} := \phi_{\cup_{k=1}^N T_k}= \prod_{k=1}^N x^{T_k}.
\end{align}
and we define its {\it effective degree} with respect to $w$ as
\begin{align}\label{main:eq:eff_degree}
    p_{\rm eff}(\mT) := \sum_{k=1}^N \lambda_k \cdot |T_k|,
\end{align}
As a concrete example for the initial weights $w = {\bf 1}_d$, we have $k=1$, $\T_1 = [d],$ and $\lambda_1 = 1$. Then the effective degree of any $\mT$ is simply $|T_1|$, which is exactly the degree of the feature. In the more general setting, the effective degree replaces the role of the polynomial degree $p$. Understanding how this quantity evolves along the iterations is key to obtaining guarantees for $\mathtt{IRKM}(\alpha)$.

Consider the empirical weights for the inner product kernel $(K_w)_{ij} = g\round{\nicefrac{\inner{\sqrt{w}\odot x^{(i)},\sqrt{w}\odot x^{(j)}}}{d}}$ and set $K_{w,\lambda}:=K_w+\lambda I_n$. As in Eq.~\eqref{main:eqn:deriv_est}, direct manipulations show
\begin{align}\label{main:eqn:deriv_est_w}
         \E_n\round{\partial_{r}  \hat{f}_{w}}^2= \frac{w_r^2}{nd^2} \norm{    K_{w}' X_r K_{w,\lambda}^{-1} y}^2_2.
\end{align}
Next, in parallel to Eq.~\eqref{main:eqn:K_decomp_cube}, we can decompose $K_w$ on Fourier blocks truncated now by effective degree:
\begin{align}\label{main:eq:kw_decomp}
    K_w = \Phi_{\alp}\widehat{D} \Phi_{\alp}\tran + \rho I_n+\Delta_1,
\end{align}
where
\begin{align*}
      \alp:=\biground{\mT = (T_1,T_2,\ldots,T_N)\;\middle|\; T_k\in 2^{\T_k}~~{\rm and}~p_{\rm eff}(\mT) \leq p+\delta},
\end{align*}
The diagonal entries of $\hat D$ scale as
\begin{align*}
     \widehat{D}_{\mT,\mT} = \Theta_d(d^{- p_{\rm eff}(\mT)}),
\end{align*}
which is a direct analog of $D$.  Note that for the initial weights $w = {\bf 1}_d$, since the effective degree is simply the degree of the feature, the set $\alp$ contains all Fourier features with degree not exceeding $p$. Thus $\alp$ is the analogues of all Fourier features of degree bounded by $p$.

Define the dual coordinates  $\beta_w:=K_{w,\lambda}\inv y$. In parallel to Eq.~\eqref{main:eqn:basic:needed}, the following term will dominate the empirical weights:
\begin{align*}
\E_n\round{\partial_{r}  \hat{f}_{w}}^2\approx    \norm{\beta_w \tran X_r\Phi_{\alp} (\tfrac{w_r}{d}\hd'_{\alp})}_2^2,
\end{align*}
where $\widehat D'$ is the diagonal matrix appearing in the decomposition of $K_w'$. 
Fix now a coordinate index $r$ and let $\iota_r$ be the unique index satisfying $r\in\T_{\iota_r}$.
We now define
\begin{align*}
    \A_{p_{\delta}}^+(r):=\biground{\mT = (T_1,T_2,\ldots,T_N)\;\middle|\; T_k\in 2^{\T_k}~~{\rm and}~0< p_{\rm eff}(\mT)-(p+\delta) < \lambda_{\iota_r}}.
\end{align*}
Henceforth, we will suppress the subscript $r$ on $\iota_r$ if it is clear from context.
We will see that $\A_{p_{\delta}}^+(r)$ plays the role of degree $p+1$ feasrures in the equiweighted case.
Analogous to \eqref{main:eq:agop} we can derive the estimates for the general weights:
\begin{equation}\label{main:eq:agop_w}
\E_n(\partial_r\hat f_w)^2
\approx \sum_{\mT\in \A_{p_{\delta}}^+(r)} \Theta_d(d^{-2|p_{\rm eff}(\mT) -(p+\delta)|})\E(\partial_r f^*_{\mT})^2  +  \sum_{\mT\in \alp} \E(\partial_r f^*_{\mT})^2 + O_{d,\P}(d^{-\lambda_{\iota}} \sigma_\varepsilon^2).
\end{equation}
The DN-estimator for general weights can be analogously derived:
\begin{equation}\label{main:eq:nd_w}
w_r\cdot\D_r(w)
\approx \sum_{\mT\in \A_{p_{\delta}}^+(r)\sqcup \alp} \Theta_d(d^{-|p_{\rm eff}(\mT) -(p+\delta)|})\E(\partial_r f^*_{\mT})^2   + O_{d,\P}(d^{-\lambda_{\iota}+\epsilon} \sigma_\varepsilon^2).
\end{equation}
Here for simplicity of presentation, we assume there is no $\mT$ such that $p_{\mathrm{eff}}(\mT) = p + \delta$, since this case has to be treated separately.

\paragraph{Algorithm $\mathtt{IRKM}(\alpha)$ learns leap $(p+1)$ functions.} Suppose now that $f^*$ has leap complexity $p+1$.
For any averaging parameters $\alpha\in(0,1)$, the two estimators \eqref{main:eq:agop_w} and \eqref{main:eq:nd_w} together give
\begin{align*}
    &~~~~(1-\alpha)\E_n(\partial_r\hat f_w)^2 + \alpha w_r\cdot\D_r(w)\\
    &\approx (1-\alpha)\sum_{\mT\in \alp} \E(\partial_r f^*_{\mT})^2 + \alpha\sum_{\mT\in \A_{p_{\delta}}^+(r)} \Theta_d(d^{-|p_{\rm eff}(\mT) -(p+\delta)|})\E(\partial_r f^*_{\mT})^2 + O_{d,\P}(d^{-\lambda_{\iota}+\epsilon} \sigma_\varepsilon^2).
\end{align*}

From the estimate we conclude that a feature $x_{\bf T}$ that depends on coordinate $x_r$ can be identified when its effective degree~\eqref{main:eq:eff_degree} does not exceed $p+\delta + \lambda_{\iota_r}$, where $\lambda_{\iota_r}$ is initially $1$. As the iterations proceed, one can show that any coordinate $r$ that is learned---meaning that one of the summands in the expression is nonzero---will be assigned a large weight after the normalization step (Line~\ref{line:normalizeintro} in Algorithm~\ref{alg:rfm}), that is of order $d^{\eta}$ with $\eta\in(0,1)$. This reweighting reduces the effective degree of every feature that depends on  $x_r$, thereby enabling those features to be learned in subsequent iterations. 

Consider any feature that is the product of learned and unlearned coordinates.
A careful accounting of normalization, safeguarding and sparsity shows that (1) each unlearned coordinate retains weight of the order $ \Theta_d(1)$; hence the corresponding $\lambda_\iota = 1$ and it contributes one to the effective degree, and (2) the combined contribution of the learned coordinates can be controlled to remain below 
$\delta$ since these coordinates are assigned large weights. Consequently, in the regime $n = d^{p+\delta}$, any feature whose effective degree is less than $p+\delta+1$ can be identified as discussed, and can contain at most $p+1$ new coordinates (see Figure~\ref{fig:effective_degree} for a visual illustration). Since $f^*$ is assumed to be finite degree, this careful accounting also shows that a constant number of iterations of $\mathtt{IRKM}(\alpha)$ suffices to learn leap $(p+1)$ component of $f^*$.
\begin{figure}[H]
    \centering
    \includegraphics[width=0.6\linewidth]{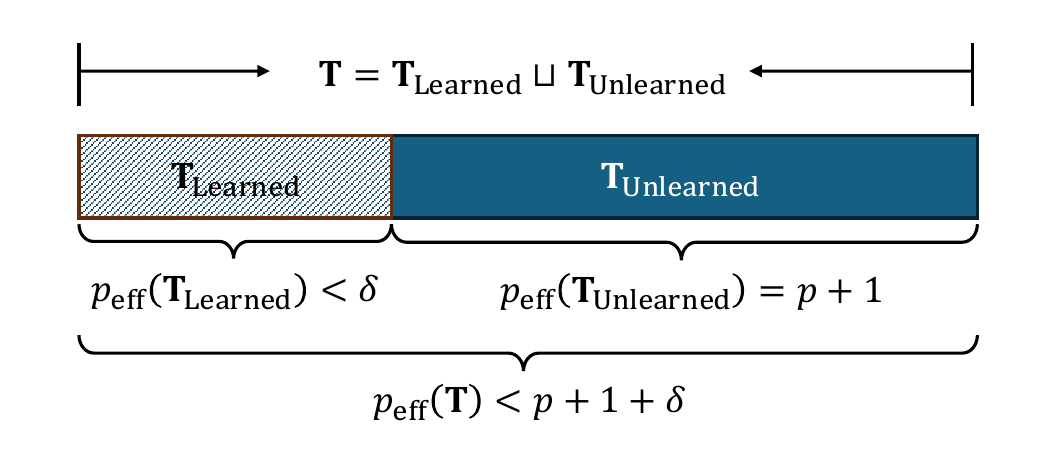}
\caption{Illustration of the effective degree of a feature as the product of learned and unlearned coordinates.}
    \label{fig:effective_degree}
\end{figure}

%% file: sections/numerical.tex
\section{Numerical experiments}\label{sec:numerics}
In this section, we demonstrate  the efficiency of $\mathtt{IRKM}(\alpha)$ (Algorithm~\ref{alg:rfm}) when learning sparse functions across diverse data distributions and with higher leap complexity than considered in the introduction. Our experimental results show that $\mathtt{IRKM}(\alpha)$ often outperforms neural networks on these tasks including epistasis detection, and achieves comparable performance to state-of-the-art methods on UCI classification tasks~\cite{fernandez2014we}. 
The code for $\mathtt{IRKM}(\alpha)$ that we use is available at \url{https://github.com/zhm1995/IRKM}.

For fixed training size $n$, we train both neural networks and kernel machines using $2n$ samples, drawing fresh $n$ samples at each $\mathtt{IRKM}(\alpha)$ training step.  All neural networks reported in this section are fully connected and trained with Adam, using widths and batch sizes in $\{128, 256, 512\}$, depth in $\{2,3,4,5,6\}$ and learning rate in $\{10^{-3}, 10^{-4}\}$. We report the smallest test loss across these configurations.  

Generally speaking, in our experiments, wider and deeper networks tend to perform better. As a concrete illustration, Figure~\ref{fig:network_depth} presents the test performance of fully connected neural networks across depths from 2 to 6 and width in $\{128, 256, 512\}$  for the experimental settings in Figure~\ref{fig:task_1}, \ref{fig:task_2} and \ref{fig:task_3} of the introduction. We observe that the test performance is improved as the depth or the width increases, and this trend persists in all our experiments. 

\begin{figure}[H]
    \centering
    \begin{subfigure}[b]{0.32\linewidth}
        \centering
        \includegraphics[width=\linewidth]{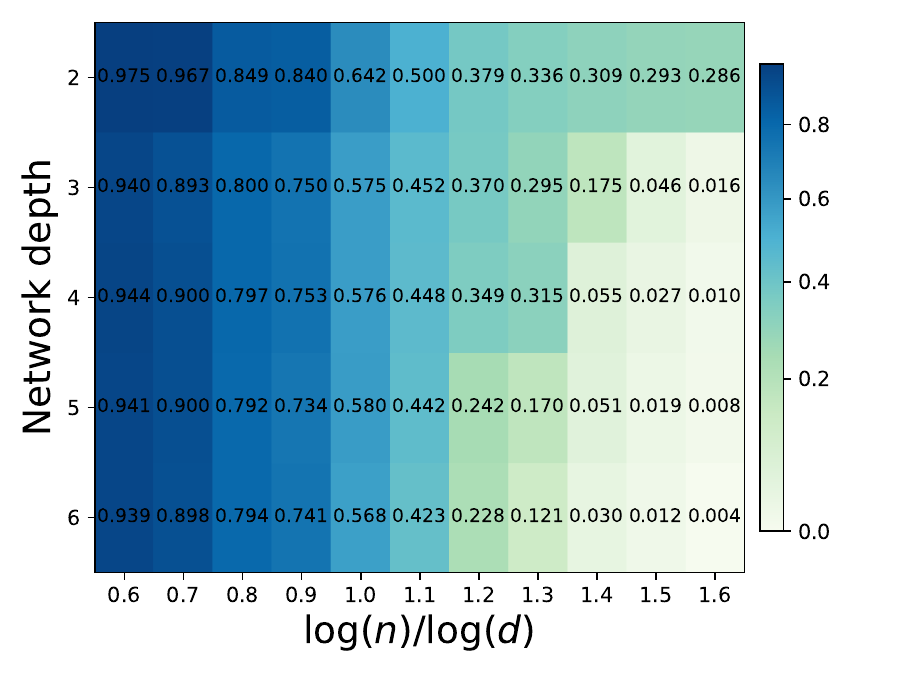}
    \end{subfigure}
    \begin{subfigure}[b]{0.32\linewidth}
        \centering
        \includegraphics[width=\linewidth]{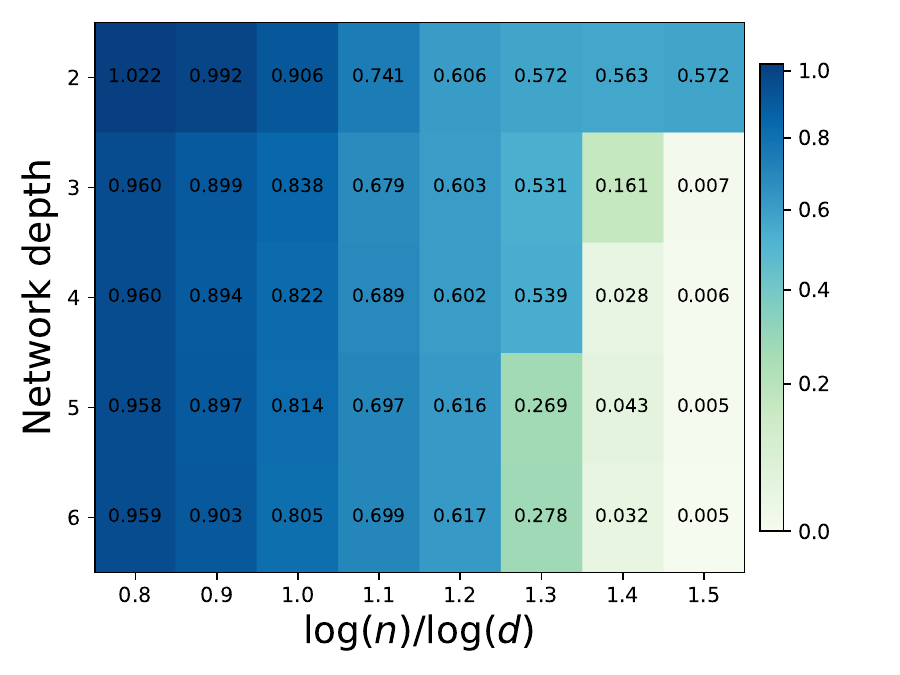}
    \end{subfigure}
            \begin{subfigure}[b]{0.32\linewidth}
        \centering
        \includegraphics[width=\linewidth]{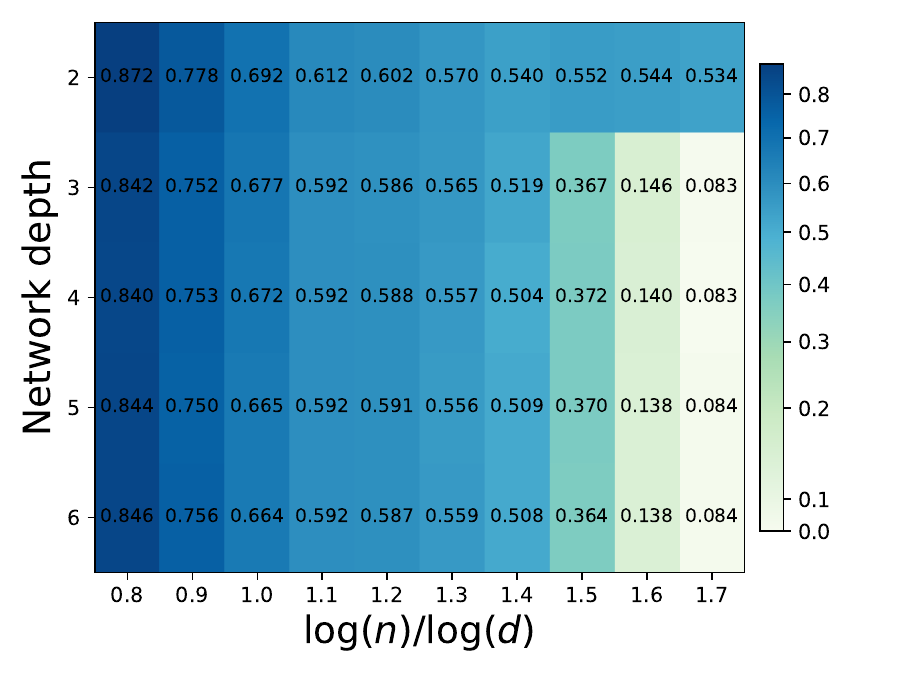}
         \end{subfigure}\\
    \begin{subfigure}[b]{0.32\linewidth}
        \centering
        \includegraphics[width=\linewidth]{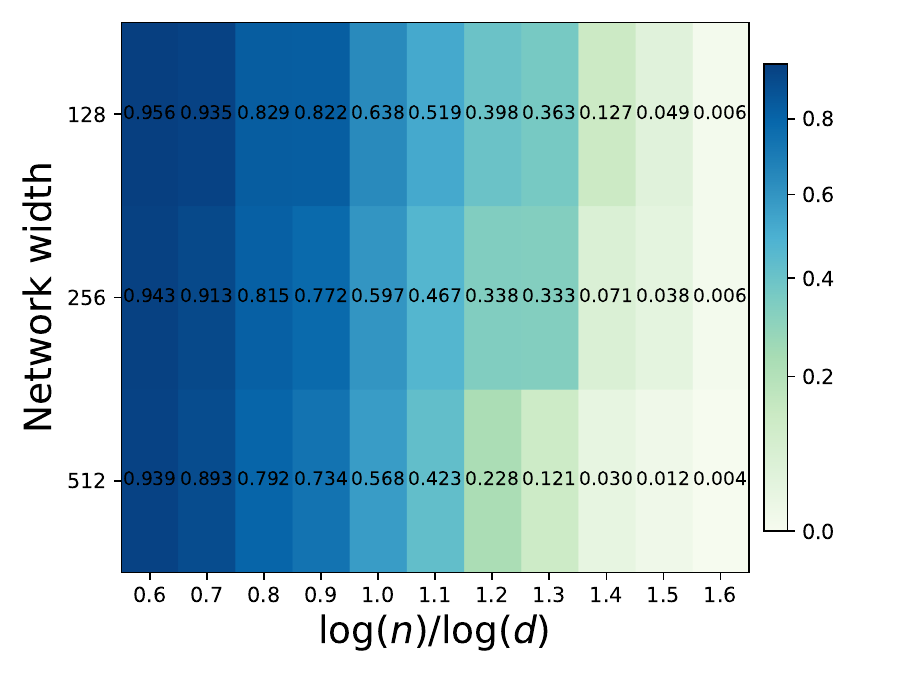}
        \caption{}
    \end{subfigure}
    \begin{subfigure}[b]{0.32\linewidth}
        \centering
        \includegraphics[width=\linewidth]{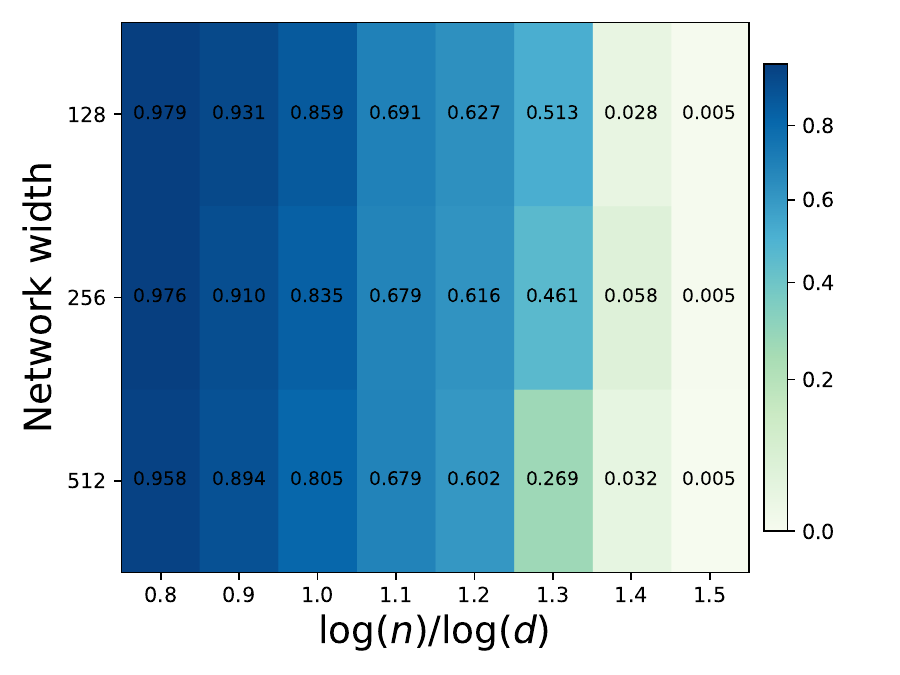}
    \caption{}
    \end{subfigure}
            \begin{subfigure}[b]{0.32\linewidth}
        \centering
        \includegraphics[width=\linewidth]{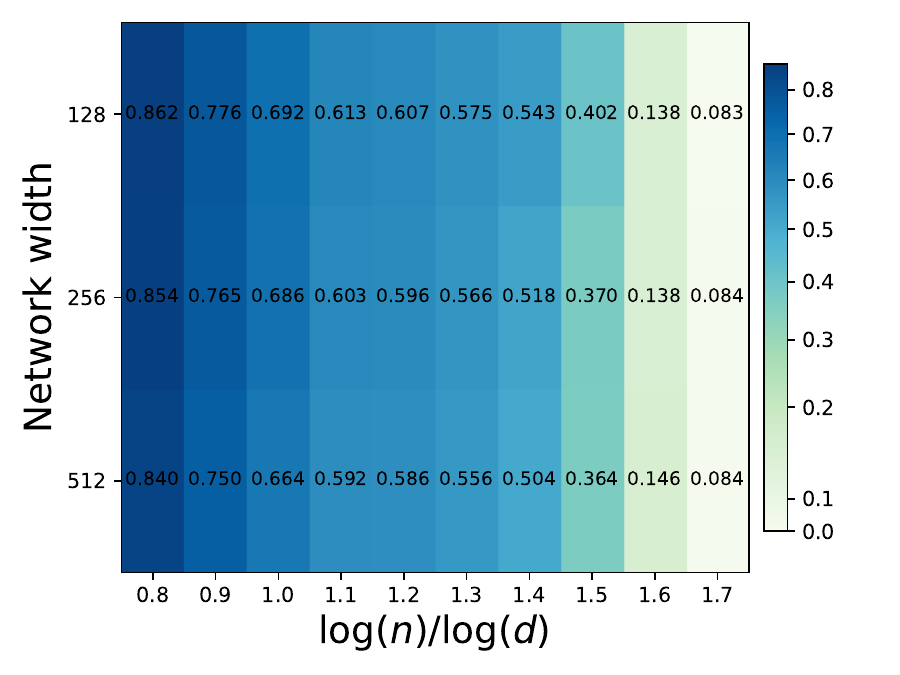}
        \caption{}
         \end{subfigure}\\         
    \caption{{\bf Best test performance of neural networks across various widths and depths.} Columns (a), (b) and (c) correspond to the experimental setting in Figure~\ref{fig:task_1}, \ref{fig:task_2} and \ref{fig:task_3} respectively. We train a fully-connected network by Adam,  varying the widths and batch sizes as $\{128, 256, 512\}$, depth as $\{2,3,4,5,6\}$ and learning rate as $\{10^{-3}, 10^{-4}\}$ and report the smallest test loss. For all experiments, we report an average of $5$ runs.}
    \label{fig:network_depth}
\end{figure}

\subsection{Evaluation of $\mathtt{IRKM}(\alpha)$  on Gaussian and CIFAR-10 data}\label{subsec:gauss_cifar}
Though our theoretical analysis of $\mathtt{IRKM}(\alpha)$ focuses on hypercube data, we now empirically demonstrate its impressive performance across multiple input distributions. Beyond the hypercube experiments in the introduction, we evaluate $\mathtt{IRKM}(\alpha)$ for learning sparse functions with inputs drawn from Gaussian distributions and the CIFAR-10 dataset~\cite{krizhevsky2009learning}. 

Analogously to Figures~\ref{fig:task_1}, \ref{fig:task_2} and \ref{fig:task_3} for hypercube data, we run $\mathtt{IRKM}(\tfrac{1}{2})$ on samples from Gaussian distributions to learn the same target functions.  
The results are reported in Figures~\ref{fig:task_2_gauss_cifar}A,~\ref{fig:task_2_gauss_cifar}B and \ref{fig:task_3_gaussian}A, respectively. 
For the CIFAR-10 dataset, we randomly select $200$ pixels from each image and learn sparse functions and multi-index models based on these selected pixels. See the results in Figure~\ref{fig:task_2_gauss_cifar}C and~\ref{fig:task_3_gaussian}B respectively. The conclusion is exactly the same in all cases as for hypercube data.

\begin{figure}[H]
  \centering

  \begin{minipage}{\linewidth}
    \makebox[0pt][l]{\textbf{A}\hspace{0.8em}}%
    \renewcommand\thesubfigure{A\arabic{subfigure}}
    \setcounter{subfigure}{0}

    \begin{subfigure}[b]{0.38\linewidth}
        \centering
        \includegraphics[width=\linewidth]{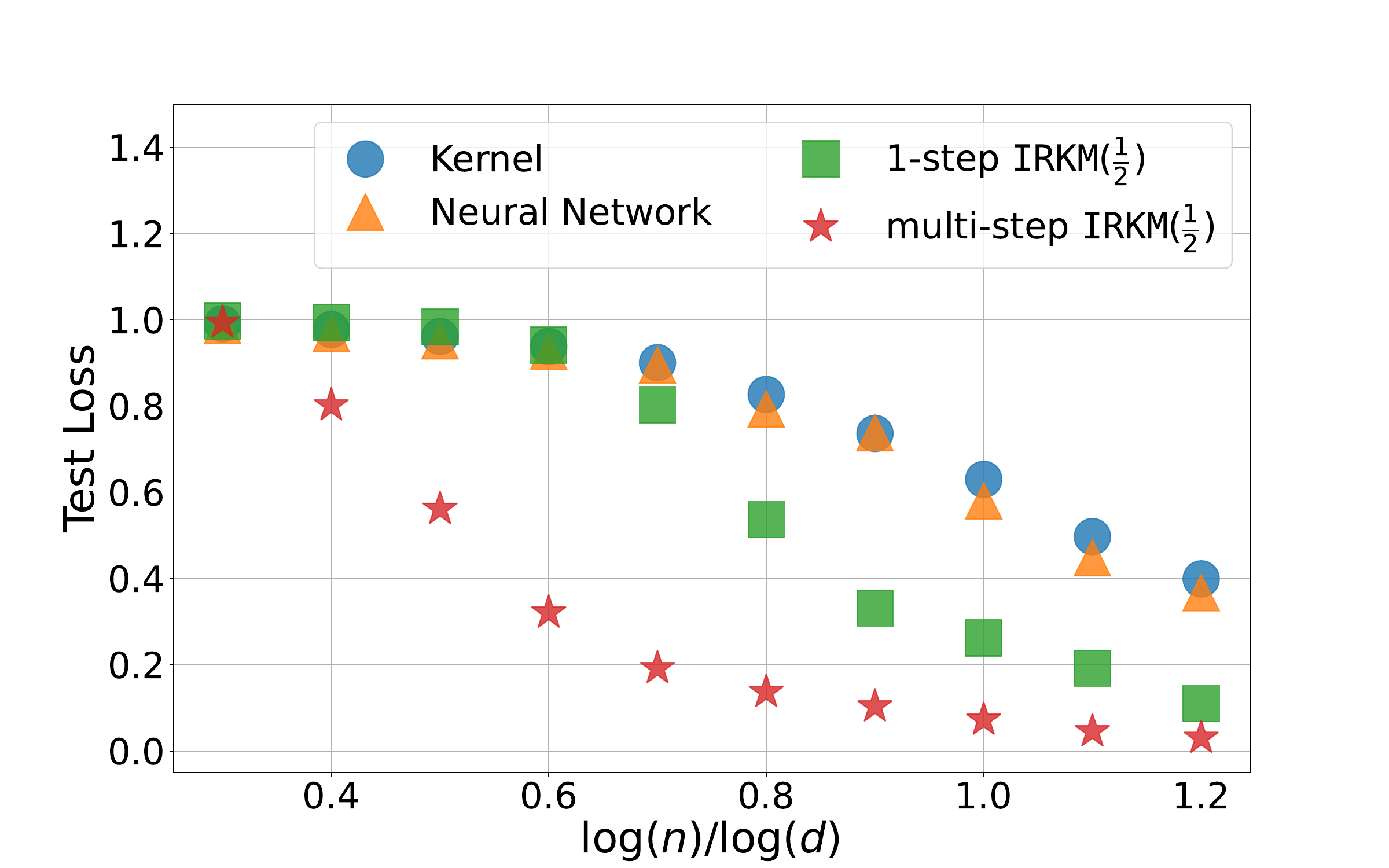}
    \end{subfigure}
    \begin{subfigure}[b]{0.3\linewidth}
        \centering
        \includegraphics[width=\linewidth]{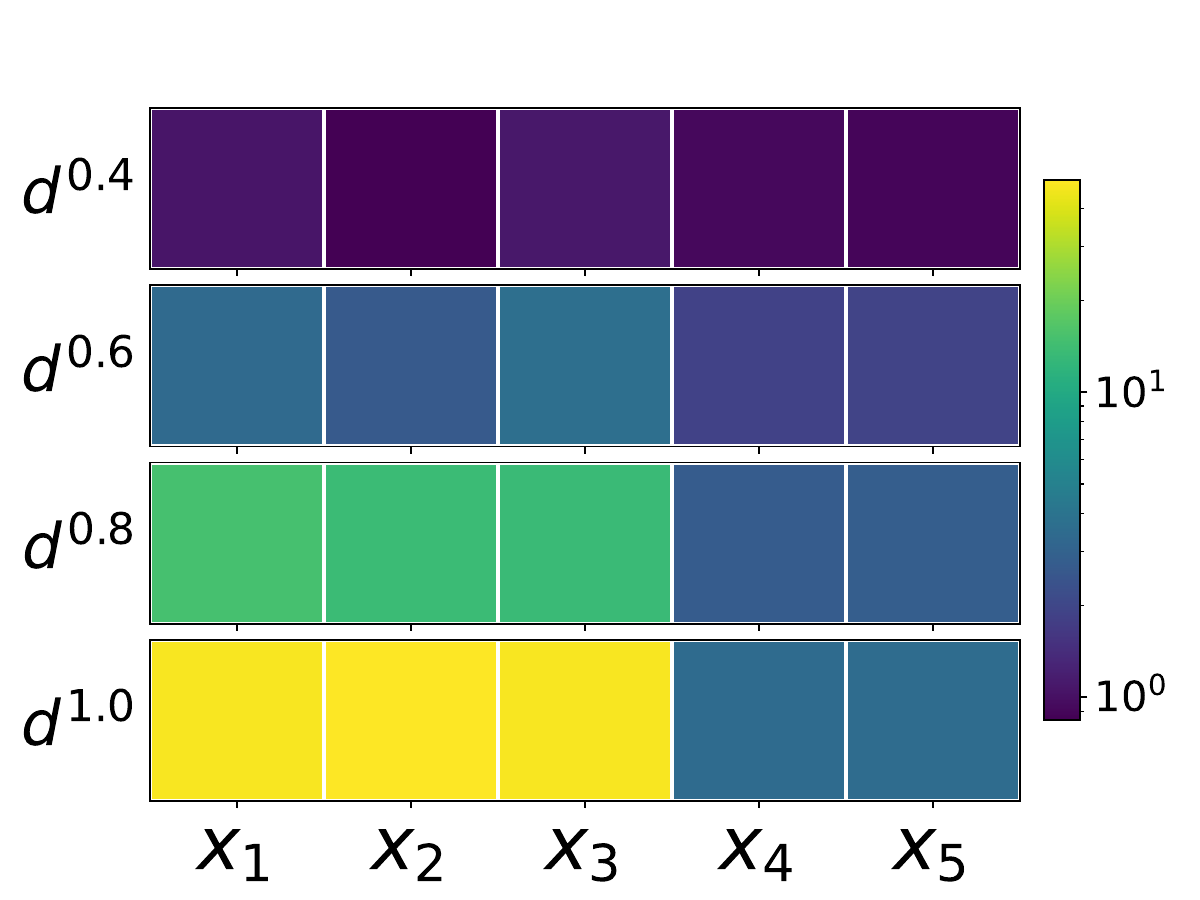}
    \end{subfigure}
        \begin{subfigure}[b]{0.3\linewidth}
        \centering
        \includegraphics[width=\linewidth]{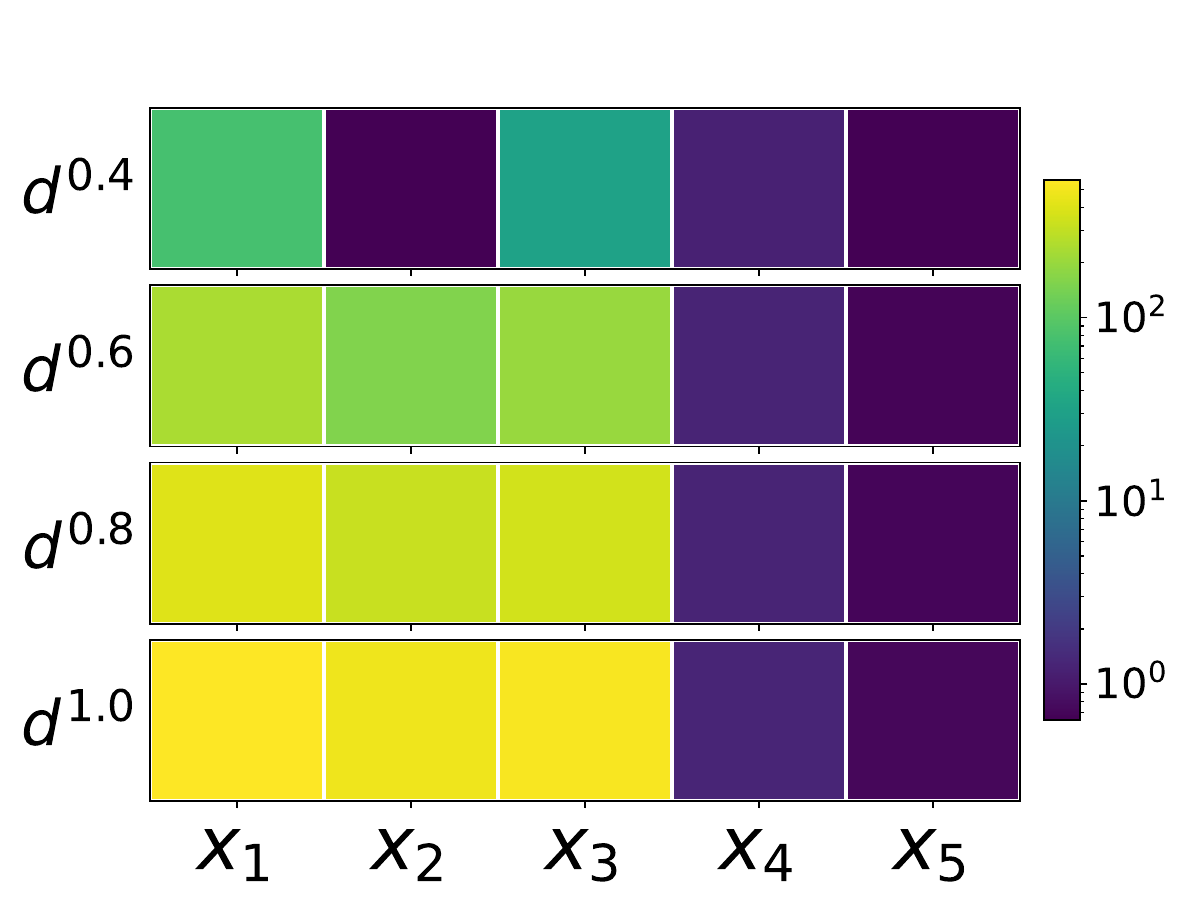}
    \end{subfigure}
  \end{minipage}
  \vspace{0.9em}
  
  \begin{minipage}{\linewidth}
    \makebox[0pt][l]{\textbf{B}\hspace{0.8em}}%
    \renewcommand\thesubfigure{B\arabic{subfigure}}
    \setcounter{subfigure}{0}

    \begin{subfigure}[b]{0.38\linewidth}
      \centering
      \includegraphics[width=\linewidth]{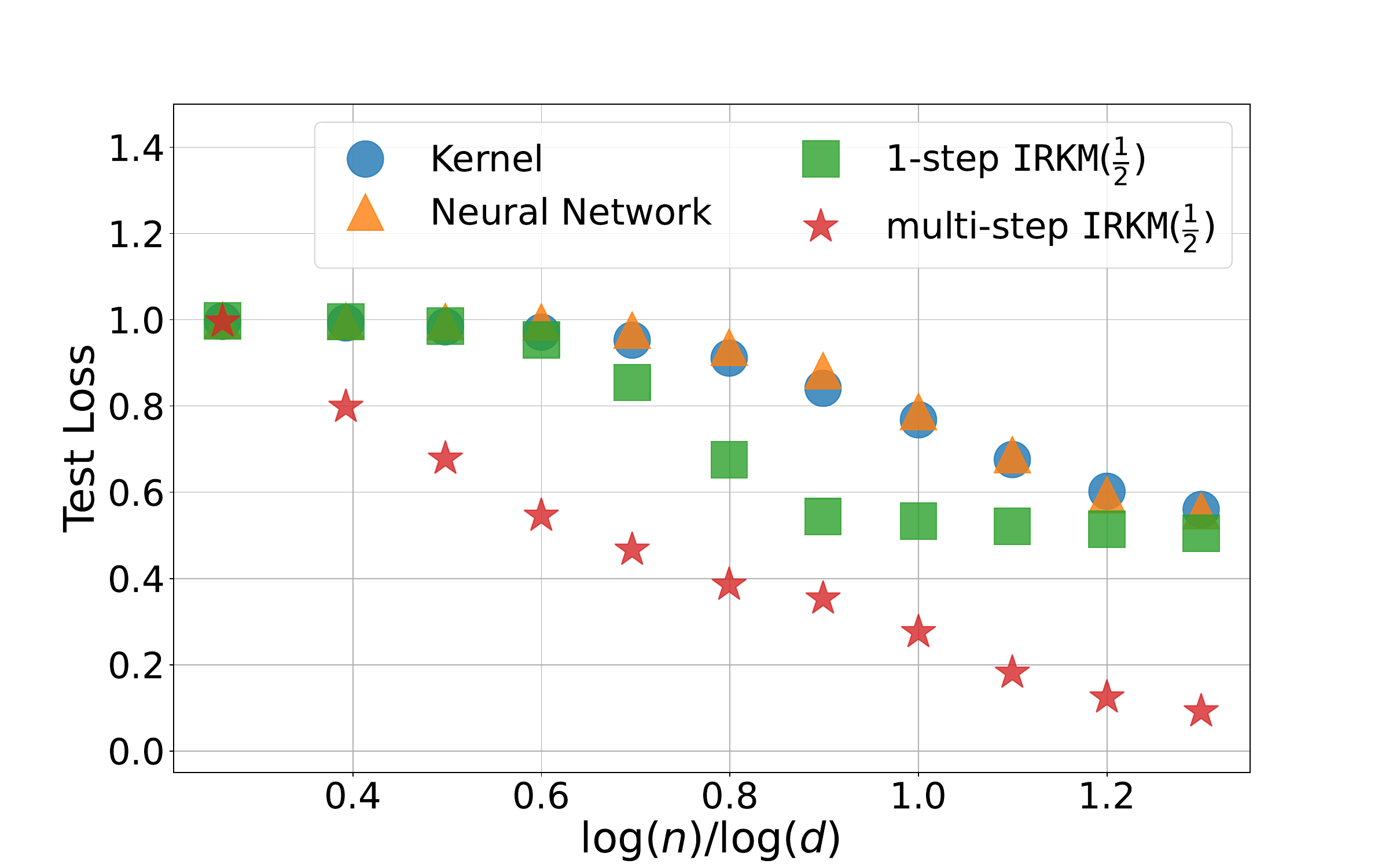}
    \end{subfigure}\hfill
    \begin{subfigure}[b]{0.3\linewidth}
      \centering
      \includegraphics[width=\linewidth]{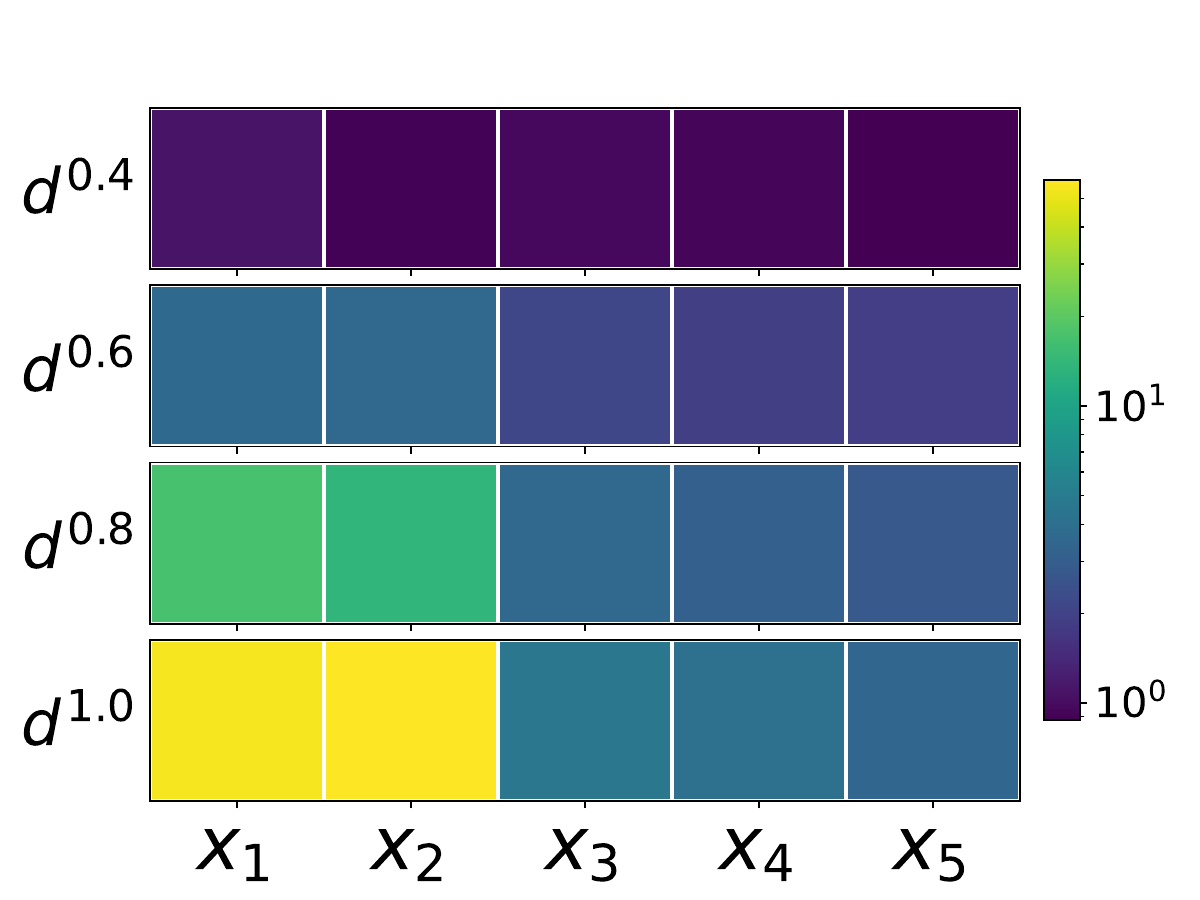}
    \end{subfigure}\hfill
    \begin{subfigure}[b]{0.3\linewidth}
      \centering
      \includegraphics[width=\linewidth]{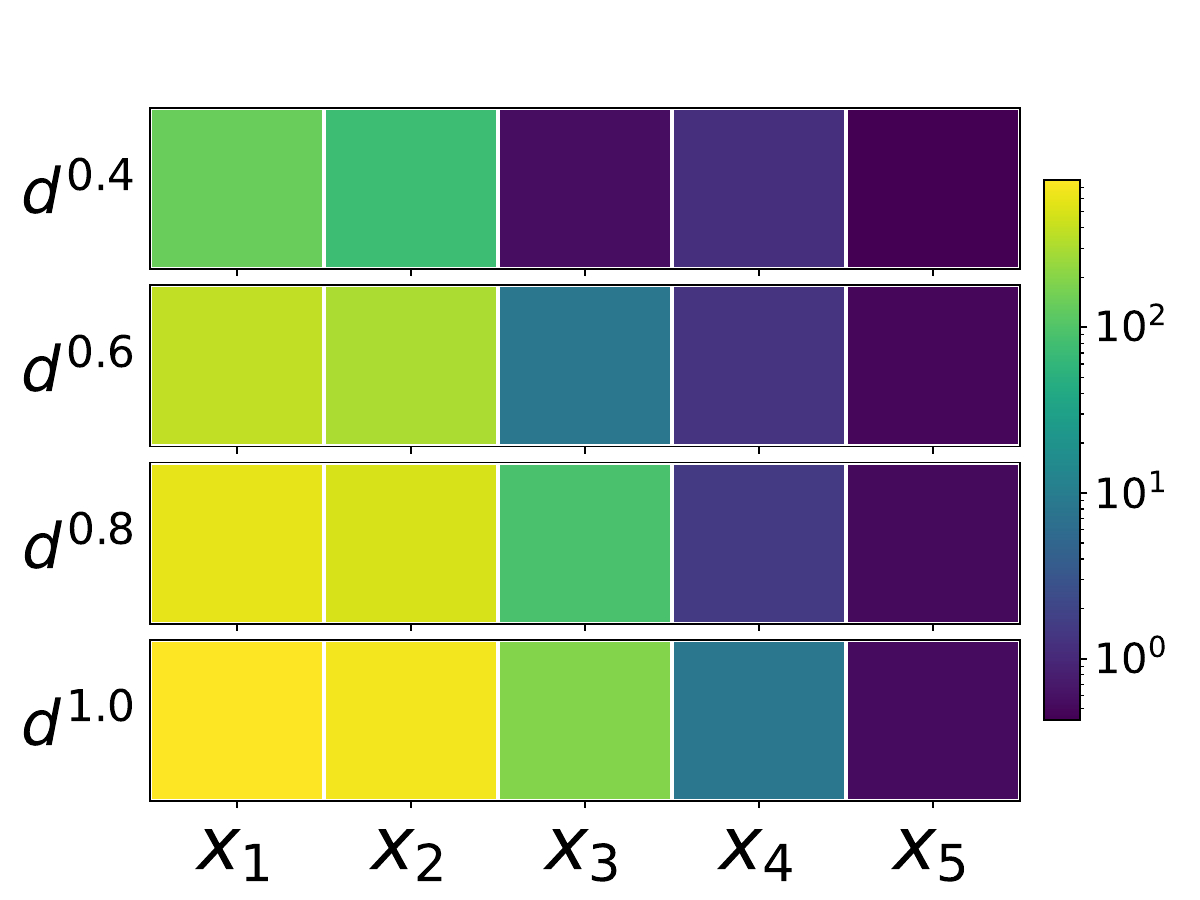}
    \end{subfigure}
  \end{minipage}

  \vspace{0.9em} %

  \begin{minipage}{\linewidth}
    \makebox[0pt][l]{\textbf{C}\hspace{0.8em}}%
    \renewcommand\thesubfigure{\arabic{subfigure}}
    \setcounter{subfigure}{0}

    \begin{subfigure}[b]{0.38\linewidth}
      \centering
      \includegraphics[width=\linewidth]{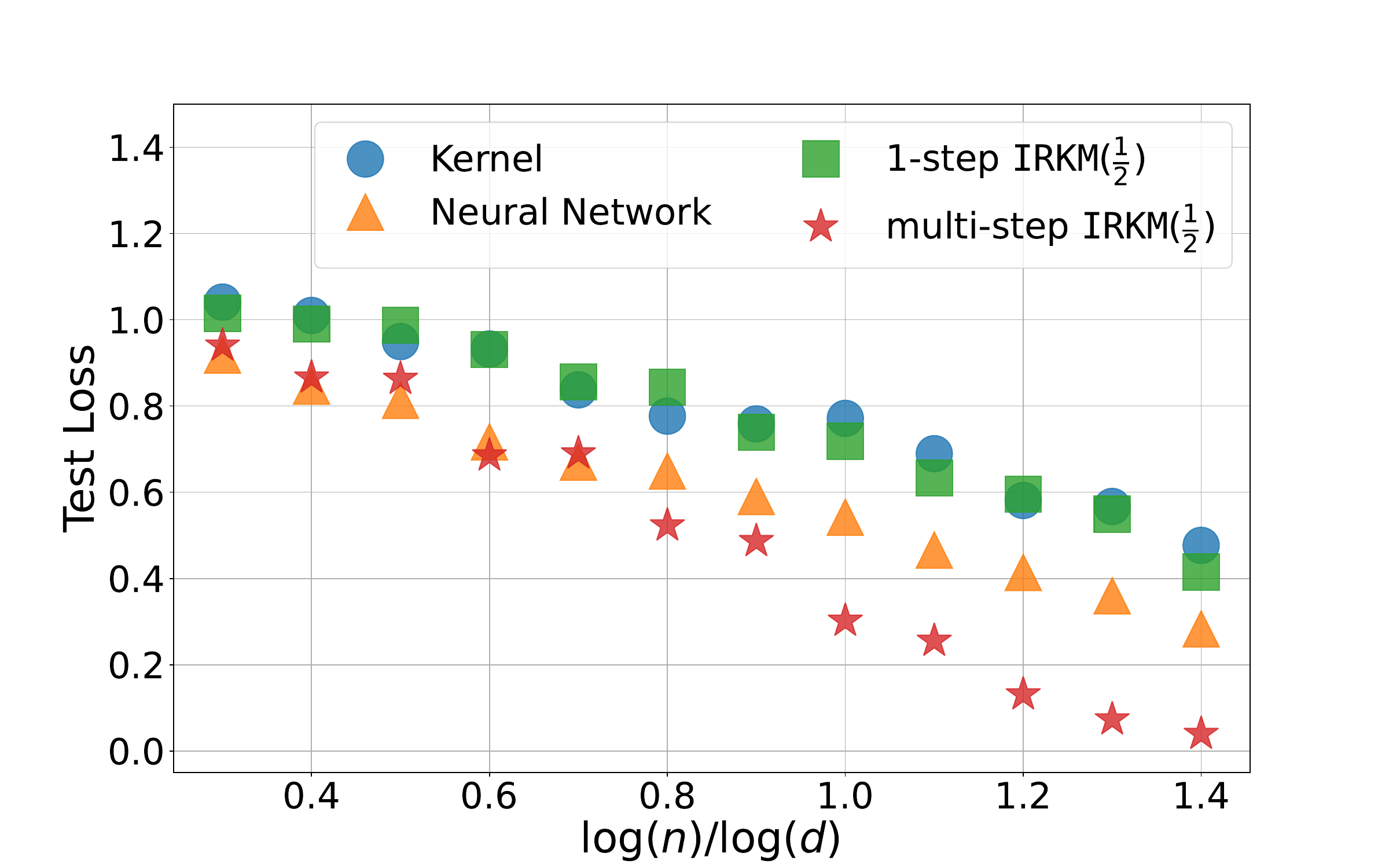}
      \caption{Generalization error}\label{fig:B1}
    \end{subfigure}\hfill
    \begin{subfigure}[b]{0.3\linewidth}
      \centering
      \includegraphics[width=\linewidth]{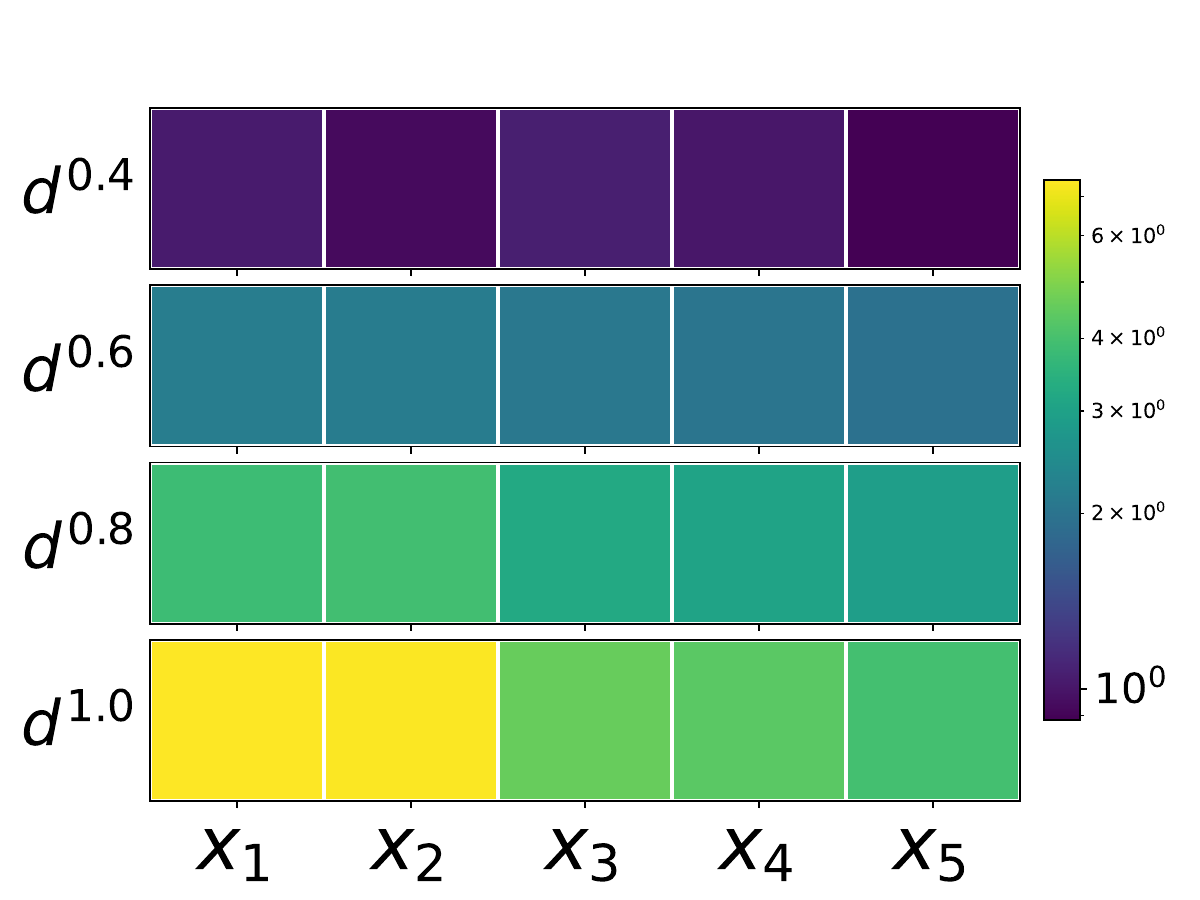}
      \caption{One‑step identification}\label{fig:B2}
    \end{subfigure}\hfill
    \begin{subfigure}[b]{0.3\linewidth}
      \centering
      \includegraphics[width=\linewidth]{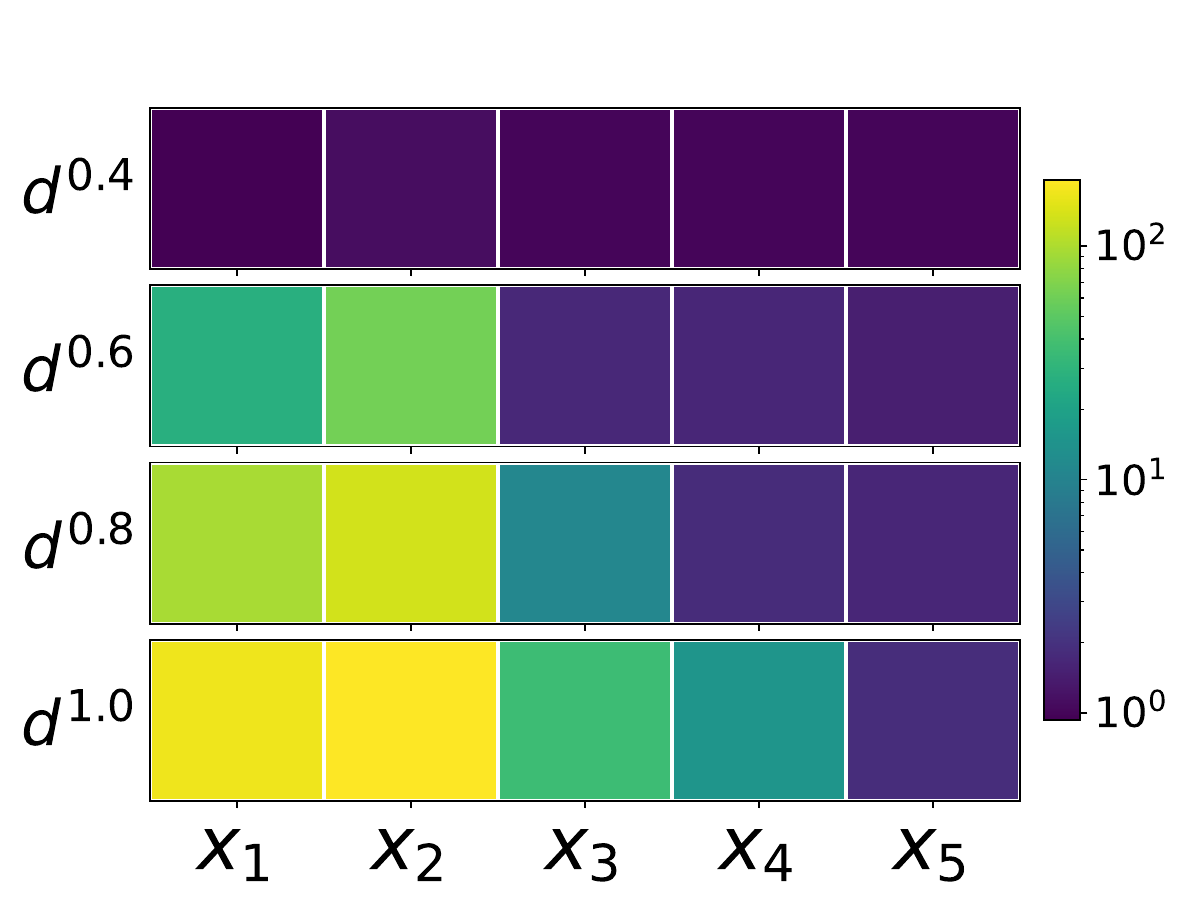}
      \caption{Multi‑step identification}\label{fig:B3}
    \end{subfigure}
  \end{minipage}
\vspace{0.9em}
  \caption{
  {\bf Column.} {\bf (1)} Test loss for kernel machines, neural networks and $\mathtt{IRKM}(\tfrac{1}{2})$. %
    {\bf (2)} The empirical coordinate weights at the first step. %
   {\bf (3)} The empirical coordinate weights at  step $T$. \\
    {\bf Row.}
\textbf{A:} synthetic data $x\sim\mathcal N(0,I_{500})$,
label $y=x_1+x_2+x_3+x_1x_2x_3+\varepsilon$.  
\textbf{B:} synthetic data $x\sim\mathcal N(0,I_{500})$,
label $y=x_1+x_2+x_1x_2x_3+x_1x_2x_3x_4+\varepsilon$.  
\textbf{C:} $200$ random pixels from CIFAR‑10, same label rule as row B. The noise $\varepsilon$ is i.i.d. drawn from $\mathcal{N}(0,0.1^2)$.
We use the Laplacian kernel for both kernel machines and $\mathtt{IRKM}(\tfrac{1}{2})$. The $\mathtt{IRKM}(\tfrac{1}{2})$ algorithm is iterated for at most $T=20$ steps for row A, and $10$ steps for row B and C. We train a fully-connected network by Adam.  For all experiments, we report an average of $10$ runs.}
  \label{fig:task_2_gauss_cifar}
\end{figure}

\begin{figure}[H]
    \centering
  \begin{minipage}{\linewidth}
    \makebox[0pt][l]{\textbf{A}\hspace{0.8em}}%
    \renewcommand\thesubfigure{A\arabic{subfigure}}
    \setcounter{subfigure}{0}
    
    \begin{subfigure}[b]{0.32\linewidth}
        \centering
        \includegraphics[width=\linewidth]{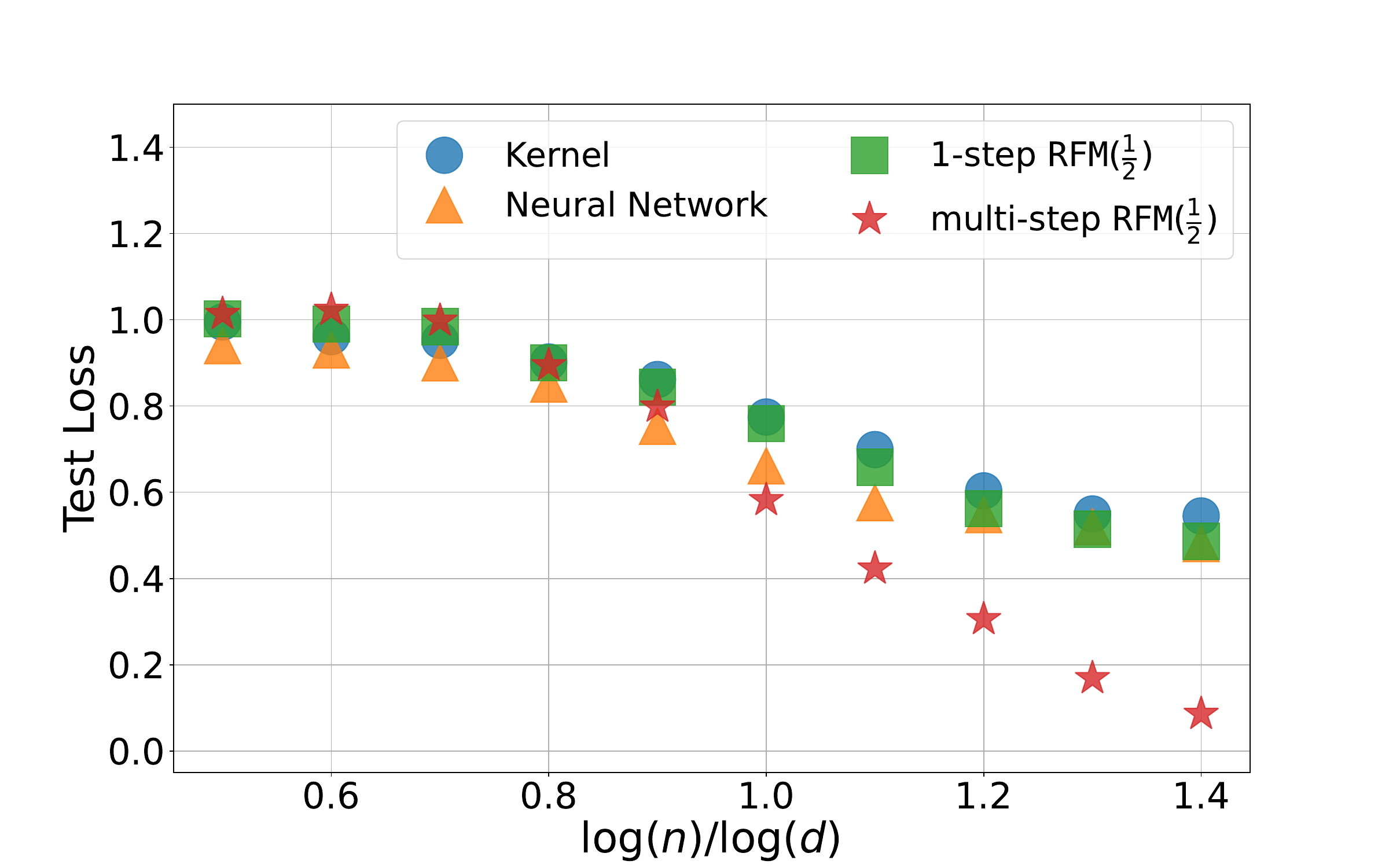}
        \caption{Generalization error}
    \end{subfigure}
    \begin{subfigure}[b]{0.32\linewidth}
        \centering
        \includegraphics[width=\linewidth]{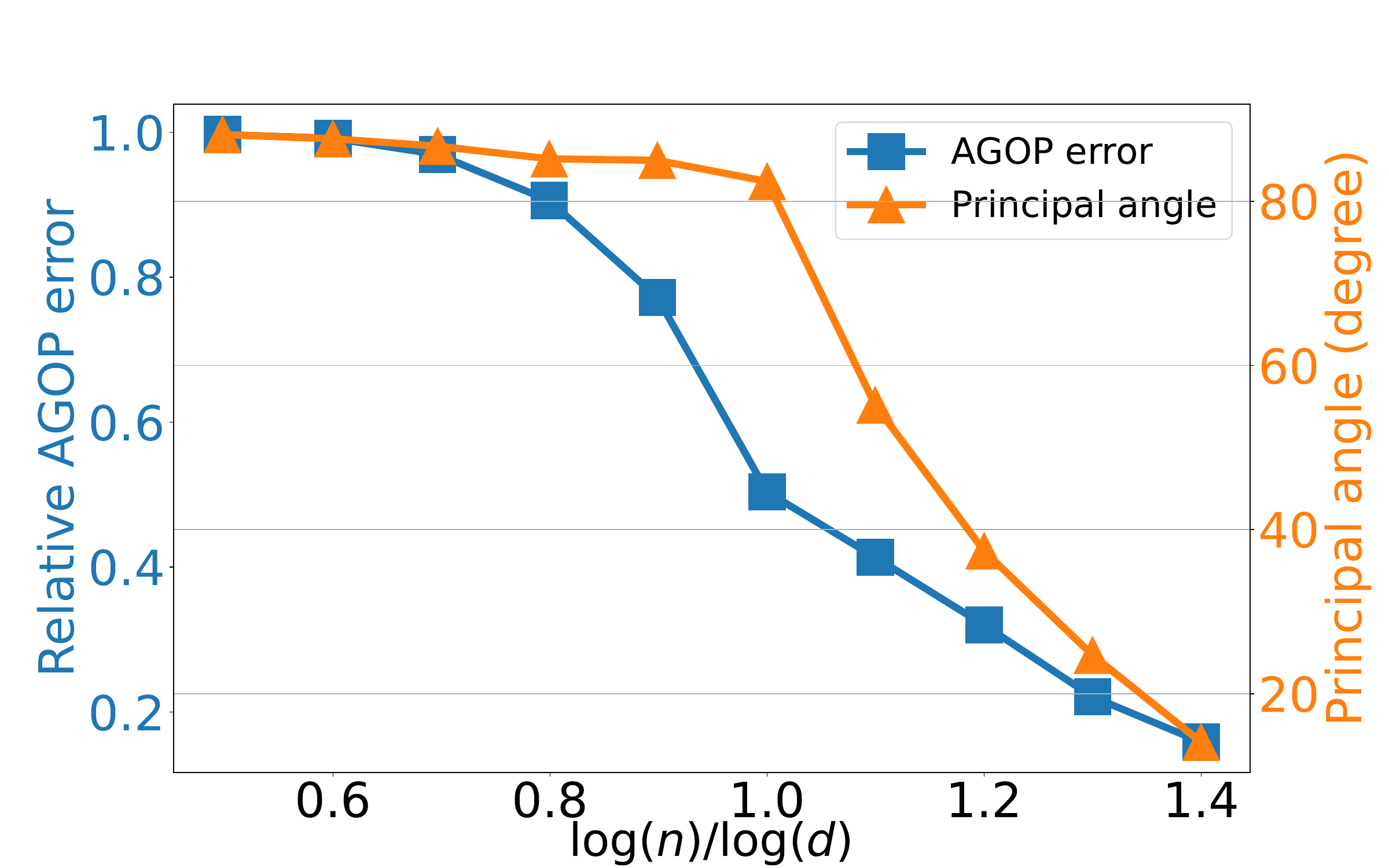}
    \caption{AGOP error}
    \end{subfigure}
            \begin{subfigure}[b]{0.32\linewidth}
        \centering
        \includegraphics[width=\linewidth]{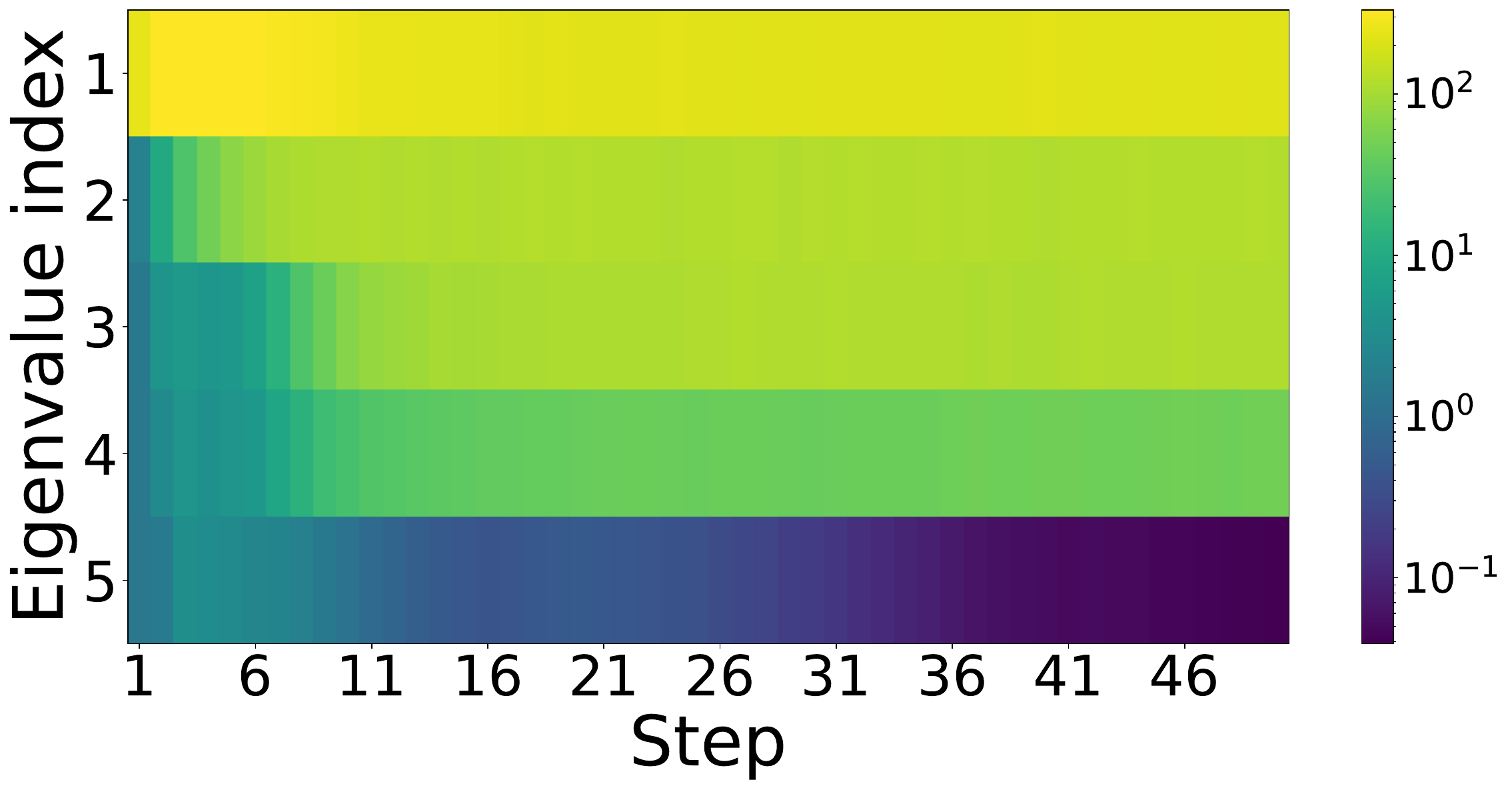}
        \caption{Multi-step identification}
         \end{subfigure}
  \end{minipage}
\vspace{0.9em}

 \begin{minipage}{\linewidth}
    \makebox[0pt][l]{\textbf{B}\hspace{0.8em}}%
    \renewcommand\thesubfigure{B\arabic{subfigure}}
    \setcounter{subfigure}{0}
    
        \begin{subfigure}[b]{0.32\linewidth}
        \centering
        \includegraphics[width=\linewidth]{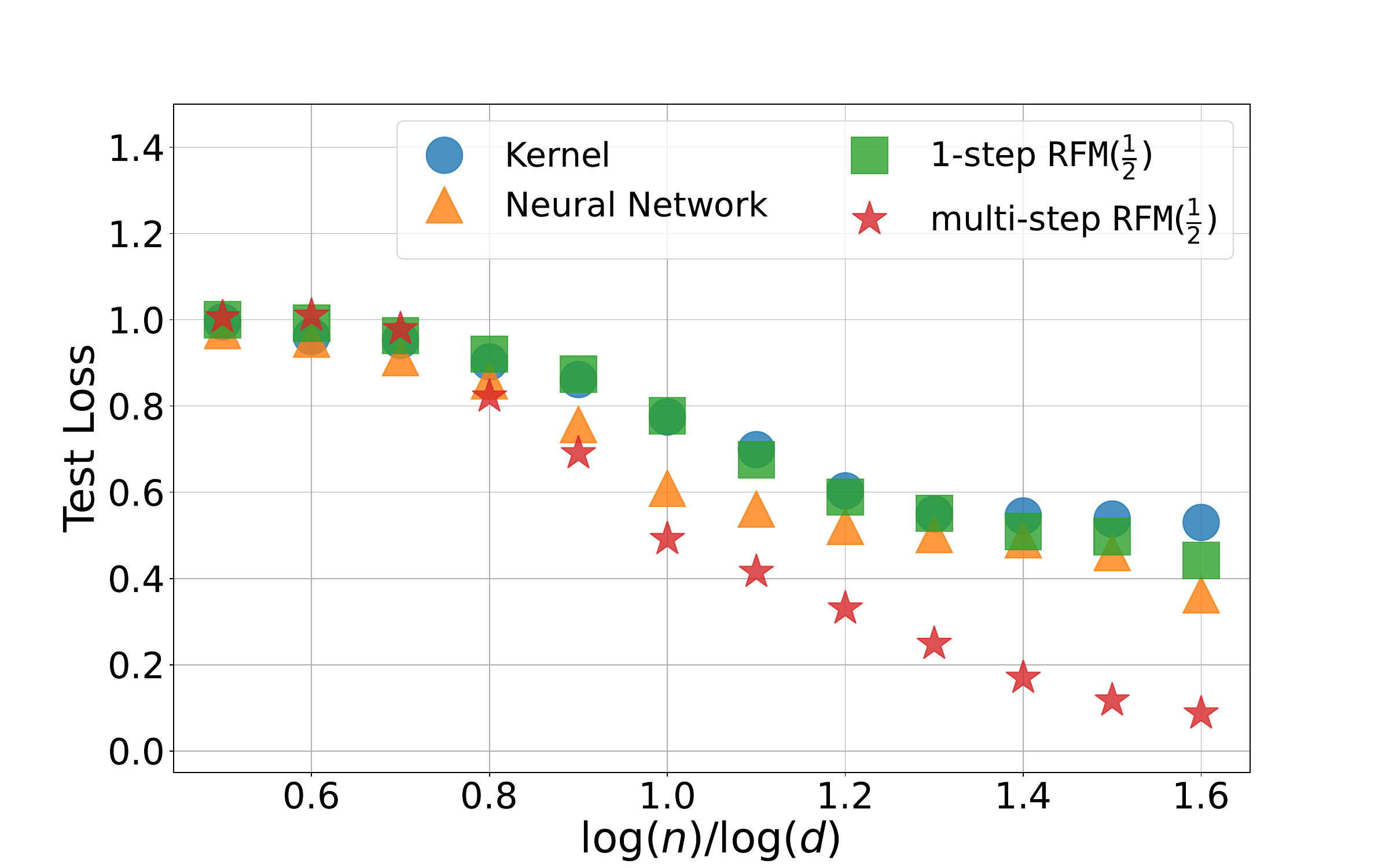}
        \caption{Generalization error}
    \end{subfigure}
    \begin{subfigure}[b]{0.32\linewidth}
        \centering
        \includegraphics[width=\linewidth]{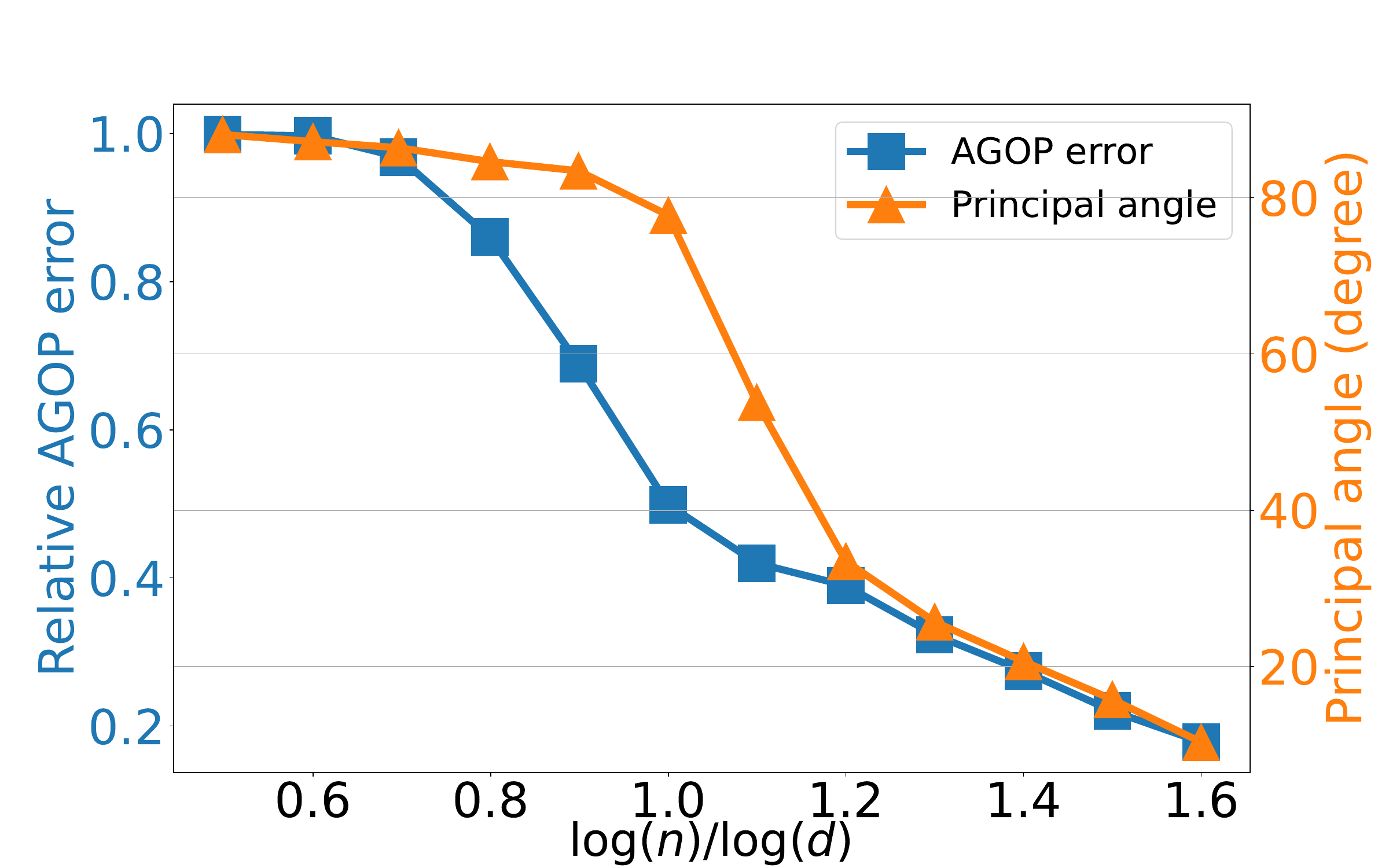}
    \caption{AGOP error}
    \end{subfigure}
         \begin{subfigure}[b]{0.32\linewidth}
        \centering
        \includegraphics[width=\linewidth]{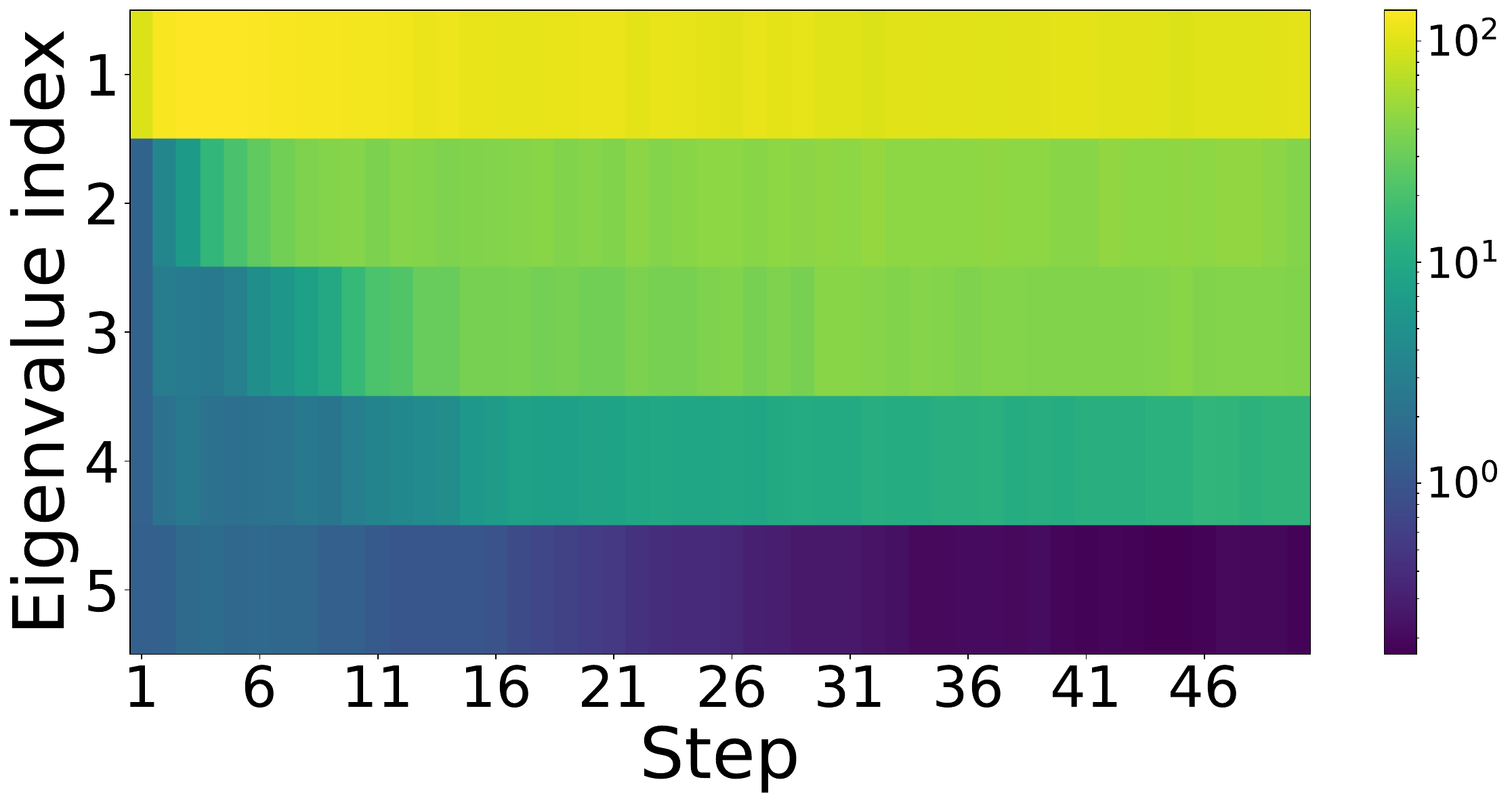}
       \caption{Multi-step identification}
    \end{subfigure}
  \end{minipage}
\vspace{0.9em}
    \caption{{\bf %
    (1) Test loss for kernel machines, neural networks and $\mathtt{RFM}(\tfrac{1}{2})$) (Algo.~\ref{alg:fullrfm}). %
    (2) Relative AGOP error and the principal angle between the top-4 eigenspace of AGOPs.  %
    (3) Eigenvalues of AGOP at step $1-50$ in the regime $n = d^{1.4}$. %
   } 
   \textbf{A:} synthetic data $x\sim\mathcal N(0,I_{500})$,
label $y=z_1+z_2+z_3+z_1z_2z_3+\varepsilon$.
   \textbf{B:} $200$ random pixels from ZCA whitened CIFAR‑10, same label rule as row A. The noise $\varepsilon$ is i.i.d. drawn from $\mathcal{N}(0,0.1^2)$ and  $z = Ux$ where $U$ is a random rotation matrix. We use the Laplacian kernel for both kernel machines and $\mathtt{RFM}(\tfrac{1}{2})$. We train a fully-connected network by Adam. For all experiments, we report an average of $10$ runs.}
    \label{fig:task_3_gaussian}
\end{figure}

\subsection{Learning functions with leap complexity two}
So far, we have experimented only with functions that have leap complexity one. 
We now consider target functions with leap complexity two: 
\begin{align*}
    f^*(x) = x_1 x_2 + x_1x_2x_3x_4,
\end{align*}
and show that $\mathtt{RFM}(\alpha)$ outperforms neural networks on both hypercube and Gaussian synthetic datasets. Our experimental results appear in Figures~\ref{fig:task_4}.

\begin{figure}[H]
  \centering
  \begin{minipage}{\linewidth}
    \makebox[0pt][l]{\textbf{A}\hspace{0.8em}}%
    \renewcommand\thesubfigure{A\arabic{subfigure}}
    \setcounter{subfigure}{0}

    \begin{subfigure}[b]{0.38\linewidth}
        \centering
        \includegraphics[width=\linewidth]{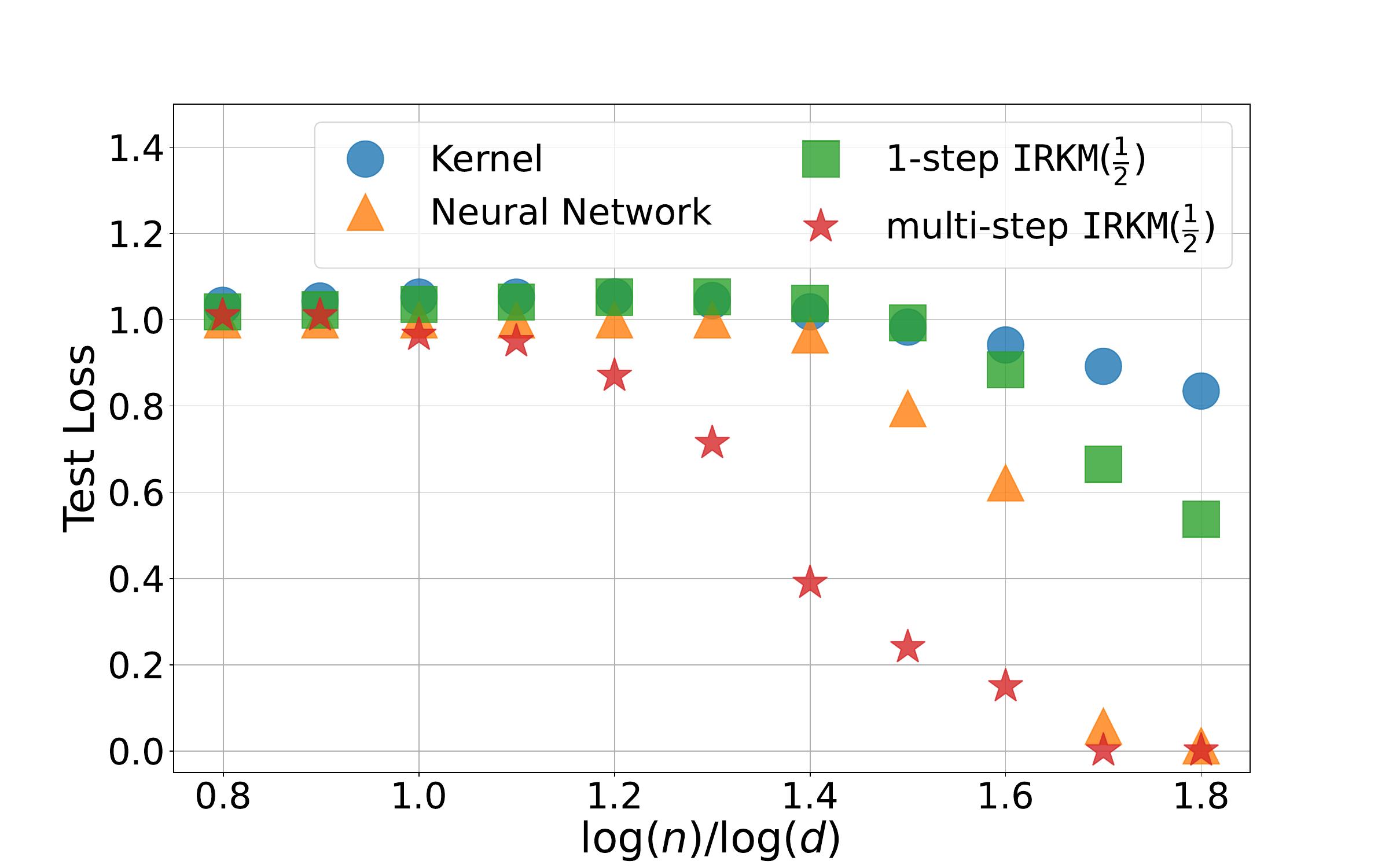}
        \caption{Generalization error}
    \end{subfigure}
    \begin{subfigure}[b]{0.3\linewidth}
        \centering
        \includegraphics[width=\linewidth]{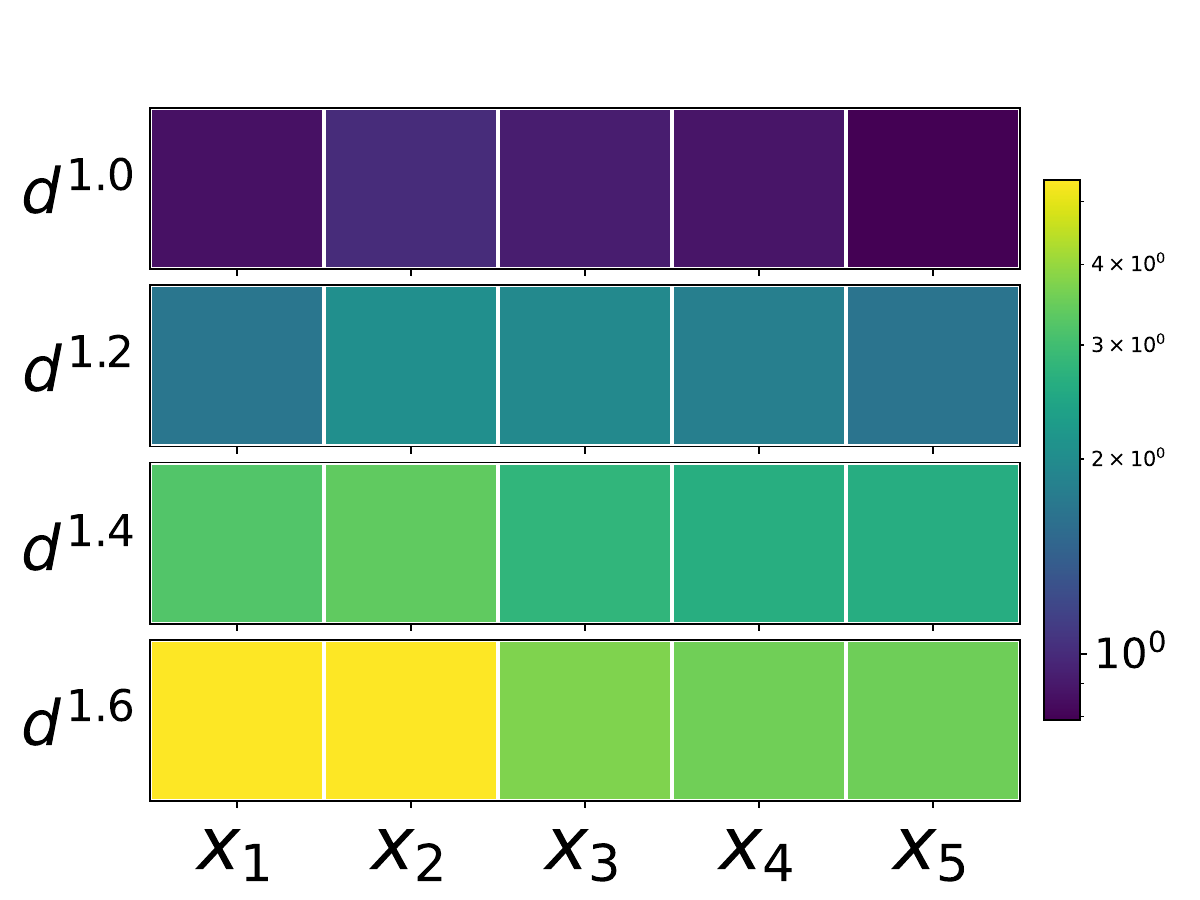}
    \caption{One step identification }
    \end{subfigure}
        \begin{subfigure}[b]{0.3\linewidth}
        \centering
        \includegraphics[width=\linewidth]{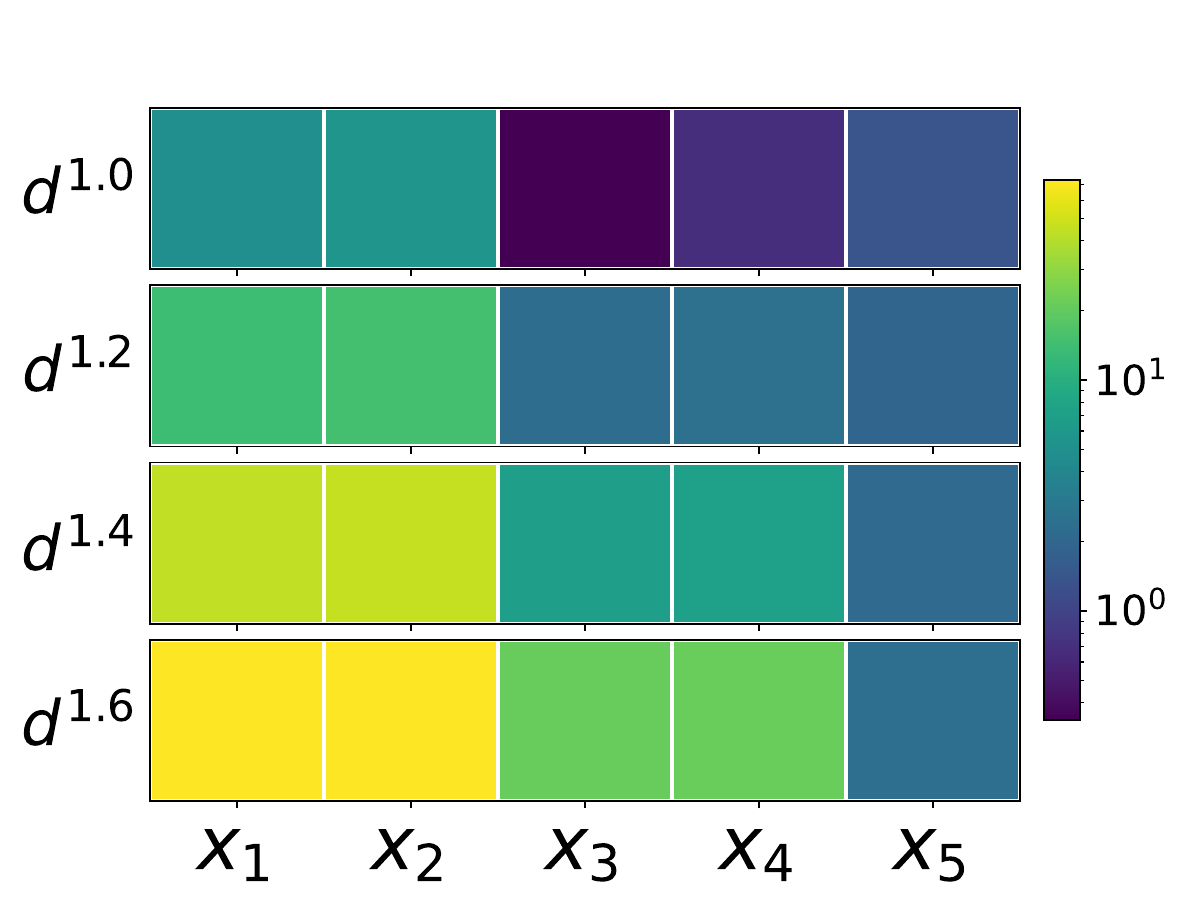}
        \caption{Multi-step identification}
    \end{subfigure}
  \end{minipage}

  \vspace{0.9em} %

  \begin{minipage}{\linewidth}
    \makebox[0pt][l]{\textbf{B}\hspace{0.8em}}%
    \renewcommand\thesubfigure{B\arabic{subfigure}}
    \setcounter{subfigure}{0}

   \begin{subfigure}[b]{0.38\linewidth}
        \centering
        \includegraphics[width=\linewidth]{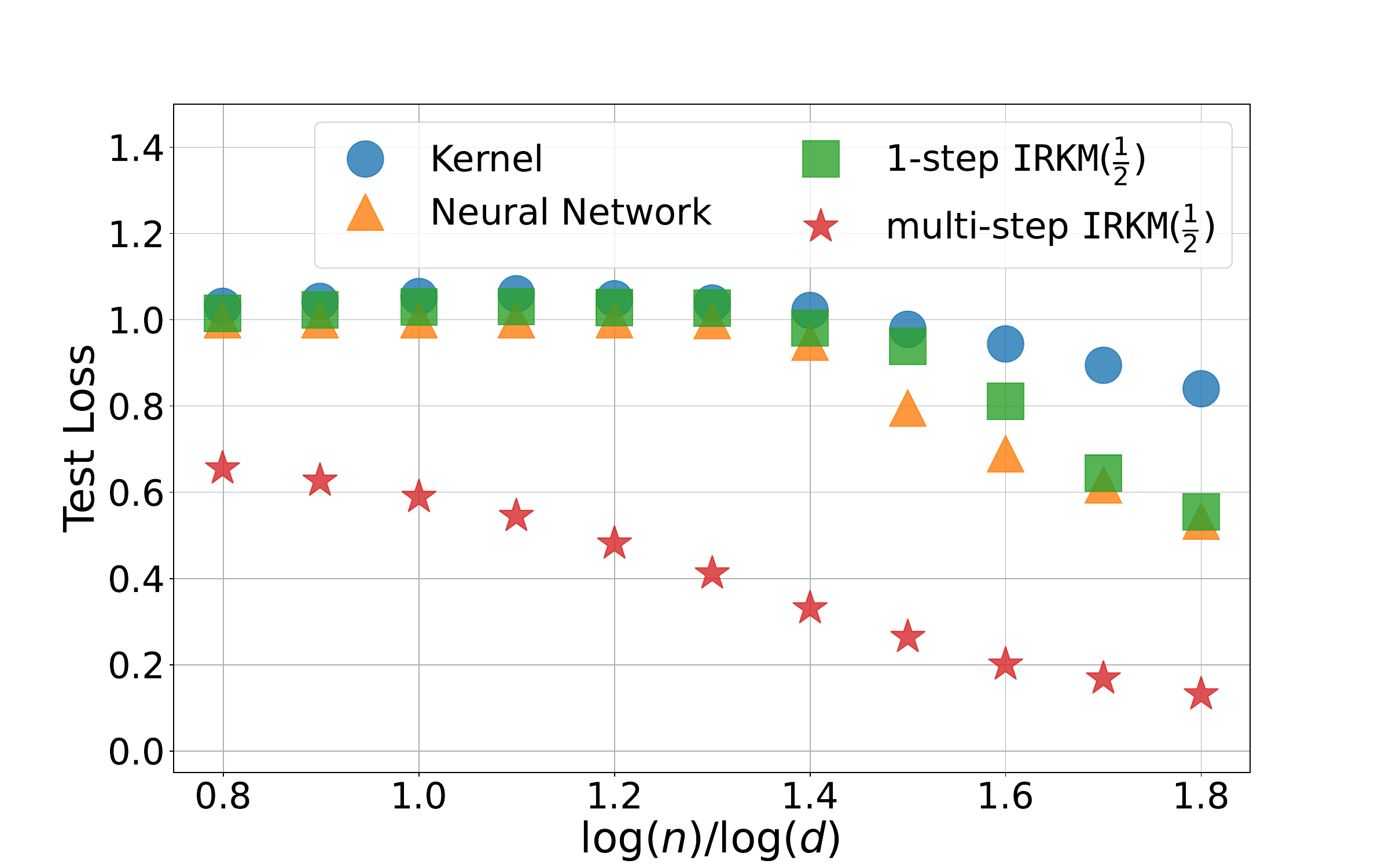}
        \caption{Generalization error}
    \end{subfigure}
    \begin{subfigure}[b]{0.3\linewidth}
        \centering
        \includegraphics[width=\linewidth]{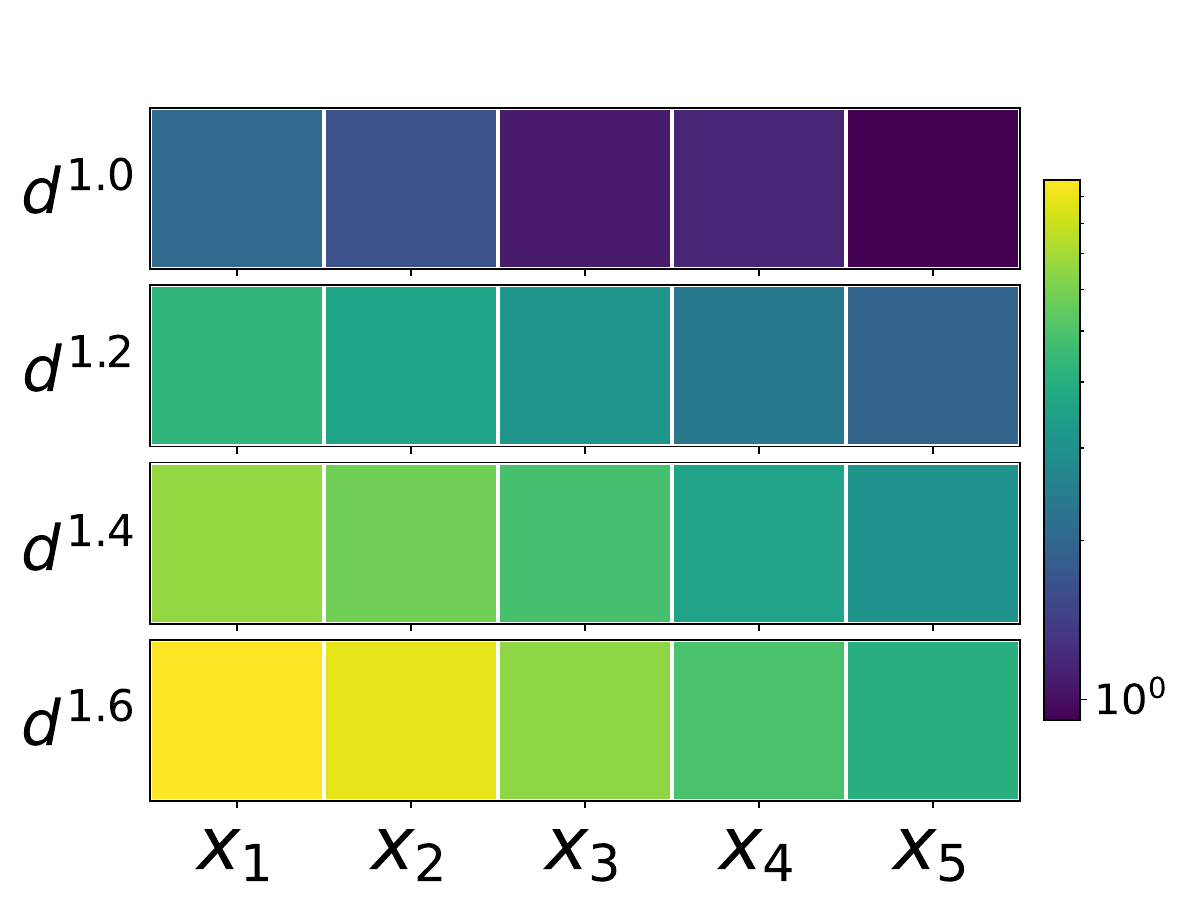}
    \caption{One step identification }
    \end{subfigure}
        \begin{subfigure}[b]{0.3\linewidth}
        \centering
        \includegraphics[width=\linewidth]{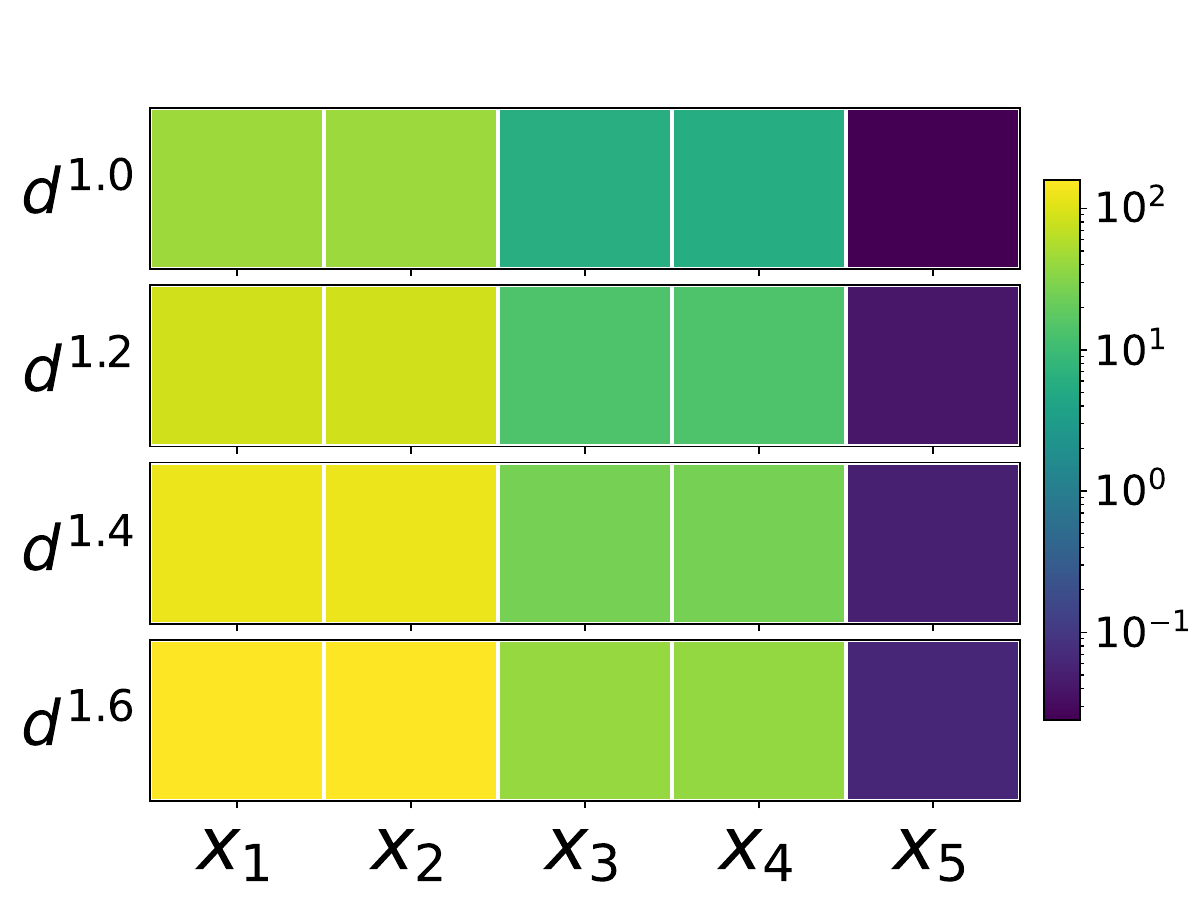}
        \caption{Multi-step identification}
    \end{subfigure}
  \end{minipage}

  \caption{{\bf %
    (A1,B1) Test loss for kernel machines, neural networks and $\mathtt{RFM}(\tfrac{1}{2})$. %
    (A2,B2) The empirical coordinate weights at the first step. %
    (A3,B3) The empirical coordinate weights at step $T$.} The samples are i.i.d. uniformly drawn from the hypercube $\{\pm 1\}^d$ (Row A) and $ \mathcal{N}(0,I_d)$ (Row B) with $d= 100$. The label is $y = x_1x_2 + x_1x_2x_3x_4 + \varepsilon$ with $\varepsilon\sim \mathcal{N}(0,0.1^2)$. We use the Laplacian kernel for both kernel machines and $\mathtt{RFM}(\tfrac{1}{2})$. The $\mathtt{RFM}(\alpha)$ algorithm is iterated for at most $T=20$ steps. We train a fully-connected network by Adam.  For all experiments, we report an average of $10$ runs.}
  \label{fig:task_4}
\end{figure}

\subsection{Effect of the mixing parameter $\alpha$ on  $\mathtt{IRKM}(\alpha)$ performance}
We vary the value of the mixing parameter $\alpha$ and evaluate the test performance of $\mathtt{IRKM}(\alpha)$ . Overall, the results indicate that setting $\alpha \approx \frac{1}{2}$ yields the best performance. The corresponding experimental results are presented in Figure~\ref{fig:ablation_alpha}.

\begin{figure}[H]
    \centering
    \begin{subfigure}[b]{0.45\linewidth}
        \centering
        \includegraphics[width=\linewidth]{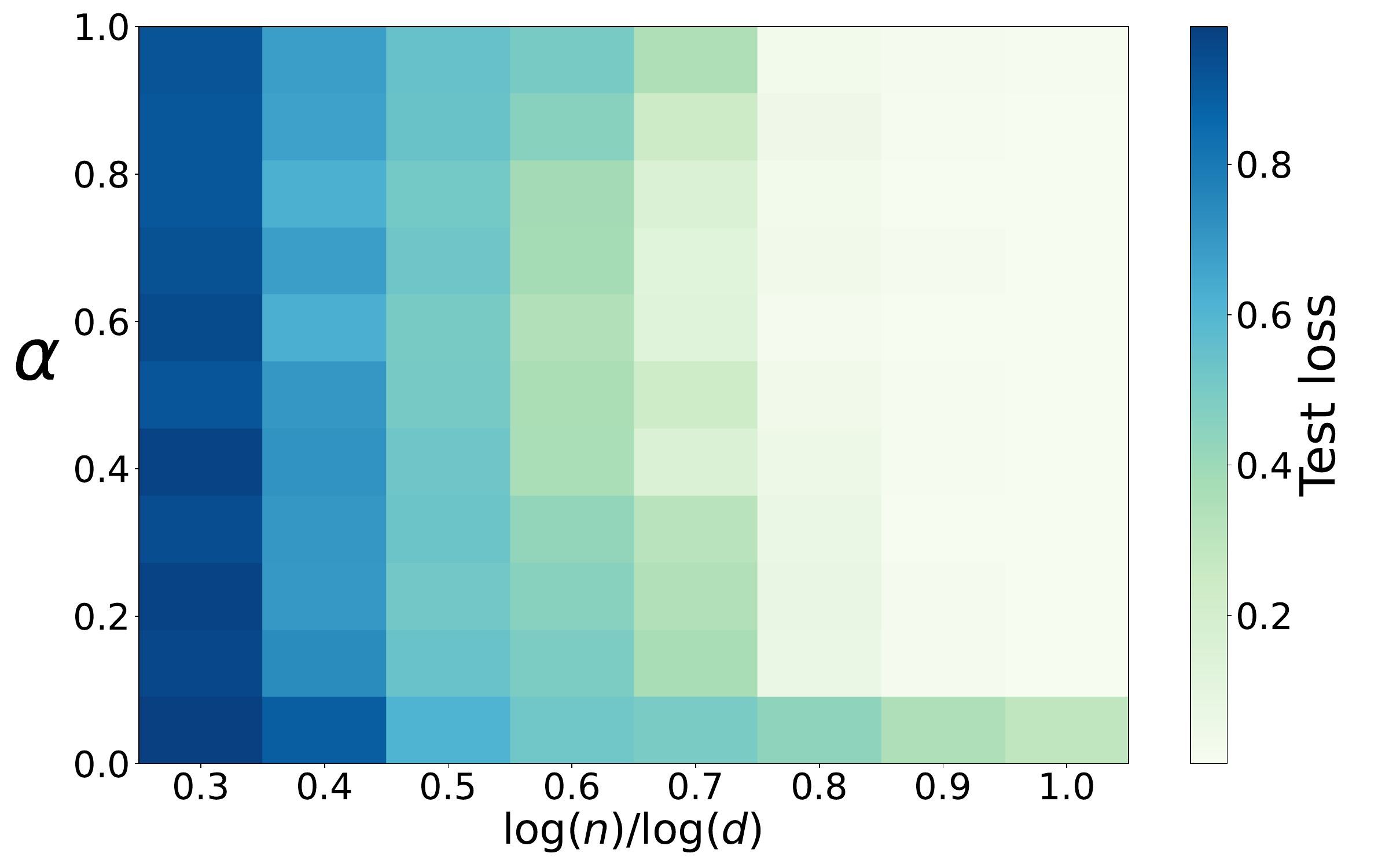}
        \caption{}
    \end{subfigure}
    \begin{subfigure}[b]{0.45\linewidth}
        \centering
        \includegraphics[width=\linewidth]{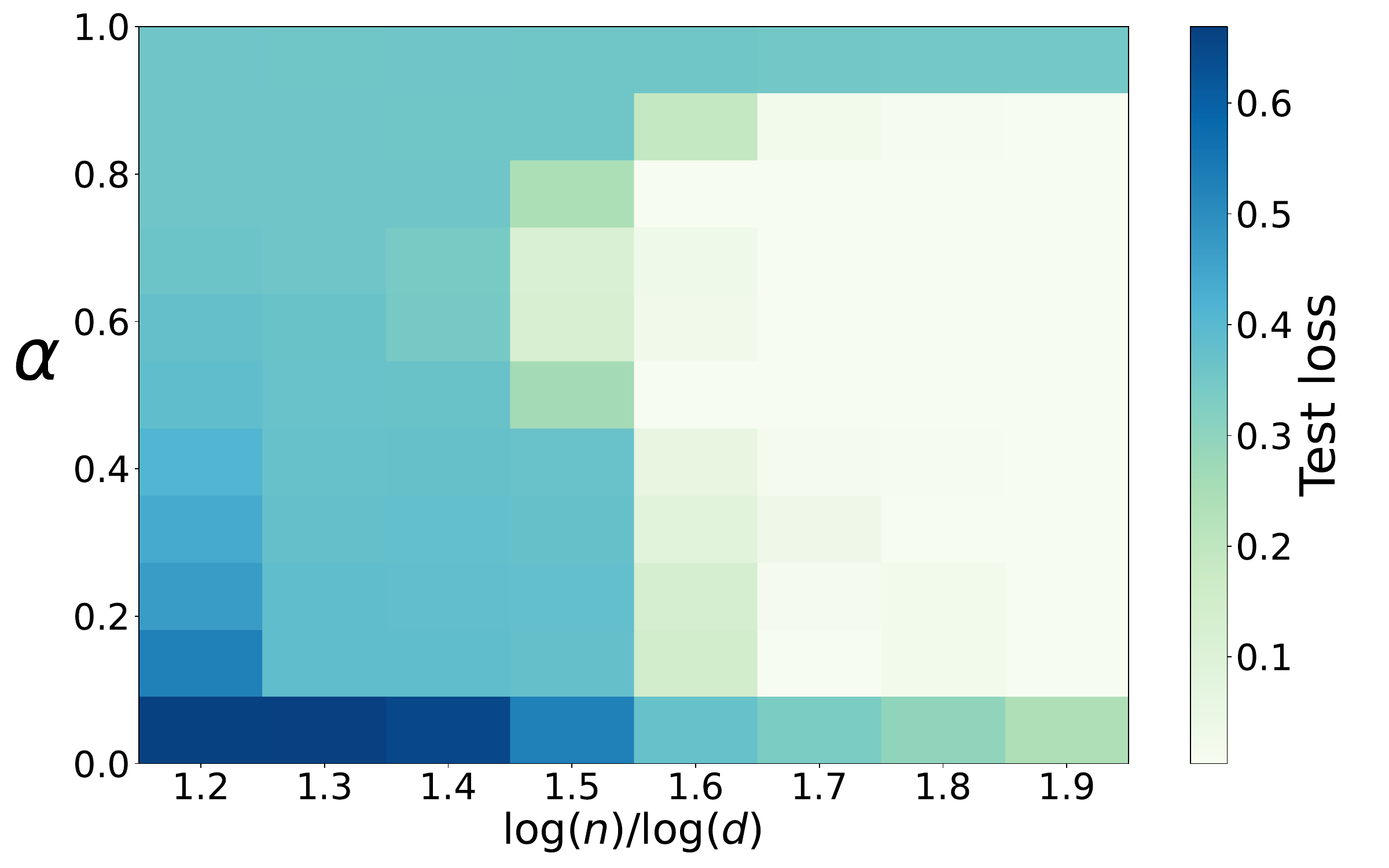}
    \caption{ }
    \end{subfigure}\\
    \centering
    \begin{subfigure}[b]{0.45\linewidth}
        \centering
        \includegraphics[width=\linewidth]{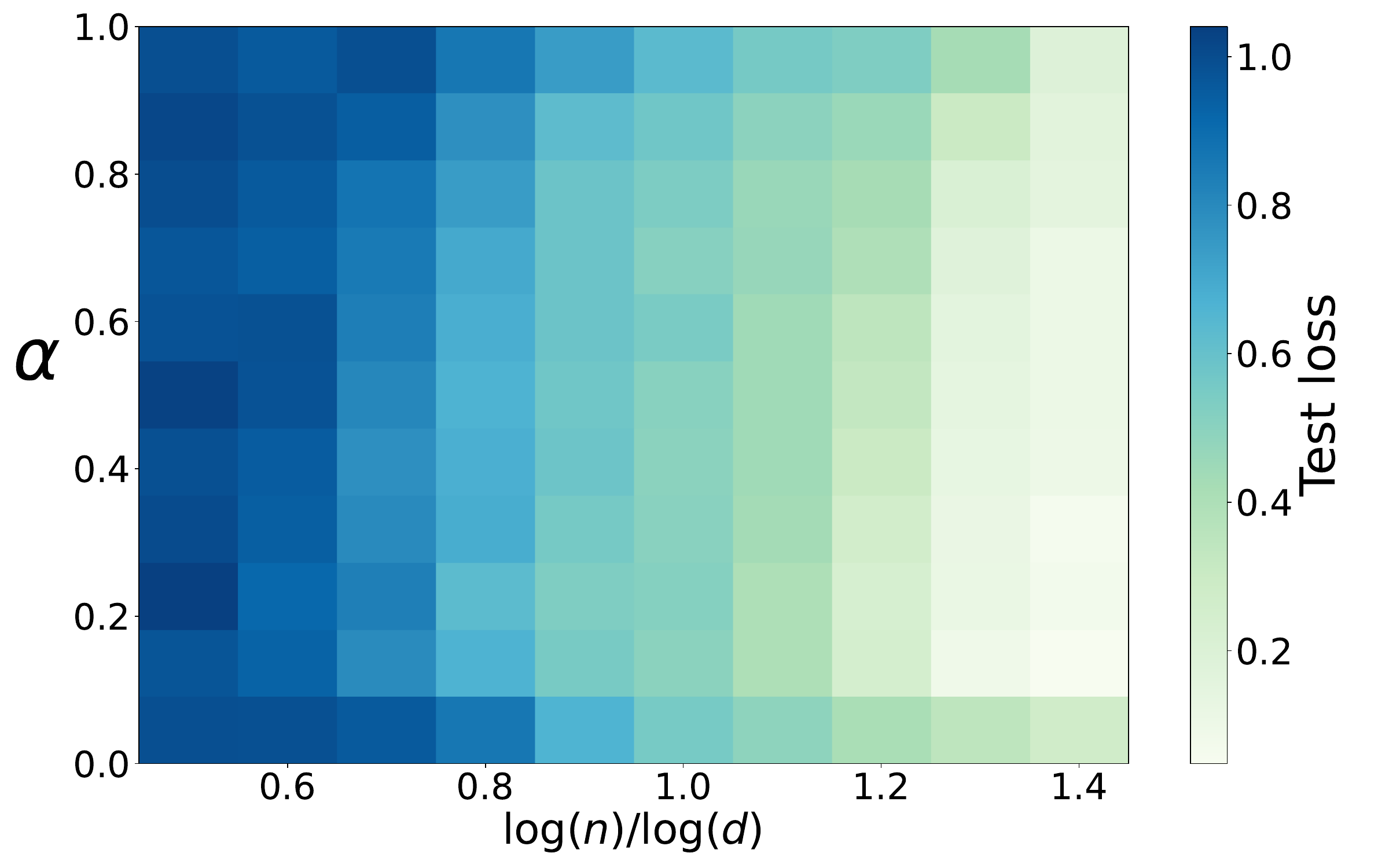}
        \caption{}
    \end{subfigure}
    \begin{subfigure}[b]{0.45\linewidth}
        \centering
        \includegraphics[width=\linewidth]{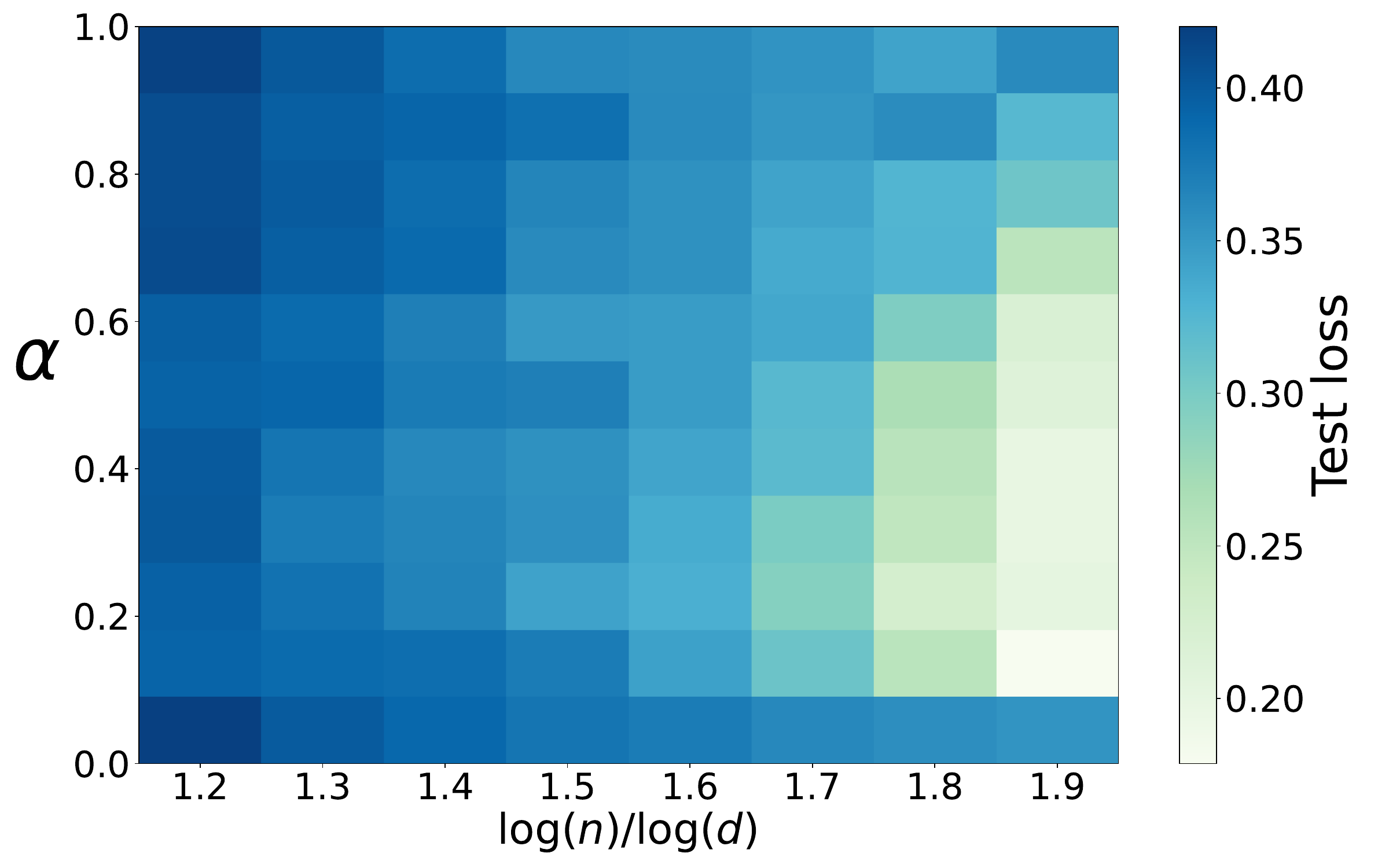}
    \caption{ }
    \end{subfigure}\\
\caption{{\bf Test loss of $\mathtt{IRKM}(\alpha)/\mathtt{RFM}(\alpha)$   with varying mixing parameter $\alpha$.} 
Panel~(a): $y = x_1 + x_2 + x_1x_2x_3 + x_1x_2x_3x_4 + \varepsilon$, $x\sim \{\pm 1\}^{500}$. 
Panel~(b): $y = x_1 + x_2x_3 + x_1x_2x_3x_4x_5 + \varepsilon$, $x\sim \{\pm 1\}^{100}$. 
Panel~(c): $y = z_1 + z_2 + z_1z_2z_3 + z_1z_2z_3z_4 + \varepsilon$,  $z = Ux$ where $x\sim \{\pm 1\}^{500}$ and $U$ is a random rotation matrix. 
Panel~(d): $y = x_1 + x_2x_3 + x_1x_2x_3x_4x_5 + \varepsilon$, $x\sim \mathcal{N}(0,I_{100})$.
The noise $\varepsilon$ is i.i.d. drawn from $ \mathcal{N}(0,0.1^2)$. 
$\mathtt{IRKM}(\alpha)/\mathtt{RFM}(\alpha)$ uses the Laplacian kernel, is run for $T=20$ steps for Panel (a,b,d) and for $T=50$ steps for Panel (c), and results are averaged over $10$ runs.}
    \label{fig:ablation_alpha}
\end{figure}

\subsection{Performance on tabular datasets}
We evaluate $\mathtt{IRKM}(\alpha)$ on tabular datasets that potentially require learning sparse functions. Using the benchmark from~\cite{fernandez2014we}, which compares $179$ machine learning methods across $121$ tabular classification tasks, we benchmark $\mathtt{IRKM}(\alpha)$ against neural networks, random forests, and RFM, previously reported as state-of-the-art for these datasets~\cite{rfm_science}. Note that in each iteration of $\mathtt{IRKM}(\alpha)$, we reuse the entire training set instead of sampling fresh data.
We evaluate the performance based on the following standard metrics:
\begin{itemize}
    \item Average accuracy: The average accuracy of the classifier across all datasets.
    \item P90/P95: The percentage of datasets where the classifier achieved accuracy within 90\%/95\% of the best model.
    \item Friedman rank: The average rank of the classifier across all datasets.
\end{itemize}

Table~\ref{tab:uci} shows that $\mathtt{IRKM}(0)$ achieves comparable results to RFM, while both significantly outperform random forests and neural networks. This in particular suggests that the prediction functions in these data sets tend to be sparse; that is, many feature coordinates do not significantly influence the label. We report only the results for $\mathtt{IRKM}(0)$, as this setting achieves the highest average accuracy.

\begin{table}[htbp]
\centering
\begin{tabular}{lcccc}
\toprule
Classifier $(\lambda)$ & Avg. Accuracy $(\%)$&  P90  & P95 &Friedman Rank \\
\midrule
IRKM (ours) & $84.69$ & $92.56$ &$79.34$ & $25.80$ \\
RFM  & $85.15$ & $93.39$ &$78.51$ & $20.37$\\
Random Forest& $80.01$ & $75.43$ & $64.46$ & $62.21$ \\
Neural Network & $78.46$ & $64.46$ & $46.28$ & $64.71$ \\
\bottomrule
\end{tabular}
\caption{Test performance of $\mathtt{IRKM}(0)$ against RFM, Random Forest, and Neural Networks.\label{tab:uci}}
\end{table}

\subsection{Performance on Epistatic Interactions}
In Genome-Wide Association Studies (GWAS), epistasis detection involves identifying genetic variant interactions that jointly influence disease phenotypes. This task is important in biology as it can reveal the underlying mechanisms of complex diseases. The challenge is difficult because combinations of SNPs may contribute to disease risk while individual SNPs show no detectable effects. This creates nonlinear sparse functions in high-dimensional data where only a small subset of genetic variants are truly relevant.

We evaluate our approach using the GAMETES~\cite{urbanowicz2012gametes} simulator. GAMETES generates pure epistatic interactions through penetrance tables where only multi-locus genotype combinations (AA, Aa, aa) predict phenotype, and individual loci show no main effects in these models. This provides controlled ground truth for assessing our algorithm's ability to identify relevant feature subsets among genetic variants.
Each genetic variant (i.e., SNP)  takes values $\{0, 1, 2\}$ corresponding to the count of minor alleles, and the classification target is binary disease status.

We conduct experiments with 2-way epistatic interactions, namely the label depends on 2 SNPs, varying the heritability parameter across $\{0.05, 0.1, 0.2\}$. Higher heritability values correspond to stronger genetic effects and improved signal-to-noise ratios. Classification performance is evaluated using the AUC metric (Area Under the ROC Curve), comparing $\mathtt{IRKM}(\tfrac{1}{2})$ against neural networks and random forests. We consider fully-connected networks with depth $6$ and width $512$, trained by Adam with mini-batch size $128$. Note that in each iteration of $\mathtt{IRKM}(\tfrac{1}{2})$, we reuse the entire training set instead of sampling fresh data. The results are presented in Figure~\ref{fig:epistasis}.

\begin{figure}[H]
  \centering
    \begin{subfigure}[b]{0.32\linewidth}
        \centering
        \includegraphics[width=\linewidth]{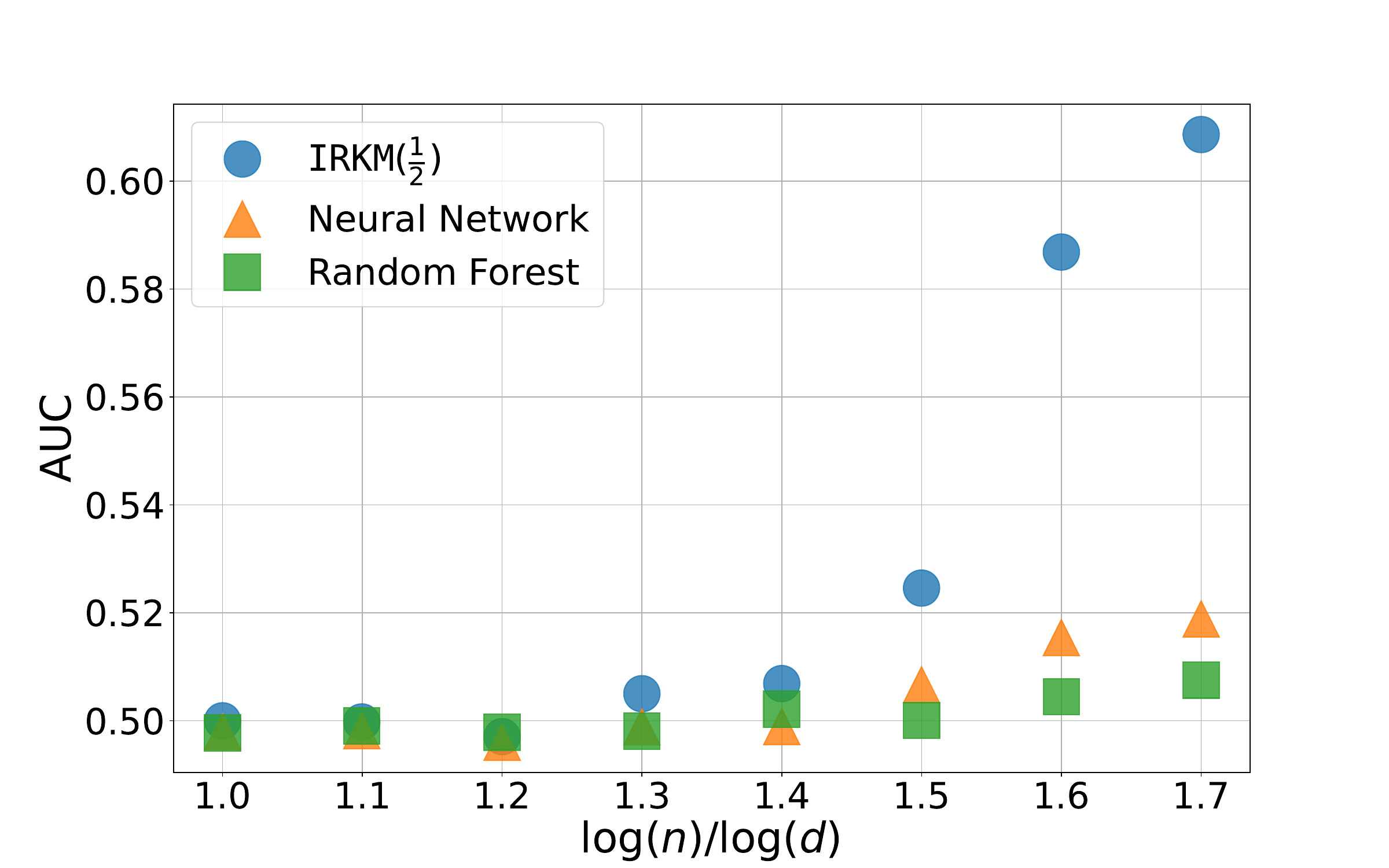}
        \caption{Heritability = 0.05}
    \end{subfigure}
    \begin{subfigure}[b]{0.32\linewidth}
        \centering
        \includegraphics[width=\linewidth]{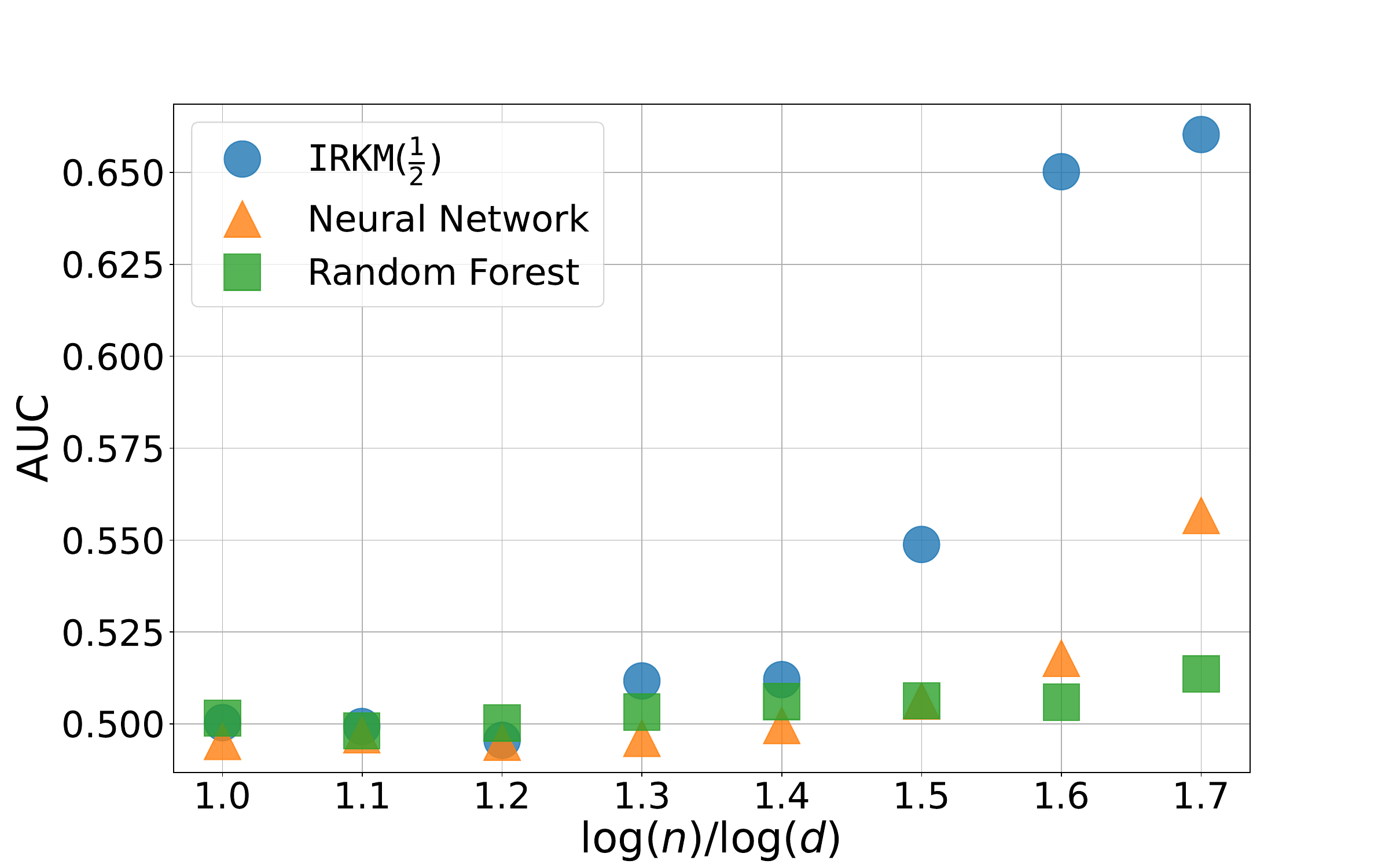}
    \caption{Heritability = 0.1 }
    \end{subfigure}
        \begin{subfigure}[b]{0.32\linewidth}
        \centering
        \includegraphics[width=\linewidth]{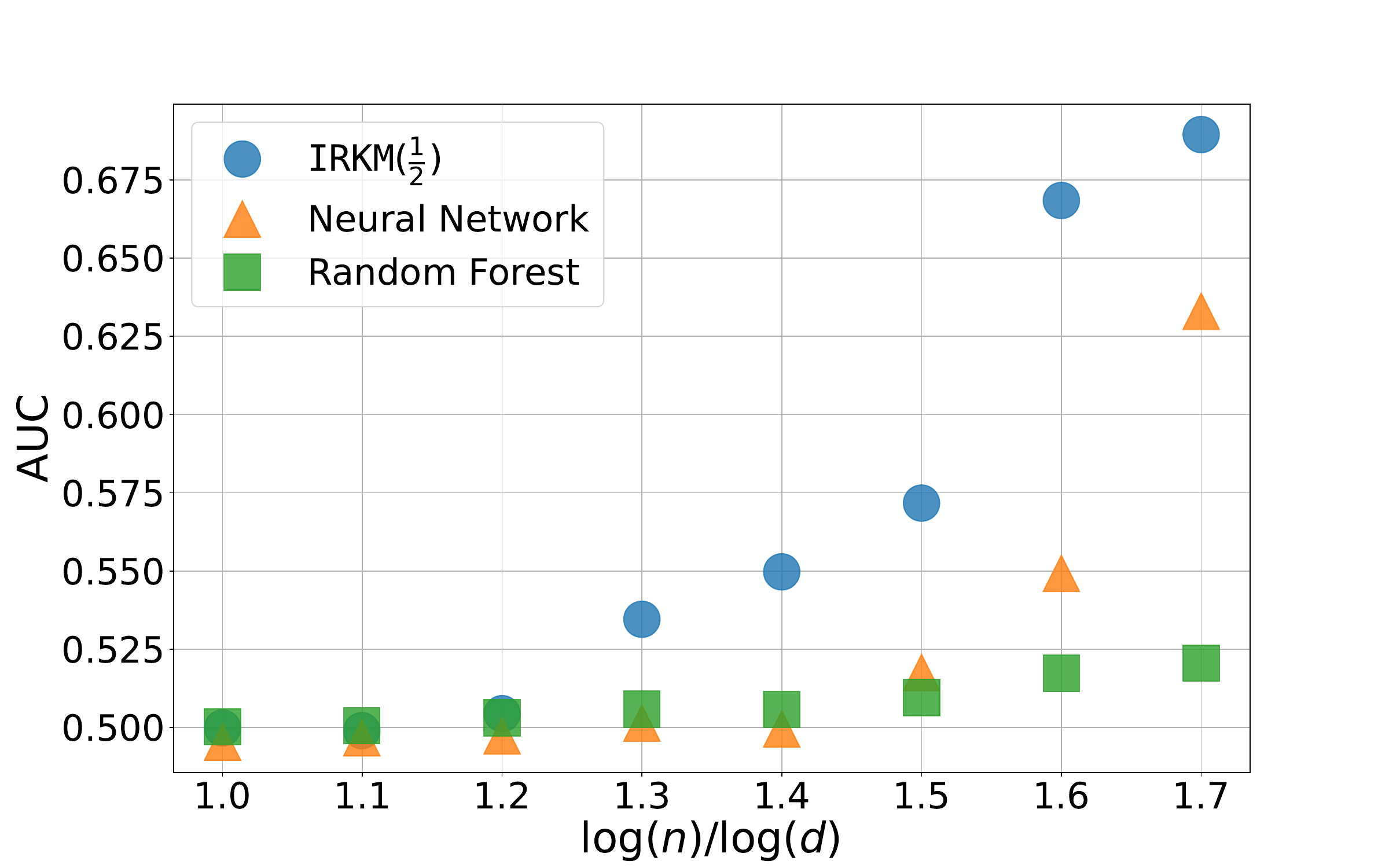}
        \caption{Heritability = 0.2}
    \end{subfigure}

  \caption{{\bf AUC for $\mathtt{IRKM}(\tfrac{1}{2})$, neural networks and random forest.
   } The datasets are generated by GAMETES simulator with 100 SNPs, where the minor allele frequency is set to $0.2$ for both interacting loci. We set the number of epistatic attributes to $2$ and vary heritability across $0.05$, $0.1$, and $0.2$. We report the average of $10$ independent runs.}
  \label{fig:epistasis}
\end{figure}

%% file: sections/conclusion.tex
\section{Conclusion}

In this work, we have demonstrated that iteratively retrained kernel machines can efficiently learn coordinate features and hierarchical structures, thereby challenging the prevailing belief that these phenomena are unique to neural networks. 
Specifically, we established two key results: (1) the empirical coordinate weights of the KRR predictor can serve as an asymptotically consistent estimator for the population coordinate weights, even while the generalization error remains large; and (2) IRKM can learn functions with leap complexity $p$ using roughly $d^{p-1}$ samples, thereby demonstrating that kernel methods effectively learn hierarchical structure.
These findings suggest that kernel-based methods are much more powerful than previously thought and their feature-learning ability deserves further investigation. More broadly, our work contributes to understanding what makes learning algorithms effective, suggesting that adapting to the intrinsic structure of data may be more fundamental than the specific model class that is employed.

%% file: sections/appendix_prelim.tex
\section{Orthogonal polynomials: Isometric properties and kernel approximation}
\input{sections_prelim/iso_poly}
\input{sections_prelim/kernel_approx}

%% file: sections_prelim/iso_poly.tex
\subsection{Isometric properties of orthogonal polynomials}\label{sec:iso_poly}
In this section, we establish some isometric properties of covariance-like matrices induced by orthogonal polynomials. Setting the stage, fix a probability space $(D,\mathcal{A},\mu_d)$ with $D\subset\R^d$ and a set 
$\{\phi_j\}_{j\in \bar{\mathcal{S}}}$
of orthonormal polynomials with respect to $\mu_d$, indexed by some set $\bar{\mathcal{S}}$. For any finite set $\S\subset \bar{\mathcal{S}}$ and a point $x\in D$, we define the concatenated vector:
$$\phi_{\S}(x):=(\phi_j(x))_{j\in \S}\in \R^{|\S|}.$$
Finally given a sequence of points $X=(x^{(1)},\ldots, x^{(n)})$ we may stack $\phi_{\S}(x^{(i)})$ as rows to form the matrix
$$\Phi_{\S}(X)=[\phi_j(x^{(i)})]\in \R^{n\times |\S|}.
$$
Note that the rows of $\Phi_{\S}(X)$ are indexed by the data points and the columns by the polynomials in $\S$. To simplify notation, we will often omit $X$ from the symbol $\Phi_{\S}(X)$, whenever it is clear from context. Throughout, we will assume that the data points $x^{(1)},\ldots, x^{(n)}$ are sampled independently from $\mu$. 

We impose the following hypercontractivity assumption for the rest of the section, mimicking equation \eqref{eqn:hypercontrac} for our two running examples. More precisely, we will assume hypercontractivity of the product measure $\mu_d\times \mu_d$ on $\R^d\times \R^d$. Hypercontractivity of the original measure $\mu_d$ then follows trivially.

\begin{assumption}[Hypercontractivity]\label{ass:hypercontr}
{\rm
There exist constants $C_{l,q}>0$ indexed by integers $l,q\in \mathbb{N}$, such that any polynomial $f$ on $\R^d\times \R^d$ of degree at most $\ell$ satisfies:
$$\|f\|_{{L_q}(\mu_d\times \mu_d)}\leq C_{\ell,q}\cdot \|f\|_{L_2(\mu_d\times \mu_d)}\qquad \forall q\geq 2.$$}
\end{assumption}
In particular, Assumption~\ref{ass:hypercontr} holds in our two running examples---Gaussian and uniform on hypercube---with $C_{l,q}=(q-1)^{l/2}$, since the product measure of Gaussians (respectively uniform on hypercube) is Gaussian (respectively uniform on hypercube). 
The rest of the section is devoted to establishing isometry properties of the matrix $\Phi_{\S}$ in the high-dimensional regime $n,d\to \infty$.  We begin by estimating the operator norm of $\Phi_{\S}$.
\begin{lemma}[Norm control]\label{lem:norm_control}
     Suppose Assumption~\ref{ass:hypercontr} holds. Consider the regime $n = d^{p+\delta}$ where $\delta\in(0,1)$ is a constant. Fix a set $\S\subseteq\bar{\S}$ indexing polynomials of degree at most $l$. Then the following are true:
     \begin{enumerate}
    \item\label{it:smallS} in the regime $|\S| \leq Cd^{p+\delta_0}$ with $\delta_0\in (-\infty, \delta)$ the estimate  $\|\Phi_{\S}\|_{\rm op}=O_{d,\mathbb{P}}(\sqrt{n})$ holds.
    \item\label{it:bigS} in the regime $|\S| \geq Cd^{p+\delta_0}$ with $\delta_0\in (\delta,\infty)$ the estimate  $\|\Phi_{\S}\|_{\rm op}=O_{d,\mathbb{P}}(\sqrt{|\S|})$ holds. 
     \end{enumerate}
\end{lemma}

\noindent The exact asymptotics of $\Phi_{\S}$ depend on the relative scale of $n$ to $|\S|$, meaning:
$$\frac{n}{|\S|}\asymp d^{q} ~{\rm with}~ \underbrace{q>0}_{\rm Case~ I} ~{\rm or }~\underbrace{q<0}_{\rm Case~ II}.$$
The following two theorems show (respectively) that in the first case $\lambda>0$, the empirical covariance matrix $\frac{1}{n}\Phi_{\S}^\top\Phi_{\S}$ is asymptotically equal to the identity $I_{|\S|}$, while in the second case $\lambda<0$, the Gram matrix $\frac{1}{|S|}\Phi_{\S}\Phi_{\S}^\top$ is asymptotically equal to the identity $I_{n}$.
Both Lemma~\ref{lem:norm_control} and Theorem~\ref{lemma:phi_id_2} are proved in Appendix~\ref{proof:phi_id_2} and closely follow the proof template of \cite[Theorem 6(b)]{mei2022generalization}.

\begin{theorem}[Asymptotics in Case I]\label{lemma:phi_id_2}
  Suppose that Assumption~\ref{ass:hypercontr} holds and consider the regime $n = d^{p+\delta}$ where $\delta\in(0,1)$ is a constant.    Fix a set $\S\subseteq\bar{\S}$ of cardinality  $|\S| \leq Cd^{p+\delta_0}$, with $\delta_0\in (-\infty,\delta)$, indexing polynomials of degree at most $\ell$. Then for any $\delta'\in (0,\delta-\delta_0)$ there exists a constant  $C'$ satisfying 
   $$\E\brac{\snorm{\frac{1}{n}\Phi_{\S}\tran \Phi_{\S} - I_{|\S|}}} \leq C' d^{-\delta'/2}\sqrt{\log(n)}.$$
   Consequently,   with probability at least $1 - \tfrac{C'}{\sqrt{\log(n)}}$, we have
$
    \snorm{\Phi_{\S}(X)\tran \Phi_{\S}(X)/ n - I_{|\S|}} \leq d^{-\delta'/2 + \epsilon}
$.
\end{theorem}

Next, we argue that in the complementary setting when $n$ is small compared to $|\S|$, the Gram matrix $\Phi_{\S}\Phi_{\S}^\top=[\langle 
\phi_{\S}(x^{(i)},\phi_{\S}(x^{(j)})\rangle]_{i,j=1}^n$ is close to a multiple of the identity. The proof appears in  Appendix~\ref{proof:gram_matrix} and closely follows the proof template of \cite[Lemma 3]{mei2022generalization} based on matrix decoupling. 

\begin{theorem}[Asymptotics in Case II]\label{lemma:gram_matrix} Suppose that Assumption~\ref{ass:hypercontr} holds and consider the regime $n = d^{p+ \delta}$ where $\delta\in(0,1)$ is a constant. Fix a set $\S\subseteq\bar{\S}$ of cardinality  $|\S| \geq Cd^{p+\delta_0}$, with $\delta_0\in (\delta,\infty)$, indexing polynomials of degree at most $\ell$. Then for any $\epsilon>0$ there exists a constant $C'$ satisfying:
    \begin{equation}\label{eqn:bound_needed}
 \E\brac{\snorm{ \frac{1}{|\S|}\Phi_{\S}   \Phi_{\S}\tran  
 - I_n}} \leq C'\left(d^{-\frac{p+\delta}{2}+2\epsilon}+\sqrt{\log(n)}\cdot d^{-\frac{\delta_0-\delta}{2}+\epsilon}\right),
\end{equation}
Consequently, in the case $p\geq 1$, $\delta_0=1$, and $\epsilon\in (0,\delta)$  with probability at least $1 - \tfrac{C'}{\sqrt{\log(n)}}$, we have
$
    \snorm{ \Phi_{\S}   \Phi_{\S}\tran  
 / |\S|- I_n} \leq d^{-\tfrac{1-\delta}{2}+\epsilon}
$.
\end{theorem}

%% file: sections_prelim/kernel_approx.tex
\subsection{Approximating kernels by orthogonal polynomials}\label{sec:kernel_approx}
In this section, we discuss approximating kernels by quadratic forms in orthogonal polynomials. Throughout, we fix an inner-product kernel \eqref{eq:inner_product_kernel} satisfying the following regularity assumption. Note that Assumption~\ref{assump:g_0} implies that both kernels $K(z,x)$ and $K'(z,x)$ satisfy Assumption~\ref{assump:g}. This is important because  later parts of the paper  will invoke the results of this section with $K$ replaced by $K'$.

\begin{assumption}[Regularity of the kernel]\label{assump:g}
There exists a constant $\varepsilon\in (0,1)$ such that the function $g$ in \eqref{eq:inner_product_kernel} is Lipschitz continuous on $(-1-\varepsilon,1+\varepsilon)$ and is analytic on $(-\varepsilon,\varepsilon)$.
\end{assumption}

Our goal is to approximate $K(x,y)$ by a quadratic form $K(x,y)\approx \phi(x)^\top D\phi(y)$ where $\phi\colon\R^d\to\R^p$ has orthonormal polynomials as its coordinate functions. Ideally, $D$ should be a relatively simple matrix (e.g., diagonal). We proceed by first forming a Taylor approximation of $g$ at the origin and forming the approximate kernel:
$$g_{m}(t):=\sum_{k=0}^{m}\tfrac{g^{(k)}(0)}{k!}t^k\qquad \textrm{and}\qquad K_{m}(x,y) := g_{m}\round{\tfrac{\inner{x,y}}{d}}.$$
The following theorem shows that the deviation $K_{m}(X,X)-K(X,X)$ is asymptotically equivalent to a multiple of the identity under favorable conditions. Namely in the regime, $n=d^{p+\delta}$ with $\delta\in (0,1/2)$ the right side of \eqref{eqn:taylor_approx} tends to zero for any approximation order $m\geq 2p$.
The proof appears in Appendix~\ref{proof:poly_approx}. The approximation of kernels has been previously studied in the linear regime (i.e., $n\asymp d$) by \cite{elkaroui2010spectrum} and in the quadratic regime (i.e., $n\asymp d^2$) by~\cite{pandit2024universality}.

\begin{theorem}[Taylor approximation of kernels with Gaussian data]\label{thm:poly_approx} Consider independent, mean zero, isotropic random vectors $x^{(1)},\ldots,x^{(n)}$ in $\R^d$ and suppose that the coordinates of each vector $x^{(i)}$ are independent and $\sigma$-subGaussian. Suppose that we are in the regime $n = d^{p+ \delta}$, where $\delta\in(0,1/2)$ is a constant. Fix an arbitrary constant $C>2$ and approximation order $m\geq 2p$. Then if $d/\log(d) > c_\epsilon$, the estimate 
\begin{equation}\label{eqn:taylor_approx}
\snorm{K(X,X) - K_{m}(X,X) - (g(1) - g_{m}(1)) I_n} \le  c_{g,m}\left(\sqrt{\frac{(C\log(n))^{m+1}}{d^{m-2p+1-2\delta}}}+\sqrt{\frac{C\log(n)}{d}}\right),
\end{equation}
holds with probability at least $1-\frac{4}{n^{C-2}}$, where $c_{g,m},c_\epsilon<\infty$ are constants that depend only on $m$, $\sup_{k\geq m+1 }g^{(k)}(0)$, $\sigma^2$,  $\epsilon$, and the Lipchitz constant of $g$.
\end{theorem}

In light of Theorem~\ref{thm:poly_approx}, we may now approximate the kernel matrix by $K_m(X,X)$ plus a multiple of the identity. Note that each entry of $K_m(X,X)$ is linear in the powers of inner products $(\langle x^{(i)},x^{(j)}\rangle/d)^k$. It remains therefore to express these powers as a quadratic form in an orthogonal polynomial basis. We do so in our two running examples in the following two sections.

\subsubsection{Uniform measure on the hypercube}
In this section, we focus on the uniform measure $\tau_d$ over the hypercube $\mathbb{H}^d=\{-1,1\}^d$ as discussed in Section~\ref{sec:fourier_exp}. We consider the Fourier orthogonal basis $\phi_\lambda:=x^\lambda$ for each $\lambda\in \{0,1\}^d$. As usual, for any set $\S\subset\{0,1\}^d$, we define the concatenated vector $\phi_S(x)=\{\phi_{\lambda}(x)\}_{\lambda\in S}\in \R^{|\S|}$. Stacking the vectors $\Phi_{\lambda}(x^{(i)})$ as rows yields a matrix $\Phi(X)\in \R^{n\times |\S|}$. 
It will be useful to isolate degree $k$ basis elements:
\begin{align*}
\S_k&=\{\lambda\in \{0,1\}^d: |\lambda|=k\}.
\end{align*}
We set $\Phi_{\S_k}(X)$ and $\Phi_{\leq k}(X)$ to be, respectively, the submatrices of $\Phi$ indexed by degree $k$ and degree at most $k$ basis elements. 
We will suppress the symbol $X$ from $K(X,X)$ and $\Phi(X)$ throughout the section in order to simplify the notation.
We stress that all probabilistic statements are in reference to the random vectors $\{x^{(i)}\}_{i=1}^n$ sampled independently from $\tau_d$. 

In parallel to Theorem~\ref{thm:poly_approx}, we present the Taylor approximation of kernels for data on the hypercube. We provide a proof for a more general parameterized kernel $K_w(x,y) = g(\inner{\sqrt{w}\odot x,\sqrt{w}\odot y}/d)$, which encompasses the special case $w = \mathbf{1}_d$ below. The proof appears in Appendix~\ref{proof:taylor_approx_kernel_m}.

\begin{theorem}[Taylor approximation of kernels on hypercube]\label{thm:poly_approx_cube} Consider independent, mean zero, isotropic random vectors $x^{(1)},\ldots,x^{(n)}$ in $\R^d$ and suppose that the coordinates of each vector $x^{(i)}$ are independent and are uniformly drawn from $\H^d$. Suppose that we are in the regime $n = d^{p+ \delta}$, where $\delta\in(0,1)$ is a constant. Fix an arbitrary constant $C>2$ and approximation order $m> 2p$.  Then, then the estimate 
\begin{equation}\label{eqn:taylor_approx_cube}
\snorm{K(X,X) - K_{m}(X,X) - (g(1) - g_{m}(1)) I_n} \le  c_{g,m}\sqrt{\frac{(C\log(n))^{m+1}}{d^{m-2p+1-2\delta}}}
\end{equation}
holds with probability at least $1-\frac{4}{n^{C-2}}$, where $c_{g,m}<\infty$ is a constant that depends only on $m$, $\{g^{(k)}(0)\}_{k=1}^{m}$, and $L_m$. 
\end{theorem}

We begin with the following (deterministic) lemma that expresses the matrix $\round{\frac{XX\tran}{d}}^{\odot k}$ as a diagonal quadratic form acting on the Fourier basis. We further decompose the diagonal matrix into the dominant part, corresponding to the degree $k$ basis elements, and the small error part, corresponding to the lower-order basis elements. 

\begin{lemma}[Conversion from a polynomial kernel to a Fourier basis]\label{lemma:each_mono_cube}
For any $k>0$, the following holds
\begin{align}
    \frac{1}{k!}\round{\frac{XX\tran}{d}}^{\odot k} = d^{-k}\cdot \Phi_{\S_k}\Phi_{\S_k}\tran + \sum_{ \substack{j:\, 0\leq j<k, \\ k-j~\mathrm{is~even}}}\Phi_{\S_j} \widetilde{D}_{\S_j}\Phi_{\S_j}\tran,
\end{align}
where each $\widetilde{D}_{\S_j}$ is a diagonal matrix with all entries on the order of $\Theta_d(d^{-(j+k)/2})$.
\end{lemma}

The following lemma---the main part of the section---shows the asymptotic equivalence of $K(X,X)$ to a sum of a multiple of the identity and a quadratic form in Fourier basis elements of degree at most $p$. The proof of the lemma appears in Section~\ref{proof:kernel_to_mono}.

\begin{lemma}[Fourier basis approximation of a kernel]\label{lemma:kernel_to_mono}
Consider the regime $n = d^{p+\delta}$ where $\delta \in (0,1)$ is a constant. Then for any $\epsilon\in (0,\delta)$, the following holds:
\begin{align}
    &\snorm{K - \Phi_{\leq p} D \Phi_{\leq p}\tran - \round{ g(1) - g_{p}(1)}I_{n}} = O_{d,\P}\round{ \frac{\log(n)}{d^{\tfrac{1-\delta}{2}-\epsilon}}},\label{eq:kernel_to_mono}\\
    &\snorm{K - \Phi_{\leq p} D \Phi_{\leq p}\tran - \frac{g^{(p+1)}(0)}{d^{p+1}}\offd\round{ \Phi_{\S_{p+1}}\Phi_{\S_{p+1}}\tran}- \round{ g(1) - g_{p}(1)}I_{n}} = O_{d,\P}\round{ \frac{\log(n)}{d^{\tfrac{2-\delta}{2}-\epsilon}}},\label{eq:kernel_to_mono_advanced}
\end{align}
where $D$ is a diagonal matrix satisfying $\|D_{\S_k} -g^{(k)}(0)d^{-k}I_{|\S_k|}\|_{\rm op} = O_d(d^{-k-1})$ for $k=0,\ldots, p$.
\end{lemma}

Note that the estimate \eqref{eq:kernel_to_mono} allows us to approximate $K$ by the simple expression $\Phi_{\leq p} D \Phi_{\leq p}\tran + \round{ g(1) - g_{p}(1)}I_{n}$ up to an error $\approx d^{-(1-\delta)/2}$. In contrast, 
\eqref{eq:kernel_to_mono_advanced} provides a more sophisticated approximation of $K$ but with a much better error rate $\approx d^{-(2-\delta)/2}$. This improved error rate will be important for some of our arguments.

\subsubsection{Hermite expansion of polynomial kernel}
Next, we focus on the Gaussian measure $\gamma_d$ on $\R^d$ as discussed in Section~\ref{sec:hermite}. We consider the Hermite orthogonal basis $\phi_{\alpha}=h_\alpha$ for each multi-index $\alpha\in \mathbb{N}^d$. As usual, for any set $\S\subset\mathbb{N}^d$, we define the concatenated vector $\Phi_S(x)=\{\phi_{\alpha}(x)\}_{\alpha\in S}\in \R^{|\S|}$. Stacking the vectors $\Phi_{\lambda}(x^{(i)})$ as rows yields a matrix $\Phi(X)\in \R^{n\times |\S|}$. 
It will be useful to isolate degree $k$ basis elements:
\begin{align*}
\S_k&=\left\{\alpha\in \mathbb{N}^d: \sum_{i=1}^d\alpha_i=k\right\}.
\end{align*}
We set $\Phi_{\S_k}(X)$ and $\Phi_{\leq k}(X)$ to be, respectively, the submatrices of $\Phi$ indexed by degree $k$ and degree at most $k$ basis elements. 
We will suppress the symbol $X$ from $K(X,X)$ and $\Phi(X)$ throughout the section in order to simplify the notation.
All probabilistic statements in this section are in reference to  the random vectors $\{x^{(i)}\}_{i=1}^n$ sampled independently from $\gamma_d$.

As in the hypercube setting of the previous section, we estimate the matrix  $\round{\frac{XX\tran}{d}}^{\odot k}$ by a quadratic form acting on the Hermite basis. The distinction is now that this quadratic form is not diagonal in the simplest case of $k= 2$.\footnote{Note that the claim in \cite{ba2023learning} that the quadratic form in the Hermite basis is diagonal is inaccurate as can be seen from \eqref{eqn:second_orderequiv}.} Namely, in the hypercube case, we saw that $\frac{1}{k!}\round{\frac{XX\tran}{d}}^{\odot k}$ is well-approximated by $d^{-k}\Phi_{\S_k}\Phi_{\S_k}^\top$. The analogous statement is not true in the Gaussian setting and instead the correct approximation for $k=2$ takes the form in Lemma~\ref{lem:conv_quad_herm}. The proof of the lemma is elementary and appears in Section~\ref{sec:proofoflem:conv_quad_herm}.
In the lemma and in the rest of the section, we will order the second-order Hermite polynomials $\phi_{\S_2}$ so that the pure second-order parts $\{h_{2e_i}\}_{i=1}^d$ appear first and the mixed parts $\{h_{e_i+e_j}\}_{1\leq i<j\leq d}$ appear second.

\begin{lemma}[Conversion from quadratic kernel to a Hermite basis]\label{lem:conv_quad_herm}
The equality holds:
\begin{equation}\label{eqn:second_orderequiv}
    \frac{1}{2}\round{\frac{XX\tran}{d}}^{\odot 2}=\Phi_{\leq 2}\cdot \begin{pmatrix}
        &\frac{1}{2d}&0 &D_{13}\tran\\
        &0 &0 &0\\
        &D_{13} &0 &  \frac{1}{d^2}I_{|\S_2|}
    \end{pmatrix} \cdot\Phi^\top_{\leq 2},
\end{equation}
where  we define the matrix $D_{13}:=(\frac{1}{\sqrt{2}d^2}{\bf 1}_{d}; 0_{|\S_2|-d})\in \R^{|\S_2|}$.
\end{lemma}

The case $k=2$, described in Lemma~\ref{lem:conv_quad_herm}, appears to be exceptional. Indeed, the expansions of $\round{\frac{XX\tran}{d}}^{\odot k}$ for the larger values of $k=3,4$ are diagonal, as the following lemma shows, and we conjecture this to be true for all $k>4$. We note, however, that the combinatorial argument we use in the case $k=3,4$ appears to be difficult to generalize. The proof of Lemma~\ref{lemma:hermite_decomposition_p=3} appears in Appendix~\ref{proof:hermite_decomposition_p=3}.

\begin{lemma}[Conversion from a polynomial kernel to a Hermite basis for $k=3,4$]\label{lemma:hermite_decomposition_p=3} Consider the regime $n = d^{2+\delta}$ where $\delta \in(0,1/2)$. Then the following estimates hold:
\begin{align}
      &\snorm{\frac{1}{3!}\round{\frac{XX\tran}{d}}^{\odot 3} - \frac{1}{d^3}\Phi_{\S_3}\Phi_{\S_3}\tran - {\frac{1}{2d^2}}\Phi_{\S_1}\Phi_{\S_1}\tran} = O_{d,\P}(d^{\delta-1/2}),\label{eqn:k3}\\
      &\snorm{\frac{1}{4!}\round{\frac{XX\tran}{d}}^{\odot 4} - \frac{1}{d^4}\Phi_{\S_4}\Phi_{\S_4}\tran - \frac{1}{4d^2}\Phi_{\S_0}\Phi_{\S_0}\tran} = O_{d,\P}(d^{\delta- 1/2}).\label{eqn:k4}
    \end{align}
\end{lemma}

The following lemma---the main result of the section---parallels Lemma~\ref{lemma:kernel_to_mono} for the uniform distribution on the hypercube.
We note, however, the key difference is that the polynomial kernel is well-approximated by a sum of the identity and a quadratic form in the Hermite basis with a non-diagonal arrowhead matrix.  The proof appears in Appendix~\ref{proof:kernel_hermite_p=2}.

\begin{lemma}[Kernel to Hermite conversion]\label{lemma:kernel_hermite_p=2}
Consider the regime $n = d^{2+\delta}$ where $\delta \in (0,1/2)$ is a constant. Then for any $\delta_1\in (\delta,1/2)$, the following holds:    
    \begin{align}
        \snorm{K - \Phi_{\leq 2} D \Phi_{\leq 2}\tran-\round{g(1) - g_{2}(1)} I_n} = O_{d,\P}\round{d^{\delta_1-1/2}},
    \end{align}
    where
    \begin{align}\label{eq:D_p=2}
      D = \begin{pmatrix}
        &g(0)+\frac{g^{(2)}(0)}{2d}+\frac{g^{(4)}(0)}{4d^2} &0 &D_{13}\tran\\
        &0 &\round{\frac{g^{(1)}(0)}{d} + \frac{g^{(3)}(0)}{2d^2}}I_{|\S_1|} &0\\
        &D_{13} &0 & \frac{1}{d^2}g^{(2)}(0) I_{|\S_2|}
    \end{pmatrix}, 
    \end{align}
and $D_{13}= \round{\frac{g^{(2)}(0)}{\sqrt{2}d^2}\textbf{1}_d; \mathbf{0}_{|\S_2| -d}} \in \R^{|\S_2|}$.
\end{lemma}

%% file: sections/appendix_A.tex
\section{Proofs of main results}
\input{sections_appendix_a/proof_gauss}

\input{sections_appendix_a/hypercube}
\input{sections_appendix_a/main_one_step_learning}

\input{sections_appendix_a/one_step_gradient}

\input{sections_appendix_a/M_hat}

\input{sections_appendix_a/one_step_generalization}

%% file: sections_appendix_a/proof_gauss.tex
\subsection{Proof of Theorem~\ref{thm:main_gauss}} \label{proof:gauss}   
We prove the theorem only in the case $p=2$, since the case $p=1$ follows by similar (and simpler) arguments.
We analyze each coordinate $r=1,\ldots, d$ individually.
For simplicity of notation, we omit the input $X$ in $K(X,X)$ and $K'(X,X)$ throughout and set $K_{\lambda}=K+\lambda I$.
Summing \eqref{eq:derivative} with $z=x^{(i)}$ yields the convenient expression: 
\begin{align}
     &~~~\frac{1}{n}\sum_{i=1}^n \round{\partial_r f_{\rm KRR} (x^{(i)})}^2=\frac{1}{nd^2} \|K'X_r K_{\lambda}^{-1}y\|^2_2,\label{eqn:deriv_est}
\end{align}
where $X_r \in \R^{d\times d}$ is a diagonal matrix with entry $(X_r)_{i,i} = x_r^{(i)}$ for $i\in[d]$. By the assumption $f^*\in L_2(\gamma)$, we can expand $f^*$ into Hermite polynomials:
\begin{equation}\label{eqn:hermite_expans_target}
    y_i = f^*(x^{(i)}) = \phi(x^{(i)}) \tran b =\sum_{\lambda\in\mathbb{N}^d}b_{\lambda}\phi_{\lambda}(x^{(i)}),
\end{equation}
for some coefficients $b_{\lambda}$ satisfying $\|b\|_2<\infty$. 
Since $\E[y^2_i]=\E|f^*|^2 +\sigma_\varepsilon^2$ we deduce the estimate $\E[\|y\|_2^2] = n (\E|f^*|^2+\sigma_\varepsilon^2) = n(\norm{f^*}^2_{L_2} + \sigma_\varepsilon^2)$. Markov's inequality  subsequently shows $\norm{y}^2 = O_{d,\P}(n\log(n)\cdot (\norm{f^*}_{L_2}^2+\sigma_\varepsilon^2))$. We summarize this observation in the following proposition.
\begin{proposition}\label{prop:bound_y}
With probability at least $1-1/\log(n)$ we have $\norm{y}^2 \leq n\log(n)(\norm{f^*}_{L_2}^2+\sigma_\varepsilon^2)$.
\end{proposition}

Applying Lemma~\ref{lemma:kernel_hermite_p=2} with the two Kernels $K$ and $K'$ shows that for any $\epsilon>0$ there exist matrices $\Delta,\Delta_1$ satisfying $\snorm{\Delta},\snorm{\Delta_1} = O_{d,\P}(d^{\delta - 1/2+\epsilon})$ and 
\begin{align}
    K &= \Phi_{\leq 2}D\Phi_{\leq 2}\tran + \rho I_n + \Delta_1\label{eqn:K_decomp},\\
    K' &= \Phi_{\leq 2}D'\Phi_{\leq 2}\tran + \rho' I_n + \Delta,\label{eqn:K'_decomp}    
\end{align}
where we define $\rho := g(1) - g_{2}(1)$ and $\rho':= g'(1) - g_3(1)$. With the expression \eqref{eqn:K'_decomp} in place of $K'$, equation \eqref{eqn:deriv_est} becomes
\begin{align*}
    \frac{1}{n}\sum_{i=1}^n \round{\partial_r f_{\rm KRR} (x^{(i)})}^2= \frac{1}{nd^2} \norm{\round{ \underbrace{ \Phi_{\leq 2}D'\Phi_{\leq 2}\tran X_r K_{\lambda}^{-1}}_{=:T_1} + \underbrace{(\rho'  I_n + \Delta)X_r K_{\lambda}^{-1}}_{=:T_2}}y}^2_2.
\end{align*}
Expanding the square, we conclude
\begin{align}
    \frac{1}{n}\sum_{i=1}^n \round{\partial_r f_{\rm KRR} (x^{(i)})}^2 = \frac{1}{nd^2}\left(\|T_1y\|^2_2+\|T_2y\|^2_2+2\langle T_1y, T_2y \rangle\right).
\end{align}
We will analyze each summand separately. Specifically, we will verify the following claim, from which Theorem~\ref{thm:main_gauss} follows immediately.

\begin{claim}\label{cl:main_claim}
The following estimates hold for any $\epsilon>0$ uniformly over all coordinates $r$: 
\begin{enumerate}[label=(\alph*)]
    \item $\frac{1}{nd^2}\|T_2y\|^2_2 = O_{d,\P}\left(d^{-2+\epsilon} \right) \cdot\round{\norm{f^*}_{L_2}^2+\sigma_\varepsilon^2}$\label{cl:main_claimT2}
    \item $\frac{1}{nd^2}\norm{T_1y }_2^2 = \E\brac{\round{\partial_r f^*_{\leq 2} (x)}^2}+ O_{d,\P}\round{d^{\delta - 1/2+\epsilon}+ d^{-\delta/2+\epsilon}}\cdot \norm{f^*}_{L_2}^2+ O_{d,\P}\round{d^{-\delta/2}} \cdot \norm{f^*_{>2}}_{L_{2+\eta}}^2\\
 +O_{d,\P}(d^{\delta-1/2+\epsilon} )\cdot \sigma_\varepsilon^2,$\label{cl:main_claimT1}
    \item\label{cl:main_claimT3} $\frac{1}{nd^2}\langle T_1y, T_2y \rangle= O_{d,\P}(d^{-1+\epsilon/2})\cdot  \norm{f^*}_{L_2}^2 +O_{d,\P}(d^{\delta/2-5/4 + \epsilon})\cdot \sigma_\varepsilon^2 +  O_{d,\P}\round{d^{-\delta/2}} \cdot \norm{f^*_{>2}}_{L_{2+\eta}}^2$. 
\end{enumerate}
\end{claim}
The rest of the section is devoted to proving Claim~\ref{cl:main_claim}. Note that Item~\ref{cl:main_claimT3} follows immediately from  Item~\ref{cl:main_claimT2} and \ref{cl:main_claimT1} and the Cauchy-Schwartz inequality. Throughout the proof, all asymptotic terms $o_{d,\P}$ and $O_{d,\P}$ that appear will always be uniform over the coordinates $r=1,\ldots, d$.

\subsubsection{Proof of Item~\ref{cl:main_claimT2} of Claim~\ref{cl:main_claim}}\label{subsec:t2}
Sub-multiplicativity of the operator norm directly implies
\begin{align*}
     \norm{T_2y}  &\leq (|\rho'|+\|\Delta\|_{\rm op})   \snorm{X_r}\norm{K_{\lambda}^{-1}y}_2.
\end{align*}
 Then the expression \eqref{eqn:K_decomp} directly implies 
\begin{align}\label{eq:k_lambda_inv}
    \snorm{K_\lambda\inv} \leq 1/(\rho+\lambda -\snorm{\Delta_1}) = O_{d,\P}(1).
\end{align}
where we use the fact that $D$ is PSD since $D$ is diagonally dominant. Therefore, taking into account Proposition~\ref{prop:bound_y}, we deduce

\begin{equation}\label{eqn:we_gonna_need}
\norm{K_\lambda\inv y}_2\leq \snorm{K_\lambda\inv}\|y\|_2\leq O_{d,\mathbb{P}}\left(\sqrt{n\log(n)\round{\norm{f^*}_{L_2}^2+\sigma_\varepsilon^2}}\right)
\end{equation}
Since the diagonal of $X_r$ has the same distribution as a standard Gaussian vector $v^{(r)}$ in $\R^n$, we have 
\begin{equation}\label{eqn:nnorm_bound_gauss}
\sup_{r}\snorm{X_r} = \sup_{r}\|v^{(r)}\|_{\infty}=\max_{i,r}|v^{(r)}_i|={O}_{d,\P}(\sqrt{\log(nd)}).
\end{equation}
Consequently, we conclude 
$\frac{1}{nd^2}\|T_{2}y\|_2^2 ={O}_{d,\P}\left(\frac{ \log^2(d)}{d^{2}}\cdot \round{\norm{f^*}_{L_2}^2+\sigma_\varepsilon^2}\right) = {O}_{d,\P}\left(d^{-2+\epsilon}\cdot \round{\norm{f^*}_{L_2}^2+\sigma_\varepsilon^2}\right)$ for any $\epsilon>0$
as desired. %

\subsubsection{Proof of Item~\ref{cl:main_claimT1} of Claim~\ref{cl:main_claim}}\label{subsec:t1}
We first simplify $ T_1 T_1\tran$. By Theorem~\ref{lemma:phi_id_2}, there exists a matrix $\widetilde\Delta$ satisfying
\begin{align*}
   \tfrac{1}{n}\Phi_{\leq 2}\tran \Phi_{\leq 2}  = I + \widetilde{\Delta}\qquad \textrm{and}\qquad \snorm{\widetilde{\Delta}} = O_{d,\P}(d^{-\delta/2+\epsilon}\log(d)).
\end{align*}
Using this approximation and defining $\beta=K_{\lambda}^{-1}y$, we therefore expand the square to obtain
\begin{align}
\frac{1}{nd^2}\|T_1y\|^2&=\frac{1}{nd^2}\beta^\top[X_r\Phi_{\leq 2}D' \Phi_{\leq 2}\tran\Phi_{\leq 2} D'\Phi_{\leq 2}\tran X_r]\beta\notag\\
&=\frac{1}{d^2}\beta^\top[X_r\Phi_{\leq 2}D' (I+\widetilde\Delta) D'\Phi_{\leq 2}\tran X_r]\beta\notag\\
&=\tfrac{1+O(\|\widetilde{\Delta}\|_{\rm op})}{d^2}\cdot \|\beta \tran X_r\Phi_{\leq 2} D' \|_2^2.\label{eqn:basic:needed}
\end{align}
Next, we argue that the contribution in this expression due to second-order terms $\Phi_{\S_2}$ is negligible. Namely, let us decompose $\Phi_{\leq 2}$ and $D'$ into blocks as follows: 
$$\Phi_{\leq 2}=[\Phi_{\leq 1}, \Phi_{\S_2}]\qquad \textrm{and}\qquad D'=\begin{bmatrix}
D'_{\leq 1} \\ D'_{\S_2}
\end{bmatrix}=\begin{bmatrix}
D'_{\leq 1,\leq 1} &  D'_{\leq 1, \S_2}\\
D'_{\S_2,\leq 1} & D'_{\S_2,\S_2}
\end{bmatrix}.$$
 Using this decomposition, we successively compute
\begin{align}
  \frac{1}{d^2} \|\beta\tran X_r \Phi_{\leq 2} D'\|^2_2 &= \frac{1}{d^2}\|\beta\tran X_r\round{\Phi_{\leq 1}{D}'_{\leq 1} + \Phi_{\S_2}{D}'_{\S_2}}\|^2_2\nonumber\\
     &= \norm{\frac{1}{d}\beta\tran X_r \Phi_{\leq 1}{D}'_{\leq 1}}_2^2 +\norm{\frac{1}{d}\beta\tran X_r\Phi_{\S_2}{D}'_{\S_2}}_2^2\nonumber\\
     &~~~+ \frac{2}{d^2}\langle\beta\tran X_r\Phi_{\leq 1}{D}'_{\leq 1}, \beta\tran X_r\Phi_{\S_2}{D}'_{\S_2}\rangle.\label{eqn:rhs_negl}
\end{align}

\noindent We will show that only the first term on the right side of \eqref{eqn:rhs_negl} does not vanish as $d$ tends to infinity. 
To this end, the following proposition shows that the second term on the right of \eqref{eqn:rhs_negl} is negligible. The proof appears in Appendix~\ref{proof:s2_s2}.

\begin{proposition}\label{prop:s2_s2}
    For any $\epsilon>0$, the estimate holds uniformly over all coordinates $r$:   
 $$\norm{\frac{1}{d}\beta\tran X_r\Phi_{\S_2}{D}'_{\S_2}}_2^2 = O_{d,\P}(d^{2\delta-1+\epsilon})\cdot \round{\norm{f^*}_{L_2}^2 + \sigma_\varepsilon^2}.$$
\end{proposition}

\noindent In order to show that the last term on the right side of \eqref{eqn:rhs_negl} is negligible, it will be useful to control the first term and then apply Cauchy-Schwarz. To this end, we decompose the first term as 
\begin{equation}\label{eqn:blup_needed}
\norm{\frac{1}{d}\beta\tran X_r \Phi_{\leq 1}{D}'_{\leq 1}}_2^2=\norm{\frac{1}{d}\beta\tran X_r \Phi_{\leq 1}{D}'_{\leq 1,\leq 1}}_2^2+\norm{\frac{1}{d}\beta\tran X_r \Phi_{\leq 1}{D}'_{\leq 1,\S_2}}_2^2.
\end{equation}
The following proposition shows that the second term on the right side of \eqref{eqn:blup_needed} is negligible. The proof appears in Appendix~\ref{proof:l1_to_s2} 

\begin{proposition}\label{prop:l1_to_s2}
For any $\epsilon>0$, the estimate holds uniformly over all coordinates $r$: 
    \begin{align}\label{eq:l1_to_s2}
        \norm{\frac{1}{d}\beta\tran X_r\Phi_{\leq 1}{D}'_{\leq 1,\S_2}}_2^2 = O_{d,\P}(d^{2\delta-1+\epsilon})\cdot \round{\norm{f^*}_{L_2}^2 + \sigma_\varepsilon^2}.
    \end{align}
\end{proposition}

\noindent Thus, it remains to analyze the first term on the right side of \eqref{eqn:blup_needed}. This is the main part of the entire argument. The proof of the proposition appears in Appendix~\ref{proof:<_l1_l1_l1}.

\begin{proposition}\label{prop:keymain}
For any $\epsilon>0$, the following estimate holds uniformly over all coordinates $r$:
\begin{align*}
\norm{\frac{1}{d}\beta\tran X_r \Phi_{\leq 1}{D}'_{\leq 1,\leq 1}}_2^2 &= \E\brac{\round{\partial_r f^*_{\leq 2} (x)}^2} +O_{d,\P}\round{d^{\delta - 1/2+\epsilon}+ d^{-\delta/2+\epsilon}}\cdot \norm{f^*}_{L_2}^2+ O_{d,\P}\round{d^{-\delta/2}} \cdot \norm{f^*_{>2}}_{L_{2+\eta}}^2\\
&~~~+O_{d,\P}(d^{-(\delta+1)/2+\epsilon} )\cdot \sigma_\varepsilon^2.
\end{align*}
\end{proposition}

In particular, Propositions~\ref{prop:s2_s2} and \ref{prop:keymain} together with the Cauchy Schwartz inequality imply that the last term in \eqref{eqn:rhs_negl} is bounded by $ O_{d,\P}(d^{\delta-1/2+\epsilon/2})\cdot \round{\norm{f^*}_{L_2} + \sigma_\varepsilon}\cdot \norm{\partial_r f_{\leq 2}^*}_{L_2}$ and therefore we arrive at the bound 
\begin{align*}
   \frac{1}{d^2} \|\beta\tran X_r \Phi_{\leq 2} D'\|^2_2&=\E\brac{\round{\partial_{r} f^*_{\leq 2} (x)}^2}+ O_{d,\P}\round{d^{\delta - 1/2+\epsilon}+ d^{-\delta/2+\epsilon}}\cdot \norm{f^*}_{L_2}^2+ O_{d,\P}\round{d^{-\delta/2}} \cdot \norm{f^*_{>2}}_{L_{2+\eta}}^2\\
   &~~~+O_{d,\P}(d^{\delta-1/2+\epsilon} )\cdot \sigma_\varepsilon^2,
\end{align*}
where we apply AM-GM inequality for the cross terms.

Lastly, combining this expression with \eqref{eqn:basic:needed}  completes the proof.

%% file: sections_appendix_a/hypercube.tex
\subsection{Proof of Theorem~\ref{thm:main_hypercube}}\label{proof:main_hypercube}
The proof follows along similar lines as the proof of Theorem~\ref{thm:main_gauss}. Therefore, we will omit some of the common steps in the argument when they are clear and focus mostly on the key differences.

By the assumption $f^*\in L_2(\tau)$ where $\tau$ is the uniform measure over the hypercude $\H^d$, we can expand $f^*$ into Fourier orthogonal basis:
\begin{equation}\label{eqn:fourier_expans_target}
    y_i = f^*(x^{(i)}) = \phi(x^{(i)})\tran b =\sum_{\lambda\in\mathbb{N}^d}b_{\lambda}\phi_{\lambda}(x^{(i)}),
\end{equation}
for some coefficients $b_{\lambda}$ satisfying $\|b\|_2<\infty$.

First, observe that equation \eqref{eqn:deriv_est} and Proposition~\ref{prop:bound_y} hold verbatim. 
Applying Lemma~\ref{lemma:kernel_to_mono} with the two kernels $K$ and $K'$ shows that for any $\epsilon>0$ there exist matrices $\Delta,\Delta_1$ satisfying $\snorm{\Delta},\snorm{\Delta_1} = O_{d,\P}(d^{(\delta-1)/2+\epsilon})$ and 
\begin{align}
    K &= \Phi_{\leq p}D\Phi_{\leq p}\tran + \rho I_n + \Delta_1\label{eqn:K_decomp_cube},\\
    K' &= \Phi_{\leq p}D'\Phi_{\leq p}\tran + \rho' I_n + \Delta,\label{eqn:K'_decomp_cube} 
\end{align}
where we define $\rho := g(1) - g_{p}(1)$ and $\rho':= g'(1) - g_{p+1}(1)$, and the diagonal matrices $D$ and $D'$ satisfy $\|D_{\S_k} -g^{(k)}(0)d^{-k}I_{|\S_k|}\|_{\rm op} = O_d(d^{-k-1})$ and $\|D'_{\S_k} -g^{(k+1)}(0) d^{-k}I_{|\S_k|}\|_{\rm op} = O_d(d^{-k-1})$ for $k=0,\ldots, p$.

With the expression \eqref{eqn:K'_decomp_cube} in place of $K'$, equation \eqref{eqn:deriv_est} becomes
\begin{align*}
    \frac{1}{n}\sum_{i=1}^n \round{\partial_{r} \hat{f} ( x^{(i)})}^2= \frac{1}{nd^2} \norm{\round{ \underbrace{ \Phi_{\leq p}D'\Phi_{\leq p}\tran X_r K_{\lambda}^{-1}}_{=:T_1} + \underbrace{(\rho'  I_n + \Delta)X_r K_{\lambda}^{-1}}_{=:T_2}}y}^2_2.
\end{align*}
Expanding the square, we conclude
\begin{align}\label{eq:main_cube_decomp}
    \frac{1}{n}\sum_{i=1}^n \round{\partial_{r} \hat{f}  ( x^{(i)})}^2 = \frac{1}{nd^2}\left(\|T_1y\|^2_2+\|T_2y\|^2_2+2\langle T_1y, T_2y \rangle\right).
\end{align}
We verify the following claim from which Theorem~\ref{thm:main_hypercube} follows immediately.

\begin{claim}\label{cl:main_claim_cube}
The following estimates hold for any $\epsilon>0$ uniformly over all coordinates $r$: 
\begin{enumerate}[label=(\alph*)]
    \item $\frac{1}{nd^2}\|T_2y\|^2_2 \leq O_{d,\P}(d^{-2+\epsilon})\cdot(\norm{f^*}_{L_2}^2+\sigma_\varepsilon^2)$\label{cl:main_claimT2_cube}
    \item 
    $\left|\frac{1}{nd^2}\norm{T_1y }_2^2 -\norm{\partial_{r} f^*_{\leq p} }_{L_2}^2 - \frac{d^{2\delta-2}\cdot \round{g^{(p+1)}(0)}^2}{(\rho+\lambda)^2}\cdot \norm{\partial_{r} f^*_{p+1} }_{L_2}^2\right|\\
    \leq
O_{d,\P}(1)\cdot\left( \frac{\norm{f^*}_{L_2}^2+\sigma_\varepsilon^2}{d}+ \frac{\norm{\partial_r f^*_{\leq p}}_{L_2}^2 + d^{2\delta -2}\norm{\partial_r f^*_{ p+1}}_{L_2}^2}{d^{\delta/2-\epsilon} + d^{(1-\delta)/2-\epsilon}} +\frac{\norm{f^*}_{L_2}\norm{\partial_{r} f^*_{\leq p} }_{L_2}}{d^{1/2}}+ \frac{(\norm{f^*}_{L_2}+\sigma_\varepsilon)  \norm{\partial_{r} f^*_{p+1} }_{L_2}}{d^{3/2-\delta}}\right). 
$\label{cl:main_claimT1_cube}

\item\label{cl:main_claimT3_cube} $\frac{1}{nd^2}|\langle T_1y, T_2y \rangle|\\
\leq O_{d,\P}(1)\cdot\left( \frac{\norm{f^*}_{L_2}^2+\sigma_\varepsilon^2}{d^{3/2-\epsilon}} + \frac{\norm{f^*}_{L_2}\cdot \norm{\partial_{r} f^*_{\leq p} }_{L_2}}{d^{1-\epsilon}} + \frac{(\norm{f^*}_{L_2}+\sigma_\varepsilon)\cdot\norm{\partial_{r} f^*_{p+1} }_{L_2}}{d^{2-\delta-\epsilon}}\right)$.
\end{enumerate}
\end{claim}
\noindent Plugging Claim~\ref{cl:main_claim_cube} (a,b,c) into Eq.~(\ref{eq:main_cube_decomp}) yields:
\begin{align*}
    &~~~\abs{\E_n (\partial_r \hat f)^{2} -\E{\round{\partial_{r} f^*_{\leq p} }^2} - \frac{d^{2\delta-2}\cdot \round{g^{(p+1)}(0)}^2}{(\rho+\lambda)^2} \E(\partial_r f^*_{p+1})^2}\\
    &\leq O_{d,\P}(1)\cdot\left( \frac{\norm{f^*}_{L_2}^2+\sigma_\varepsilon^2}{d}+ \frac{\norm{\partial_r f^*_{\leq p}}_{L_2}^2 + d^{2\delta -2}\norm{\partial_r f^*_{ p+1}}_{L_2}^2}{d^{\delta/2-\epsilon} + d^{(1-\delta)/2-\epsilon}} +\frac{\norm{f^*}_{L_2}\norm{\partial_{r} f^*_{\leq p} }_{L_2}}{d^{1/2}}+ \frac{(\norm{f^*}_{L_2}  +\sigma_\varepsilon)\cdot \norm{\partial_{r} f^*_{p+1} }_{L_2}}{d^{3/2-\delta}}\right).
\end{align*}

The rest of the section is devoted to proving Claim~\ref{cl:main_claim_cube}.  Note that Item~\ref{cl:main_claimT3_cube} follows immediately from  Items~\ref{cl:main_claimT2_cube} and \ref{cl:main_claimT1_cube} and the Cauchy-Schwartz and Young's inequality. Throughout the proof, all asymptotic terms $o_{d,\P}$ and $O_{d,\P}$ that appear will always be uniform over the coordinates $r=1,\ldots, d$.

\subsubsection{Proof of Item~\ref{cl:main_claimT2_cube} of Claim~\ref{cl:main_claim_cube}}\label{subsec:t2_cube}
The proof is identical to the proof of Item~\ref{cl:main_claimT2} of  Claim~\ref{cl:main_claim} with \eqref{eqn:nnorm_bound_gauss} replaced by the trivial bound $\sup_r \|X_r\|_{\rm op}\leq 1$.
We refer the readers to Section~\ref{subsec:t2}.%

\subsubsection{Proof of Item~\ref{cl:main_claimT1_cube} of Claim~\ref{cl:main_claim_cube}}\label{subsec:t1_cube}

Similar to Eq.~(\ref{eqn:basic:needed}), we can write
\begin{align}
\frac{1}{nd^2}\|T_1y\|^2&=\frac{1}{nd^2}\beta^\top[X_r\Phi_{\leq p}D' \Phi_{\leq p}\tran\Phi_{\leq p} D'\Phi_{\leq p}\tran X_r]\beta\notag\\
&=\frac{1}{d^2}\beta^\top[X_r\Phi_{\leq p}D' (I+\widetilde\Delta) D'\Phi_{\leq p}\tran X_r]\beta\notag\\
&=\tfrac{1+O(\|\widetilde{\Delta}\|_{\rm op})}{d^2}\cdot \|\beta \tran X_r\Phi_{\leq p} D' \|_2^2,\label{eqn:basic:needed_cube}
\end{align}
where $\snorm{\widetilde{\Delta}} = O_{d,\P}(d^{-\delta/2+\epsilon})$. 

We next decompose $\norm{\frac{1}{d}\beta \tran X_r\Phi_{\leq p} D' }_2^2$. First notice that the following holds for any set $S\subseteq[d]$:
\begin{align*}
    x_r x^{S} = \begin{cases}
        x^{S\cup \{r\}},~~\textrm{if}~~r\notin S;\\
         x^{S\slash r},~~\textrm{otherwise}.\\
    \end{cases} 
\end{align*}
Define now the two sets:
\begin{align}\label{eq:s_j_0_1}
    \S_j^0=\{S\in \S_j: r\notin S\}\qquad \textrm{and}\qquad \S_j^1=\{S\in \S_j: r\in S\}.
\end{align}
It follows that every column of $X_r\Phi_{\leq p}$ has the form $\phi_{S}$ for some $\displaystyle S\in \left(\displaystyle\bigsqcup_{j=0}^{p-1} S_j^0\right)\cup \left(\displaystyle\bigsqcup_{j=1}^{p+1} S_j^{1}\right)$.
Consequently, there exist diagonal matrices $R_{j}^0$ and $R_{j}^1$ such that the following holds:
\begin{equation}\label{eqn:decompo_summm}
  \norm{\beta \tran X_r\Phi_{\leq p} \round{\frac{1}{d}D'} }_2^2 =   \sum_{j=0}^{p-1}\|\beta \tran \Phi_{\S_j^0} R_{j}^0\|_2^2   + \sum_{j=1}^{p+1}\|\beta \tran \Phi_{\S_j^1} R_{j}^1 \|_2^2.
\end{equation}
\begin{claim}
    The following estimate holds:
    \begin{align}
        \snorm{R_{j}^0 - g^{(j+2)}(0) d^{-j-2}\cdot I} &= O_d(d^{-j-3}),~~\textrm{for}~~j = 0,\ldots,p-1, \label{eq:def_r_j_0}\\
        \snorm{R_{j}^1 - g^{(j)}(0) d^{-j}\cdot I} &= O_d(d^{-j-1}),~~\textrm{for}~~j = 1,\ldots,p+1.\label{eq:def_r_j_1}
    \end{align}
\end{claim}
\begin{proof}
     Recall that $D'$ is a diagonal matrix with entry $D'_{S,S} =  g^{(|S|+1)}(0)d^{-|S|} + O_d(d^{-|S|-1})$. It is straightforward to see that $R_j^0$ is a diagonal matrix with entry $g^{(j+2)}(0) d^{-j-2} + O_d(d^{-j-3})$, which establishes \eqref{eq:def_r_j_0}. Similar reasoning yeilds Eq.~\eqref{eq:def_r_j_1}.
\end{proof}

We now group all the $p$-th order terms in \eqref{eqn:decompo_summm}. Namely, define the block-diagonal matrix $R$ with blocks $R_{\S_j^1,\S_j^1} = R_j^1$ for $j\leq p$, $R_{\S_j^0,\S_j^0} = R_j^0$ for $j\leq p-1$, and $R_{\S_p^{0},\S_p^{0}}=0$. Then equation \eqref{eqn:decompo_summm} becomes 
\begin{align}
    \norm{\beta \tran X_r\Phi_{\leq p} \round{\frac{1}{d}D'} }_2^2 = \norm{\beta\tran \Phi_{\leq p}R}_2^2 +  \norm{\beta\tran \Phi_{\S_{p+1}^1}R_{{p+1}}^1}_2^2\label{eqn:rhs_negl_cube}.
\end{align}
 We now estimate the first term on the right side of \eqref{eqn:rhs_negl_cube}.
The proof of the proposition appears in Appendix~\ref{proof:<_l1_l1_l1_cube}.

\begin{proposition}\label{prop:<_l1_l1_l1_cube}
For any $\epsilon>0$, the estimate holds uniformly over all coordinates $r$:
    \begin{align}\label{eq:<_l1_l1_l1_cube}
&~~~\abs{\norm{\beta\tran \Phi_{\leq p}R }_2^2 -\norm{\partial_{r} f^*_{\leq p} }_{L_2}^2  }\notag\\
&=  O_{d,\P}\round{d\inv}\cdot \round{\norm{f^*}_{L_2}^2+\sigma_\varepsilon^2} + O_{d,\P}(d^{-\delta/2+\epsilon} + d^{(\delta-1)/2+\epsilon})\cdot \norm{\partial_r f^*_{\leq p}}_{L_2}^2\notag \\
&~~~+O_{d,\P}(d\half)\cdot \norm{\partial_r f^*_{\leq p}}_{L_2} \norm{f^*}_{L_2} .
    \end{align}
\end{proposition}

\noindent We next analyze the second term on the right side of \eqref{eqn:rhs_negl_cube}.
We defer the proof to Appendix~\ref{proof:s2_s2_cube}.

\begin{proposition}\label{prop:s2_s2_cube}
For any $\epsilon>0$, the estimate holds uniformly over all coordinates $r$:  
\begin{align}\label{eq:s2_s2_cube}
    &~~~\abs{ \norm{\beta\tran \Phi_{\S_{p+1}^1}R_{{p+1}}^1}_2^2-\frac{d^{2\delta-2} \cdot  \round{g^{(p+1)}(0)}^2}{(\rho+\lambda)^{2}}\norm{\partial_{r} f^*_{p+1} }_{L_2}^2}\notag\\
    &=      O_{d,\P}(d^{-1})\cdot  (\norm{f^*}_{L_2}^2+\sigma_\varepsilon^2)
       +O_{d,\P}(d^{
     \frac{3}{2}\delta-2+\epsilon} + d^{\frac{5}{2}(\delta-1)
     +\epsilon})\cdot \norm{\partial_{r} f^*_{p+1}}_{L_2}^2\notag\\
     &~~~+O_{d,\P}(d^{\delta - 3/2 })\cdot \norm{\partial_{r} f^*_{p+1} }_{L_2} \cdot  (\norm{f^*}_{L_2}+\sigma_\varepsilon).
\end{align}
\end{proposition}
\noindent Propositions~\ref{prop:<_l1_l1_l1_cube}  and \ref{prop:s2_s2_cube} together give the bound
\begin{align}
  &~~~\abs{  \norm{\frac{1}{d}\beta\tran X_r \Phi_{\leq p}{D'}}_2^2-\norm{\partial_{r} f^*_{\leq p} }_{L_2}^2 -\frac{d^{2\delta-2} \cdot  \round{g^{(p+1)}(0)}^2}{(\rho+\lambda)^{2}}\norm{\partial_{r} f^*_{p+1}}_{L_2}^2  }\notag\\
  &=
O_{d,\P}(d^{-1}) (\norm{f^*}_{L_2}^2+\sigma_\varepsilon^2)
+  
O_{d,\P}(d^{-\delta/2+\epsilon} + d^{(\delta-1)/2+\epsilon})\cdot \round{\norm{\partial_r f^*_{\leq p}}_{L_2}^2+ d^{2\delta -2}\norm{\partial_r f^*_{ p+1}}_{L_2}^2}\notag \\
&~~~+O_{d,\P}(d\half)\cdot \norm{\partial_r f^*_{\leq p}}_{L_2} \norm{f^*}_{L_2} 
+O_{d,\P}(d^{\delta-3/2})\cdot \norm{\partial_{r} f^*_{p+1} }_{L_2} \cdot (\norm{f^*}_{L_2}+\sigma_\varepsilon).
\end{align}
Combining this expression with \eqref{eqn:basic:needed_cube} completes the proof.

%% file: sections_appendix_a/main_one_step_learning.tex
\subsection{One-step learning progress and proof of Theorem~\ref{cor:multi_step}}\label{sec:one_step_learning}
We state the paper's main result, which characterizes each step's learning progress in the implementation of  IRKM (Algorithm~\ref{alg:rfm}). This result generalizes Theorem~\ref{thm:main_hypercube}, which can be regarded as the first step of IRKM where $w= \mathbf{1}_d$. The proof of Theorem~\ref{cor:multi_step} follows directly from the result of this theorem.

We begin with some notation. Given a nonnegative vector $w \in \R^d_{+}$, we define the parameterized kernel $K_w$ by the expression
\begin{align}\label{eq:inner_product_kernel_weighted}
    K_w(x,y) = g\left(\frac{\langle \sqrt{w} \odot x, \sqrt{w} \odot y\rangle}{d}\right).
\end{align}
The KRR prediction function induced by $K_w$ then is given by
\begin{align}\label{eq:krr_m}
    \hat{f}_w (z) = K_w(z,X)\beta\qquad \textrm{where}\qquad \beta:=\round{K_w(X,X) + \lambda I_n}^{-1} y.
\end{align}
We now define the active coordinates for a weight vector $w$ as follows.

\begin{definition}[Active indices]\label{def:act_idx_new} 
{\rm
    Consider a weight vector $w\in \R^d_{+}$ with $\norm{w}_1 = d$ and fix two vectors $\vs,\bm{\gamma} \in \R_{\geq 0}^N$.
    A partition $\pi_N = \{\T_1,\ldots,\T_N\}$ of $[d]$ is called {\em $(\vs,\bm{\gamma})$-active} for $w$ if equalities $|\T_k| = \Theta_d(d^{s_k})$ hold and $$w_r=\Theta_{d}(d^{\gamma_k})\quad \forall r\in \T_k.$$
}
\end{definition}
\noindent We define the function $\pi_N: [d] \rightarrow N$ that  maps each element to its corresponding partition index
\begin{align}
\pi_N(r) = \iota \iff r \in \mathcal{T}_\iota.
\end{align}
Given a tuple $\mT = (T_1, T_2,\ldots,T_N)$ where $T_k\in 2^{\T_k}$ for $k\in[N]$, we define its corresponding Fourier feature
\begin{align}
    \phi_{\mT} := \phi_{\cup_{k=1}^N T_k}= \prod_{k=1}^N \underbrace{\prod_{j\in T_k} x_j}_{:=\phi_{T_k}}.
\end{align}
Then we proceed to define the effective degree of $\phi_\mT$:
\begin{definition}[Effective degree]\label{def:eff_degree}
Let $\pi_N = \{\T_1,\ldots,\T_N\}$ be $(\vs,\vgamma)$-active for vector $w\in \R^d_+$. For any tuple $\mT = (T_1,T_2,\ldots,T_N)$ where $T_k\in 2^{\T_k}$ for $k\in [N]$, the effective degree of its corresponding Fourier feature $\phi_{\mT}$ with respect to $w$ is
\begin{align}
    p_{\rm eff}(\mT) :=\sum_{k=1}^N (1-\gamma_k)\cdot |T_k|.
\end{align}

\end{definition}

Now we are ready to present the theorem. We defer the proof to Appendix~\ref{proof:learning_hypercube_gradient} and \ref{proof:learning_hypercube_gradient_rkhs}  for the empirical weights estimator \eqref{eq:gradient_error_bound} and DN-estimator \eqref{eq:rkhs_gradient_error_bound}, respectively. The proof for the generalization bound \eqref{eq:generalization_error_bound} appears in Appendix~\ref{proof:learning_hypercube_generalization}.

\begin{theorem}[One-step learning progress]\label{thm:learning_hypercube}
Suppose that Assumptions~\ref{assump:g_0} and \ref{assump:finite_basis_hypercube} hold.  Consider the regime $n = d^{p+\delta}$ where $\delta \in(0,1)$ is a constant. Consider a weight vector $w\in \R^d_+$ satisfying $\|w\|_1=d$ and let  $\pi_N = \{\T_1,\ldots,\T_N\}$ be  $(\vs,\vgamma)$-active  for $w$. Suppose moreover for all $k\in[N]$, equality $s_k + \gamma_k= 1$  holds if and only if $s_k=1$.  Define the family of sets 
\begin{align}\label{eq:a_leq_p}
   \aslp&:=\biground{\mT = (T_1,T_2,\ldots,T_N)\;\middle|\; T_k\in 2^{\T_k}~~{\rm and}~p_{\rm eff}(\mT) < p+\delta},\\
   \aep&:=\biground{\mT = (T_1,T_2,\ldots,T_N)\;\middle|\; T_k\in 2^{\T_k}~~{\rm and}~p_{\rm eff}(\mT) = p+\delta},\\
   \A_{p_{\delta}}^+(r)&:=\biground{\mT = (T_1,T_2,\ldots,T_N)\;\middle|\; T_k\in 2^{\T_k}~~{\rm and}~p_{\rm eff}(\mT) \in (p+\delta,p+\delta+1-\gamma_{\pi_N(r)}]},
\end{align}
for $r\in[d]$, and the truncation of the target function 
\begin{align*}
    f^*_{\mT} := b_{\mT} \phi_{\mT},~~~f^*_{\A}:= \sum_{\mT\in \A} f^*_{\mT},~~{\rm for}~~\A \in\{\aslp,\aep, \A_{p_{\delta}}^+(r)\}.
\end{align*}
Given $\mT$, define a function
\begin{align}\label{eq:p_eff_dist}
    \kappa(\mT) :=p_{\rm eff}(\mT) - (p+\delta).
\end{align}
Define now the effective sampling power $\emp = \min(\ell, \floor{\frac{p+\delta}{1-\max_{k\in[N]}\gamma_k}})$, and suppose that condition \eqref{eqn:nondegegn} holds with $q=\emp$. Then the parameterized KRR predictor $\hat f_{\lm}$ admits the generalization bound
\begin{align}\label{eq:generalization_error_bound}
\norm{\hat{f}_{\lm} - f^*_{{\A_{< p}}} - \Theta_d(1)\cdot f^*_{{\A_{= p}}}}_{L_2}^2 = o_{d,\P}(1)\cdot (\norm{f^*}_{L_2}^2+\sigma_\varepsilon^2),
\end{align}
and  for any $\epsilon>0$, the bound on the empirical weight estimator error
\begin{align}\label{eq:gradient_error_bound}
   &~~~\abs{\E_n (\partial_r \hat f_\lm)^{2} - (1+o_{d,\P}(1))\round{\E{\round{\partial_{r} f^*_{\A_{< p}} }^2} + \Theta_d(1) \cdot \E{\round{\partial_{r} f^*_{\A_{= p}} }^2} + \sum_{\mT \in \A_{p_{\delta}}^+(r)}\Theta_d(1)\cdot d^{-2\kappa(\mT)} \E{\round{\partial_{r} f^*_{\mT}}^2}}}\notag\\
   &\leq O_{d,\P}\round{ d^{\gamma_{\pi_N(r)}-1 }}\cdot (\norm{f^*}_{L_2}^2+\sigma_\varepsilon^2),
\end{align}
and the derivative norm estimator error:
\begin{align}\label{eq:rkhs_gradient_error_bound}
   &~~~\abs{~w_r\cdot\D_r(w) - (1+o_{d,\P}(1))\round{\sum_{\mT \in \{ \alp, \A_{p_{\delta}}^+(r)\}}\Theta_d(1)\cdot d^{-\kappa(\mT)} \E{\round{\partial_{r} f^*_{\mT}}^2}}}\notag\\
   &\leq O_{d,\P}\round{ d^{\gamma_{\pi_N(r)}-1 +\epsilon}}\cdot (\norm{f^*}_{L_2}^2+\sigma_\varepsilon^2),
\end{align}
simultaneously over all $r=1,\ldots,d$.

\end{theorem}

\subsection{Proof of Theorem~\ref{cor:multi_step}}\label{proof:multi_step}
We will prove the results by iteratively applying Theorem~\ref{thm:learning_hypercube}. Throughout the proof, we will use the following notation for the estimated coordinate weights at iteration $t$:
\begin{align*}
    w^{[1]}(t) &= \underbrace{\E_n (\partial_w \hat f_\lm(t))^{2}}_{:=\hat{w}} +   {\bf 1}_d\epsilonS,\\
    w^{[2]}(t) &=\D(w(t))\odot w(t) +  {\bf 1}_d\epsilonS,
\end{align*}
and
\begin{align}\label{eq:normalized_w}
    \bar{w}^{[1]}(t) = d\cdot \frac{w^{[1]}(t)}{\norm{w^{[1]}(t)}_1},~~~\bar{w}^{[2]}(t) = d\cdot \frac{w^{[1]}(t)}{\norm{w^{[2]}(t)}_1}.
\end{align}
Additionally, we denote 
\begin{align*}
    \bar{w}(t) = (1-\alpha)\cdot \bar{w}^{[1]}(t)  + \alpha \cdot \bar{w}^{[2]}(t).
\end{align*}
For convenience, we denote by $w(t)$ the coordinate weights before the update at step $t$.  
Accordingly, after applying the update rule, we have the equality $$w(t+1) = \bar{w}(t).$$
Since the target function $f^*$ is assumed to be of degree at most $\ell$ (Assumption~\ref{assump:finite_basis_hypercube}), we have $\norm{f^*}_{L_2}^2 = O_d(1)$ therefore we omit the norm terms in the bound for simplicity of the notation.

In the proof, we will show that within $O_d(1)$ many iterations of the algorithm, the set $\aslp$ will contain all $\mT$ where features $\phi_{\mT}$ have non-zero Fourier coefficients, i.e., $b_\mT \neq 0$ for $\mathsf{L}_{p+1} f^*$. 
The desired generalization bound can be immediately obtained following Eq.~\eqref{eq:generalization_error_bound} in Theorem~\ref{thm:learning_hypercube}.

\subsubsection{Supporting Claims and Propositions}
We start with a useful claim that shows that the empirical coordinate weights for non-influential coordinates remain of the order $\Theta_{d,\P}(1)$ after normalization during the whole iterations. 
\begin{claim}\label{claim:unlearned}
For any coordinate $r$ that is not influential for $f^*$, the following holds for all $t\geq 0$
\begin{align}
    \bar{w}_r^{[1]}(t), ~~~\bar{w}_r^{[2]}(t)= \Theta_{d,\P}(1).
\end{align}
\end{claim}
\begin{proof}
We prove the result by induction. Observe that for any $t$, we have
\begin{align*}
    \norm{\E_n (\partial_w \hat f_\lm(t))^{2}}_1,~~\norm{\D(w(t))\odot w(t)}_1\lesssim d^s,
\end{align*}
where $s$ is the sparsity.

With the assumption $s\leq 1-\eta$, we immediately have 
\begin{align*}
    \norm{ w^{[1]}(t)}_1,~~\norm{ w^{[2]}(t)}_1 = \Theta_{d,\P}(d^{1-\eta}).
\end{align*}
At step $t= 0$, since $\gamma_{\pi_N(r)} = 0$, we have 
\begin{align*}
     w^{[1]}(t) &= \Theta_{d,\P}(d^{-1}) + d^{-\eta} = \Theta_{d,\P}(d^{-\eta}),\\
      w^{[2]}(t) &= \Theta_{d,\P}(d^{-1+\epsilon}) + d^{-\eta} = \Theta_{d,\P}(d^{-\eta}).
\end{align*}
Then applying the normalization \eqref{eq:normalized_w} gives the claimed bound.

Suppose that at step $t = \tau$ the claimed bound holds. At step \(t = \tau+1\), the inductive hypothesis guarantees that \(\gamma_{\pi_N(r)} = 0\). Therefore, by applying the same analysis as in the previous step, we obtain the claimed bound at \(t = \tau+1\). This completes the inductive step and thus the proof.
\end{proof}

We define an auxiliary estimate 
\begin{align}
    w^\aux_r:= (1-\alpha)\cdot \hat{w}_r+ \alpha\cdot \D_r(w)\odot w.
\end{align}
Combining the bounds \eqref{eq:gradient_error_bound} and \eqref{eq:rkhs_gradient_error_bound}, the estimate $ w^\aux$ satisfies:
\begin{align}\label{eq:combined_gradient_error_bound}
    &~~~\abs{w^\aux_r - (1+o_{d,\P}(1))\round{\E{\round{\partial_{r} f^*_{\A_{< p}} }^2} + \Theta_d(1) \cdot \E{\round{\partial_{r} f^*_{\A_{= p}} }^2} + \sum_{\mT \in \A_{p_{\delta}}^+(r)}\Theta_d(1)\cdot d^{-\kappa(\mT)} \E{\round{\partial_{r} f^*_{\mT}}^2}}}\notag\\
   &\leq O_{d,\P}\round{ d^{\gamma_{\pi_N(r)}-1 + \epsilon }}\cdot (\norm{f^*}_{L_2}^2+\sigma_\varepsilon^2).
\end{align}
Since the non-influential coordinates dominate the \(\ell_{1}\)-norm,  
normalizing the weights jointly, i.e., normalizing \(w^{\mathrm{aux}}\),  
induces the same scaling across all coordinates as normalizing the weights 
separately, as performed in Algorithm~\ref{alg:rfm}.  We formalize the result below and omit the proof since it follows directly from the arguments in the proof of Claim~\ref{claim:unlearned}.
\begin{claim}
For any $t\geq 0$ and $r\in[d]$, the following equality holds:
\begin{align}
    \bar{w}^\aux_r(t) \asymp \bar{w}_r(t).
\end{align}
\end{claim}
Therefore, for simplicity of presentation, we establish the learning result by 
analyzing the following update rule, which replaces 
lines~\ref{line:safeintro_1}--\ref{line:mixing} in Algorithm~\ref{alg:rfm}:
\begin{align}\label{eq:new_update_rule}
    w = \frac{ w^\aux + \epsilonS}{\norm{ w^\aux +  {\bf 1}_d\epsilonS}_1}.
\end{align}
We similarly denote the normalized $w^\aux_r$ for each $r\in[d]$ by $ \bar{w}^\aux_r$.

Now we define fully learned, weakly learned and unlearned coordinates based on the magnitude of ${w}^\aux_r$ in Table~\ref{tab:w_order}:
\begin{table}[H]
\centering
\begin{tabular}{c|c|c|c}
\toprule
\makecell{\textbf{Learning status of } \\ \textbf{ coordinate $r$}} & 
\makecell{\textbf{Magnitude of} $w^\aux$} & 
\makecell{\textbf{Magnitude of $w^\aux+\epsilonS$} \\ ($\epsilonS = d^{-\eta}$)} & 
\makecell{\textbf{Magnitude of $\bar{w}^\aux$}} \\ 
\midrule
Fully learned & $\Theta_d(1)$ & $\Theta_d(1)$ & $\Theta_d(d^\eta)$\\
\midrule
Weakly learned & $\Theta_d(d^{-\eta'})$, $\eta'\in(0,\eta)$ & $\Theta_d(d^{-\eta'} )$ & $\Theta_d(d^{\eta-\eta'})$\\
\midrule
Unlearned & $O_d(d^{-\eta})$ & $\Theta_d(d^{-\eta})$ & $\Theta_d(d^{0})$ \\
\bottomrule
\end{tabular}
\caption{Magnitude of coordinate weights during a single iteration.  }
\label{tab:w_order}
\end{table}
Next we define a feature being learned by $\hat{f}$:
\begin{definition}[Learned feature]
  A feature $\phi_{\mT}$ is said to be learned by $\hat{f}$, if the following holds:
  \begin{align*}
      \abs{\inner{\phi_{\mT},\hat{f}}_{L_2} -b_\mT} = o_{d,\P}(1).
  \end{align*}
\end{definition}

The following proposition shows that the order of $  \bar{w}^\aux_r$ is non-decreasing, which implies that the learned features remain learned and new features will be learned as the interaction step $t$ increases.
\begin{claim}
For each empirical coordinate weight $  \bar{w}_r$, its order of magnitude with respect to $d$ is non-decreasing as a function of the iteration step $t$  throughout the first $O_d(1)$ iterations of Algorithm~\ref{alg:rfm}.
\end{claim}
\begin{proof}
Based on the update rule~\eqref{eq:gradient_error_bound} and \eqref{eq:rkhs_gradient_error_bound}, we can deduce that for any unlearned coordinate $r$, the empirical weight $  \bar{w}_r$ can increase the order only when they become fully or weakly learned, otherwise $\bar{w}_r(t)$ remains at the order $\Theta_d(1)$ with $\gamma_{\pi_N(r)} = 1$. For fully learned and weakly learned coordinates, the results are proved in Claim~\ref{claim:fully_learned}
 and \ref{claim:partial_to_partial} respectively.
 \end{proof}

\noindent Next we give a claim for the fully learned coordinates:
\begin{claim}\label{claim:fully_learned}
For any fully learned coordinate $r$, it will remain fully learned until the algorithm terminates.
\end{claim}
\begin{proof}
Note that when coordinate $r$ is fully learned, there must exist a feature $\phi_{\mT}$ that depends on $x_r$ such that $p_{\rm eff}(\mT)\leq p+\delta$ and $\norm{\partial_r f^*_{\mT}}_{L_2}>0$. Furthermore,  any coordinate upon which $\phi_{\mT}$ depends is also fully learned. For this specific feature, a recursive argument shows that the corresponding weight $\bar{w}$ of all coordinates on which $\phi_{\mT}$ depends remains the order unchanged, specifically $\Theta_{d,\P}(d^\eta)$, and the condition $p_{\rm eff}(\mT) \leq p+\delta$ remains satisfied until algorithm termination. This establishes that coordinate $r$ remains fully learned throughout the process.
\end{proof}

To prove the desired result, i.e.,  $\norm{\hat{f} - \mathsf{L}_{p+1} f^*}_{L_2} = o_{d,\P}(1)$, we will show for any feature $\phi_{\mT}$ that is learned by $\hat{f}$, feature $\phi_{\mT'}$ with $|\mT' \setminus \mT| \leq p+1$ is also learned when the algorithm terminates.  Here we abuse the notation by treating the tuple $\mT=(T_1,\ldots,T_N)$ as a set $\cup_{j=1}^N T_k$. To prove the result, it suffices to show that $\phi_{\mT'} := \phi_{\mT} \phi_{S}$ for any $S\in 2^{[d]}$ with $|S|\leq p+1$ is learned by $\hat{f}$.

\subsubsection{Main proof}

The following claim shows that the order of the weight of weakly learned coordinates will monotonically increase until it reaches the maximum order $\eta$:

\begin{claim}\label{claim:partial_to_partial}
For any coordinate $r$ that is weakly learned, we have the estimate
\begin{align}\label{eq:partial_to_partial_aux}
    \bar{w}^\aux_r  = \Theta_{d,\P}( d^{\eta-{\kappa}(\mT_r')})\cdot \E{\round{\partial_{r} f^*_{\mT_r'}}^2},
\end{align}
where
\begin{align}
    \mT_r' = \argmax_{\mT \in \A_{p_{\delta}}^+(r)~:~\E{\round{\partial_{r} f^*_{\mT}}^2}>0} \kappa(\mT).
\end{align}
Furthermore, there exists a constant $\sigma>0$ independent of the step $t$ such that the following holds:
\begin{align}\label{eq:w_evolve}
       \lambda_{\pi_N(r)}(t+1) \leq   \max(1-\eta, \lambda_{\pi_N(r)}(t) -  \sigma).
\end{align}
\end{claim}
\begin{proof}

We prove the result by induction. 
Note that there exists a time step where the coordinate $r$ changes from unlearned to weakly learned after this step. 
Without loss of generality, we assume this step is $t=1$. 
Correspondingly, we have $\lambda_{\pi_N(r)}(1) =1$ and $\bar{w}^\aux_r(1) = \Theta_d(d^{\sigma})$ with some constant $\sigma \in(0,\eta)$, where 
\begin{align*}
    \sigma := \eta -\kappa(\mT_r'(1)).
\end{align*}
Since $ \kappa(\mT_r'(1)) <\lambda_{\pi_N(r)}(1)  =1$, the error term in Eq.~\eqref{eq:combined_gradient_error_bound} is strictly smaller than $w^\aux_r$ hence we have the estimate  \eqref{eq:partial_to_partial_aux}. 
Consequently, we have $w_r(2) = \bar{w}_r(1) \gtrsim \Theta_d(d^{\sigma})$ which implies
$$\lambda_{\pi_N(r)}(2) \leq 1 - \sigma.$$ Then we complete the bound \eqref{eq:w_evolve} at $t=1$.

Now we suppose that at step $t=\tau$ the estimate \eqref{eq:partial_to_partial_aux}  and the inequality \eqref{eq:w_evolve} hold and we consider $t = \tau+1$. 
A simple calculation shows
\begin{align*}
   \lambda_{\pi_N(r)}(\tau+1) =  1- \eta + {\kappa}(\mT'_r(\tau)). 
\end{align*}
 We assume  ${\kappa}(\mT'_r(\tau))>0$ otherwise coordinate $r$ has been fully learned at step $\tau+1$ which finishes the proof. Notice that the inductive hypothesis indicates that  $\lambda_{\pi_N(r)}(\tau)$ will decrease by at least $ \sigma$:
 \begin{align*}
     \lambda_{\pi_N(r)}(\tau+1) \leq \lambda_{\pi_N(r)}(\tau) -   \sigma.
 \end{align*}
Therefore, considering unlearned and fully learned coordinates, we can deduce that until step $\tau+1$, the empirical coordinates weight $w$ is non-decreasing in terms of the order of the magnitude.  Consequently, the effective dimension will also decrease by at least $\sigma$:
\begin{align*}
  p_{\rm eff}(\mT_r'(\tau+1))\leq  p_{\rm eff}(\mT_r'(\tau)) -  \sigma
\end{align*}
which leads to the decrease of $\kappa(\mT_r'(\tau+1)):$
\begin{align}\label{eq:kappa_tau_plus_1}
  \kappa(\mT_r'(\tau+1)) \leq {\kappa}(\mT_r'(\tau)) -  \sigma.
\end{align}
If $\kappa(\mT_r'(\tau+1))\leq 0$, then coordinate $r$ is fully learned at step $\tau+2$, therefore we only need to consider the case when $\kappa(\mT_r'(\tau+1))>0$.

From Eq.~\eqref{eq:kappa_tau_plus_1} and the condition $\sigma\in(0,\eta)$, we immediately have
\begin{align*}
  \lambda_{\pi_N(r)}(\tau+1) =1 - \eta +  {\kappa}(\mT_r'(\tau)) \geq {1-\eta+ {\kappa}(\mT_r'(\tau+1))}+  \sigma > \kappa(\mT_r'(\tau+1)),
\end{align*}
which completes the inductive step for Eq.~\eqref{eq:partial_to_partial_aux} at step $\tau+1$.
Additionally, with $ \lambda_{\pi_N(r)}(\tau+2)= 1-\eta+ {\kappa}(\mT_r'(\tau+1))$, we can deduce
\begin{align*}
  \lambda_k(\tau+2) \leq  \max(1-\eta, \lambda_k(\tau+1)  - \sigma )
\end{align*}
which completes the inductive step for Eq.~\eqref{eq:w_evolve} thus completing the proof.
\end{proof}

With the results of Claim~\ref{claim:partial_to_partial}, we can deduce that for any coordinate that is weakly learned,  it will be fully learned in $O_d(1)$ many steps. Next, we show that at most $p+1$ unlearned coordinates can be learned (fully or weakly) in a single step, which then indicates that the algorithm can learn functions with leap complexity $p+1$. To that end, we consider the possible tuple $\mT$ that satisfies:
\begin{align*}
    \argmax_{\mT = (T_1,\ldots,T_N),~\lambda_1 = 1} |T_1|,~~~s.t.~~~p_{\rm eff} (\mT) < p+\delta + \eta.
\end{align*}
Note that it is sufficient to consider $\mT$ with only two entries, i.e., $\mT = (T_1,T_2)$ with $\lambda_1 = 1$ and $\lambda_2 = 1-\eta$, i.e., fully learned and  unlearned coordinates. For those weakly learned coordinates, Claim~\ref{claim:partial_to_partial} ensures that they will ultimately be fully learned, so the format of $\mT$ reduces to one with two entries. Given the assumption $\eta> 1 - \frac{\delta}{\ell+1}$, we can deduce the effective degree satisfies
\begin{align}\label{eq:condition_p_plus_one}
   \lambda_2\cdot \ell + \lambda_1\cdot(p+1)= (1-\eta)\cdot \ell + 1\cdot (p+1) < p+\delta +\eta,
\end{align}
which implies the maximum value of $|T_1|$ is $p+1$. We conclude the proof.

%% file: sections_appendix_a/one_step_gradient.tex
\subsection{Proof of Theorem~\ref{thm:learning_hypercube} (Gradient error)}\label{proof:learning_hypercube_gradient}

\subsubsection{Preparing Lemmas, Propositions and Notation}
We start with some useful notations.
We group the tuple $\mT$ based on the cardinality of each $T_k$.
We define the collection of tuples $\S_\vm$ where $\vm = (m_1,\ldots,m_N)\in \mathbb{N}^N$:
\begin{align}\label{eq:def_vm}
    \S_\vm = \biground{ \mT = (T_1,\ldots,T_N) ~\big\vert~ T_k\in 2^{\T_k},~~|T_k| = m_k~~\textrm{for}~~k\in[N]}.
\end{align}
Then $\phi_{\S_{\vm}}$ is defined as
\begin{align}\label{def:psi_sm}
    \phi_{\S_{\vm}} := \round{\phi_{\mT}}_{\mT \in \S_{\vm}} \in \R^{\vs\tran \vm}.
\end{align}
Clearly, given $\vm$, any $\phi_{\mT}$ for $\mT \in \S_\vm$ has identical effective degree $\vlambda\tran \vm$. Since there is a one-to-one correspondence between the tuple $\mT$ and the set $\cup_{k=1}^N T_k$, we sometimes abuse the notation by treating $\mT$ as $\cup_{k=1}^N T_k$ when it is clear from the context. We sometimes use $\mT$ in subscripts, such as in $D_{\mT,\mT}$, to denote the submatrix of $D$ corresponding to the rows and columns indexed by the elements in $\mT$.

For simplicity of the notation, given the order of the weights in each $T_k$, i.e., $(\gamma_1,\ldots,\gamma_N)$, we denote a vector
\begin{align}
    \vlambda := (\lambda_1,\ldots,\lambda_N),~~~\lambda_k = 1-\gamma_k~~ {\rm for}~~ k\in[N].
\end{align}
Following Definition~\ref{def:eff_degree}, we define the {\it adjusted effective degree} for $\mT$, which is
\begin{align}
    p_{\rm aeff}(\mT):= \sum_{k=1}^N s_k\cdot |T_k|.
\end{align}
Clearly, we have $p_{\rm aeff}(\mT) \leq p_{\rm eff}(\mT)$ since $s_k\leq \lambda_k$.

We define two constants that measure the difference of the (adjusted) effective degree of $\phi_{\mT} \in \phi_{\S_\vm}$ from $p+\delta$: 
\begin{align}\label{eq:zeta_1}
    \zeta_1&:= \min_{\vm:\vlambda\tran\vm \leq p+\delta}\round{p+\delta -\vs\tran \vm}\\
    \zeta_2&:= \min_{\vm:\vlambda\tran\vm > p+\delta}\round{\vlambda\tran \vm- p-\delta }.\label{eq:zeta_2}
\end{align}
\begin{claim}\label{claim:zeta_1_2}
    The constants $\zeta_1,\zeta_2$ are bounded from zero.
\end{claim}
\begin{proof}
By definition, $\zeta_2>0$ trivially holds. Now we show $\zeta_1>0$ also holds. By assumption, $s_k = \lambda_{k}$ holds iff $s_k = 1$. Therefore, we can see that if one of $m_k>0$ and the corresponding $s_k<1$, then $\vs\tran\vm <\vlambda\tran \vm \leq p+\delta$ which implies $\zeta_1>0$. Otherwise, if $s_k = 0$ or $1$ for all $m_k>0$, then $\zeta_1 = p+\delta - \vs\tran \vm  \geq \delta >0$.
\end{proof}

\noindent For simplicity, we denote 
\begin{align}\label{eq:zeta}
    \zeta:= \min(\zeta_1,\zeta_2)>0.
\end{align}
Lastly, we denote the maximum order of $w_r$ as 
\begin{align}
    \gamma_{\max} := \max_{k\in[N]}\gamma_k < 1.
\end{align}
\paragraph{Taylor approximation of kernels.}We similarly denote $$g_{w,m}(t) := \sum_{k=0}^m \frac{g_{w}^{(k)}(0)}{k!}t^k~~{\rm and}~~K_{w,m}( x^{(i)}, x^{(j)}) := g_{w,m}\round{\frac{\inner{ x^{(i)}, x^{(j)}}}{d}}.$$ 
For simplicity of notation, we denote $ z := \sqrt{w}\odot x$. Invoking Lemma~\ref{lemma:each_mono_cube}, $K_w$ can be decomposed into
\begin{align}\label{eq:k_z_z}
    K_{w} (X,X) = K(Z,Z) = \Phi(Z) D\Phi(Z)\tran
\end{align}
where $D$ is a diagonal matrix satisfying $\|D_{\S_k} -g^{(k)}(0)d^{-k}I_{|\S_k|}\|_{\rm op} = O_d(d^{-k-1})$ for $k=0,1,\ldots, \infty$.

We present a similar result to Theorem~\ref{thm:poly_approx} for parameterized kernels. The proof appears in Appendix~\ref{proof:taylor_approx_kernel_m}.

\begin{theorem}[Taylor approximation of parameterized kernels]\label{thm:taylor_approx_kernel_m} Consider independent, mean zero, isotropic random vectors $x^{(1)},\ldots,x^{(n)}$ in $\R^d$ and suppose that the coordinates of each vector $x^{(i)}$ are independent and are uniformly drawn from $\H^d$. Suppose that we are in the regime $n = d^{p+ \delta}$, where $\delta\in(0,1)$ is a constant. Fix an arbitrary constant $C>2$ and approximation order $\emp\geq \frac{2(p+\delta)}{1-\gamma}$.  Given $w \in \R^d_+$ that satisfies  $\norm{w}_1 = d$, and $\norm{w}_{\infty} = O_d(d^{\gamma})$ with $\gamma\in[0,1)$, if ${d}^{1-\gamma}/\log(d) >c_\epsilon $ then the estimate 
\begin{equation}\label{eqn:taylor_approx_rfm}
\snorm{K_{w}(X,X) - K_{w,\emp}(X,X) - (g(1) - g_{\emp}(1)) I_n} \le  c_{g,\emp}\sqrt{\frac{(C\log(n) \norm{w}_\infty)^{\emp+1}}{d^{\emp-2p+1-2\delta}}}
\end{equation}
holds with probability at least $1-\frac{4}{n^{C-2}}$, where $c_{g,\emp},c_\epsilon<\infty$ are constants that depend only on $\emp$, $\sup_{k\geq \emp+1}g^{(k)}(0)$, and $\epsilon$. 
\end{theorem}

Plugging $$\left\lceil\frac{2p+2\delta+2}{1-\mgamma}-1\right\rceil := \widetilde{p}_{\mgamma}$$ as $\emp$ into Eq.~\eqref{eqn:taylor_approx_rfm}, we immediately have
\begin{align}\label{eqn:taylor_approx_rfm_result}
 \snorm{K_{w}(X,X) - K_{w,\widetilde{p}_\mgamma}(X,X) - (g(1) - g_{\emp_\mgamma}(1)) I_n} = O_{d,\P}(d\half).
\end{align}

Now we develop the isometric properties of transformed features $\phi_{\S_\vm}(Z)$ corresponding to Theorem~\ref{lemma:gram_matrix}, which stands with a scaling below. See the proof in Appendix~\ref{proof:gram_matrix_cube_m}.
\begin{lemma}[Isometric properties of scaled orthogonal monomials]\label{lemma:gram_matrix_cube_m}
Consider the regime $n = d^{p+\delta}$. For any $\vm$ that satisfies $\norm{\vm}_1 \leq \emp_\gamma$, the following holds for any $\epsilon>0$:
\begin{align}
    \snorm{\frac{ \Phi_{\S_{\vm}}(Z)\round{\Phi_{\S_{\vm}}(Z)}\tran}{d^{(\vs+\mathbf{1}-\vlambda)\tran \vm}} - \mu_{\S_{\vm}} I} = O_{d,\P}(d^{p+\delta -\vs\tran \vm +\epsilon} + d^{(p+\delta - \vs\tran\vm)/2+\epsilon}),
\end{align}
where 
\begin{align}\label{eq:def_mu}
    \mu_{\S_{\vm}} := \frac{1}{d^{(\vs+\mathbf{1}-\vlambda)\tran \vm}} \E_z\brac{\norm{\phi_{\S_{\vm}}(z)}_2^2}= \Theta_d(1)
\end{align}
\end{lemma}

\noindent Corresponding to Eq.~(\ref{eq:a_leq_p}), we now write $\A$ in terms of the union of $\S_\vm$:
\begin{align*}
   \alp&:=   \bigcup_{\vm: \lmlp} \S_{\vm},\\
   \A_{p_\delta}^+(r)&:=\bigcup_{\vm: \vlambda\tran \vm\in(p+\delta,p+\delta+\lambda_{\pi_N(r)}]} \S_{\vm},
\end{align*}
and we additionally denote
\begin{align*}
       \agp&:=   \bigcup_{\substack{\vm: \lmgp\\ \norm{\vm}_1\leq \emp_\mgamma } } \S_{\vm}.
\end{align*}
Accordingly, we can further decompose $K(Z,Z)$ based on Eq.~\eqref{eq:k_z_z} as
\begin{align}\label{eq:k_z_decomp}
    K(Z,Z) =  \Phi_{\alp}(Z)D_{\alp}\Phi_{\alp}(Z)\tran + \Phi_{\agp}(Z)D_{\agp}\Phi_{\agp}(Z)\tran + g(1) - g_{\emp_\mgamma}(1).
\end{align}
In parallel to Lemma~\ref{lemma:kernel_to_mono}, we have the following estimate for $K(Z,Z)$. We defer the proof to Appendix~\ref{proof:kernel_to_mono_m}.
\begin{lemma}\label{lemma:kernel_to_mono_m}
Consider the regime $n = d^{p+\delta}$. Consider a weight vector $w\in \R^d_+$ satisfying $\|w\|_1=d$ and let the coordinate set be an $(\vs,\vlambda)$-active set for $w$. The following holds for any $\epsilon>0$:
\begin{align}
    &\snorm{K(Z,Z) -  \Phi_{\alp}(Z)D_{\alp}\Phi_{\alp}(Z)\tran  - \rho I_n} =O_{d,\P}(d^{-\zeta_2/2+\epsilon}),\label{eq:kernel_to_mono_m}\\
    &\snorm{K(Z,Z) -  \Phi_{\alp}(Z)D_{\alp}\Phi_{\alp}(Z)\tran  - \sum_{\vm\in J_1}\offd\round{\Phi_{\S_{\vm}}(Z)D_{\S_{\vm}}\Phi_{\S_{\vm}}(Z)\tran}- \rho I_n}= O_{d,\P}(d^{-1+\epsilon}),\label{eq:kernel_to_mono_m_advanced}
\end{align}
where 
\begin{align}\label{eq:rho_cube_m}
   \rho&= \sum_{\substack{\vm: \lmgp,\\\norm{\vm}_1\leq \emp_\mgamma}} d^{(\vs-\vlambda)\tran \vm} g^{(\norm{\vm}_1)}(0)\mu_{\S_{\vm}} + g(1) - g_{\emp_\mgamma}(1) = \Theta_d(1),\\
   J_1&=\{\vm~\vert~\vlambda\tran\vm >p+\delta~~{\rm and}~~\vs\tran \vm< p+2\}
\end{align}

\end{lemma}

Through the proof, we consistently treat functions with argument $X$ or $x$ and we will specify when arguments are $z$. Since $z =\sqrt{w}\odot x$, we can find a diagonal matrix $W $ to transform from $\phi(X)$ to $\phi(Z)$: for any multi-index $S$, we define
\begin{align}\label{eq:w}
    W_{S,S} = w^{S}.
\end{align}
Then we can write $\phi(Z)$ as
\begin{align*}
  \phi(Z)   = \phi(X) W^{1/2}.
\end{align*}
Lastly, for a diagonal matrix $D$, we use $D_{\S}$ where $\S$ is a set to denote a sub-matrix of $D$ with entries $D_{S,S}$ with $S\in \S$.

\subsubsection{Main proof}
Given the parameterized kernel Eq.~\eqref{eq:inner_product_kernel_weighted}, for a new input $ x$, the KRR prediction function is given by
\begin{align}
    \hat{f}_{w}( x) = K_{w}( x,X)(K_{w}(X,X)+\lambda I_n)^{-1}y\qquad \textrm{where}\qquad \beta:=\round{K_{w}(X,X) + \lambda I_n}^{-1} y,
\end{align}
where $\lambda>0$ is the ridge parameter. 

We will estimate the following:
\begin{align}\label{eqn:deriv_est_cube_m}
    \frac{1}{n}\sum_{i=1}^n \round{\partial_{r}  \hat{f}_{w}( x^{(i)})}^2= \frac{w_r^2}{nd^2} \norm{    K_{w}' X_r K_{w,\lambda}^{-1} y}^2_2.
\end{align}
For simplicity of notation, we denote
\begin{align*}
    \widehat{D}:= WD.
\end{align*}
In parallel to Eq.~\eqref{eqn:K_decomp_cube} and \eqref{eqn:K'_decomp_cube}, applying Lemma~\ref{lemma:kernel_to_mono_m} with the two Kernels $K_{w}$ and $K'_{w}$ shows that for any $\epsilon>0$ there exist matrices $\Delta,\Delta_1$ satisfying $\snorm{\Delta},\snorm{\Delta_1} = O_{d,\P}(d^{-\zeta_2/2+\epsilon})$ and 
\begin{align}
    K_{w} &= \Phi_{\alp}\hd_{\alp}\Phi_{\alp}\tran + \rho I_n + \Delta_1\label{eqn:K_decomp_cube_m},\\
    K'_{w} &= \Phi_{\alp}\hd'_{\alp} \Phi_{\alp}\tran + \rho' I_n + \Delta,\label{eqn:K'_decomp_cube_m}
\end{align}
where $\rho$ is defined in Eq.~\eqref{eq:rho_cube_m} and $\rho'$ is defined as
$$\rho' :=  \sum_{\substack{\vm: \lmgp,\\\norm{\vm}_1\leq \emp_\mgamma}} d^{(\vs-\vlambda)\tran \vm} g^{(\norm{\vm}_1)}(0)\mu_{\S_{\vm}} + g'(1) - g_{\emp_\mgamma+1}(1).$$
With the expression \eqref{eqn:K'_decomp_cube_m} in place of $K_w'$, equation \eqref{eqn:deriv_est_cube_m} becomes
\begin{align*}
    \frac{1}{n}\sum_{i=1}^n \round{\partial_{r}  \hat{f}_{w} ( x^{(i)})}^2= \frac{w_{r}^2}{nd^2} \norm{\round{ \underbrace{ \Phi_{\alp}\hd'_{\alp} \Phi_{\alp}\tran X_r K_{w,\lambda}^{-1}}_{=:T_1} + \underbrace{(\rho'  I_n + \Delta)X_r K_{w,\lambda}^{-1}}_{=:T_2}}y}^2_2.
\end{align*}
Expanding the square we have
\begin{align}\label{eq:main_cube_decomp_m}
    \frac{1}{n}\sum_{i=1}^n \round{\partial_{r}  \hat{f}_{w} ( x^{(i)})}^2 =\frac{w_{r}^2}{nd^2}\left(\|T_1y\|^2_2+\|T_2y\|^2_2+2\langle T_1y, T_2y \rangle\right).
\end{align}
We will analyze each addend separately. We suppose $\iota = \pi_N(r)$, i.e., $r\in \T_\iota$ for some $k\in [N]$. 
Specifically, we will show
\begin{claim}\label{cl:main_claim_cube_m}
The following estimates hold for any $\epsilon>0$ uniformly over all coordinates $r\in \T_\iota$ where $c_r,c_\mT>0$ are constants:
\begin{enumerate}[label=(\alph*)]
    \item $\frac{w_{r}^2}{nd^2}\|T_2y\|^2_2 =
O_{d,\P}\left(d^{-2\lambda_\iota+\epsilon}\right)\cdot (\norm{f^*}_{L_2}^2+\sigma_\varepsilon^2)$;\label{cl:main_claimT2_cube_m}
    \item $
        \abs{\frac{w_{r}^2}{nd^2}\norm{T_1y }_2^2 -\norm{\partial_r f^*_{\aslp}}_{L_2}^2 - c_r\cdot \norm{\partial_r f^*_{\aep}}_{L_2}^2 - \sum_{\mT \in \A_{p_{\delta}}^+(r)}c_{\mT}\cdot d^{-2\kappa(\mT)} \norm{\partial_r f^*_{\mT}}_{L_2}^2}\\ 
        =O_{d,\P}\round{ d^{- \lambda_\iota + \epsilon}}\cdot (\norm{f^*}_{L_2}^2 +\sigma_\varepsilon^2)+O_{d,\P}(d^{-\zeta/2+\epsilon})\cdot  \round{\norm{\partial_r f^*_{\alp}}_{L_2}^2  + \sum_{\mT \in \A_{p_{\delta}}^+(r)}d^{-2\kappa(\mT)}\cdot \norm{\partial_r f^*_{\mT}}_{L_2}^2 }\\
+O_{d,\P}\round{d^{(-\lambda_\iota+\epsilon)/2
 }}\cdot\norm{f^*}_{L_2} \norm{\partial_r f^*_{\alp}}_{L_2} +O_{d,\P}\round{d^{-\lambda_\iota/2
 }}\cdot \sum_{\mT \in \A_{p_{\delta}}^+(r)}d^{-\kappa(\mT)}\cdot \norm{\partial_r f^*_{\mT}}_{L_2} \norm{f^*}_{L_2};
    $
\label{cl:main_claimT1_cube_m}

\item\label{cl:main_claimT3_cube_m} $\frac{w_{r}^2}{nd^2}\langle T_1y, T_2y \rangle= O_{d}\round{d^{-\frac{3}{2} \lambda_\iota +\epsilon} }\cdot \norm{f^*}_{L_2}^2+ O_{d,\P}(d^{-\lambda_\iota + \epsilon} )\norm{f^*}_{L_2} \norm{\partial_r f^*_{\alp}}_{L_2} \\
+ \sum_{\mT \in \A_{p_{\delta}}^+(r)}\cdot  d^{-2\kappa(\mT)-\lambda_\iota + \epsilon} \norm{\partial_r f^*_{\mT}}_{L_2}\norm{f^*}_{L_2}$.
\end{enumerate}
\end{claim}

Plugging Claim~\ref{cl:main_claim_cube_m} (a,b,c) into Eq.~\eqref{eq:main_cube_decomp_m} yields:
\begin{align}\label{eq:gradient_error_bound_complete}
    &~~~\abs{\E_n (\partial_r \hat f_\lm)^{2} - \E{\round{\partial_{r} f^*_{\A_{< p}} }^2} - c_r\cdot \E{\round{\partial_{r} f^*_{\A_{= p}} }^2} - \sum_{\mT \in \A_{p_{\delta}}^+(r)}c_{\mT}\cdot d^{-2\kappa(\mT)} \E{\round{\partial_{r} f^*_{\mT}}^2}}\notag\\
    &= O_{d,\P}\round{ d^{- \lambda_\iota +\epsilon}}\cdot (\norm{f^*}_{L_2}^2+\sigma_\varepsilon^2) +O_{d,\P}(d^{-\zeta/2+\epsilon})\cdot  \round{\norm{\partial_r f^*_{\alp}}_{L_2}^2  + \sum_{\mT \in \A_{p_{\delta}}^+(r)}d^{-2\kappa(\mT)}\cdot \norm{\partial_r f^*_{\mT}}_{L_2}^2 }\notag\\
&~~~+O_{d,\P}\round{d^{(-\lambda_\iota+\epsilon)/2
 }}\cdot\norm{f^*}_{L_2} \norm{\partial_r f^*_{\alp}}_{L_2} +O_{d,\P}\round{d^{-\lambda_\iota/2
 }}\cdot \sum_{\mT \in \A_{p_{\delta}}^+(r)}d^{-\kappa(\mT)}\cdot \norm{\partial_r f^*_{\mT}}_{L_2} \norm{f^*}_{L_2}.
\end{align}

Similarly, the rest of the section is devoted to proving Claim~\ref{cl:main_claim_cube_m}.  Note that Item~\ref{cl:main_claimT3_cube_m} follows immediately from  Item~\ref{cl:main_claimT2_cube_m} and \ref{cl:main_claimT1_cube} and the Cauchy-Schwartz inequality. Throughout the proof, all asymptotic terms $o_{d,\P}$ and $O_{d,\P}$ that appear will always be uniform over the coordinates $r=1,\ldots, d$.

\subsubsection{Proof of Item~\ref{cl:main_claimT2_cube_m} of Claim~\ref{cl:main_claim_cube_m}}\label{subsec:t2_cube_m}

The proof is identical to the proof of Claim~\ref{cl:main_claim} (a). We refer the readers to  Section~\ref{subsec:t2} for more details. 

We  can yield the following bound:
\begin{align*}
    \frac{w_{r}^2}{nd^2}\|T_2y\|^2_2 = O_{d,\P}\left(w_{r}^2 \cdot\frac{\log(d)}{d^2}\right)\cdot (\norm{f^*}_{L_2}^2+\sigma_\varepsilon^2) = O_{d,\P}\left(d^{-2\lambda_\iota + \epsilon}\right)\cdot (\norm{f^*}_{L_2}^2+\sigma_\varepsilon^2),
\end{align*}
where $\epsilon>0$ is an arbitrarily small constant.

\subsubsection{Proof of Item~\ref{cl:main_claimT1_cube_m} of Claim~\ref{cl:main_claim_cube_m}}\label{subsec:t1_cube_m}

We denote $\beta = K_{w,\lambda}\inv y$.  Similar to Eq.~(\ref{eqn:basic:needed}), we can write
\begin{align}
\frac{w_{r}^2}{nd^2}\|T_1y\|^2&=\frac{w_{r}^2}{nd^2}\beta^\top[X_r\Phi_{\alp}\hd'_{\alp}\Phi_{\alp}\tran\Phi_{\alp} \hd'_{\alp}\Phi_{\alp}\tran X_r]\beta\notag\\
&=\frac{w_{r}^2}{d^2}\beta^\top[X_r\Phi_{\alp}\hd'_{\alp} (I+\widetilde\Delta)  \hd'_{\alp}\Phi_{\alp}\tran X_r]\beta\notag\\
&=\tfrac{1+O_d(\|\widetilde{\Delta}\|_{\rm op})}{d^2}\cdot \|w_{r}\beta \tran X_r\Phi_{\alp} \hd'_{\alp}\|_2^2,\label{eqn:basic:needed_cube_m}
\end{align}
where $\snorm{\widetilde{\Delta}} = O_{d,\P}(d^{-\zeta_1/2+\epsilon})$.

Notice that $\alp = \cup_{\lmlp} \S_{\vm}$. Let $e_\iota$ be $\iota$-th standard basis vector.
We  decompose $X_r\Phi_{\alp}$ into  components: 
\begin{align}\label{eq:fine_split_cube_m}
    \{\Phi_{\S_{\vm}^0}\}_{\lmlp}~~\textrm{and}~~\{\Phi_{\S_{\vm}^1}\}_{\vlambda\tran (\vm - e_\iota) \leq p+\delta},
\end{align}
where 
\begin{align}
    \S_{\vm}^1 := \biground{\mT~\vert~ r\in \cup_{k=1}^N T_k~\textrm{and}~|T_k| = m_k,~~k\in[N]},\\
    \S_{\vm}^0 := \biground{\mT~\vert~ r\notin \cup_{k=1}^N T_k~\textrm{and}~|T_k| = m_k,~~k\in[N]}.
\end{align}
We correspondingly decompose $\hd'_{\alp}$. We can show there exist diagonal matrices $\{R_{\vm}^0\}_\vm$ and $\{R_{\vm}^1\}_\vm$ such that the following holds:
\begin{align}\label{eqn:rhs_negl_cube_m_raw}
  &~~~\norm{\beta \tran X_r\Phi_{\alp} \round{\frac{w_r}{d} \hd_{\alp}' }}_2^2\notag\\
  &=   \sum_{\vlambda\tran \vm \leq p+\delta}\|\beta \tran \Phi_{\S_{\vm}^0} R_{\vm}^0\|_2^2   + \sum_{\vlambda\tran (\vm-e_\iota)\leq p+\delta}\|\beta \tran \Phi_{\S_{\vm}^1} R_{\vm}^1 \|_2^2,
\end{align}

\begin{claim}
    The following estimate holds:
    \begin{align}
        \snorm{R_{\vm}^0 - g^{(\norm{\vm}_1+2)}(0) \cdot w_r^2\cdot W_{\S_{\vm}^0}d^{-\norm{\vm}_1-2}} &= w_r^2\cdot O_d(d^{-\vlambda\tran \vm-3 }), \label{eq:def_r_j_0_m}\\
        \snorm{R_{\vm}^1 - g^{(\norm{\vm}_1)}(0)\cdot  W_{\S_{\vm}^1}d^{-\norm{\vm}_1}} &= O_d(d^{-\vlambda\tran \vm-1 }),\label{eq:def_r_j_1_m}
    \end{align}
\end{claim}
\begin{proof}
 Recall that $D'$ is a diagonal matrix with entry $D'_{\S_\vm} =  g^{(\norm{\vm}_1+1)}(0)d^{-\norm{\vm}_1} + O_d(d^{-\norm{\vm}_1-1})$. Also each entry of $W_{\S_\vm}$ is of the order $d^{\norm{\vm}_1 - \vlambda\tran \vm}$ based on Eq.~\eqref{eq:w}.

 Based on Eq.~\eqref{eqn:rhs_negl_cube_m_raw}, we can see that  $\Phi_{\S_\vm^0}R_\vm^0$ corresponds to $x_r\cdot\Phi_{\S^1_{\vm+e_\iota}} \frac{w_r}{d}D'_{\S^1_{\vm+e_\iota}}$ since $\phi_{S_{\vm}^0} = x_r\phi_{S_{\vm+e_\iota}^1}$. Then we can deduce that 
 \begin{align*}
     R_\vm^0 = \frac{w_r}{d}W_{\S_{\vm+e_\iota}^1}D'_{\S_{\vm+e_\iota}^1} = \frac{w_r^2}{d} W_{\S_\vm^0} D'_{\S_{\vm+e_\iota}^1}.
     \end{align*}
Plugging the expression of $W$ and $D'$ into the above equation completes the proof for Eq.~\eqref{eq:def_r_j_0_m}. A similar analysis applies to
\begin{align*}
    R_\vm^1 = \frac{w_r}{d}W_{\S_{\vm-e_\iota}^0}D'_{\S_{\vm-e_\iota}^0} = \frac{1}{d}W_{\S_{\vm}^1}D'_{\S_{\vm-e_\iota}^0}.
\end{align*} 
\end{proof}

Regarding Eq.~\eqref{eqn:rhs_negl_cube_m_raw}, we group the features based on the inequalities (1) $\vlambda\tran\vm \leq p+\delta$ and (2) $\vlambda\tran\vm >p+\delta$. Then we can define a diagonal matrix $R$ with each block $R_{\S_\vm^1,\S_\vm^1} = R_\vm^1$ and $R_{\S_\vm^0,\S_\vm^0} = R_\vm^0$  such that Eq.~\eqref{eqn:rhs_negl_cube_m_raw} can be written as
\begin{align}
\norm{\beta \tran X_r\Phi_{\alp} \round{\frac{w_r}{d} \hd'_{\alp} }}_2^2= \norm{\beta\tran \Phi_{\alp}R}_2^2 +  \sum_{\vm:\vlambda\tran\vm \in (p+\delta, p+\delta +\lambda_\iota]}\norm{\beta\tran \Phi_{\S_{\vm}^1}R_{\vm}^1}_2^2\label{eqn:rhs_negl_cube_m}.
\end{align}
We start by analyzing the first term on the right side of \eqref{eqn:rhs_negl_cube_m}.
The proof of the proposition appears in Appendix~\ref{proof:<_l1_l1_l1_cube_m}.

\begin{proposition}\label{prop:<_l1_l1_l1_cube_m}
For any $\epsilon>0$, the following estimate holds uniformly for all coordinates $r \in \T_\iota$ where $c_r>0$ is a constant:
    \begin{align}\label{eq:l1_l1_cube_m}
        &~~~\abs{ \norm{\beta\tran   \Phi_{\alp}R}_2^2 - \norm{\partial_r f^*_{\aslp}}_{L_2}^2 -c_r\cdot \norm{\partial_r f^*_{\aep}}_{L_2}^2}\notag\\
&=O_{d,\P}\round{d^{-\lambda_\iota} }\cdot \norm{f^*}_{L_2}^2 +   O_{d,\P}(d^{-\zeta/2+\epsilon})\cdot  \norm{\partial_r f^*_{\alp}}_{L_2}^2+O_{d,\P}\round{d^{-\lambda_\iota/2 }}\cdot \norm{f^*}_{L_2}\norm{\partial_r f^*_{\alp}}_{L_2},
    \end{align}
\end{proposition}

\noindent We proceed to analyze the second term on the right side of \eqref{eqn:rhs_negl_cube_m}. The proof appears in Appendix~\ref{proof:s2_s2_cube_m}.

\begin{proposition}\label{prop:s2_s2_cube_m}
For any $\epsilon>0$, the estimate holds uniformly for all coordinates $r \in \T_\iota$ and $\vm\in\R^N$ that satisfies $\vlambda \tran \vm \in (p+\delta,p+\delta+\lambda_\iota]$:  \begin{align}\label{eq:s2_s2_cube_m_bound1}
    &~~~\abs{\norm{\beta\tran \Phi_{\S_{\vm}^1}R_{\vm}^1}_2^2-\round{g^{(\norm{\vm}_1)}(0)}^2\norm{(\rho+\lambda)\inv d^{p+\delta -\norm{\vm}_1 }b_{\S_{\vm}^1}\tran W_{\S_{\vm}^1}}_2^2}\notag\\
    &= O_{d,\P}(d^{-\lambda_\iota })\cdot (\norm{f^*}^2_{L_2}+\sigma_\varepsilon^2) +  d^{2(p+\delta - \vlambda\tran \vm)}\cdot O_{d,\P}(d^{-\zeta/2+\epsilon})\cdot \norm{\partial_r f^*_{\S_{\vm}^1}}_{L_2}^2 \notag\\
     &~~~+ O_{d,\P}(d^{p+\delta - \vlambda\tran \vm  - \lambda_\iota/2 }) \cdot \norm{\partial_r f^*_{\S_{\vm}^1}}_{L_2} (\norm{f^*}_{L_2}+\sigma_\varepsilon).
\end{align}
\end{proposition}

\noindent Following the triangle inequality,  the second term in Eq.~\eqref{eqn:rhs_negl_cube_m} can be approximated as
\begin{align}\label{eq:s2_s2_cube_m_bound2}
&~~~\abs{\sum_{\vm:\vlambda\tran\vm \in (p+\delta, p+\delta +\lambda_\iota]}\norm{\beta\tran \Phi_{\S_{\vm}^1}R_{\vm}^1}_2^2 -\sum_{\vm:\vlambda\tran\vm \in (p+\delta, p+\delta +\lambda_\iota]}\round{g^{(\norm{\vm}_1)}(0)}^2\norm{(\rho+\lambda)\inv d^{p+\delta -\norm{\vm}_1 }b_{\S_{\vm}^1}\tran W_{\S_{\vm}^1}}_2^2}\notag\\
&= O_{d,\P}(d^{- \lambda_\iota })\cdot (\norm{f^*}^2_{L_2}+\sigma_\varepsilon^2) + \sum_{\vm:\vlambda\tran\vm \in (p+\delta, p+\delta +\lambda_\iota]} d^{2(p+\delta - \vlambda\tran \vm)}\cdot O_{d,\P}(d^{-\zeta/2+\epsilon})\cdot \norm{\partial_r f^*_{\S_{\vm}^1}}_{L_2}^2 \notag\\
&~~~+\sum_{\vm:\vlambda\tran\vm \in (p+\delta, p+\delta +\lambda_\iota]} O_{d,\P}(d^{p+\delta - \vlambda\tran \vm - \lambda_\iota/2}) \cdot \norm{\partial_r f^*_{\S_{\vm}^1}}_{L_2} (\norm{f^*}_{L_2}+\sigma_\varepsilon).
\end{align}
We can further simplify the above equation.
Note that $|W_{\mT,\mT}| = \Theta_d(d^{\norm{\vm}_1 - \vlambda\tran \vm})$ for any $\mT\in \S_{\vm}^1$. We can deduce that there exists a constant $c_{\mT}>0$ such that 
\begin{align}
    \round{g^{(\norm{\vm}_1)}(0)}^2\norm{(\rho+\lambda)\inv d^{p+\delta -\norm{\vm}_1 }b_{\mT}\tran W_{\mT,\mT}}^2 =c_\mT \cdot d^{2(p+\delta - \vlambda\tran \vm)}\cdot \norm{b_{\mT}}_2^2 = c_\mT d^{-2\kappa(\mT)}\cdot \norm{b_{\mT}}_2^2
\end{align}
where $\kappa(\mT):= \vlambda\tran \vm-p+\delta$ is defined in Eq.~\eqref{eq:p_eff_dist}.

Finally with the notation $\zeta:=\min(\zeta_1,\zeta_2)$,  Eq.~\eqref{eq:s2_s2_cube_m_bound2} can be reduced to
\begin{align}\label{eq:s2_s2_cube_m_bound3}
&~~~\abs{\sum_{\vm:\vlambda\tran\vm \in (p+\delta, p+\delta +\lambda_\iota]}\norm{\beta\tran \Phi_{\S_{\vm}^1}R_{\vm}^1}_2^2 -\sum_{\mT \in \A_{p_{\delta}}^+(r)}c_{\mT}\cdot d^{-2\kappa(\mT)} \norm{b_{\mT}}_2^2}\notag\\
&= O_{d,\P}(d^{- \lambda_\iota })\cdot (\norm{f^*}^2_{L_2}+\sigma_\varepsilon^2) +  \sum_{\mT \in \A_{p_{\delta}}^+(r)}d^{-2\kappa(\mT)}\cdot O_{d,\P}(d^{-\zeta/2+\epsilon})\cdot \norm{b_{\mT}}_2^2 \notag\\
     &~~~+ \sum_{\mT \in \A_{p_{\delta}}^+(r)}d^{-\kappa(\mT)}\cdot  O_{d,\P}(d^{ - \lambda_\iota/2 }) \cdot \norm{b_{\mT}}_2 (\norm{f^*}_{L_2}+\sigma_\varepsilon).
\end{align}
Eq.~\eqref{eq:l1_l1_cube_m} and \eqref{eq:s2_s2_cube_m_bound3} together give
\begin{align}
 &~~~\abs{\norm{\beta \tran X_r\Phi_{\alp} \round{\frac{w_r}{d} \hd'_{\alp} }}_2^2 - \norm{\partial_r f^*_{\alp}}_{L_2}^2 - c_r\cdot \norm{\partial_r f^*_{\aep}}_{L_2}^2 - \sum_{\mT \in \A_{p_{\delta}}^+(r)}c_{\mT}\cdot d^{-2\kappa(\mT)} \norm{b_{\mT}}_2^2}\notag\\ 
 &=O_{d,\P}\round{ d^{- \lambda_\iota }}\cdot (\norm{f^*}_{L_2}^2+\sigma_\varepsilon^2) +O_{d,\P}(d^{-\zeta/2+\epsilon})\cdot  \norm{\partial_r f^*_{\alp}}_{L_2}^2  + \sum_{\mT \in \A_{p_{\delta}}^+(r)}d^{-2\kappa(\mT)}\cdot O_{d,\P}(d^{-\zeta/2+\epsilon})\cdot \norm{\partial_r f^*_{\mT}}_{L_2}^2 \notag\\
 &~~~+O_{d,\P}\round{d^{-\lambda_\iota/2
 }}\cdot(\norm{f^*}_{L_2}+\sigma_\varepsilon)\norm{\partial_r f^*_{\alp}}_{L_2} +\sum_{\mT \in \A_{p_{\delta}}^+(r)}d^{-\kappa(\mT)}\cdot  O_{d,\P}(d^{ - \lambda_\iota/2 }) \cdot \norm{\partial_r f^*_{\mT}}_{L_2} (\norm{f^*}_{L_2}+\sigma_\varepsilon).
\end{align}
Combining this expression with \eqref{eqn:basic:needed_cube_m} completes the proof.

%% file: sections_appendix_a/M_hat.tex
\subsection{Derivative norm estimator}\label{proof:learning_hypercube_gradient_rkhs}
The proof is analogous to the proof for equation~\eqref{eq:gradient_error_bound} hence we adopt the notation used in Appendix~\ref{proof:learning_hypercube_gradient}. We will estimate the following:
\begin{align}\label{eqn:deriv_est_rkhs}
   w_r\cdot\D_r(w) := \frac{w_r}{nd} y\tran K_{w,\lambda}\inv X_r K'_w X_r K_{w,\lambda}\inv y
\end{align}
With the expression \eqref{eqn:K'_decomp_cube_m} in place of $K_w'$, equation \eqref{eqn:deriv_est_rkhs} becomes
\begin{align}\label{eq:main_rkhs_decomp}
    w_r\cdot\D_r(w) = \frac{w_{r}}{nd} \norm{ \underbrace{(\hd'_{\alp})^{\frac{1}{2}}\Phi_{\alp}\tran X_r K_{w,\lambda}\inv}_{:=T_1} y}_2^2 + \frac{w_r}{nd}\norm{\underbrace{(\rho'I_n+\Delta)^{\frac{1}{2}} X_r K_{w,\lambda}\inv}_{:=T_2} y}_2^2.
\end{align}
We will analyze each addend separately. We suppose $\iota = \pi_N(r)$, i.e., $r\in \T_\iota $ for some $\iota \in [N]$. 
Specifically, we will show
\begin{claim}\label{cl:main_claim_rkhs}
The following estimates hold for any $\epsilon>0$ uniformly over all coordinates $r\in \T_\iota$ :
\begin{enumerate}[label=(\alph*)]
    \item $\frac{w_{r}}{nd}\|T_2y\|^2_2 =
O_{d,\P}\left(d^{-\lambda_\iota + \epsilon}\right)\cdot (\norm{f^*}_{L_2}^2+\sigma_\varepsilon^2)$;\label{cl:main_claimT2_rkhs}
    \item $
       \abs{\frac{w_r}{nd}\norm{T_1 y}_2^2 - \sum_{\mT \in  \alp\sqcup \A_{p_{\delta}}^+}\Theta_d(d^{-\abs{\kappa(\mT}}) \cdot \norm{\partial_r f_{\mT}}_{L_2}^2}\notag\\ =O_{d,\P}\round{ d^{- \lambda_\iota +\epsilon}}\cdot (\norm{f^*}_{L_2}^2+\sigma_\varepsilon^2)  +  O_{d,\P}(d^{-\zeta/2+\epsilon})\cdot\sum_{\mT \in  \alp\sqcup \A_{p_{\delta}}^+} d^{-\abs{\kappa(\mT)}}\cdot \norm{\partial_r f^*_{\mT}}_{L_2}^2 \notag\\
+O_{d,\P}(d^{ - \lambda_\iota/2 + \epsilon }) \cdot \sum_{\mT \in  \alp\sqcup \A_{p_{\delta}}^+}d^{-\abs{\kappa(\mT)}/2}\cdot   \norm{\partial_r f^*_{\mT}}_{L_2} (\norm{f^*}_{L_2}+\sigma_\varepsilon). $
\label{cl:main_claimT1_rkhs}

\end{enumerate}
\end{claim}

Plugging Claim~\ref{cl:main_claim_rkhs} (a,b) into Eq.~\eqref{eq:main_rkhs_decomp} yields:
\begin{align}\label{eq:gradient_error_bound_complete_rkhs}
    &~~~\abs{w_r\D_r(w)  - \sum_{\mT \in \alp\sqcup \A_{p_{\delta}}^+}\Theta_d(d^{-\abs{\kappa(\mT}}) \cdot \norm{\partial_r f_{\mT}}_{L_2}^2}\notag\\
    &= O_{d,\P}\round{ d^{- \lambda_\iota +\epsilon}}\cdot (\norm{f^*}_{L_2}^2+\sigma_\varepsilon^2)  +  O_{d,\P}(d^{-\zeta/2+\epsilon})\cdot\sum_{\mT \in \alp\sqcup \A_{p_{\delta}}^+}d^{-\abs{\kappa(\mT)}}\cdot \norm{\partial_r f^*_{\mT}}_{L_2}^2 \notag\\
&~~~~+O_{d,\P}(d^{ - \lambda_\iota/2 + \epsilon }) \cdot \sum_{\mT \in \alp\sqcup \A_{p_{\delta}}^+}d^{-\abs{\kappa(\mT)}/2}\cdot   \norm{\partial_r f^*_{\mT}}_{L_2} (\norm{f^*}_{L_2}+\sigma_\varepsilon).
\end{align}
Similarly, the rest of the section is devoted to proving Claim~\ref{cl:main_claim_rkhs}.   Throughout the proof, all asymptotic terms $o_{d,\P}$ and $O_{d,\P}$ that appear will always be uniform over the coordinates $r=1,\ldots, d$.

\subsubsection{Proof of Item~\ref{cl:main_claimT2_rkhs} of Claim~\ref{cl:main_claim_rkhs}}\label{subsec:t2_rkhs}

The proof is almost identical to the proof of Claim~\ref{cl:main_claim} (a). We refer the readers to  Section~\ref{subsec:t2} for more details. 
We  can similarly obtain the following bound:
\begin{align*}
    \frac{w_{r}}{nd}\|T_2y\|^2_2 = O_{d,\P}\left(w_{r}\cdot\frac{\log(d)}{d}\right)\cdot (\norm{f^*}_{L_2}^2+\sigma_\varepsilon^2) = O_{d,\P}\left(d^{-\lambda_\iota + \epsilon}\right)\cdot (\norm{f^*}_{L_2}^2+\sigma_\varepsilon^2),
\end{align*}
where $\epsilon>0$ is an arbitrarily small constant.

\subsubsection{Proof of Item~\ref{cl:main_claimT1_rkhs} of Claim~\ref{cl:main_claim_rkhs}}\label{subsec:t1_rkhs}

We denote $\beta = K_{w,\lambda}\inv y$. 
Following Eq.~\eqref{eq:fine_split_cube_m}, we  decompose $X_r\Phi_{\alp}$ into  components: 
\begin{align}\label{eq:fine_split_rkhs}
    \{\Phi_{\S_{\vm}^0}\}_{\lmlp}~~\textrm{and}~~\{\Phi_{\S_{\vm}^1}\}_{\vlambda\tran (\vm - e_\iota) \leq p+\delta}.
\end{align}
We correspondingly decompose $\hd'_{\alp}$. We can show there exist diagonal matrices $\{R_{\vm}^0\}_\vm$ and $\{R_{\vm}^1\}_\vm$ such that the following holds:
\begin{align}\label{eqn:rhs_negl_rkhs_raw}
  &~~~\norm{\beta \tran X_r\Phi_{\alp} \round{\frac{w_r}{nd} \hd_{\alp}' }^{\frac{1}{2}}}_2^2\notag\\
  &=   \sum_{\vlambda\tran \vm \leq p+\delta}\|\beta \tran \Phi_{\S_{\vm}^0} R_{\vm}^0\|_2^2   + \sum_{\vlambda\tran (\vm-e_\iota)\leq p+\delta}\|\beta \tran \Phi_{\S_{\vm}^1} R_{\vm}^1 \|_2^2,
\end{align}

\begin{claim}
    The following estimate holds:
\begin{align}
    (R_{\vm}^0)_{\mT,\mT} &=    \round{g^{(\norm{\vm}_1+2)}(0)\cdot \frac{w_r^2}{nd}W_{\mT,\mT}  d^{-\norm{\vm}_1 -1}(1 + O_d(d\inv))}^{\frac{1}{2}},~~~\forall~\mT \in \S_{\vm}^0; \notag\\
    (R_{\vm}^1)_{\mT,\mT} &=    \round{g^{(\norm{\vm}_1)}(0)\cdot \frac{1}{nd}W_{\mT,\mT}  d^{-\norm{\vm}_1 + 1}(1 + O_d(d\inv))}^{\frac{1}{2}},~~~\forall~\mT \in \S_{\vm}^1.\label{eq:def_r_j_1_rkhs}
\end{align}

\end{claim}
\begin{proof}
 Recall that $D'$ is a diagonal matrix with entry $D'_{\S_\vm} =  g^{(\norm{\vm}_1+1)}(0)d^{-\norm{\vm}_1} + O_d(d^{-\norm{\vm}_1-1})$. 
 Based on Eq.~\eqref{eqn:rhs_negl_rkhs_raw}, we can see that  $\Phi_{\S_\vm^0}R_\vm^0$ corresponds to $x_r\cdot\Phi_{\S^1_{\vm+e_\iota}} (\frac{w_r^2}{nd}\hd'_{\S^1_{\vm+e_\iota}})^{\frac{1}{2}}$ since $\phi_{S_{\vm}^0} = x_r\phi_{S_{\vm+e_\iota}^1}$. Then we can deduce that 
 \begin{align*}
     R_\vm^0 = \round{\frac{w_r}{nd}W_{\S_{\vm+e_\iota}^1}D'_{\S_{\vm+e_\iota}^1}}^{\frac{1}{2}} =  \round{\frac{w_r^2}{nd}W_{\S_{\vm}^0}D'_{\S_{\vm+e_\iota}^1}}^{\frac{1}{2}}.
     \end{align*}
Plugging the expression of $D'$ into the above equation completes the proof for Eq.~\eqref{eq:def_r_j_0_m}. A similar analysis applies to
\begin{align*}
    R_\vm^1 = \round{\frac{w_r}{nd}W_{\S_{\vm-e_\iota}^0}D'_{\S_{\vm-e_\iota}^0}}^{\frac{1}{2}} = \round{\frac{1}{nd}W_{\S_{\vm}^1}D'_{\S_{\vm-e_\iota}^0}}^{\frac{1}{2}}.
\end{align*} 
\end{proof}

Regarding Eq.~\eqref{eqn:rhs_negl_rkhs_raw}, we group the features based on the inequalities (1) $\vlambda\tran\vm \leq p+\delta$ and (2) $\vlambda\tran\vm >p+\delta$. Then we can define a diagonal matrix $R$ with each block $R_{\S_\vm^1,\S_\vm^1} = R_\vm^1$ and $R_{\S_\vm^0,\S_\vm^0} = R_\vm^0$  such that Eq.~\eqref{eqn:rhs_negl_rkhs_raw} can be written as
\begin{align}
\norm{\beta \tran X_r\Phi_{\alp} \round{\frac{w_r}{nd} \hd'_{\alp} }^{\frac{1}{2}}}_2^2= \norm{\beta\tran \Phi_{\alp}R}_2^2 +  \sum_{\vm:\vlambda\tran\vm \in (p+\delta, p+\delta +\lambda_\iota]}\norm{\beta\tran \Phi_{\S_{\vm}^1}R_{\vm}^1}_2^2\label{eqn:rhs_negl_rkhs}.
\end{align}
We start by analyzing the first term on the right side of \eqref{eqn:rhs_negl_rkhs}.
The proof of the proposition appears in Appendix~\ref{proof:<_l1_l1_l1_rkhs}.

\begin{proposition}\label{prop:<_l1_l1_l1_rkhs}
For any $\epsilon>0$, the following estimate holds uniformly for all coordinates $r \in \T_\iota$:
    \begin{align}\label{eq:l1_l1_rkhs}
        &~~~\abs{ \norm{\beta\tran   \Phi_{\alp}R}_2^2 -\sum_{\vm:\vlambda\tran\vm\leq p+\delta}\Theta_d(d^{\vlambda\tran \vm - p-\delta})\cdot \norm{\partial_r f^*_{\S_\vm^1} }_{L_2}^2}\notag\\
&=O_{d,\P}\round{d^{-\lambda_\iota} }\cdot \norm{f^*}_{L_2}^2 +  \sum_{\vm:\vlambda\tran \vm \leq p+\delta}O_{d,\P}( d^{-\zeta/2 + \epsilon+ (\vlambda\tran\vm - p - \delta)})\cdot  \norm{\partial_r f^*_{\S_{\vm}^1}}_{L_2}^2 \notag \\
&~~~~+O_{d,\P}\round{d^{-\lambda_\iota/2 }}\cdot \norm{f^*}_{L_2}\norm{\partial_r f^*_{\alp}}_{L_2},
    \end{align}
\end{proposition}

\noindent We proceed to analyze the second term on the right side of \eqref{eqn:rhs_negl_rkhs}. The proof appears in Appendix~\ref{proof:s2_s2_rkhs}.

\begin{proposition}\label{prop:s2_s2_rkhs}
For any $\epsilon>0$, the estimate holds uniformly for all coordinates $r \in \T_\iota$ and $\vm\in\R^N$ that satisfies $\vlambda \tran \vm \in (p+\delta,p+\delta+\lambda_\iota]$:  \begin{align}\label{eq:s2_s2_rkhs_bound1}
    &~~~\abs{\norm{\beta\tran \Phi_{\S_{\vm}^1}R_{\vm}^1}_2^2- \Theta_d(d^{p+\delta - \vlambda\tran \vm}) \cdot \norm{b_{\S_{\vm}^1}}_2^2}\notag\\
    &= O_{d,\P}(d^{-\lambda_\iota +\epsilon })\cdot (\norm{f^*}^2_{L_2}+\sigma_\varepsilon^2) +  d^{p+\delta - \vlambda\tran \vm}\cdot O_{d,\P}(d^{-\zeta/2+\epsilon})\cdot \norm{\partial_r f^*_{\S_{\vm}^1}}_{L_2}^2 \notag\\
     &~~~+ O_{d,\P}(d^{(p+\delta - \vlambda\tran \vm)/2  - \lambda_\iota/2 +\epsilon }) \cdot \norm{\partial_r f^*_{\S_{\vm}^1}}_{L_2} (\norm{f^*}_{L_2}+\sigma_\varepsilon).
\end{align}
\end{proposition}

\noindent Following the triangle inequality,  the second term in Eq.~\eqref{eqn:rhs_negl_rkhs} can be approximated as
\begin{align}\label{eq:s2_s2_rkhs_bound2}
&~~~\abs{\sum_{\vm:\vlambda\tran\vm \in (p+\delta, p+\delta +\lambda_\iota]}\norm{\beta\tran \Phi_{\S_{\vm}^1}R_{\vm}^1}_2^2 -\sum_{\vm:\vlambda\tran\vm \in (p+\delta, p+\delta +\lambda_\iota]}\Theta_d(d^{p+\delta - \vlambda\tran \vm}) \cdot \norm{b_{\S_{\vm}^1}}_2^2}\notag\\
&= O_{d,\P}(d^{- \lambda_\iota + \epsilon })\cdot (\norm{f^*}^2_{L_2}+\sigma_\varepsilon^2) + \sum_{\vm:\vlambda\tran\vm \in (p+\delta, p+\delta +\lambda_\iota]} d^{p+\delta - \vlambda\tran \vm}\cdot O_{d,\P}(d^{-\zeta/2+\epsilon})\cdot \norm{\partial_r f^*_{\S_{\vm}^1}}_{L_2}^2 \notag\\
&~~~+\sum_{\vm:\vlambda\tran\vm \in (p+\delta, p+\delta +\lambda_\iota]} O_{d,\P}(d^{(p+\delta - \vlambda\tran \vm)/2 - \lambda_\iota/2 + \epsilon}) \cdot \norm{\partial_r f^*_{\S_{\vm}^1}}_{L_2} (\norm{f^*}_{L_2}+\sigma_\varepsilon).
\end{align}
Lastly, Eq.~\eqref{eq:l1_l1_rkhs} and \eqref{eq:s2_s2_rkhs_bound2} together give
\begin{align}
 &~~~\abs{\norm{\beta \tran X_r\Phi_{\alp} \round{\frac{w_r}{nd} \hd'_{\alp} }^{\frac{1}{2}}}_2^2 - \sum_{\mT \in \alp\sqcup \A_{p_{\delta}}^+}\Theta_d(d^{-\abs{\kappa(\mT}}) \cdot \norm{\partial_r f_{\mT}}_{L_2}^2}\notag\\ &=O_{d,\P}\round{ d^{- \lambda_\iota +\epsilon}}\cdot (\norm{f^*}_{L_2}^2+\sigma_\varepsilon^2)  + \sum_{\mT \in \alp\sqcup \A_{p_{\delta}}^+}d^{-\abs{\kappa(\mT)}}\cdot O_{d,\P}(d^{-\zeta/2+\epsilon})\cdot \norm{\partial_r f^*_{\mT}}_{L_2}^2 \notag\\
 &~~~+\sum_{\mT \in \alp\sqcup \A_{p_{\delta}}^+}d^{-\abs{\kappa(\mT)}/2}\cdot  O_{d,\P}(d^{ - \lambda_\iota/2 + \epsilon }) \cdot \norm{\partial_r f^*_{\mT}}_{L_2} (\norm{f^*}_{L_2}+\sigma_\varepsilon).
\end{align}
Then we complete the proof.

%% file: sections_appendix_a/one_step_generalization.tex
\subsection{Proof of Theorem~\ref{thm:learning_hypercube} (Generalization bound)}\label{proof:learning_hypercube_generalization}

We denote $z^{(i)} =\sqrt{w}\odot x^{(i)}$. Plugging the expression
$$\hat{f}_{w}(x) = K_w(x,X)K_{w,\lambda}\inv y = K(\sqrt{w}\odot x,Z)K_{w,\lambda}\inv y $$ 
into the left-hand side of Eq.~\eqref{eq:generalization_error_bound} gives
\begin{align*}
    &~~~~\E_x\brac{\round{\hat{f}_{w}(x) - f^*_{\aslp}(x) - \sqrt{c}\cdot f^*_{\aep}(x)}^2} \\
    &=\underbrace{y\tran K_{w,\lambda}\inv \E_x[K(Z,\sqrt{w}\odot x)K(\sqrt{w}\odot x,Z)]K_{w,\lambda}\inv y}_{:=T_1} +   \round{\E_x\brac{{f^*_{\aslp}(x)}^2} + c\cdot \E_x\brac{{f^*_{\aep}(x)}^2}} \\
    &~~~- 2 \underbrace{\E_x\brac{(f^*_{\aslp}(x)+\sqrt{c}\cdot f^*_{\aep}(x))K(\sqrt{w}\odot x,Z)} K_{w,\lambda}\inv y}_{:=T_2},
\end{align*}
where we denote $K_{w,\lambda} = K_{\lambda}(Z,Z) = K(Z,Z)+\lambda I_n$ and $c>0$ is a constant.

We will analyze each summand separately. Specifically, we will verify the following claim, from which the above equation follows immediately.

\begin{claim}\label{cl:main_claim_learning}
There exists a constant $c>0$ such that the following estimates hold 
\begin{enumerate}[label=(\alph*)]
    \item $ \abs{T_{1} - \E_x\brac{{f^*_{\aslp}(x)}^2} - c\cdot \E_x\brac{{f^*_{\aep}(x)}^2} }= o_{d,\P}(1)\cdot (\norm{f^*}_{L_2}^2+\sigma_\varepsilon^2)$,\label{cl:main_claim_learning_1}
    \item $\abs{T_{2} - \E_x\brac{{f^*_{\aslp}(x)}^2} - c\cdot \E_x\brac{{f^*_{\aep}(x)}^2}} = o_{d,\P}(1)\cdot (\norm{f^*}_{L_2}^2+\sigma_\varepsilon^2)$. \label{cl:main_claim_learning_2}
\end{enumerate}
\end{claim}

Before presenting the proof of the claims, we give a useful orthogonal decomposition of the kernel.  The proof appears in Appendix~\ref{proof:mercer}.
\begin{lemma}\label{prop:mercer}
Let $g: \R \to \R$ satisfy assumption~\ref{assump:g_0} and fix   $w \in \R^d_+$ with  $\norm{w}_\infty = O_d(d^\gamma)$ for some $\gamma<1$. For every $x,y\in \H^d := \{\pm 1\}^d$, endowed with the uniform measure \(\tau_{d}\), the kernel $K_w(x,y) :=g(x\tran \Diag(w) y/d)$  admits an orthogonal expansion in $L_2(\H^d\times \H^d,\tau_d \otimes \tau_d)$:
\begin{align}
    K_w(x,y) =\sum_{r=0}^\infty \sum_{S: |S| = r} \lambda_S \phi_{S}(x)\tran \phi_{S}(y),
\end{align}
where  \(\phi_{S}(x):=\prod_{i\in S}x_{i}\) and $\lambda_S =  (1+O_d(d^{-2(1-\gamma)}))\cdot g^{(|S|)}(0)d^{-|S|}w^S$.
\end{lemma}

\subsubsection{Proof of Claim~\ref{cl:main_claim_learning_1}}
 We define diagonal matrices $\Lambda_{\alp}$ and $V$ with each entry
 \begin{align}\label{eq:def_v_generalization}
     \Lambda_{S,S} = \lambda_S,~~~ V_{S,S} = \frac{n\lambda_s}{n\lambda_s + (\rho+\lambda)},
 \end{align}
for $S\in \alp$, where $\lambda_S$ is defined in Proposition~\ref{prop:mercer}.  We proceed to define the constant $c$ as
\begin{align}\label{eq:def_c_generalization}
c:= \frac{\sum_{S\in \aep}V_{S,S}^2 b_S^2}{\sum_{S\in \aep} b_S^2}.
\end{align}
 For simplicity of notation, we denote $$M:=  \E_x[K(Z,\sqrt{w}\odot x)K(\sqrt{w}\odot x,Z)].$$
Similar to Lemma~\ref{lemma:kernel_to_mono_m}, we show $M$ can be approximated using Fourier features. 
The proof appears in Appendix~\ref{proof:approx_H}.
\begin{proposition}\label{prop:approx_H}
The following estimate holds
\begin{align}
   \E_x[K(Z,\sqrt{w}\odot x)K(\sqrt{w}\odot x,Z)] = \underbrace{\Phi_{\alp}(X)\Lambda_{\alp}^2\Phi_{\alp}(X)\tran}_{:= M_{\alp}}  + \Delta_M,
\end{align}
where $\snorm{\Delta_M} = O_{d,\P}(d^{-p-\delta -\zeta_2})$.
\end{proposition}

\noindent Consequently, we decompose $T_1$ into
\begin{align*}
    T_1 = \underbrace{y\tran K_{\lm,\lambda}\inv M_{\alp} K_{\lm,\lambda}\inv y}_{T_{11}} + \underbrace{y\tran K_{\lm,\lambda}\inv\Delta_M K_{\lm,\lambda}\inv y}_{T_{12}}.
\end{align*}
A simple calculation shows $|T_{12}| = o_{d,\P}(1)$:
\begin{align}\label{eq:bound_t12}
    |T_{12}| &\leq \norm{y}_2^2 \snorm{K_{\lm,\lambda}\inv}^2 \snorm{\Delta_M} \notag \\
    &= O_d(n\log(d))\cdot (\norm{f^*}_{L_2}^2+\sigma_\varepsilon^2)\cdot O_{d,\P}(1) \cdot O_{d,\P}(d^{-p-\delta-\zeta_2})\notag\\
    &= O_{d,\P}(d^{- \zeta_2}\log(d))\cdot (\norm{f^*}_{L_2}^2+\sigma_\varepsilon^2).
\end{align}
where the bound $ \snorm{K_{w,\lambda}\inv} \leq 1/(\rho+\lambda -\snorm{\Delta_1}) = O_{d,\P}(1)$ is implied by  the expression \eqref{eqn:K_decomp_cube_m} and $\norm{y}_2^2 = O_{d,\P}(n\log(d))\cdot (\norm{f^*}_{L_2}^2+\sigma_\varepsilon^2)$ follows from Proposition~\ref{prop:bound_y}.

Next we proceed to analyze $T_{11}$, which can be estimated by the proposition below. The proof can be found in Appendix~\ref{proof:t11_est}.
\begin{proposition}\label{prop:t11_est}
The following estimate holds for any $\epsilon>0$:
\begin{align}\label{eq:t11_est}
    &~~~\abs{T_{11} - \E_x\brac{{f^*_{\aslp}(x)}^2} - c\cdot \E_x\brac{{f^*_{\aep}(x)}^2} }\notag\\
    &=  O_{d,\P}(d^{-\zeta_1/2+\epsilon} + d^{-\zeta_2/2+\epsilon} + d^{-2+2\mgamma+\epsilon})\cdot (\norm{f^*}_{L_2}^2+\sigma_\varepsilon^2),
\end{align}
where $\mgamma:= \max_{k\in[N]}\gamma_k$.
\end{proposition}

\noindent Combining Eq.~\eqref{eq:bound_t12} and \eqref{eq:t11_est} completes the proof.

\subsubsection{Proof of Claim~\ref{cl:main_claim_learning_2}}
The claim is directly implied by the following proposition. See the proof in Appendix~\ref{proof:t2_est}.
\begin{proposition}\label{prop:t2_est}
The following estimate holds for any $\epsilon>0$:
\begin{align}\label{eq:t2_est}
    \abs{T_{2} - \E_x\brac{{f^*_{\aslp}(x)}^2} - c\cdot \E_x\brac{{f^*_{\aep}(x)}^2}}= O_{d,\P}(d^{-\zeta_1/2+\epsilon} + d^{-\zeta_2/2+\epsilon})\cdot (\norm{f^*}_{L_2}^2+\sigma_\varepsilon^2).
\end{align}
\end{proposition}

\noindent We have finished the proof for Claim~\ref{cl:main_claim_learning_1} and \ref{cl:main_claim_learning_2} thus completing the proof.

%% file: sections/appendix_B.tex
\input{sections_appendix_b/proof_prelim}

\input{sections_appendix_b/prop_gauss}
\input{sections_appendix_b/prop_hypercube}

\input{sections_appendix_b/prop_gradient}

\input{sections_appendix_b/prop_gradient_rkhs}

\input{sections_appendix_b/prop_generalization}

%% file: sections_appendix_b/proof_prelim.tex
\section{Proof of propositions for main results}
\subsection{Missing proofs for Section~\ref{sec:iso_poly} and~\ref{sec:kernel_approx}}

\subsubsection{Proof of Lemma~\ref{lem:norm_control} and  Theorem~\ref{lemma:phi_id_2}.}\label{proof:phi_id_2}

We will prove Lemmas~\ref{lem:norm_control} and \ref{lemma:phi_id_2} simultaneously. To this end, following the setting of Lemma~\ref{lem:norm_control} consider the regime $n = d^{p+\delta}$ where $\delta\in(0,1)$ is a constant. Fix a set $\S\subseteq\bar{\S}$ indexing polynomials of degree at most $l$. Note that the rows of $\Phi_\S$ are independent isotropic random vectors in $\R^n$. Therefore, applying~\cite[Theorem 5.45]{vershynin2010introduction} yields the estimate
\begin{align}
    \E\brac{ \snorm{\tfrac{1}{n}\Phi_{\S}\tran \Phi_{\S} - I_{|\S|}} } \leq \frac{C_1\sqrt{m \log \min(|\S|,n)}}{\sqrt{n}},\label{eq:phi_id_S}
\end{align}
where we set $m := \E\brac{\max_{i\leq n}\norm{\phi_\S( x^{(i)})}^2}$ and $C_1>0$ is an absolute constant.

We next upper bound $m$. To this end, to simplify notation we let $\phi_{i,\S}$ be shorthand for $\phi_\S( x^{(i)})$.
Now, for any $q>1$, we successively estimate
\begin{align}
     \frac{m}{|S|} &= \frac{1}{|S|}\E\brac{\max_{i\leq n}\norm{\phi_{i,\S}}^2}\nonumber\\
     &\leq \E\brac{\max_{i\leq n} \max_{j\in\S } \round{\phi_{i,j}}^2}\label{eqn:sum_1} \\
     &\leq \left(\E\brac{\max_{i\leq n} \max_{j\in\S } (\phi_{i,j})^{2q}}\right)^{1/q}\label{eqn:lyap}\\
     &\leq \left(\sum_{i\leq n,~j\in\S}  \E(\phi_{i,j})^{2q}\right)^{1/q}\label{eqn:sum_2}\\
     &\leq (n|S|)^{1/q} \max_{i\leq n,~j\in\S}\E\brac{ (\phi_{i,j})^{2q}}^{1/q},\label{eq:phi_phi_bound_1}
\end{align}
where \eqref{eqn:sum_1} and \eqref{eqn:sum_2} follow from bounding the maximum by a sum and \eqref{eqn:lyap} follows from Lyapunov’s Inequality.
The hypercontractivity Assumption~\ref{ass:hypercontr}  implies
\begin{align}\label{eq:phi_phi_bound_2}
    \E\brac{ (\phi_{i,j})^{2q}}^{1/q} \leq C_{2q,l}^2\E\brac{ (\phi_{i,j})^2}  =  C_{2q,l}^2.
\end{align}
Consequently, plugging Eq.~(\ref{eq:phi_phi_bound_2}) into Eq.~(\ref{eq:phi_phi_bound_1}) yields
\begin{align}\label{eq:phi_phi_m}
    m \leq n^{1/q}|\S|^{1+1/q} C_{2q,l}^2.
\end{align}
Plugging Eq.~(\ref{eq:phi_phi_m}) into Eq.~(\ref{eq:phi_id_S}), we obtain
\begin{align}
     \E\brac{ \snorm{\tfrac{1}{n}\Phi_{\S}\tran \Phi_{\S} - I_{|\S|}} } &\leq   \frac{C_1C_{2q,l}\sqrt{n^{1/q}|\S|^{1+1/q} \log \min(|S|,n)}}{\sqrt{n}}\label{eq:phi_phi_final_bound}.
\end{align}
In particular, we deduce the bound on the operator norm
$$\E\|\Phi_{\S}\|^2_{\rm op}\leq n+C_1C_{2q,l}\sqrt{(n|\S|)^{1+1/q} \log \min(|S|,n)}$$
Let us look at two cases. 

\noindent{\bf Case 1: large $|\S|$.}
Suppose that $|\S|\geq C d^{p+\delta_0}$ with $\delta_0>\delta$. Then simple arithmetic shows that there exists $\epsilon>0$ such that the inequality $(n|\S|)^{1+1/q}\leq C^{-1}|\S|^{2-2\epsilon}$ holds for all large $q$. Therefore, we deduce $$\E\|\Phi_{\S}\|^2_{\rm op}\lesssim n+ |\S|^{1-\epsilon}\sqrt{\log(n)}\lesssim |\S|d^{-(\delta_0-\delta)}+|\S|^{1-\epsilon}\sqrt{\log(n)}.$$
Markov's inequality therefore ensures $\|\Phi_{\S}\|_{\rm op}=O_{d,\mathbb{P}}(\sqrt{|\S|})$ as claimed in item~\ref{it:bigS} of Lemma~\ref{lem:norm_control}.

\noindent{\bf Case 2: small $|\S|$.} Suppose that $|\S| \leq Cd^{p+\delta_0}$ with $\delta_0< \delta$. Then from \eqref{eq:phi_phi_final_bound}, we obtain
\begin{align}
     \E\brac{ \snorm{\tfrac{1}{n}\Phi_{\S}\tran \Phi_{\S} - I_{|\S|}} }
     &= C_1 C_{2q,l}C^{\frac{1}{2}+\frac{1}{2q}} d^{-(\frac{\delta-\delta_0}{2}-\frac{2p+\delta+\delta_0}{2q})} \sqrt{\log \min(Cd^p,n)}\label{eq:phi_phi_final_bound2}.
\end{align}
Therefore for any $\epsilon>0$, by choosing $\delta$ sufficiently large, we may ensure the estimate: 
$$\E\brac{ \snorm{\Phi_{\S}\tran \Phi_{\S}/ n - I_{|\S|}} } \lesssim d^{-(\delta-\delta_0-\epsilon)/2}\sqrt{\log (n)}.$$ Item~\ref{it:smallS} of Lemma~\ref{lem:norm_control} now follows directly from Markov's inequality and the triangle inequality. Similarly, setting $\delta'=\delta-\delta_0-\epsilon$ and applying Markov's inequality, then absorbing the logarithmic factors into $\epsilon$ concludes the proof of Lemma~\ref{lemma:phi_id_2}.

\subsubsection{Proof of Lemma~\ref{lemma:gram_matrix}}\label{proof:gram_matrix}

 In the proof, we omit the subscript of $\Phi$ to simplify notation and set $\phi_i = \phi( x^{(i)})$. Define the off-diagonal entries of $\tfrac{1}{|\S|}\Phi   \Phi\tran$ by
\begin{align*}
 \Delta_{i,j}
 :=\begin{cases}
      &\frac{1}{|\S|}\phi_i\tran \phi_j,~~\text{if}~ i\neq j,\\
      &0, ~~\text{otherwise}.
 \end{cases}
\end{align*}
We now split $\tfrac{1}{|\S|}\Phi   \Phi\tran$ into the diagonal and off-diagonal parts. Namely, the triangle inequality yields 
\begin{align}\label{eq:g_m_triangle}
 \E\brac{\snorm{ \tfrac{1}{|\S|}\Phi   \Phi\tran  
 - I_n}} \leq \E\brac{\snorm{\tfrac{1}{|\S|}\Diag(\Phi   \Phi\tran)  
 - I_n}}+ \E\brac{\snorm{\Delta}}.
\end{align}
We will show that both terms on the RHS of Eq.~(\ref{eq:g_m_triangle}) are small. 

\paragraph{Bounding $\E\brac{\snorm{{\tfrac{1}{|\S|}}\Diag(\Phi   \Phi\tran)  
 - I_n}}$.} Using Lyapunov's inequality for any $q\geq 1$, we deduce
\begin{align}\label{eq:g_m_bound_diag_main}
    \E\brac{\snorm{\tfrac{1}{|\S|}\Diag(\Phi   \Phi\tran)  
 - I_n}} &=   \E\brac{\max_i\abs{ \|\phi_i\|^2 /|\S| -1}}\nonumber\\
 &\leq \round{\E\brac{\max_i\round{ \|\phi_i\|^2 /|\S| -1}^{2q}}}^{1/(2q)}\nonumber\\
 &\leq \round{\sum_{i=1}^n\E\brac{\round{ \|\phi_i\|^2/|\S| -1}^{2q}}}^{1/(2q)} \nonumber\\
 &= n^{1/(2q)} \round{\E\brac{\round{ \|\phi_i\|^2/|\S| -1}^{2q}}}^{1/(2q)}.
\end{align}
The hypercontractivity Assumption~\ref{ass:hypercontr} then implies
\begin{align}\label{eq:g_m_bound_hyper}
    \round{\E\brac{\round{ \|\phi_i\|^2 /|\S| -1}^{2q}}}^{1/(2q)} \leq C_{2l,2q} \round{\E\brac{\round{ \|\phi_i\|^2 /|\S| -1}^{2}}}^{1/2}.
\end{align}
Taking into account the equality $\E\brac{(\phi_i)_s (\phi_i)_t} = \delta_{st}$, we deduce
\begin{align*}
  \E\brac{\round{ \|\phi_i\|^2 /|\S| -1}^{2}} &= \E\brac{ \|\phi_i\|^4}/|\S|^2 - 1 \nonumber\\
  &= \frac{1}{|\S|^2}\round{|\S|\cdot \E\brac{(\phi_i)_s^4} + |\S| (|\S|-1)} - 1 \nonumber\\
  &=\round{\E\brac{(\phi_i)_s^4}-1} / |\S|.
\end{align*}
The hypercontractivity Assumption~\ref{ass:hypercontr} again implies
\begin{align*}
    \E\brac{(\phi_i)_s^4} \leq C_{l,4}^{4}\cdot \E[(\phi_i)_s^2]^2 = C_{l,4}^{4}. 
\end{align*}
Therefore, we conclude
\begin{align}\label{eq:g_m_bound_2}
     \E\brac{\round{ \|\phi_i\|^2/|\S| -1}^{2}} \leq \round{C_{l,4}^{4}- 1}/ |\S|.
\end{align}
Plugging Eq.~(\ref{eq:g_m_bound_2}) into Eq.~(\ref{eq:g_m_bound_hyper}) we obtain
\begin{align}\label{eq:g_m_bound_1}
      \round{\E\brac{\round{ \|\phi_i\|^2 /|\S| -1}^{2q}}}^{(1/2q)} \leq C_{2l,2q} \round{\round{C_{l,4}^{4}- 1}/ |\S|}^{1/2}.
\end{align}
Plugging Eq.~(\ref{eq:g_m_bound_1}) into Eq.~(\ref{eq:g_m_bound_diag_main}) yields the estimates
\begin{align}
     \E\brac{\snorm{\tfrac{1}{|\S|}\Diag(\Phi   \Phi\tran)  
  - I_n}} &\leq n^{1/(2q)} C_{2l,2q} \round{\round{C_{l,4}^{4}- 1}/ |\S|}^{1/2}\nonumber\\
 &\leq C_{2l,2q}(C_{l,4}^{4}- 1)^{1/2} C^{-1/2} d^{\frac{p+\delta}{2q} -\frac{p+\delta}{2}},\label{eq:g_m_diag_bound}
\end{align}
where we use the fact that $|\S| \geq C d^{p+\delta_0}$.

\paragraph{Bound $\E\brac{\snorm{\Delta}}$.} Now we consider the second term in Eq.~(\ref{eq:g_m_triangle}). We first apply the matrix decoupling from~\cite{vershynin2010introduction} in Lemma 5.60. For $T_1,T_1 \subseteq [n]$, we define $\Delta_{T_1,T_2} = \round{\Delta_{i,j}}_{i\in T_1, j\in T_2}$. By Lemma~\ref{lemma:m_decouple},
 we have
 \begin{align*}
     \E\brac{\snorm{\Delta}} \leq 4 \sup_{T\subseteq[n]}\E\brac{\snorm{\Delta_{T,T^c}}},
 \end{align*}
 where $T^c$ denotes the complement of $T$.

For any set $J\subset[n]$, we let $\E_J$ denote the expectation with respect to $\{ x^{(i)}\}_{i\in J}$ and conditionally on $\{ x^{(i)}\}_{i\in J^c}$. We fix $T \in [n]$ and apply Lemma~\ref{lemma:m_norm} conditionally on $\{ x^{(i)}\}_{i\in T^c}$, thereby obtaining
\begin{align}\label{eq:g_m_off_diag_bound_12}
    \E_T \brac{\snorm{\Delta_{T,T^c}}} \leq \round{\Sigma(T) \cdot n}^{1/2} + c \cdot \round{\Gamma(T) \cdot \log(n)}^{1/2},
\end{align}
where $\Sigma(T) := \snorm{\E_{ x^{(j)}}\brac{\Delta_{T^c,j}\Delta_{T^c,j}^\top}}$ for some $j \in T$ and $\Gamma(T):= \E_T \brac{\max_{j\in T}\norm{\Delta_{j,T^c}}^2}$.

Using the tower rule, We split $\E$ into $\E_{T^c}\E_{T}$ and deduce
\begin{align}
    \E\brac{\snorm{\Delta}} &\leq 4 \sup_{T\subseteq[n]}\E\brac{\snorm{\Delta_{T,T^c}}}\notag \\
    &= 4 \sup_{T\subseteq[n]}\E_{T^c}\E_T \brac{\snorm{\Delta_{T,T^c}}}\notag\\
    &\leq 4 \sup_{T\subseteq [n]} \E_{T^c}\brac{\round{\Sigma(T) \cdot n}^{1/2}} + c  \cdot4 \sup_{T\subseteq [n]} \E_{T^c} \brac{\round{\Gamma(T) \cdot \log(n)}^{1/2}}\label{eqn:last_plugin_decouple}\\
    &\leq 4 n^{1/2}\sup_{T\subseteq [n]} \round{\E_{T^c}\brac{\Sigma(T)}}^{1/2} + 4c\log^{1/2}(n)  \cdot \sup_{T\subseteq [n]} \round{\E_{T^c}\brac{\Gamma(T)}}^{1/2}.\label{eq:g_m_off_diag_bound_main}
\end{align}
where \eqref{eqn:last_plugin_decouple} follows from \eqref{eq:g_m_off_diag_bound_12} and \eqref{eq:g_m_off_diag_bound_main} follows from
Jensen's inequality. 

We next  bound separately the two terms on the right side of \eqref{eq:g_m_off_diag_bound_main}.

\paragraph{Step 1: Bound $\E_{T^c}\brac{\Sigma(T)}$.}
Note that each entry of $\Delta_{T^c,j}\Delta_{T^c,j}^\top$ takes the form $\phi_i\tran \phi_j \phi_j\tran \phi_l / |\S|^2$ for $ i,l\in T^c, j\in T$. Since $\E_{ x^{(j)}}\brac{\phi_j \phi_j\tran} = I_{|\S|}$, we have each entry of $\E_{ x^{(j)}}\brac{\Delta_{T^c,j}\Delta_{T^c,j}^\top}$ takes the form $\phi_i\tran \phi_l /|\S|^2$. Therefore, we deduce
\begin{align*}
        \Sigma(T) = \snorm{\E_{ x^{(j)}}\brac{\Delta_{T^c,j}\Delta_{T^c,j}^\top}} = \snorm{\round{\Phi\Phi\tran}_{T^c,T^c} }/ |\S|^2 \leq \E_T\brac{\snorm{\Phi\Phi\tran }}/ |\S|^2,
\end{align*}
where in the last inequality, we use Cauchy's interlacing theorem~\cite{horn2012matrix}. Specifically, the theorem shows that adding a row and a column to a Hermitian matrix will only increase its operator norm. 

Note that splitting $\Phi\Phi^\top$ into its diagonal and off-diagonal parts yields the estimate 
\begin{align*}
  \E_T\brac{\snorm{\Phi\Phi\tran }}/ |\S|^2 \leq   \E_T\brac{\snorm{\Delta}}/|\S| + \E_T\brac{\max_{j\in [n]} \norm{\phi_j}^2} /|\S|^2.
\end{align*}
Taking the expectation and we have
\begin{align}\label{eq:g_m_off_diag_bound_2}
    \E_{T^c}\brac{\Sigma(T)} \leq \E\brac{\snorm{\Delta}}/|\S| + \E\brac{\max_{j\in [n]} \norm{\phi_j}^2} /|\S|^2.
\end{align}
We control the second term  using Lyapunov's inequality: for any $q\geq 1$, we have
\begin{align}\label{eq:g_m_off_diag_lyap}
    \E\brac{\max_{j\in [n]}\norm{\phi_j}^2}  &\leq  |\S|\cdot\E\brac{\max_{i\in \S}\max_{j\in [n]}(\phi_j)_i^2} \notag\\
    &\leq |\S|\cdot \round{\E\brac{\max_{i\in \S, j\in[n]}(\phi_j)_i^{2q}}}^{1/q}\notag\\
    &\leq |\S|\cdot\round{\sum_{i\in \S}\sum_{j\in [n]}\E\brac{\round{\phi_j}_i^{2q}}}^{1/q}\notag\\
    &= n^{1/q} |\S|^{1+1/q} \round{\E\brac{\round{\phi_j}_i^{2q}}}^{1/q}.
\end{align}
Using the hypercontractivity Assumption~\ref{ass:hypercontr}, we obtain
\begin{align*}
    \round{\E\brac{\round{\phi_j}_i^{2q}}}^{1/q} \leq C_{l,2q}^2\round{\E\brac{\round{\phi_j}_i^{2}}} = C_{l,2q}^2.
\end{align*}
Consequently, we deduce
\begin{align}\label{eq:g_m_off_diag_bound_1.5}
    \E\brac{\max_{j\in [n]}\norm{\phi_j}^2} \leq C_{l,2q}^2 |\S|^{1+1/q} n^{1/q}.
\end{align}
Plugging this estimate into Eq.~(\ref{eq:g_m_off_diag_bound_2}) yields
\begin{align}\label{eq:g_m_off_diag_bound_4}
     \E_{T^c}\brac{\Sigma(T)} \leq \E\brac{\snorm{\Delta}}/|\S| + C_{l,2q}^2 |\S|^{(1/q)-1} n^{1/q}.
\end{align}

\paragraph{Step 2: Bound $\E_{T^c}\brac{\Gamma(T)}$.}
Using Lyapunov's inequality, for any $q\geq1$,  we have
\begin{align}\label{eq:g_m_off_diag_bound_3}
    \E_{T^c}\E_T \brac{\max_{i\in T}\norm{\Delta_{i,T^c}}^2} &= \E \brac{\max_{i\in T}\norm{\Delta_{i,T^c}}^2}\nonumber\\
    &\leq n\E \brac{\max_{i\in T}\max_{ j\in T^c} \Delta_{i,j}^2}\nonumber\\
    &\leq n \E \brac{\max_{i\in T}\max_{ j\in T^c} \Delta_{i,j}^{2q}}^{1/q}\nonumber\\
    &\leq n \round{\sum_{i\in T}\sum_{j\in T^c}\E\brac{\Delta_{i,j}^{2q}}}^{1/q}\nonumber\\
    &\leq n^{1+2/q} \E\brac{\Delta_{i,j}^{2q}}^{1/q}.
\end{align}
Since $\Delta_{i,j}$ is a degree $2l$ polynomial, the hypercontractivity Assumption~\ref{ass:hypercontr} implies
\begin{align}\label{eq:g_m_off_diag_bound_3.5}
    \E\brac{\Delta_{i,j}^{2q}}^{1/q} \leq C_{2l,2q}^2\E\brac{\Delta_{i,j}^{2}} = \frac{C_{2l,2q}^2}{|\S|}.
\end{align}
Consequently, plugging the above inequality into Eq.~(\ref{eq:g_m_off_diag_bound_3}) we have
\begin{align}\label{eq:g_m_off_diag_bound_5}
    \E_{T^c}\brac{\Gamma(T)} \leq \frac{n^{1+2/q} C_{2l,2q}^2}{|\S|}.
\end{align}
Combining Eq.~(\ref{eq:g_m_off_diag_bound_4}) and Eq.~(\ref{eq:g_m_off_diag_bound_5}) with Eq.~(\ref{eq:g_m_off_diag_bound_main}) we obtain:
\begin{align*}
   \E\brac{\snorm{\Delta}}  &\leq 4 \round{\frac{n}{|\S|}}^{1/2}\round{\E\brac{\snorm{\Delta}}+C_{l,2q}^2 |\S|^{1/q} n^{1/q}}^{1/2} + 4cC_{2l,2q}^2\frac{ n^{\frac{1}{2}+\frac{1}{q}}\log^{1/2}(n)  }{\sqrt{|\S|}}\\
   &\leq  4 \round{\frac{n}{|\S|}}^{\frac{1}{2}}\round{\E\brac{\snorm{\Delta}}}^{\frac{1}{2}} +  4 C_{l,2q} n^{\frac{1}{2} +\frac{1}{2q}} |\S|^{\frac{1}{2q}-\frac{1}{2}}+ 4cC_{2l,2q}^2\frac{ n^{\frac{1}{2}+\frac{1}{q}}\log^{\frac{1}{2}}(n)  }{\sqrt{|\S|}},
\end{align*}
where the last inequality is obtained from the elementary estimate $\sqrt{a+b}\leq \sqrt{a}+\sqrt{b}$ for $a,b\geq 0$.
Solving the quadratic inequality with respect to $\sqrt{\E\brac{\snorm{\Delta}}} $ we obtain:
\begin{align*}
     \sqrt{\E\brac{\snorm{\Delta}}} \leq 2 \round{\frac{n}{|\S|}}^{1/2} + 2\sqrt{ \frac{n}{|\S|} +  C_{l,2q} n^{\frac{1}{2} +\frac{1}{2q}} |\S|^{\frac{1}{2q}-\frac{1}{2}}+ cC_{2l,2q}^2\frac{ n^{\frac{1}{2}+\frac{1}{q}}\log^{1/2}(n)  }{\sqrt{|\S|}} }.
\end{align*}
Squaring both sides and using the inequality $(a+b)^2\leq 2a^2+2b^2$ we conclude
$$\frac{1}{8}\E\brac{\snorm{\Delta}}\leq \frac{2n}{|\S|}+  C_{l,2q} n^{\frac{1}{2}+ \frac{1}{2q}} |\S|^{\frac{1}{2q}-\frac{1}{2}}+ cC_{2l,2q}^2\frac{ n^{\frac{1}{2}+\frac{1}{q}}\log^{1/2}(n)  }{\sqrt{|\S|}}$$
Taking into account that $|\S|\gtrsim d^{p+\delta_0}$ and $n=d^{p+\delta}$, we therefore conclude
$$\E\brac{\snorm{\Delta}}\lesssim d^{\delta-\delta_0}+C_{l,2q} d^{\frac{\delta-\delta_0}{2}+\frac{2p+\delta+\delta_0}{2q}}+C_{l,2q}^2\sqrt{\log(n)}d^{\frac{\delta-\delta_0}{2}+\frac{p+\delta}{q}}.$$
Combining this estimate with \eqref{eq:g_m_triangle} and \eqref{eq:g_m_diag_bound} we conclude
\begin{align*}
\E\brac{\snorm{ \Phi   \Phi\tran  
 /|\S|- I_n}} &\lesssim C_{2l,2q}(C_{l,4}^{4}- 1)^{1/2}  d^{\frac{p+\delta}{2q} -\frac{p+\delta}{2}}+ d^{\delta-\delta_0}+C_{2l,2q} d^{\frac{\delta-\delta_0}{2}+\frac{2p+\delta+\delta_0}{2q}}\\
 &+C_{2l,2q}^2\sqrt{\log(n)}d^{\frac{\delta-\delta_0}{2}+\frac{p+\delta}{q}}.
\end{align*}
Clearly, for any $\epsilon>0$ we may choose $q$ sufficiently large so that the right side is bounded by a constant multiple of 
\begin{equation}\label{eqn:final_bound_assympt}
d^{-\frac{p+\delta}{2}+2\epsilon}+\sqrt{\log(n)}\cdot d^{-\frac{\delta_0-\delta}{2}+\epsilon},
\end{equation}
thereby completing the proof of \eqref{eqn:bound_needed}. In the case, $p\geq 1$, $\delta_0=1$, and $\epsilon\in (0,\delta)$ the first term in \eqref{eqn:final_bound_assympt} is negligible compared to the second; applying Markov's inequality  and absorbing logarithmic factors into $\epsilon$ completes the proof.

\subsubsection{Proof of Theorem~\ref{thm:poly_approx}}\label{proof:poly_approx}
First, applying Lemma~\ref{lem:angle_app} we deduce that for any $s\in (0,1)$, the estimate 
\begin{equation}\label{eqn:angle}
        \max_{i,j}\abs{\tfrac{1}{d}\inner{ x^{(i)}, x^{(j)}} - \delta_{ij}} \leq s\sigma^2,
\end{equation}
holds with probability at least $1-2n^2\exp(-c_1s^2 d)$, where $c_1>0$ is some constant.
Similarly, using Lemma~\ref{lem:norms_app} and taking a union bound over $i=1,\ldots, n$ we deduce that for any $t\in (0,\sigma^{-2})$ the estimate 
\begin{equation}\label{eqn:norm_control}
\left|\tfrac{1}{d^q}\| x^{(i)}\|^{2q}_2-1 \right|\leq q 2^{2q}t\sigma^2 \qquad\forall q\geq 1,
\end{equation}
holds with probability at least $1-2n\exp(-c_2dt^2)$. 
We assume that the events \eqref{eqn:angle} and \eqref{eqn:norm_control} occur for the remainder of the proof.

To simplify notation, let $\bar{c} := g_{m}(1) = \sum_{k=0}^{m} \frac{g^{(k)}(0)}{k!}$ and define the matrix $E(X):=K(X,X)-K_{m}(X,X)$. Then splitting $E(X)$ into its diagonal and off-diagonal parts, we may estimate
\begin{align}
\|E(X)- (  g(1) - {\bar c})   I_n\|_{\rm op}&= \|E(X)-\Diag(E(X))+\Diag(E(X))- (  g(1) - {\bar c})  I_n\|_{\rm op}\notag\\
&\leq \|E(X)-\Diag(E(X))\|_{\rm op}+\|\Diag(E(X))-(  g(1) - {\bar c})  I_n\|_{\rm op}\notag\\
&\leq \|E(X)-\Diag(E(X))\|_{F}+\max_{i=1,\ldots,n}|E(X)_{ii}-(  g(1) - {\bar c})  I_n|,\label{eqn:decomp_kern}
\end{align}
where the first inequality follows from the triangle inequality and the second follows the norm comparison $\|\cdot\|_{\rm op}\leq \|\cdot\|_{F}$.

We consider the off-diagonal entries first. By Assumption~\ref{assump:g_0},  we can write
\begin{align}
    g(t) = \sum_{m=0}^\infty \frac{g^{(m)}(0)}{m!} t^m,~~~t\in(-\epsilon,\epsilon).
\end{align}
Setting $s<\frac{\min(1/2,\epsilon)}{\sigma^2}$, we successively estimate
\begin{align}
     \|E(X)-\Diag(E(X))\|_{F}^2&=\sum_{i\neq j} \round{g(x^{(i)},x^{(j)}) - g_{m}(x^{(i)},x^{(j)})}^2\notag\\
     &\leq n^2 \sum_{k = m+1}^\infty \left(\tfrac{g^{(k)}(0)}{(k)!}\right)^2 \max_{i\neq j}\abs{\inner{x^{(i)},x^{(j)}} \nicefrac{}{} d}^{2k}\label{eqn:first_eq_needed}\\
     &\leq \frac{n^2\sup_{k\geq m+1}g^{(k)}(0)}{(m+1)!}  \sum_{k=m+1}^\infty (s\sigma^2)^{2k}\label{eqn:second_eq_needed}\\
     &\leq \frac{n^2\sup_{k\geq m+1}g^{(k)}(0)}{(m+1)!} \cdot \frac{(s\sigma^2)^{2m+2}}{1-1/4}\notag\\
     &\leq  \frac{4n^2\sup_{k\geq m+1}g^{(k)}(0) (s\sigma)^{2m+2}}{3(m+1)!}.\label{eqn:third_eq_needed}
\end{align}
where\eqref{eqn:first_eq_needed} follows from \eqref{eqn:angle} and \eqref{eqn:second_eq_needed} follows from the geometric series formula.

Moving on to the diagonal entries, we compute:
\begin{align}
|E(X)_{ii}-(  g(1) - {\bar c}) I_n|&=\abs{g\round{\tfrac{1}{d}\norm{ x^{(i)}}^2 } - g(1) + \round{{\bar c} - g_{m}\round{\tfrac{1}{d}\norm{ x^{(i)}}^2}  }}\notag\\
    &\leq \abs{g\round{\tfrac{1}{d}\norm{ x^{(i)}}^2} - g(1)} + \abs{\bar c - g_{m}\round{\tfrac{1}{d}\norm{ x^{(i)}}^2 }}\notag\\
    &\leq  L_0\abs{\tfrac{1}{d}\norm{ x^{(i)}}^2 -1}+\sum_{k=1}^{m}\frac{g^{(k)}(0) }{k!}\left|\tfrac{1}{d^k}\norm{ x^{(i)}}^{2k} -1\right|\label{eqn:lip_g}\\
    &\leq L_0s\sigma^2+ t\sigma^2\sum_{k=1}^{m}\frac{g^{(k)}(0)2^{2k}}{(k-1)!} k , \label{eqn:lip_g_fin}
\end{align}
where \eqref{eqn:lip_g} follows from Lipschitz continuity of $g$ with $L_0:= \sup_{t\in(1-\epsilon,1+\epsilon)}|g'(t)|$ and \eqref{eqn:lip_g_fin} follows from \eqref{eqn:norm_control}.

Therefore combining \eqref{eqn:decomp_kern}, \eqref{eqn:third_eq_needed}, and \eqref{eqn:lip_g_fin} we obtain the estimate
\begin{align}\label{eq:basic_prev_plug}
    \|E(X)-(g(1)-\bar c)I_n\|_{\rm op} &\leq \sqrt{\frac{4n^2\sup_{k\geq m+1}g^{(k)}(0) }{3(m+1)!}}(s\sigma)^{m+1}+L_0s\sigma^2+\sum_{k=1}^{m}\frac{g^{(k)}(0)2^{2k} }{(k-1)!} t\sigma^2.
\end{align}
Defining the constant $C_{g,m}:= \max\round{\sqrt{\frac{4\sup_{k\geq m+1}g^{(k)}(0) }{3(m+1)!}}, \sum_{k=1}^{m}\frac{g^{(k)}(0) 2^{2k}}{(k-1)!},1} < \infty$ and we therefore have the bound
\begin{equation}\label{eqn:basic_plugin}
\|E(X)-(g(1)-\bar c)I_n\|_{\rm op}\leq C_{g,m}\cdot( n (s\sigma^2)^{m+1}+L_0s\sigma^2+t\sigma^2).
\end{equation}
Then for any $C>2$, we may set $s:=\sqrt{C\log(n)/c_1d}$ and $t:=\sqrt{(C-1)\log(n)/c_2d}$ thereby ensuring that the events \eqref{eqn:angle} and \eqref{eqn:norm_control} occur with probability at least $1-\frac{4}{n^{C-2}}$. Plugging these values into \eqref{eqn:basic_plugin} yields the estimate 
$$\|E(X)-(g(1)-\bar c)I_n\|_{\rm op}\leq c\cdot C_{g,m}\left(C^{\frac{m+1}{2}}d^{\frac{2p-m-1}{2}+\delta}\log^{\frac{m+1}{2}}(n)\sigma^{2m+2}+(L_0+1)\sigma^2\sqrt{\frac{C\log(n)}{d}}\right),$$
where $c$ is a numerical constant. 

To meet the assumption $s<\frac{\min(1/2,\epsilon)}{\sigma^2}$, $d$ needs to satisfy $d/\log(d)> \frac{\min(1/2,\epsilon)^2 c_1}{C(p+\delta)\sigma^4}:=c_\epsilon$. 
The proof is complete.

\subsubsection{Proof of Lemma~\ref{lemma:each_mono_cube}}\label{proof:each_mono_cube}

Let $C_d(k)$ consists of all vectors $\lambda\in\mathbb{N}^d$ satisfying $\lambda\geq 0$ and $\sum_{i=1}^d\lambda_i=k$. Then the multinomial expansion of $(x^\top z)^k$ for any $x,y\in\R^d$ yields the expression
\begin{equation}\label{eqn:multinom_expans}
    \frac{1}{k!}\round{\frac{x\tran z}{d}}^k = \frac{1}{d^k} \sum_{\lambda\in C_d(k)} \frac{1}{\lambda!}x^\lambda z^\lambda.
\end{equation}
Now if $x$ and $y$ lie in the hypercube $\mathbb{H}^d$, equality  $x^{\lambda}=x^{\lambda\,{\rm mod}\, 2}$ holds. Thus we may now rewrite \eqref{eqn:multinom_expans} as 

$$\sum_{\lambda\in C_d(k)} \frac{1}{\lambda!}x^\lambda z^\lambda=\sum_{\lambda\in \{0,1\}^d} D_{\lambda} x^{\lambda}y^{\lambda},$$
where we define the coefficients 
\begin{equation}\label{eqn:defining_Dl}
D_{\lambda}=\sum_{\substack{\alpha \in C_d(k): \\ \lambda = \alpha\, {\rm mod}\,2}} \frac{1}{\alpha!}.
\end{equation}
In particular, it is clear from \eqref{eqn:defining_Dl} that if $k-|\lambda|$ is odd or strictly negative then equality $D_{\lambda}=0$ holds. Moreover, for any $\lambda\in \{0,1\}^d$ with $|\lambda|=k$, the only summand on the right-side of \eqref{eqn:defining_Dl} corresponds to $\alpha=\lambda$ and therefore equality $D_{\lambda}=\frac{1}{\lambda!}=1$ holds. More generally, for any $\lambda \in \{0,1\}^d$ such that  $k-|\lambda|$ is even and nonnegative, the number of terms in the sum in \eqref{eqn:defining_Dl} is exactly the cardinality of the set $C_d(\frac{k-|\lambda|}{2})$, which is ${\frac{k-|\lambda|}{2}+d-1\choose d-1}=\Theta(d^{(k-|\lambda|)/2})$. Dividing $D_{\lambda}$ by $d^k$ completes the proof.

\subsubsection{Proof of Lemma~\ref{lemma:kernel_to_mono}}\label{proof:kernel_to_mono}

We first present a proof for Eq.~\eqref{eq:kernel_to_mono}.
To this end, Theorem~\ref{thm:poly_approx_cube} with $m=2p+1$ yields the estimate: 
\begin{align}\label{eq:mono_taylor_approx_bound1}
    \snorm{K - K_{2p+1} - (g(1) - g_{2p+1}(1))I_n} = O_{d,\P}\left(\frac{\log^{p+1}(n)}{d^{1-\delta}}\right).
\end{align}
Invoking Lemma~\ref{lemma:each_mono_cube} for each $k$, we may write:
\begin{align*}
  \frac{1}{k!}\round{\frac{XX\tran}{d}}^{\odot k} =  d^{-k}\cdot \Phi_{\S_k}\Phi_{\S_k}\tran + \sum_{ \substack{j:\, 0\leq j<k, \\ k-j~\mathrm{is~even}}}\Phi_{\S_j} \widetilde{D}_{\S_j}\Phi_{\S_j}\tran,
\end{align*}
where each $\widetilde{D}_{\S_j}$ is a diagonal matrix with all entries on the order of $\Theta(d^{-(j+k)/2})=O(d^{-j-1})$.
Since $k-j$ is even and positive we deduce $k\geq j+2$. Therefore every entry of $\widetilde{D}_{\S_j}$ is bounded by $O(d^{-j-1})$.
Thus, merging $\Phi_{\S_j} \widetilde{D}_{\S_j}\Phi_{\S_j}\tran$ with the lower-order bases, we can write
\begin{align}\label{eqn:K_expression_2p_plus_1}
    K_{2p+1} = \Phi_{\leq 2p+1} D_{\leq 2p+1} \Phi_{\leq 2p+1}\tran,
\end{align}
where $D$ is a diagonal matrix satisfying
\begin{equation}\label{eqn:diagonal_expression}
\|D_{\S_k} - g^{(k)}(0)d^{-k}I_{\S_k}\|_{\rm op} = O_d(d^{-k-1}),
\end{equation}
for each $k=0,\ldots, 2p+1$.
We now claim that the terms in the product corresponding to $k=p+1,\ldots, 2p+1$ are close to a multiple of the identity. To see this, using the triangle inequality we estimate 
\begin{align*}
    \snorm{\sum_{k=p+1}^{2p+1}  \Phi_{\S_k}  D_{\S_k} \Phi_{\S_k}\tran - \sum_{k=p+1}^{2p+1} \frac{g^{(k)}(0)}{k!}I_n} \leq \sum_{k=p+1}^{2p+1}\left\|\Phi_{\S_k}  D_{\S_k}\Phi_{\S_k}\tran -\frac{g^{(k)}(0)}{k!}I_n\right\|_{\rm op}.
\end{align*}
Using the triangle inequality with the expression \eqref{eqn:diagonal_expression}, we subsequently deduce
\begin{equation}\label{eqn:main_bound_exprsq}
    \snorm{ \Phi_{\S_k}  D_{\S_k} \Phi_{\S_k}\tran - \frac{g^{(k)}(0)}{k!}\cdot I_{n}} \leq  \frac{g^{(k)}(0)}{k!}\snorm{  \frac{k!}{d^k}\Phi_{\S_k}   \Phi_{\S_k}\tran  - I_{n}} + O_d(d^{-k-1})\cdot    \snorm{ \Phi_{\S_k} }^2.
    \end{equation}
We now bound the two terms separately. To this end, note the estimate $|\S_k|=\frac{d^{k}}{k!}+O_d(d^{k-1})$. Therefore Lemma~\ref{lem:norm_control} implies $\|\Phi_{\S_k}\|^2_{\rm op}=O_{d,\mathbb{P}}(d^k)$ and hence the second term on the right side of \eqref{eqn:main_bound_exprsq} is $O_{d,\mathbb{P}}(d^{-1})$. In order to control the first term, we invoke Theorem~\ref{lemma:gram_matrix} yielding
\begin{align}\label{eq:mono_taylor_approx_bound2}
    \snorm{\tfrac{1}{|\S_k|} \Phi_{\S_k}   \Phi_{\S_k}\tran  
 - I_n} \leq O_{d,\mathbb{P}}\left(d^{-\tfrac{k-p-\delta}{2}+\epsilon}\log(n)\right),
\end{align}
for any $\epsilon\in (0,\delta)$. By the triangle inequality, replacing $|\S_k|$ by $d^k/k!$ in \eqref{eq:mono_taylor_approx_bound2} incurs an error of size
\begin{align}\label{eq:mono_taylor_approx_bound3}
    \left|\tfrac{k!}{d^{k}}-\tfrac{1}{|\S_k|}\right|\cdot \|\Phi_{\S_k}\|^2_{\rm op}= \frac{||\S_k|/d^k-\frac{1}{k!}|}{|S_k|/k!} \|\Phi_{\S_k}\|^2_{\rm op}=O_d(d^{-1}). 
\end{align}
Therefore, we conclude that the first term on the right side of \eqref{eqn:main_bound_exprsq} is on the order of
\begin{align}
    O_{d,\mathbb{P}}\left(d^{-\tfrac{k-p-\delta}{2}+\epsilon}\log(n)  +d\inv\right).
\end{align} 
Combining with the bound for the second term on the right side of \eqref{eqn:main_bound_exprsq}, we obtain
\begin{align}\label{eq:each_taylor_bound}
 \snorm{ \Phi_{\S_k}  D_{\S_k} \Phi_{\S_k}\tran - \frac{g^{(k)}(0)}{k!}\cdot I_{n}} =  O_{d,\mathbb{P}}\left(d^{-\tfrac{k-p-\delta}{2}+\epsilon}\log(n)  +d\inv\right).
 \end{align}
Note that the term with $k = p+1$ will dominate the summation, thus, we deduce 
$$\sum_{k=p+1}^{2p+1}  \Phi_{\S_k}  D_{\S_k} \Phi_{\S_k}\tran=\underbrace{\sum_{k=p+1}^{2p+1} \frac{g^{(k)}(0)}{k!}I_n}_{=g_{2p+1}(1)-g_{p}(1)}+O_{d,\mathbb{P}}\left(d^{-\tfrac{1-\delta}{2}+\epsilon}\log(n)\right).$$
Combining this expression with \eqref{eq:mono_taylor_approx_bound1} completes the proof for Eq.~\eqref{eq:kernel_to_mono}.

Now we proceed to prove Eq.~\eqref{eq:kernel_to_mono_advanced}. To this end, applying Theorem~\ref{thm:poly_approx_cube} with $m=2p+3$ yields the estimate: 
\begin{align}\label{eq:mono_taylor_approx_extra_bound1}
    \snorm{K - K_{2p+3} - (g(1) - g_{2p+3}(1))I_n} = O_{d,\P}\left(\frac{\log^{p+2}(n)}{d^{2-\delta}}\right) = O_{d,\P}(d\inv).
\end{align}
By the same reasoning that led to Eq.~\eqref{eqn:K_expression_2p_plus_1}, we can write $K_{2p+3}$ as 
\begin{align}\label{eqn:K_expression_2p_plus_3}
    K_{2p+3} = \Phi_{\leq 2p+3} D_{\leq 2p+3} \Phi_{\leq 2p+3}\tran.
\end{align}
We can bound the term with $k\geq p+2$ by letting $k = p+2$ in Eq.~\eqref{eq:each_taylor_bound}, which yields
\begin{align}\label{eq:feature_p_plus_2}
    \snorm{  \Phi_{\S_k}  D_{\S_k} \Phi_{\S_k}\tran - \frac{g^{(k)}(0)}{k!}I_n} = O_{d,\mathbb{P}}\left(d^{-\tfrac{2-\delta}{2}+\epsilon}\log(n)\right).
\end{align}
For $k = p+1$, we provide the following approximation:
\begin{align*}
     &~~~\snorm{\Phi_{\S_{p+1}}  D_{\S_{p+1}} \Phi_{\S_{p+1}}\tran - \frac{g^{(p+1)}(0)}{(p+1)!}\cdot I_{n} - \frac{g^{(p+1)}(0)}{d^{p+1}}\offd\round{\Phi_{\S_{p+1}}   \Phi_{\S_{p+1}}\tran  }}\\
     &\overset{(a)}{\leq} \snorm{\frac{g^{(p+1)}(0)}{d^{p+1}}\diag\round{\Phi_{\S_{p+1}}   \Phi_{\S_{p+1}}\tran  } - \frac{g^{(p+1)}(0)}{(p+1)!}\cdot I_{n}} + O_d(d^{-p-2})\cdot \snorm{\Phi_{\S_{p+1}}   \Phi_{\S_{p+1}}\tran }\\
     &\overset{(b)}{=}\snorm{\frac{g^{(p+1)}(0)}{(p+1)! |\S_{p+1}|}\diag\round{\Phi_{\S_{p+1}}   \Phi_{\S_{p+1}}\tran  } - \frac{g^{(p+1)}(0)}{(p+1)!}\cdot I_{n}} + O_d(d\inv) + O_d(d^{-p-2})\cdot \snorm{\Phi_{\S_{p+1}}   \Phi_{\S_{p+1}}\tran }\\
     &\overset{(c)}{=} O_{d,\P}(d\inv),
\end{align*}
where $(a)$ follows from Eq.~\eqref{eqn:diagonal_expression}, $(b)$ follows from Eq.~\eqref{eq:mono_taylor_approx_bound3} and $(c)$ follows from Lemma~\ref{lem:norm_control} and the fact $\diag\round{\Phi_{\S_{p+1}}   \Phi_{\S_{p+1}}\tran }/ |\S_{p+1}| = I_n$.
Combining with Eq.~\eqref{eq:feature_p_plus_2}  we deduce 
$$\snorm{\sum_{k=p+1}^{2p+3}  \Phi_{\S_k}  D_{\S_k} \Phi_{\S_k}\tran - \underbrace{\sum_{k=p+1}^{2p+3} \frac{g^{(k)}(0)}{k!}}_{=g_{2p+3}(1)-g_{p}(1)}I_n - \frac{g^{(p+1)}(0)}{d^{p+1}}\offd\round{\Phi_{\S_{p+1}}   \Phi_{\S_{p+1}}\tran }} = O_{d,\mathbb{P}}\left(d^{-\tfrac{2-\delta}{2}+\epsilon}\log(n)\right).$$
Combining this expression with Eq.~\eqref{eq:mono_taylor_approx_extra_bound1} completes the proof.

\subsubsection{Proof of Lemma~\ref{lem:conv_quad_herm}}\label{sec:proofoflem:conv_quad_herm}
Let us partition the second-order Hermite polynomials $\S_2$ into the pure and mixed parts as 
$$\phi_{\S_2}=\left(\underbrace{\{h_{2e_i}\}_{i=1}^d}_{=:\phi^s_{\S_2}(x)}, \underbrace{\{h_{e_i+e_j}\}_{1\leq i<j\leq d}}_{=:\phi_{\S_2}^m(x)}\right).$$
For any $x,y\in\R^d$, expanding the square yields the expression
 \begin{align*}
    \frac{1}{2}\round{x\tran y}^2 &= \frac{1}{2}\round{\sum_{i=1}^d {x_iy_i}}^2\\
    &=\frac{1}{2}\sum_{i=1}^d x_i^2 y_i^2+\sum_{1\leq i<j\leq d} (x_ix_j)\cdot(y_iy_j)\\
    &=\frac{1}{2}\sum_{i=1}^d (\sqrt{2}h_2(x_i)+1) (\sqrt{2}h_2(y_i)+1) + \sum_{1\leq i< j\leq d}(h_1(x_i) h_1( x_j))\cdot (h_1(y_i) h_1( y_j))\\
    &=\frac{1}{2}\langle \sqrt{2}\phi_{\S_2}^s(x)+h_0(x)\cdot {\bf 1}_d,\sqrt{2}\phi_{\S_2}^s(y)+h_0(y)\cdot {\bf 1}_d\rangle+\langle \phi_{\S_2}^m(x),\phi_{\S_2}^m(y)\rangle\\
    &=\langle \phi_{\S_2}(x),\phi_{\S_2}(y) \rangle+\frac{d}{2}\langle \phi_{\S_0}(x),\phi_{\S_0}(y)
    \rangle+\frac{1}{\sqrt{2}}\langle {\bf 1}_d, \phi_{\S_2}^s(x)+\phi_{\S_2}^s(y)\rangle.
\end{align*}
The equality \eqref{eqn:second_orderequiv} follows immediately.

\subsubsection{Proof of Lemma~\ref{lemma:hermite_decomposition_p=3}}\label{proof:hermite_decomposition_p=3}
For now, let $k\in \mathbb{N}$ be arbitrary. Observe that since $\round{x\tran y}^k$ is a degree $k$ polynomial in $x$ and in $y$, it can be expressed in terms of Hermite polynomials:
\begin{align}\label{eq:mono_to_hermite}
  \left(x^\top y\right)^k=\sum_{\substack{
    \alpha,\,\beta \in \mathbb{N}^d \\
    \|\alpha\|_1 \le k \\
    \|\beta\|_1 \le k
}} C[\alpha,\beta] h_{\alpha}(x)h_{\beta}(y),  
\end{align}
for some constants $C[\alpha,\beta]$. In particular, We may regard $C[\cdot, \cdot]$ as a matrix satisfying 
$$\frac{1}{k!}\round{\frac{XX\tran}{d}}^{\odot k}=\frac{1}{k!d^k}\Phi C  \Phi^\top.$$
Our immediate goal is therefore to estimate the values  $C[\alpha,\beta]$. We will do so by relating $C[\alpha,\beta]$ to the derivatives of the function $\phi(x,y):=(x^\top y)^k$.
Namely, multiplying both sides of \eqref{eq:mono_to_hermite} by $h_{\alpha}(x)$ and $h_{\beta}(y)$ and taking expectations, we deduce 
\begin{align*}
C[\alpha,\beta]&=\E_{x,y}[\phi(x,y) h_{\alpha}(x) h_{\beta}(y)]\\
&=\frac{1}{\sqrt{\alpha!}}\E_{y}[[\E_x\partial_x^{(\alpha)}\phi(x,y)] h_{\beta}(y)]\\
&=\frac{1}{\sqrt{\alpha!\beta!}}\E_{x,y}\partial_{x,y}^{(\alpha,\beta)}\phi(x,y),
\end{align*}
where the second and third equalities follow from \eqref{eqn:gradient formula}. In order, to compute the derivatives of $\phi$ we apply the multinomial expansion to obtain  
\begin{equation}\label{eqn:multinom_expan}
    \phi(x,y) =  \sum_{\lambda\in C_d(k)} \frac{k!}{\lambda!}x^\lambda y^\lambda,
\end{equation}
where we let $C_d(k)$ be the set consisting of all vectors $\lambda\in\mathbb{N}^d$ satisfying $\lambda\geq 0$ and $\|\lambda\|_1=k$.
Now for any  $\alpha,\beta\in \mathbb{N}^d$ with $\|\alpha\|_1\le k$ and $\|\beta\|_1\le k$, differentiating the equation \eqref{eqn:multinom_expan} yields the expression
$$\E[\partial^{(\alpha,\beta)}_{x,y} \phi(x,y)]=\sum_{\lambda\in C_d(k)} \frac{k!}{\lambda!}\cdot(\E\partial^{(\alpha)}_x x^{\lambda})\cdot (\E\partial^{(\beta)}_y y^{\lambda})= \sum_{\lambda\in\Lambda_{(\alpha,\beta)}} \tfrac{k!\lambda!}{(\lambda-\alpha)!(\lambda-\beta)!}\cdot \E[x^{\lambda-\alpha}] \E[y^{\lambda-\beta}],$$
where we define the set
$$\Lambda_{\alpha,\beta}:=\{\lambda\in C_d(k): \lambda-\alpha~ \textrm{and}~\lambda-\beta ~\textrm{are~nonnegative~and~even}\}.$$
Summarizing, we have learned that $C[\alpha,\beta]$ can be computed as:
\begin{align}\label{eq:C_scale}
 C[\alpha,\beta]=\frac{1}{\sqrt{\alpha!\beta!}}\sum_{\lambda\in\Lambda_{\alpha,\beta}} \tfrac{k!\lambda!}{(\lambda-\alpha)!(\lambda-\beta)!}\cdot \E[ x^{\lambda-\alpha}] \E[\vy^{\lambda-\beta}].
\end{align}
Note  that since each summand in \eqref{eq:C_scale} has constant order, equality $C[\alpha,\beta]=\Theta(|\Lambda_{\alpha,\beta}|)$ holds.
Using the formula \eqref{eq:C_scale}, one may easily compute the diagonal elements of $C$ exactly for the two cases of interest $k=3$ and $k=4$. Moreover, for any principle submatrix $C[\mathcal{I},\mathcal{I}]$, we may estimate the operator norm of the diagonal block as  
$$\|\Phi_{\I}\cdot\diag(C[\I,\I])\cdot \Phi_{\I}^\top\|_{\rm op}\leq \|\diag(C[\I,\I])\|_{\rm op}\cdot \|\Phi_{\I}\|_{\rm op}^2.$$
Using Lemma~\ref{lem:norm_control} we may then further bound 
 $\|\Phi_{\I}\|_{\rm op}^2$ as $O_{d,\mathbb{P}}(\sqrt{n})$ or as $O_{d,\mathbb{P}}(\sqrt{|\I|})$ depending on the ratio $|\I|/n$. We summarize in Tables~\ref{table:p=3_hermite} and Table~\ref{table:p=4_hermite} the diagonal entries of $C$, decomposed as a sum of diagonal blocks along with their operator norm.

\begin{table}[H]
\centering
\begin{tabular}{c c  c c}
\toprule
$\I$ & $\J$ &  $C[\I,\J]$ &  {$d^{-k}\snorm{\Phi_{\I} C \Phi_{\J}\tran}$}\\ 
\midrule
$e_i$ & $e_i$  & $3d$  &$O_{d,\P}(d^{\delta})$ \\
\midrule
$e_i$ & $e_i$  & $6$  &$O_{d,\P}(d^{-1+\delta})$ \\
\midrule
$3e_i$ & $3e_i$ & $6$  & $O_{d,\P}(d^{-1+\delta})$ \\
\midrule
$2e_i+e_j$~$(i\neq j)$ & $2e_i+e_j$  & $6$& $O_{d,\P}(d^{-1+\delta})$ \\
\midrule
$e_i+e_j+e_k$ ~$(i,j,k)$ distinct & $e_i+e_j+e_k$ & $6$ & $O_{d,\P}(1)$ \\
\bottomrule
\end{tabular}
\caption{ ($p=3$) Diagonal elements of $C$ arranged in blocks along with their operator norm.}\label{table:p=3_hermite}
\end{table}

\begin{table}[H]
\centering
\begin{tabular}{c c  c c}
\toprule
$\I$ & $\J$ &  $C[\I,\J]$ &  $d^{-k}\snorm{\Phi_{\I} C \Phi_{\J}\tran}$\\ 
\midrule
0 & 0 & $6d^2$ & $O_{d,\P}(d^\delta)$   \\
\midrule
0 & 0 & $1-6d$ &$O_{d,\P}(d^{-1+\delta})$     \\
\midrule
$e_i+e_j$~$(i\neq j)$ & $e_i+e_j$  & $24d+48$ & $O_{d,\P}(d^{-1+\delta})$\\
\midrule
$2e_i$ & $2e_i$ & $12d+60$ &$O_{d,\P}(d^{-1+\delta})$\\
\midrule
$4e_i$ & $4e_i$  &$24$ & $O_{d,\P}(d^{-2+\delta})$ \\
\midrule
$3e_i+e_j$~$(i\neq j)$ & $3e_i+e_j$ &$24$ &$O_{d,\P}(d^{-2+\delta})$ \\
\midrule
$2e_i+e_j+e_k$ \\($i,j,k$~\rm distinct) & $2e_i+e_j+e_k$  &$24$ &$O_{d,\P}(d^{-1})$\\
\midrule
$2e_i+2e_j$ ~$(i\neq j)$& $2e_i+2e_j$ &$24$&$O_{d,\P}(d^{-2+\delta})$  \\
\midrule
$e_i+e_j+e_k+e_\ell$ \\
$(i,j,k,\ell)$ distinct & $e_i+e_j+e_k+e_\ell$ & $24$ & $O_{d,\P}(1)$ \\
\bottomrule
\end{tabular}
\caption{($p=4$) Diagonal elements of $C$ arranged in blocks along with their operator norm.}\label{table:p=4_hermite}
\end{table}

In particular, we deduce that the two displayed equations \eqref{eqn:k3} and \eqref{eqn:k4} hold with the matrix $XX\tran$ replaced by its diagonal $\diag(XX\tran)$. It remains therefore to argue that matrix $d^{-k}\Phi E\Phi^\top$ corresponding to off-diagonal terms $E:= C-\diag(C)$ has a small operator norm. We do so for $k=3, 4$ by identifying submatrices $E[\I,\J]$, for  finitely many sets $\I$ and $\J$, which cover all entries of $E$ and so that the operator norm $\frac{1}{k!d^k}\|\Phi_{\I}E[\I,\J]\Phi_{\J}^\top\|_{\rm op}$ tends to zero as $d$ increases. In each case, we bound the norm as 
$$\|\Phi_{\I}E[\I,\J]\Phi_{\J}^\top\|_{\rm op}\leq \|E[\I,\J]\|_{\rm op}\|\Phi_{\I}\|_{\rm op}\cdot\|\Phi_{\J}\|_{\rm op}.$$
The operator norms $\|\Phi_{\I}\|_{\rm op}$ and $\|\Phi_{\J}\|_{\rm op}$ can be bounded using Lemma~\ref{lem:norm_control}. In order to effectively bound 
$\|E[\I,\J]\|_{\rm op}$, we choose $\I$ and $\J$ in such a way that $E[\I,\J]$, up to column/row permutations and taking the transpose, is a matrix of two possible types:
\begin{align*}
\textbf{Type~\upperR{1}~:}    \begin{pmatrix}
        c_1 & 0 &\cdots & 0 \\
        0 & c_2 & \cdots &0 \\
        \vdots & \vdots & \ddots & \vdots \\
        0 & 0 & \cdots & c_m
    \end{pmatrix},
    ~~~~
\textbf{Type~\upperR{2}~:}    \begin{pmatrix}
        \rvc_1 & \mathbf{0} &\cdots & \mathbf{0} \\
        \mathbf{0} & \rvc_2 & \cdots &\mathbf{0}\\
        \vdots & \vdots & \ddots & \vdots \\
        \mathbf{0} &\mathbf{0} & \cdots &\mathbf{c_m}
    \end{pmatrix},
\end{align*}
where $c_1,\ldots, c_m$ are real numbers and $\rvc_1,\rvc_1, \ldots,\rvc_m$ are row vectors. The operator norms of the two types of matrices are straightforward to compute. The results are summarized in Tables~\ref{table:p=3} and \ref{table:p=4}. Since the decay rate of $d^{-k}\|\Phi_{\I}E\Phi_{\J}^\top\|_{\rm op}$ is at least on the order of $O_{d,\mathbb{P}}(d^{\delta-1/2})$, the result follows.

\begin{table}[H]
\centering
\begin{tabular}{c c c c c c}
\toprule
$\I$ & $\J$ & $\Lambda_{\I,\J}$ & $E[\I,\J]$ & {\rm Type~of $E$} &  $d^{-k}\snorm{\Phi_{\I} E \Phi_{\J}\tran}$\\ 
\midrule
$e_i$ & $e_i+2e_j$ ~$(j\neq i)$ & $e_i+2e_j$ & $\Theta(1)$& \upperR{2}  & $O_{d,\P}(d^{-1/2+\delta})$\\
$e_i$ & $3e_i$  & $3e_i$ & $\Theta(1)$& \upperR{1} & $O_{d,\P}(d^{-1+\delta})$ \\
\midrule
$3e_i$ & $e_i$ & $3e_i$ & $\Theta(1)$ & \upperR{1} & $O_{d,\P}(d^{-1+\delta})$\\
\midrule
$2e_i+e_j$~$(i\neq j)$ & $e_j$ & $2e_i+e_j$& $\Theta(1)$& \upperR{2} & $O_{d,\P}(d^{-1/2+\delta})$\\
\bottomrule
\end{tabular}
\caption{All off-diagonal terms in Eq.~(\ref{eq:mono_to_hermite}) when $p=3$.}\label{table:p=3}
\end{table}

\begin{table}[H]
\centering
\begin{tabular}{c c c c c c}
\toprule
$\I$ & $\J$ & $\Lambda_{\I,\J}$ & $E[\I,\J]$ & Type of $E$ & $d^{-k}\snorm{\Phi_{\I} E \Phi_{\J}\tran}$ \\ 
\midrule
0 & $2e_i$ & $\{2e_i+2e_j\}_{j}$ & $\Theta_d(d)$ & \upperR{2} & $O_{d,\P}(d^{-1/2+\delta})$ \\
0 & $4e_i$ & $4e_i$ &$\Theta_d(1)$& \upperR{2} &  $O_{d,\P}(d^{-3/2+\delta})$\\
\midrule
$e_i+e_j$~$(i\neq j)$ & $e_i+e_j+2e_k$ ~(\rm distinct) & $e_i+e_j+2e_k$ & $\Theta_d(1)$ & \upperR{2} &$O_{d,\P}(d^{-1+\delta/2})$\\
$e_i+e_j$~$(i\neq j)$ & $3e_i+e_j$  & $3e_i+e_j$ & $\Theta_d(1)$ & \upperR{1} &$O_{d,\P}(d^{-2+\delta})$\\
\midrule
$2e_i$ & $0$ & $\{2e_i+2e_j\}_{j}$ &$\Theta_d(d)$ & \upperR{2} &$O_{d,\P}(d^{-1/2+\delta})$\\
$2e_i$ & $2e_i+2e_j$~$(i\neq j)$ & $2e_i+2e_j$ &$\Theta_d(1)$ & \upperR{2} &$O_{d,\P}(d^{-3/2+\delta})$\\
$2e_i$ & $4e_i$ & $4e_i$ &$\Theta_d(1)$ & \upperR{1} &$O_{d,\P}(d^{-2+\delta})$\\
\midrule
$4e_i$ & $0$ & $4e_i$ &$\Theta_d(1)$ & \upperR{2} &$O_{d,\P}(d^{-3/2+\delta})$\\
$4e_i$ & $2e_i$ & $4e_i$ &$\Theta_d(1)$ & \upperR{1} &$O_{d,\P}(d^{-2+\delta})$\\
\midrule
$3e_i+e_j$ ~$(i\neq j)$ & $e_i+e_j$ & $3e_i+e_j$  &$\Theta_d(1)$ & \upperR{1} &$O_{d,\P}(d^{-2+\delta})$\\
\midrule
$2e_i+e_j+e_k$ \\($i,j,k$~\rm distinct) & $e_j+e_k$ & $2e_i+e_j+e_k$ &$\Theta_d(1)$ & \upperR{2} &$O_{d,\P}(d^{-1+\delta/2})$\\
\midrule
$2e_i+2e_j$ ~$(i\neq j)$& $0$ & $2e_i+2e_j$ &$\Theta_d(1)$ & \upperR{2} &$O_{d,\P}(d^{-1+\delta})$\\
$2e_i+2e_j$~$(i\neq j)$ & $2e_i$ & $2e_i+2e_j$ &$\Theta_d(1)$ & \upperR{2} &$O_{d,\P}(d^{-3/2+\delta})$\\
$2e_i+2e_j$ ~$(i\neq j)$& $2e_j$ & $2e_i+2e_j$ &$\Theta_d(1)$ & \upperR{2} &$O_{d,\P}(d^{-3/2+\delta})$\\
\bottomrule
\end{tabular}
\caption{All off-diagonal terms in Eq.~(\ref{eq:mono_to_hermite}) when $p=4$.}\label{table:p=4}
\end{table}

\subsubsection{Proof of Lemma~\ref{lemma:kernel_hermite_p=2}}\label{proof:kernel_hermite_p=2}
Applying Theorem~\ref{thm:poly_approx} in the regime $n = d^{2+\delta}$ with $m=4$, we deduce
\begin{equation}\label{eq:taylor_kernel_p_2}  \snorm{K - K_4 - (g(1) - g_4(1))I_n} = O_{d,\P}\round{d^{\delta-1/2} \log^3(n)}.
\end{equation}
Recall that $K_4$ can be equivalently written as
\begin{align*}
    K_4 = \sum_{k=0}^4 \frac{g^{(k)}(0)}{k!} \round{\frac{XX\tran}{d}}^{\odot k}.
\end{align*}
Combining Lemma~\ref{lem:conv_quad_herm} and Lemma~\ref{lemma:hermite_decomposition_p=3} directly yields the expression:
$$K_4=\Phi_{\leq 2}D\Phi_{\leq 2}^\top-\tfrac{g^{(3)}(0)}{d^3}\Phi_{\S_3}\Phi_{\S_3}^\top-\tfrac{g^{(4)}(0)}{d^4}\Phi_{\S_4}\Phi_{\S_4}^\top.$$
It remains to argue the second two terms are close to the identity. To this end, applying Lemma~\ref{lemma:gram_matrix} with $k=3,4$ yields the expression
\begin{align*}
    \snorm{\Phi_{\S_k}\Phi_{\S_k}\tran  - |\S_k|I_n}= |\S_k|\cdot O_{d,\mathbb{P}}( d^{(\delta'-1)/2}\log(d)),
\end{align*}
for any $\delta'\in(\delta,1)$. Therefore, we deduce
\begin{align*}
\frac{g^{(k)}(0)}{d^k}\Phi_{\S_k}\Phi_{\S_k}\tran&=\frac{g^{(k)}(0)|\S_k|}{d^k}I_{n}+O_{d,\mathbb{P}}( d^{(\delta'-1)/2}\log(d))\\
&=\frac{g^{(k)}(0)}{k!}I_{n}+O_{d,\mathbb{P}}( d^{(\delta'-1)/2}\log(d)),
\end{align*}
where the last inequality uses the elementary observation $d^{-k}|\S_k|=\frac{1}{k!}+O(d^{-1})$. Therefore we deduce $$\tfrac{g^{(3)}(0)}{d^3}\Phi_{\S_3}\Phi_{\S_3}^\top+\tfrac{g^{(4)}(0)}{d^4}\Phi_{\S_4}\Phi_{\S_4}^\top=\underbrace{\sum_{k=3,4}\frac{g^{(k)}(0)}{k!}}_{=g_{4}(1)-g_2(1)}I_n+O_{d,\mathbb{P}}( d^{(\delta'-1)/2}\log(d)).$$
Combining this expression with \eqref{eq:taylor_kernel_p_2} and the triangle inequality completes the proof.

%% file: sections_appendix_b/prop_gauss.tex
\subsection{Proof of Propositions for Theorem~\ref{thm:main_gauss}}

\subsubsection{Proof of Proposition~\ref{prop:s2_s2}}\label{proof:s2_s2}
Invoking the definition of $D'$ (Eq.~(\ref{eq:D_p=2})) and using the triangle inequality we have
\begin{align*}
    \snorm{{D}'_{\S_2}} \leq \snorm{D'_{13}} + \snorm{{D}'_{\S_2,\S_2}} = \frac{g^{(3)}(0)}{\sqrt{2}d^{3/2}} + \frac{g^{(3)}(0)}{d^2} = O_{d}(d^{-3/2}).
\end{align*}
Taking into account Eq.~(\ref{eq:k_lambda_inv}) and submultiplicativity of the spectral norm, we deduce
\begin{align}
    \norm{\frac{1}{d}\beta\tran  X_r\Phi_{\S_2}{D}'_{\S_2}}_2
    &\leq \frac{1}{d}\norm{\beta}_2\snorm{X_r}\snorm{\Phi_{\S_2}}\snorm{{D}'_{\S_2}} \nonumber= O_{d,\P}(d^{\delta - 1/2}\log(d))\cdot \sqrt{\norm{f^*}_{L_2}^2 +\sigma_\varepsilon^2},
\end{align}
where the last equality follows from the estimates 
$\|\beta\|_2\leq \sqrt{n\log(n)(\norm{f^*}_{L_2}^2 +\sigma_\varepsilon^2)}$ (Eq. \eqref{eqn:we_gonna_need}), $\sup_{r}\snorm{X_r} = O_{d,\P}(\sqrt{\log(d)})$ (Eqn. \eqref{eqn:nnorm_bound_gauss}), and
$\snorm{\Phi_{\S_2}} = O_{d,\P}(\sqrt{n})$ (Lemma~\ref{lem:norm_control}). Taking the square of both sides and absorbing logarithmic factors into $\epsilon$ completes the proof.

\subsubsection{Proof of Proposition~\ref{prop:l1_to_s2}}\label{proof:l1_to_s2}
   Invoking the definition of $D'$ (Eq.~(\ref{eq:D_p=2})), we have 
\begin{align*}
    \norm{{D}'_{\leq 1,\S_2}}_2 = \norm{D'_{13}}_2 = \frac{|g^{(3)}(0)|}{\sqrt{2}}d^{-3/2}.
\end{align*}
Taking into account Eq.~(\ref{eq:k_lambda_inv}) and submultiplicativity of the spectral norm, we deduce
\begin{align*}
    \norm{\frac{1}{d}\beta\tran X_r\Phi_{\leq 1}{D}'_{\leq 1,\S_2}}_2 &\leq \frac{1}{d}\norm{\beta}_2 \snorm{X_r}\snorm{\Phi_{\leq 1}}\snorm{{D}'_{\leq 1,\S_2}}=O_{d,\P}(d^{\delta - 1/2}\log(d)),
\end{align*}
where the last equality follows from the estimates 
$\|\beta\|_2\leq \sqrt{n\log(n)(\norm{f^*}_{L_2}^2 +\sigma_\varepsilon^2)}$ (Eq. \eqref{eqn:we_gonna_need}), $\sup_{r}\snorm{X_r} = O_{d,\P}(\sqrt{\log(d)})$ (Eqn. \eqref{eqn:nnorm_bound_gauss}), and
$\snorm{\Phi_{\leq 1}} = O_{d,\P}(\sqrt{n})$ (Lemma~\ref{lem:norm_control}). 
Taking the square of both sides and absorbing logarithmic factors into $\epsilon$ completes the proof.

\subsubsection{Proof of Proposition~\ref{prop:keymain}}\label{proof:<_l1_l1_l1}
In all arguments in this section, the asymptotic terms $o_{d,\mathbb{P}}$ and $O_{d,\mathbb{P}}$ will always be uniform overall the coordinates $r=1,\ldots,d$.
Recall that all Hermite polynomials of order at most two are products of polynomials of the form 
$$h_0(x_i)=1, \qquad h_1(x_i)=x_i,\qquad h_2(x_i)=\tfrac{x_i^2-1}{\sqrt{2}},$$
evaluated along the coordinates $x_i$.
Then taking into account the definition 
$f^*_{\leq 2} ( x) = \sum_{\lambda\in \mathbb{N}^d:~\|\lambda\|_1\leq 2}b_{\lambda} h_{\lambda}(x)$, we may write 
\begin{align*}
    \E[\partial_{x_r} f^*_{\leq 2} (x)]^2 &= \E\left[\sum_{\lambda\in \mathcal{A}(r)} b_{\lambda} \partial_{x_r} h_{\lambda}(x)\right]^2\\
    &=\sum_{\lambda,\hat\lambda\in \mathcal{A}(r)} b_{\lambda}b_{\hat\lambda} \E[\partial_{x_r} h_{\lambda}(x)\partial_{x_r} h_{\hat \lambda}(x)]\\
    &=b_{e_r}^2+\sum_{i\in [d]\setminus\{r\}}b_{e_r+e_i}^2  + \sqrt{2}\cdot b_{2e_r}^2,
\end{align*}
where the last equality follows from computing the expectation $\E[\partial_{r} h_{\lambda}(x)\partial_{r} h_{\hat \lambda}(x)]$ in all cases. In particular, we may write 
$$\E[\partial_{r} f^*_{\leq 2} (x)]^2=\|b_{\leq 2}^\top E_r\|^2_2,$$
where $E_r\in \R^{L(2)\times (d+1)}$ is the matrix defined in the obvious way.

    Invoking the Hermite expansion of $f^*$ in \eqref{eqn:hermite_expans_target} we may write 
    $$y= \Phi b + \vepsilon=\underbrace{\Phi_{\leq 2}b_{\leq 2}}_{=:f^*_{\leq 2}(X)}+\underbrace{\Phi_{> 2}b_{> 2}}_{=:f^*_{>2}(x)} + \vepsilon,$$
    and similarly $$\beta=K_{\lambda}^{-1}y=\underbrace{K_{\lambda}^{-1}f^*_{\leq 2}(X)}_{=:\beta_{\leq 2}}+\underbrace{K_{\lambda}^{-1}f^*_{> 2}(X)}_{=:\beta_{>2}} + \underbrace{K_\lambda \inv \vepsilon}_{\beta_\varepsilon},$$
    where $\vepsilon :=(\varepsilon_1,\ldots,\varepsilon_n)\in\R^n$.
    
With this decomposition of $\beta$, we accordingly bound the  norm by the triangle inequality:
\begin{align}\label{eq:beta_decomp}
    &~~~\norm{\frac{1}{d}\beta\tran X_r \Phi_{\leq 1}{D}'_{\leq 1,\leq 1} - b_{\leq 2}\tran E_{r}}_2\nonumber\\
    &\leq \norm{\frac{1}{d}\beta_{\leq 2}\tran X_r \Phi_{\leq 1}{D}'_{\leq 1,\leq 1} - b_{\leq 2}\tran E_{r}}_2 + \norm{\frac{1}{d}\beta_{> 2}\tran X_r \Phi_{\leq 1}{D}'_{\leq 1,\leq 1}}_2 + \norm{\frac{1}{d}\vepsilon\tran K_\lambda \inv X_r \Phi_{\leq 1}{D}'_{\leq 1,\leq 1}}.
\end{align}
We will show all the terms on the right-hand side of the above equation tend to zero, from which the theorem will follow quickly. 
We start with the first term in the following proposition. The proof of the proposition appears in Appendix~\ref{proof:beta_<_l2}.

\begin{proposition}\label{lemma:beta_<_l2}
For any $\epsilon>0$, the estimate holds:
    \begin{align*}
         \norm{\frac{1}{d}\beta_{\leq 2}\tran X_r  \Phi_{\leq 1}{D}'_{\leq 1,\leq 1} - b_{\leq 2}\tran E_{r}}_2 =O_{d,\P}\round{d^{\delta- 1/2+\epsilon} + d^{-\delta/2+\epsilon}}\cdot \norm{\partial_r f_{\leq 2}^*}_{L_2}.
    \end{align*}
\end{proposition}

\noindent The next proposition shows that the second term on the right side of \eqref{eq:beta_decomp} is small. The proof of the proposition appears in Appendix~\ref{proof:beta_>_l2}.

\begin{proposition}\label{lemma:beta_>_l2}
   For any $\epsilon>0$, the estimate holds:
\begin{align}
      \norm{\frac{1}{d}\beta_{> 2}\tran   X_r  \Phi_{\leq 1}{D}'_{\leq 1,\leq 1}}_2 
    &= O_{d,\P}\round{d^{\delta - 1/2+\epsilon}+ d^{-\delta/2+\epsilon}}\cdot \norm{f_{> 2}^*}_{L_2}+ O_{d,\P}\round{d^{-\delta/2}} \cdot \norm{f_{>2}^*}_{L_{2+\eta}}.
\end{align}
   
\end{proposition}

\noindent Now we show that the third term on the right side of \eqref{eq:beta_decomp} is small. The proof of the proposition appears in Appendix~\ref{proof:gaussian_error_1}.
\begin{proposition}\label{prop:gaussian_error_1}
   For any $\epsilon>0$, the following estimate holds
\begin{align*}
    \norm{\frac{1}{d}\beta_\varepsilon X_r \Phi_{\leq 1}{D}'_{\leq 1,\leq 1}}_2 =  O_{d,\P}(d^{-(\delta+1)/2+\epsilon} )\cdot \sigma_\varepsilon.
\end{align*}
\end{proposition}

\noindent Consequently, with Proposition~\ref{lemma:beta_<_l2}, \ref{lemma:beta_>_l2} and \ref{prop:gaussian_error_1} we finish the bound for Eq.~(\ref{eq:beta_decomp}):
\begin{align}\label{eqn:fin_estimate_needed}
    \norm{\frac{1}{d}\beta\tran X_r \Phi_{\leq 1}{D}'_{\leq 1,\leq 1} - b_{\leq 2}\tran E_{r}}_2
    &=  O_{d,\P}\round{d^{\delta - 1/2+\epsilon}+ d^{-\delta/2+\epsilon}}\cdot \norm{f^*}_{L_2}+ O_{d,\P}\round{d^{-\delta/2}} \cdot \norm{f^*_{>2}}_{L_{2+\eta}} \notag\\
    &~~~+ O_{d,\P}(d^{-(\delta+1)/2+\epsilon} )\cdot \sigma_\varepsilon.
\end{align}
Now recall the elementary estimate 
$$|\|v\|^2_2-\|w\|^2_2|=|\langle v-w,v+w\rangle|\leq \|v-w\|_2(\|v\|+\|w\|_2),$$
that is valid for any vectors $v,w$. We therefore compute
\begin{align}
    \abs{ \norm{\frac{1}{d}\beta\tran X_r \Phi_{\leq 1}{D}'_{\leq 1,\leq 1} }_2^2 - \norm{b_{\leq 2}\tran E_{r}}_2^2} 
    &\leq \norm{\frac{1}{d}\beta\tran X_r \Phi_{\leq 1}{D}'_{\leq 1,\leq 1} - b_{\leq 2}\tran E_{r}}_2\label{eqn:bound_trick0}\\
    &~~\cdot \left(\norm{\frac{1}{d}\beta\tran X_r\Phi_{\leq 1}{D}'_{\leq 1,\leq 1}}_2 + \norm{b_{\leq 2}\tran E_{r}}_2\right).\label{eqn:bound_trick2}
\end{align}
The term \eqref{eqn:bound_trick0} is bounded by $O_{d,\P}\round{d^{\delta - 1/2+\epsilon}+ d^{-\delta/2+\epsilon}}\cdot \norm{f^*}_{L_2}+ O_{d,\P}\round{d^{-\delta/2}} \cdot \norm{f^*_{>2}}_{L_{2+\eta}}+O_{d,\P}(d^{-(\delta+1)/2+\epsilon} )\cdot \sigma_\varepsilon$ due to \eqref{eqn:fin_estimate_needed} while the term \eqref{eqn:bound_trick2} is clearly $O_{d,\mathbb{P}}(\norm{\partial_r f_{\leq 2}^*}_{L_2} )$ due to Proposition~\ref{lemma:beta_<_l2}. Then plugging these bounds into the above inequality yields an upper bound
\begin{align*}
    &O_{d,\P}\round{d^{\delta - 1/2+\epsilon}+ d^{-\delta/2+\epsilon}}\cdot \norm{f^*}_{L_2}\norm{\partial_r f_{\leq 2}^*}_{L_2}+ O_{d,\P}\round{d^{-\delta/2}} \cdot \norm{f^*_{>2}}_{L_{2+\eta}}\norm{\partial_r f_{\leq 2}^*}_{L_2}\\
    &~~~+O_{d,\P}(d^{-(\delta+1)/2+\epsilon} )\cdot \sigma_\varepsilon\norm{\partial_r f_{\leq 2}^*}_{L_2}.
\end{align*}
Using the fact $\norm{\partial_r f^*_{\leq 2}}_{L_2}^2 = O_d(1)\cdot \norm{f^*_{\leq 2}}_{L_2}^2 \leq O_d(1) \cdot \norm{f^*}_{L_2}^2$ and applying AM-GM inequality to the above bound yields
\begin{align*}
     &O_{d,\P}\round{d^{\delta - 1/2+\epsilon}+ d^{-\delta/2+\epsilon}}\cdot \norm{f^*}_{L_2}^2+ O_{d,\P}\round{d^{-\delta/2}} \cdot \norm{f^*_{>2}}_{L_{2+\eta}}^2 +O_{d,\P}(d^{-(\delta+1)/2+\epsilon} )\cdot \sigma_\varepsilon^2.
\end{align*}
Finally, combining with the fact $\|b_{\leq 2}^\top E_r\|^2_2 = \E[\partial_{r} f^*_{\leq 2} (x)]^2$ completes the proof.

\subsubsection{Proof of Proposition~\ref{lemma:beta_<_l2}}\label{proof:beta_<_l2}
Without loss of generality we focus on the first coordinate $r=1$.
We first establish some notation. We define the column vector $e_k$ to be $1$ in $k$'th coordinate and zero everywhere else. The length of $e_k$ will be clear from context.
Next, let us specify an ordering on the Hermite basis of degree at most two as 
 \begin{align*}
     \phi_{\leq 2}(x) = [&\underbrace{1}_{\phi_{\S_0}},\\
 &\underbrace{h_1(x_1),\ldots,h_1(x_d)}_{\phi_{\S_1}},\\ &\underbrace{h_2(x_1),\ldots,h_2(x_d),h_1(x_1)h_1(x_2),\ldots,h_1(x_1)h_1(x_d),\ldots,h_1(x_{d-1})h_1(x_d)}_{\phi_{\S_2}}].
 \end{align*}
In other words, the terms in $\phi_{\S_2}$ are ordered by concatenating the diagonal of the matrix $[h_1(x_{i})h_1(x_{j})]_{i,j}$ with its rows above the diagonal collapsed into a single vector. With this notation, the matrix $E_1$ can be simply written as 
$$E_{1} = \round{e_{2},\sqrt{2}e_{L(1)+1},e_
 {L(1)+d+1},e_
 {L(1)+d+2}\ldots,e_{L(1)+2d-1}}$$
 We now write $ X_1 \Phi_{\leq 1}{D}'_{\leq 1,\leq 1}$ in a Hermite basis. More precisely, a quick computation shows the equality 

$$X_1 \Phi_{\leq 1}=\Phi_{\leq 2}G,$$
where $G:=[e_2,\sqrt{2}(e_{L(1)+1}+e_1), e_{L(1)+d+1},e_{L(1)+d+2}\ldots, e_{L(1)+2d-1}]$. 
Thus multiplying both sides on the right by  ${D}'_{\leq 1,\leq 1}$ yields the equality
\begin{equation}\label{eqn:transfer}
X_1 \Phi_{\leq 1}{D}'_{\leq 1,\leq 1}=\Phi_{\leq 2}R,\qquad \textrm{where}\qquad R=G{D}'_{\leq 1,\leq 1}.
\end{equation}
Thus it remains to estimate the error 
\begin{align}\label{eqn:estimate_progress}
\left\|\frac{1}{d}\beta^{\top}_{\leq 2}\Phi_{\leq 2}R-b^\top_{\leq 2}E_1\right\|_{2}&=\left\|b^{\top}_{\leq 2}\left(\frac{1}{d}\Phi^\top_{\leq 2}K_{\lambda}^{-1}\Phi_{\leq 2}R-E_1\right)\right\|_{2}\notag \\
&\leq \sqrt{\E|f_{\leq 2}^*|^2}\cdot \left\|\frac{1}{d}\Phi^\top_{\leq 2}K_{\lambda}^{-1}\Phi_{\leq 2}R-E_1\right\|_{\rm op}.
\end{align}
We now argue in the following proposition that $d^{-1}D^{-1}R$ is close to $E_1$. The proof appears in Appendix~\ref{proof:d_inv_r}.
\begin{proposition}\label{prop:d_inv_r}
The estimate holds:
\begin{align}\label{eq:d_inv_r}
    \snorm{\frac{1}{d} D\inv R - E_{1}} = O_d(d^{-1}).
\end{align}   
\end{proposition}
\noindent Thus the right side of \eqref{eqn:estimate_progress} can be further bounded by 

\begin{equation}\label{eqn:fin_proof_bd}
\sqrt{\E|f_{\leq 2}^*|^2}\cdot\left\|\frac{1}{d}\Phi^\top_{\leq 2}K_{\lambda}^{-1}\Phi_{\leq 2}R-\frac{1}{d}D^{-1}R\right\|_{\rm op}+O_d(d^{-1}).
\end{equation}
Next, we further upper bound the right side of \eqref{eqn:fin_proof_bd} with the estimate \eqref{eqn:blup_eqnneeed1} of the following proposition. The proof appears in Section~\ref{sec:proofinverse_prop1}.
\begin{proposition}\label{prop:fait}
The two estimates hold: 
\begin{align}%
   \snorm{\frac{1}{d}\Phi_{\leq 2}\tran K_\lambda \inv \Phi_{\leq 2} R - {\frac{1}{d}D\inv R}} &= O_{d,\P}\round{d^{\delta - 1/2+\epsilon} + d^{-\delta/2+\epsilon}},\label{eqn:blup_eqnneeed1}\\
   \snorm{\frac{\sqrt{n}}{d}K_{\lambda}^{-1} \Phi_{\leq 2}R -\frac{1}{d\sqrt{n}}\Phi_{\leq 2}   D\inv R }
 &= O_{d,\P}\round{d^{\delta - 1/2+\epsilon} + d^{-\delta/2+\epsilon}}.\label{eqn:blup_eqnneeed2}
\end{align}
\end{proposition}
\noindent Combining the estimates \eqref{eqn:blup_eqnneeed1} and \eqref{eqn:fin_proof_bd} completes the proof. We will invoke the second estimate \eqref{eqn:blup_eqnneeed2} as part of the proof of Proposition~\ref{lemma:beta_>_l2}.

\subsubsection{Proof of Proposition~\ref{lemma:beta_>_l2}}\label{proof:beta_>_l2}
\begin{proof}
 Following  Eq.~\eqref{eqn:transfer}, we write $X_r \Phi_{\leq 1}{D}'_{\leq 1,\leq 1} = \Phi_{\leq 2}R_r$ for a matrix $R_r$.
Equation~\eqref{eqn:blup_eqnneeed2} from Proposition~\ref{prop:fait} shows that for  any $\epsilon>0$, the following holds:
\begin{align*}
     \snorm{\sqrt{n}K_\lambda \inv \Phi_{\leq 2} \round{\frac{1}{d}R_r} - \frac{1}{\sqrt{n}} \Phi_{\leq 2}\round{\frac{1}{d}D\inv R_r}} =O_{d,\P}\round{d^{\delta- 1/2+\epsilon} + d^{-\delta/2+\epsilon}}.
\end{align*}
By Lemma~\ref{lem:norm_control}, we have $\snorm{\Phi_{\leq 2}} = O_{d,\P}(\sqrt{n})$, and by Proposition~\ref{prop:d_inv_r}, we have $\snorm{\frac{1}{d}D\inv R_r - E_{r}} = O_d(d^{-1})$. Therefore, by triangle inequality and the submultiplicativity of the spectral norm, we have
\begin{equation}\label{eqn:neededthis est1}
     \snorm{\sqrt{n}K_\lambda \inv \Phi_{\leq 2} \round{\frac{1}{d}R_r} - \frac{1}{\sqrt{n}} \Phi_{\leq 2}E_{r}}
     =O_{d,\P}\round{d^{\delta- 1/2+\epsilon} + d^{-\delta/2+\epsilon}}.
\end{equation}
Next, observe $\norm{\Phi_{>2}b_{>2}}_2 = O_{d,\P}(\sqrt{n\log(n)\cdot \E|f_{>2}^*|^2})$ by Proposition~\ref{prop:bound_y}. Consequently, we compute
\begin{align}\label{eq:phi_l>2_l<2}
    &~~~\norm{\frac{1}{d}\beta_{> 2}\tran  X_r  \Phi_{\leq 1}{D}'_{\leq 1,\leq 1} - \frac{1}{n} b_{>2}\tran \Phi_{>2}\tran \Phi_{\leq 2}E_{r}}_2\nonumber\\
    &=\norm{b_{> 2}\tran \Phi_{>2}^\top\round{\frac{1}{d} K_{\lambda}^{-1}  \Phi_{\leq 2}R_r - \frac{1}{n}  \Phi_{\leq 2}E_{r}}_2}\nonumber\\
    &\leq \frac{1}{\sqrt{n}}\norm{\Phi_{>2}b_{>2}}_2 \snorm{\frac{\sqrt{n}}{d}K_\lambda \inv \Phi_{\leq 2}R_r - \frac{1}{\sqrt{n}} \Phi_{\leq 2}E_{r}}\nonumber\\
    &= O_{d,\P}\round{d^{\delta - 1/2+\epsilon}\log^{1/2}(d) + d^{-\delta/2+\epsilon}\log^{1/2}(d)}\cdot \sqrt{\E|f_{> 2}^*|^2},
    \end{align}
where the first line follows from the defining property $X_r \Phi_{\leq 1}{D}'_{\leq 1,\leq 1} = \Phi_{\leq 2}R_r$, the second line follows from the triangle inequality, and the third line follows from \eqref{eqn:neededthis est1}.

We next show that $\max_{r\in [d]}\norm{\frac{1}{n} b_{>2}\tran \Phi_{>2}\tran \Phi_{\leq 2}E_{r}}_2^2$ is small, from which the conclusion of the lemma follows immediately. To this end, define the vector $$F_{>2}=\begin{bmatrix} f^*_{>2}(x^{(1)})&\ldots&f^*_{>2}(x^{(n)})\end{bmatrix}\in \R^{1\times n},$$
and note the equality $F_{>2}=b_{>2}\tran \Phi_{>2}^\top$.
Taking the expectation over $x^{(1)},\cdots,x^{(n)}$ we have:
\begin{equation}\label{eqn:blup}
    \E \brac{\max_{r\in[d]}\frac{1}{n^2}\norm{F_{>2}\Phi_{\leq 2}E_r}_2^2}
    \leq \frac{1}{n^2}\E\brac{ \sum_{r=1}^d \norm{F_{>2}\Phi_{\leq 2}E_r}_2^{2}}\leq \frac{d}{n^2}\E\brac{\norm{F_{>2}\Phi_{\leq 2}E_r}_2^{2}},
\end{equation}
where the first inequality from bounding the maximum by a sum.
We now bound $\E\brac{\norm{F_{>2}\Phi_{\leq2}E_r}_2^{2}}$ as follows. Define the set $$\mathcal{A}_r:=\{\lambda\in \mathbb{N}^d: \sum_i\lambda_i\leq 2\textrm{ and }\lambda_r\geq 1\},$$ of multi-indices of order at most two that depend on  coordinate $r$. We then compute
\begin{align}
\E\brac{\norm{F_{>2}\Phi_{\leq2}E_r}_2^{2}}&=\E\langle F_{>2}^\top F_{>2},\Phi_{\leq 2}E_rE_r^\top \Phi_{\leq 2}^\top\rangle\notag\\
&=\E\left\langle F_{>2}^\top F_{>2},\sum_{\lambda\in \mathcal{A}_r}\left(1+\mathbbm{1}_{\lambda=2e_r}\right) \Phi_{\lambda}\Phi_{\lambda}^\top\right\rangle\label{eqn:Edefn}\\
&=\sum_{\lambda\in \mathcal{A}_r}\left(1+\mathbbm{1}_{\lambda=2e_r}\right) \E (F_{>2} \Phi_{\lambda})^2,\label{eqn:final_exp_square}
\end{align}
where \eqref{eqn:Edefn} follows from the definition of $E_r$. Continuing, for any $\lambda \in \mathcal{A}_r$, expanding the square we compute
\begin{align*}
\E (F_{>2} \Phi_{\lambda})^2&=\sum_{i,j=1}^n \E[f^*_{>2}(x^{(i)})\phi_{\lambda}(x^{(i)}) f^*_{>2}(x^{(j)}) \phi_{\lambda}(x^{(j)})]\\
&=\sum_{i\neq j} \E[f^*_{>2}(x^{(i)})\phi_{\lambda}(x^{(i)})]\cdot\E[ f^*_{>2}(x^{(j)}) \phi_{\lambda}(x^{(j)})]+\sum_{i=1}^n \E[f^*_{>2}(x^{(i)})^2\phi_{\lambda}(x^{(i)})^2 ]\\
&=n \cdot\E[f^*_{>2}(x)^2\phi_{\lambda}(x)^2 ]
\end{align*}
where the last equality follows from the fact that $\phi_{\lambda}$ and $f^*_{>2}$ are orthogonal for any $\lambda\in \mathbb{N}^d$ with $\lambda\leq 2$. Thus combining with \eqref{eqn:final_exp_square} and \eqref{eqn:blup} we conclude 
\begin{align*}
\E \brac{\max_{r\in[d]}\frac{1}{n^2}\norm{F_{>2}\Phi_{\leq 2}E_r}_2^2}&\leq \frac{2d}{n}\cdot |\mathcal{A}_r|\cdot \max_{\lambda\in \mathcal{A}_r} \E[f^*_{>2}(x)^2\phi_{\lambda}(x)^2 ]\\
&=O_d(d^{-\delta}) \cdot \max_{\lambda\in \mathcal{A}_r} \E[f^*_{>2}(x)^2\phi_{\lambda}(x)^2 ],
\end{align*}
where the last equality uses the fact that $\mathcal{A}_r$ has cardinality of order $d$.

It remains to show that $\E[f^*_{>2}(x)^2\phi_{\lambda}(x)^2 ]$ has constant order $O_d(1)$ for any $\lambda\in \mathcal{A}_r$. To see this, for any constant $\eta>0$,
applying Holder's inequality yields
\begin{align}\label{eq:main_thm_t11_bound5}
   \E_ x \brac{f_{>2}^*(x)^2 \phi_{\lambda}(x)^2} &\leq  \norm{f_{>2}^*(x)^2}_{L_{1+\frac{\eta}{2}}} \norm{\phi_{\lambda}^2}_{L_{1+\frac{2}{\eta}}}\\
   &=\norm{f_{>2}^*(x)}_{L_{2+\eta}}^{2} \cdot \norm{\phi_{\lambda}}_{L_{2+\frac{4}{\eta}}}^{2}.
\end{align}
By hypercontractivity of Guassian random variables (Eq.~(\ref{eqn:hypercontrac})) we have 
\begin{align*}
   \norm{\phi_{\lambda}}_{L_{2+\frac{4}{\eta}}} \leq  \left(1+\frac{4}{\eta}\right) \sqrt{\E_x\brac{\phi_{\lambda}(x)^2}} = 1+\frac{4}{\eta}.
\end{align*}
Then we conclude $  \E_ x \brac{f_{>2}^*(x)^2 \phi_{\lambda}(x)^2} = O_d(\norm{f_{>2}^*(x)}_{L_{2+\eta}}^{2})$, which implies $\E \brac{\max_{r\in[d]}\frac{1}{n^2}\norm{F_{>2}\Phi_{\leq 2}E_r}_2^2} = O_d(d^{-\delta}\cdot \norm{f_{>2}^*(x)}_{L_{2+\eta}}^{2})$. Plugging the bound into Eq.~(\ref{eq:phi_l>2_l<2}) and absorbing logarithmic factors into $\epsilon$ we complete the proof.

\subsubsection{Proof of Proposition~\ref{prop:gaussian_error_1}}\label{proof:gaussian_error_1}
 Following  Eq.~\eqref{eqn:transfer}, we write $X_r \Phi_{\leq 1}{D}'_{\leq 1,\leq 1} = \Phi_{\leq 2}R_r$ for a matrix $R_r$. 
Since $\epsilon_i\overset{\text{i.i.d.}}{\sim}\mathcal{N}(0,\sigma_\varepsilon^2) $, we have
\begin{align}\label{eq:gaussian_error_tech_bound_1}
    \E_{\vepsilon}\brac{\norm{\frac{1}{d}\vepsilon\tran K_\lambda \inv \Phi_{\leq 2}R_r}_2^2} &= \frac{\sigma_\varepsilon^2}{d^2}\Tr\round{K_\lambda\inv \Phi_{\leq 2}  R_r  R_r\tran \Phi_{\leq 2}\tran  K_\lambda \inv}\notag\\
    &=\frac{\sigma_\varepsilon^2}{d^2}\Tr\round{K_\lambda\inv \Phi_{\leq 2}D D\inv R_r  R_r\tran D\inv D\Phi_{\leq 2}\tran  K_\lambda \inv}\notag\\
    &\leq \sigma_\varepsilon^2 \cdot \snorm{K_\lambda\inv \Phi_{\leq 2}D}^2 \cdot \Tr\round{\frac{1}{d^2}D\inv R_r  R_r\tran D\inv},
\end{align}
where the last inequality follows from the fact that $|\Tr(AB)| \leq \Tr(A)\snorm{B}$ holds for Hermitian positive semidefinite matrix $A$ and arbitrary matrix $B$. 

For the term $\snorm{K_\lambda\inv \Phi_{\leq 2}D}^2$, we plug $U = D$ into Eq.~\eqref{eq:s_approx_bound_R} and yield
\begin{align*}
    K_\lambda\inv \Phi_{\leq 2}D = \frac{1}{n}E\inv \Phi_{\leq 2} V.
\end{align*}
With $\snorm{E\inv}, \snorm{V} = O_{d,\P}(1)$ established in Eq.~\eqref{eq:s_approx_bound_6}, and $\snorm{\Phi_{\leq 2}} =O_{d,\P}(\sqrt{n})$ (Lemma~\ref{lem:norm_control}), we conclude
\begin{align}\label{eq:gaussian_error_tech_bound_2}
    \snorm{ K_\lambda\inv \Phi_{\leq 2}D }^2 = \snorm{\frac{1}{n}E\inv \Phi_{\leq 2} V}^2 = O_{d,\P}(n\inv).
\end{align}
We proceed to bound $\Tr\round{\frac{1}{d^2}D\inv R_r  R_r\tran D\inv}$ in Eq.~\eqref{eq:gaussian_error_tech_bound_1}. We denote $\Delta:= \frac{1}{d} D\inv R_r - E_{r}$ and Proposition~\ref{prop:d_inv_r}  shows $\snorm{\Delta} = O_{d}(d\inv)$.  Then using Weyl’s inequality yields
\begin{align*}
 \Tr( (E_r + \Delta)(E_r + \Delta)\tran)  &= \sum_i \sigma_i(E_r + \Delta)^2\\
 &\leq \sum_i ( \sigma_i(E_r) + \snorm{\Delta})^2\\
 &\leq \norm{E_r}_F^2 + 2 \snorm{\Delta}\cdot \round{\sum_{i}\sigma_i(E_r)} + \snorm{\Delta}^2\cdot O_d(d),
\end{align*}
where $\sigma_i$ denotes the singular values and we use the fact that the rank of $E_r$ is of order $O_d(d)$ in the last inequality.

Following the definition of $E_r$, we have $\norm{E_r}_F^2 = O_d(d)$ and $\round{\sum_{i}\sigma_i(E_r)} = O_d(d)$ which concludes the bound 
\begin{align*}
    \Tr\round{\frac{1}{d^2}D\inv R_r  R_r\tran D\inv}=O_d(d).
\end{align*}
 Together with Eq.~\eqref{eq:gaussian_error_tech_bound_2}, we conclude the bound for Eq.~\eqref{eq:gaussian_error_tech_bound_1}. Then applying Markov's inequality, we obtain with probability at least $1-1/\log(d)$, the following holds:
\begin{align*}
    \norm{\frac{1}{d}\vepsilon\tran K_\lambda \inv \Phi_{\leq 2}R_r}_2^2 = O_{d}(d^{-1-\delta})\log(d).
\end{align*}
Using an arbitrarily small constant $\epsilon>0$ to absorb the logarithmic factor completes the proof.

\subsubsection{Proof of Proposition~\ref{prop:d_inv_r}}\label{proof:d_inv_r}
Recall that we have $R=GD_{\leq 1,\leq 1}$ with 
$$G=[e_2,\sqrt{2}(e_{L(1)+1}+e_1), \widetilde{I}],$$
and $\widetilde{I}:=[e_{L(1)+d+1},e_{L(1)+d+2}\ldots, e_{L(1)+2d-1}]$ is a zero-padded identity matrix. 
Let us organize $D'_{\leq 1, \leq 1}$ and $D$ in a block form as follows:  
$$D'_{\leq 1, \leq 1}=\begin{bmatrix} \gamma_1& 0\\
0 & \gamma_2 I_{|\S_1|}\end{bmatrix}\qquad\textrm{and}\qquad D=\begin{bmatrix}a& v^\top\\
v &Q\end{bmatrix}\quad \textrm{with}\quad Q:=\begin{bmatrix}q_2 I_{|\S_1|} & 0\\
0& q_3 I_{|\S_2|}\end{bmatrix}.$$
Here $a,q_1,q_2,\gamma_1,\gamma_2\in \R$ are real values and $v\in \R^{L(2)-1}$ is a column vector. The standard inversion formula for arrow-head matrices shows that $D^{-1}$ is a rank one perturbation of a diagonal matrix 
\begin{equation}\label{eqn:inverse_arrow}
D^{-1}=\begin{bmatrix} 0 & 0 \\ 0 & Q^{-1}\end{bmatrix}+\rho ww^\top,
\end{equation}
where we define
$$\rho:=\frac{1}{a-v^\top Q^{-1}v}\qquad \textrm{and}\qquad w:=\begin{bmatrix} -1 \\ Q^{-1} v\end{bmatrix}.$$
In particular, plugging in the specific values of $a,v,Q$ in our case yields the explicit expressions:
$$\rho=\frac{1}{g(0)+\frac{g^{(4)}(0)}{4d^2}}=O(1)\qquad \textrm{and}\qquad w:=\begin{bmatrix}-1\\ \mathbf{0}_{d} \\\frac{1}{\sqrt{2}}\mathbf{1}_{d}  \\ \mathbf{0}_{|\S_2|-d}\end{bmatrix}.$$
Multiplying the right side of \eqref{eqn:inverse_arrow} by $R$ we deduce: 
\begin{equation}\label{eqn:we_need_decompo}
\frac{1}{d}D^{-1}R=\frac{1}{d}D^{-1}GD'_{\leq 1,\leq 1}=\begin{bmatrix} 0 & 0 \\ 0 & Q^{-1}\end{bmatrix}GD'_{\leq 1,\leq 1}+\frac{\rho}{d} ww^\top GD'_{\leq 1,\leq 1}.
\end{equation}
We will show that the first term on the right side of \eqref{eqn:we_need_decompo} is close to $E_1$ while the second-one is negligible. We begin by explicitly computing the first term:
\begin{align*}
\frac{1}{d}\begin{bmatrix} 0 & 0 \\ 0 & Q^{-1}\end{bmatrix}GD'_{\leq 1,\leq 1}=\begin{bmatrix}\frac{\gamma_1}{dq_2}e_2 &\frac{\gamma_2\sqrt{2}}{dq_3}e_{L(1)+1}& \frac{\gamma_2}{dq_3} \widetilde{I}\end{bmatrix}.
\end{align*}
Now we compute 
$$\frac{\gamma_1}{dq_2}=\frac{\frac{g^{(1)}(0)}{d}+\frac{g^{(3)}(0)}{2d^2}+\frac{g^{(5)}(0)}{4d^3}}{\frac{g^{(1)}(0)}{d} + \frac{g^{(3)}(0)}{2d^2}}=1+O(d^{-1}),$$
and similarly
$$\frac{\gamma_2}{dq_3}=\frac{\frac{g^{(2)}(0)}{d^2} + \frac{g^{(4)}(0)}{2d^3}}{\frac{g^{(2)}(0)}{d^2}}=1+O(d^{-1}).$$
The triangle inequality therefore directly implies 
$$\left\|\frac{1}{d}\begin{bmatrix} 0 & 0 \\ 0 & Q^{-1}\end{bmatrix}GD'_{\leq 1,\leq 1}-E_1\right\|_{\rm op}=O(d^{-1}).$$
Moving on to the second term on the right side of \eqref{eqn:we_need_decompo}, a quick computation shows 
\begin{align*}
(w^\top G)D'_{\leq 1,\leq 1}&=\begin{bmatrix} 0 & -\sqrt{2}+1 & \mathbf{0}_{d-1} \end{bmatrix}D'_{\leq 1,\leq 1}=\begin{bmatrix} 0 &(-\sqrt{2}+1)\gamma_2 & \mathbf{0}_{d-1}\end{bmatrix}
\end{align*}
Therefore we deduce 
$$\left\|\frac{\rho}{d}ww^\top GD'_{\leq 1,\leq 1}\right\|_{\rm op}= \underbrace{\rho}_{=O(1)} d^{-1} \underbrace{\|w\|_{2}}_{=O(\sqrt{d})}\|\cdot \underbrace{\|w^\top GD'_{\leq 1,\leq 1}\|_{\rm op}}_{=O(d^{-1})}=O(d^{-3/2}).$$
This completes the proof.

\subsubsection{Proof of Proposition~\ref{prop:fait}}\label{sec:proofinverse_prop1}
Throughout the proof, set $U:=R/d$ and $S_{\leq 2}:=\round{I_{\leq 2} + (\rho+\lambda)(n D)\inv}\inv$. By Lemma~\ref{lemma:kernel_hermite_p=2},  there exists a matrix  $\Delta$ with $\snorm{\Delta} = O_{d,\P}(d^{\delta - 1/2+\epsilon})$ such that
 \begin{align*}
     K_{\lambda} = \Phi_{\leq 2}D \Phi_{\leq 2}\tran + (\rho+\lambda)(I_n + \Delta),
\end{align*}
where $\rho = g(1) - g_2(1)$. Setting $E:=I_n + \Delta$,
we use the  Sherman-Morrison-Woodbury formula to compute $K_{\lambda}^{-1}$ as follows:
    \begin{align}
         K_{\lambda}^{-1} \Phi_{\leq 2}
        &=\frac{1}{\rho+\lambda}\round{(\rho+\lambda)^{-1}\Phi_{\leq 2}D \Phi_{\leq 2}\tran + E}^{-1}\Phi_{\leq 2}\nonumber\\
        &=\frac{1}{\rho+\lambda} \round{E^{-1} - E^{-1}\Phi_{\leq 2}\round{(\rho+\lambda) D^{-1}+ \Phi_{\leq 2}\tran E^{-1} \Phi_{\leq 2}}^{-1}\Phi_{\leq 2}\tran E^{-1} } \Phi_{\leq 2}\nonumber\\
        &= \frac{1}{\rho+\lambda}  E^{-1} \Phi_{\leq 2}\round{I - \round{(\rho+\lambda)D^{-1}+ \Phi_{\leq 2}\tran E^{-1} \Phi_{\leq 2}}^{-1} \Phi_{\leq 2}\tran E^{-1}\Phi_{\leq 2}}  \label{eq:s_approx_bound_1}.
    \end{align}
Using the trivial identity $I - (B+A)^{-1}A = (B+A)^{-1}B$ for any $A,B$ such that $A+B$ is invertible, we deduce
\begin{align*}
    I -\round{(\rho+\lambda)D^{-1}+ \Phi_{\leq 2}\tran E^{-1} \Phi_{\leq 2}}^{-1}\Phi_{\leq 2}\tran E^{-1}\Phi_{\leq 2} 
    &=(\rho+\lambda)\round{ (\rho+\lambda)D^{-1} + \Phi_{\leq 2}\tran E^{-1}\Phi_{\leq 2} }^{-1}  D^{-1}.
\end{align*}
Plugging this expression back into Eq.~(\ref{eq:s_approx_bound_1}) and multiplying through by $U$, we obtain
\begin{align}\label{eq:s_approx_bound_R}
     K_{\lambda}^{-1} \Phi_{\leq 2}U 
    &= \frac{1}{n} E^{-1}\Phi_{\leq 2} \underbrace{\round{  (\rho+\lambda)(n D)^{-1} + \frac{1}{n}\Phi_{\leq 2}\tran E^{-1}\Phi_{\leq 2} }^{-1}}_{=:V} D\inv U.
\end{align}
We will next show that $V$ is close to $S_{\leq 2}$ and $\Phi_{\leq 2}\tran E^{-1}\Phi_{\leq 2}/n$ is be close to $I_n$.
We first show that $E^{-1}$ is close to a multiple of identity $I_n$ by computing:
\begin{align}
 \snorm{E^{-1} - I_n} &=\snorm{(I_n +\Delta)^{-1}-  I_n}\nonumber\\
   & \leq \frac{\|\Delta\|_{\rm op}}{1-\|\Delta\|_{\rm op}}\notag \\
    &= O_{d,\P}(d^{\delta - 1/2+\epsilon})\label{eq:s_approx_bound_2}.
\end{align}
Recall from Lemma~\ref{lemma:phi_id_2} that we have $\snorm{\Phi_{\leq 2}\tran \Phi_{\leq 2} /n - I_n  } = O_{d,\P}(d^{-\delta/2+\epsilon})$. 
Consequently, the triangle inequality directly implies 

\begin{align}\label{eq:phi_e_phi}
    \snorm{\frac{1}{n} \Phi_{\leq 2}\tran E^{-1}\Phi_{\leq 2} - I_n} 
    &= O_{d,\P}\round{d^{\delta - 1/2+\epsilon} + d^{-\delta/2+\epsilon}}.
\end{align}
We are now ready to show $\snorm{V - S_{\leq 2}}= o_{d,\P}(1)$. First, Eq.~(\ref{eq:phi_e_phi}) implies the bound
\begin{align*}
    \snorm{V^{-1} - S_{\leq 2}^{-1}} &=  \snorm{\frac{1}{n}\Phi_{\leq 2}\tran E^{-1}\Phi_{\leq 2} -I_n}  = O_{d,\P}\round{d^{\delta - 1/2+\epsilon} + d^{-\delta/2+\epsilon}}.
\end{align*}
Therefore, the norm of $V$ can be bounded as
\begin{align*}
    \snorm{V} = \frac{1}{\snorm{V^{-1}}} 
    &= \frac{1}{\snorm{(\rho+\lambda)(nD)\inv +\frac{1}{n}\Phi_{\leq 2}\tran E^{-1}\Phi_{\leq 2} }}  \\
    &\leq \frac{1}{\snorm{\frac{1}{n}\Phi_{\leq 2}\tran E^{-1}\Phi_{\leq 2} } 
 - \snorm{(\rho+\lambda)(nD)\inv } }\\
 &\leq \frac{1}{1 - \snorm{ \frac{1}{n}\Phi_{\leq 2}\tran E^{-1}\Phi_{\leq 2} - I_n} - \snorm{(\rho+\lambda)(nD)\inv } }.
\end{align*}
Using \eqref{eq:phi_e_phi} and ~(\ref{eqn:inverse_arrow}), when $d$ is sufficiently large, we have
$\snorm{V}  = O_{d,\P}(1).$
Since for two invertible matrices $A,B$, we have$\snorm{A-B} = \snorm{A(A^{-1} - B^{-1})B}$, we conclude
\begin{align}
    \snorm{V- S_{\leq 2}} \leq \snorm{V} \norm{S_{\leq 2}}_{\rm op}\snorm{V^{-1} - S_{\leq 2}^{-1}}=  O_{d,\P}\round{d^{\delta - 1/2+\epsilon} + d^{-\delta/2+\epsilon}}.\label{eq:s_approx_bound_3}
\end{align}
Consequently, using the triangle inequality together with Eq.~(\ref{eq:phi_e_phi}) and (\ref{eq:s_approx_bound_3}), we get
\begin{align*}
    \snorm{S_{\leq 2} - \frac{1}{n}\Phi_{\leq 2}\tran E^{-1} \Phi_{\leq 2} V } &\leq  \snorm{S_{\leq 2} - V } + \snorm{V - \frac{1}{n}\Phi_{\leq 2}\tran E^{-1} \Phi_{\leq 2} V }\\
    &\leq  \snorm{S_{\leq 2}-V }+ \snorm{V}\snorm{I - \frac{1}{n}\Phi_{\leq 2}\tran E^{-1} \Phi_{\leq 2} }\\
    &=  O_{d,\P}\round{d^{\delta - 1/2+\epsilon} + d^{-\delta/2+\epsilon}}.
\end{align*}
We are now ready to establish the claimed estimate \eqref{eqn:blup_eqnneeed1}. To this end,
using Eq.~(\ref{eq:s_approx_bound_R}) we compute
\begin{align}
   \snorm{S_{\leq L(2)}D\inv U -\Phi_{\leq 2} K_{\lambda}^{-1} \Phi_{\leq 2}U } &=  \snorm{S_{\leq 2}D\inv U -\frac{1}{n} \Phi_{\leq 2}\tran E^{-1}\Phi_{\leq 2}V D\inv U}\notag \\
   &\leq  \snorm{S_{\leq L(2)} - \frac{1}{n}\Phi_{\leq L(2)}\tran E^{-1} \Phi_{\leq 2} V }\snorm{D\inv U}\notag \\
   &= O_{d,\P}\round{d^{\delta - 1/2+\epsilon} + d^{-\delta/2+\epsilon}}.\label{eqn:fin_bound_nneded}
\end{align}
where in the last inequality we use the fact that $\snorm{D\inv U}=O_{d,\mathbb{P}}(1)$ by Proposition~\ref{prop:d_inv_r}. It remains to argue that $S_{\leq L(2)}$ is close to the identity. To see this, recall that the inverse of $D$ was computed in \eqref{eqn:inverse_arrow}. In particular, the triangle inequality shows 
$$\|(nD)^{-1}\|_{\rm op}\leq O(d^{-\delta}).$$
Consequently, setting $\hat\Delta=(\rho+\lambda)(nD)^{-1}$ we estimate
$$\|I-S_{\leq 2}\|_{\rm op}=\|I-(I+\hat \Delta)^{-1}\|_{\rm op}\leq \frac{\|\hat\Delta\|_{\rm op}}{1-\|\hat \Delta\|_{\rm op}}=O(d^{-\delta}).$$
Combining this estimate with \eqref{eqn:fin_bound_nneded} and using the triangle inequality completes the proof of \eqref{eqn:blup_eqnneeed1}.

Next, we establish the estimate \eqref{eqn:blup_eqnneeed2}. To this end, using the triangle inequality and the submultiplicativity of the spectral norm, we obtain:
\begin{align*}
&~~~\snorm{\sqrt{n}K_{\lambda}^{-1} \Phi_{\leq 2}U -\frac{1}{\sqrt{n}}\Phi_{\leq 2}  S_{\leq 2} D\inv U }\\
 &=\snorm{\frac{1}{\sqrt{n}}E^{-1} \Phi_{\leq 2}  V D\inv U  -\frac{1}{\sqrt{n}}\Phi_{\leq 2}  S_{\leq 2} D\inv U} \\
 &\leq \snorm{\frac{1}{\sqrt{n}}E^{-1} \Phi_{\leq 2}  V   -\frac{1}{\sqrt{n}}\Phi_{\leq 2}  S_{\leq 2} }\snorm{D\inv U}\\
 &\leq \frac{1}{\sqrt{n}}\snorm{(E^{-1} - I_n)\Phi_{\leq 2}  V }\snorm{D\inv U} + \frac{1}{\sqrt{n}}\snorm{\Phi_{\leq 2}\round{ V - S_{\leq 2}}}\snorm{D\inv U} \\
 &\leq \frac{1}{\sqrt{n}}\snorm{E^{-1} - I_n} \snorm{\Phi_{\leq 2}} \snorm{V}\snorm{D\inv U} + \frac{1}{\sqrt{n}}\snorm{\Phi_{\leq 2}}   \snorm{ V - S_{\leq 2}}\snorm{D\inv U}.
\end{align*}
Recall that we have already shown the estimates
\begin{align}
    &\snorm{V} = O_{d,\P}(1),~~ \snorm{E^{-1} - I_n} = O_{d,\P}(d^{\delta -1/2 +\epsilon}), \nonumber \\
    &\snorm{V - S_{\leq 2}} = O_{d,\P}\round{d^{\delta- 1/2+\epsilon} + d^{-\delta/2+\epsilon}}.\label{eq:s_approx_bound_6}
\end{align}
Additionally, by Lemma~\ref{lem:norm_control}, we have $\snorm{\Phi_{\leq 2}} = O_{d,\P}(\sqrt{n})$ and by Proposition~\ref{prop:d_inv_r} we have $\snorm{D\inv U}=O_{d,\mathbb{P}}(1)$. With these bounds, we deduce that
\begin{align*}
\snorm{\sqrt{n}K_{\lambda}^{-1} \Phi_{\leq 2}U -\frac{1}{\sqrt{n}}\Phi_{\leq 2}  S_{\leq 2} D\inv U }
 &= O_{d,\P}\round{d^{\delta - 1/2+\epsilon} + d^{-\delta/2+\epsilon}}.
\end{align*}
Moreover, recall that $\|I-S_{\leq 2}\|_{\rm op}=O(d^{-\delta})$. Consequently, using the triangle inequality, we deduce 
\begin{align*}
\snorm{\sqrt{n}K_{\lambda}^{-1} \Phi_{\leq 2}U -\frac{1}{\sqrt{n}}\Phi_{\leq 2}   D\inv U }
 &= O_{d,\P}\round{d^{\delta - 1/2+\epsilon} + d^{-\delta/2+\epsilon}},
\end{align*}
which completes the proof of \eqref{eqn:blup_eqnneeed2}.

%% file: sections_appendix_b/prop_hypercube.tex
\subsection{Proof of Propositions for Theorem~\ref{thm:main_hypercube}}

\subsubsection{Proof of Proposition~\ref{prop:<_l1_l1_l1_cube}}\label{proof:<_l1_l1_l1_cube}

The proof is analogous to that of Proposition~\ref{prop:keymain}. Let $\gE(r)$ denote the union of all sets $S\subset [d]$ containing $r$ with $|S|\leq p$, i.e., $\gE(r) = \cup_{j=1}^p \S_j^1$. Similarly, define the complement $\gE^c(r)=\round{ \cup_{j=0}^\ell \S_j}\setminus \gE(r)$.

Recall that from \eqref{eqn:grad_interp_cube} we have 
\begin{align*}
    \E[\partial_{r} f^*_{\leq p} (x)]^2 =\sum_{S\in \gE(r)} b_S^2=\|b_{\leq p}^\top E_r\|^2_2,
\end{align*}
where $E_{r}$ is the partial identity matrix of size $L(p)\times L(p)$ defined as  $(E_{r})_{\I,\I} = \mathbbm{1} \{\I \in \gE(r)\}$.
    We  now write $\beta=\beta_{\gE(r)}+\beta_{\gE^c(r)} +\beta_\varepsilon$ where we define:
\begin{align*}
    \beta_{\gE(r)}:= K_\lambda\inv (\Phi_{\gE(r)}b_{\gE(r)}),~~~\beta_{\gE^c(r)}:= K_\lambda\inv (\Phi_{\gE^c(r)}b_{\gE^c(r)}),~~~\beta_{\epsilon}:=K_\lambda \inv \vepsilon.
\end{align*}
With this decomposition of $\beta$, we accordingly bound the norm by triangle inequality:
\begin{align}\label{eq:beta_decomp_cube}
    &~~~\norm{\beta\tran \Phi_{\leq p}R - b_{\leq p}\tran E_r}_2\nonumber\\
    &\leq \norm{\beta_{\gE(r)}\tran \Phi_{\leq p}R - b_{\leq p}\tran E_r}_2 + \norm{\beta_{\gE^c(r)}\tran \Phi_{\leq p}R}_2 +\norm{\beta_{\varepsilon}\tran \Phi_{\leq p}R}_2.
\end{align}
We will show that all the terms on the right-hand-side of the above equation are small. 

The following two Propositions control the first and second terms in Eq.~(\ref{eq:beta_decomp_cube}) respectively. We defer the proofs to Appendix~\ref {proof:beta_<_l2_cube} and \ref{proof:beta_>_l2_cube} respectively.
\begin{proposition}\label{lemma:beta_<_l2_cube}
For any $\epsilon>0$, the estimate holds uniformly over all coordinates $r$:
    \begin{align*}
         \norm{\beta_{\gE(r)}\tran  \Phi_{\leq p}R - b_{\leq p}\tran E_r }_2 = O_{d,\P}(d^{-\delta/2+\epsilon} + d^{(\delta-1)/2+\epsilon})\cdot  \norm{\partial_r f^*_{\leq p}}_{L_2}.
    \end{align*}
\end{proposition}

\begin{proposition}\label{lemma:beta_>_l2_cube}
The estimate holds uniformly over all coordinates $r$:
\begin{align}
      \norm{\beta_{\gE^c(r)}\tran   \Phi_{\leq p}R}_2 
    &= O_{d,\P}\round{d^{-1/2}}\cdot\norm{f^*_{\gE^c(r)}}_{L_2}.
\end{align}
\end{proposition}

\noindent Next  we show that the third term on the right side of \eqref{eq:beta_decomp_cube} is small. The proof of the proposition appears in Appendix~\ref{proof:cube_error_1}.
\begin{proposition}\label{prop:cube_error_1}
   For any $\epsilon>0$, the following estimate holds
\begin{align*}
    \norm{\beta_{\varepsilon}\tran \Phi_{\leq p}R}_2 = O_{d,\P}(d^{-(1+\delta)/2+\epsilon})\cdot \sigma_\varepsilon.
\end{align*}
\end{proposition}

\noindent Combining the results of three propositions with Eq.~(\ref{eq:beta_decomp_cube}), using the AM-GM inequality to bound the cross-term
directly yields the estimate
\begin{align*}
   \norm{\beta\tran \Phi_{\leq p}{R} - b_{\leq p}\tran E_r}_2^2 &=O_{d,\P}\round{d\inv}\cdot\norm{f_{\gE^c(r)}^*}_{L_2}^2 + O_{d,\P}(d^{-\delta+2\epsilon} + d^{\delta-1+2\epsilon})\cdot \norm{\partial_r f^*_{\leq p}}_{L_2}^2\\
   &~~~+O_{d,\P}(d^{-1-\delta+2\epsilon})\cdot \sigma_\varepsilon^2.
\end{align*}
Using the elementary inequality $|\norm{a}_2^2-\norm{b}_2^2| \leq \norm{a-b}_2^2 +2\norm{b}_2\norm{a-b}_2$ for any vectors $a,b$ yields
\begin{align*}
    &~~~\abs{\norm{\beta\tran  \Phi_{\leq p}{R}}_2^2 -\norm{b_{\leq p}\tran E_r}_2^2} \\
    &=O_{d,\P}\round{d\inv}\cdot\norm{f_{\gE^c(r)}^*}_{L_2}^2 + O_{d,\P}(d^{-\delta+2\epsilon} + d^{\delta-1+2\epsilon})\cdot \norm{\partial_r f^*_{\leq p}}_{L_2}^2 + O_{d,\P}(d^{-1-\delta+2\epsilon})\cdot \sigma_\varepsilon^2\\
    &~~~+\norm{\partial_r f^*_{\leq p}}_{L_2}\cdot \sqrt{O_{d,\P}\round{d\inv}\cdot\norm{f_{\gE^c(r)}^*}_{L_2}^2 + O_{d,\P}(d^{-\delta+2\epsilon} + d^{\delta-1+2\epsilon})\cdot \norm{\partial_r f^*_{\leq p}}_{L_2}^2+ O_{d,\P}(d^{-1-\delta+2\epsilon})\cdot \sigma_\varepsilon^2}.
\end{align*}
Since $\sqrt{a+b}\leq \sqrt{a}+\sqrt{b}$ for any $a,b\geq 0$, we arrive at the bound
\begin{align*}
    &O_{d,\P}\round{d\inv}\cdot\norm{f_{\gE^c(r)}^*}_{L_2}^2 + O_{d,\P}(d^{-\delta/2+\epsilon} + d^{(\delta-1)/2+\epsilon})\cdot \norm{\partial_r f^*_{\leq p}}_{L_2}^2 +O_{d,\P}(d^{-1-\delta+2\epsilon})\cdot \sigma_\varepsilon^2 \\
    &~~~+O_{d,\P}(d\half)\cdot \norm{\partial_r f^*_{\leq p}}_{L_2}\norm{f_{\gE^c(r)}^*}_{L_2} + O_{d,\P}(d^{-(1+\delta)/2 + \epsilon})\sigma_\varepsilon \norm{\partial_r f^*_{\leq p}}_{L_2}.
\end{align*}
Invoking the inequality $\norm{f_{\gE^c(r)}^*}_{L_2}^2\leq \norm{f^*}_{L_2}^2$ and applying the inequality $$O_{d,\P}(d^{-(1+\delta)/2 + \epsilon})\cdot\sigma_\varepsilon \norm{\partial_r f^*_{\leq p}}_{L_2} \leq O_{d,\P}(d\inv) \cdot \sigma_\varepsilon^2 + O_{d,\P}(d^{-\delta + 2\epsilon})\cdot\norm{\partial_r f^*_{\leq p}}_{L_2}^2$$ completes the proof.

\subsubsection{Proof of Proposition~\ref{prop:s2_s2_cube}}\label{proof:s2_s2_cube}
We first simplify the expression. By \eqref{eq:def_r_j_1}, the diagonal matrix $R_{p+1}^1$ has entries $g^{(p+1)}(0)d^{-p-1} + O_d(d^{-p-2})$. Consequently, we compute
\begin{align}\label{eq:s2_s2_cube_bound_1}
    &~~~\norm{\beta\tran\Phi_{\S_{p+1}^1} R_{p+1}^1 - \frac{g^{(p+1)}(0)}{d^{p+1}}\beta\tran \Phi_{\S_{p+1}^1}}_2 \nonumber \\
    &\leq \|\beta^\top \Phi_{\S_{p+1}^1}\|_2\cdot O_d(d^{-p-2}) \notag\\
    &\leq  \norm{y}_2\snorm{K_\lambda\inv}\left\| \Phi_{\S_{p+1}^1}\right\|_{\rm op}\cdot O_d(d^{-p-2})\notag\\
    &= O_{d,\P}(d^{-2+\delta}\log^{1/2}(d))(\norm{f^*}_{L_2}+\sigma_\varepsilon),
\end{align}
where we used the fact that $\snorm{\Phi_{\S_{p+1}^1}}=\snorm{X_r}\cdot\snorm{\Phi_{\S_{p}^0}} \leq \snorm{\Phi_{\leq p}} = O_{d,\P}(\sqrt{n})$ from Lemma~\ref{lem:norm_control} and  the basic estimate $\norm{y}_2 = O_{d,\P}(\sqrt{n}\log^{1/2}(d))(\norm{f^*}_{L_2}+\sigma_\varepsilon)$. The bound $ \snorm{K_\lambda\inv} \leq 1/(\rho+\lambda -\snorm{\Delta_1}) = O_{d,\P}(1)$ is implied by  the expression \eqref{eqn:K_decomp_cube}.

We  now split the whole index set $\cup_{j=0}^\ell \S_j$ into the disjoint union $\Q\sqcup \S_{p+1}^1$ for the family of sets $\Q=\cup_{j=0}^\ell \S_j\setminus \S_{p+1}^1$.
Then we correspondingly decompose $y$ and $\beta$ as 
\begin{align*}
    y &= \Phi_{\Q}b_{\Q} + \Phi_{\S_{p+1}^1}b_{\S_{p+1}^1} + \vepsilon,\\
    \beta &= \underbrace{K_{\lambda}^{\inv} \Phi_{\Q}b_{\Q}}_{=:\,\beta_{\Q}}+\underbrace{K_{\lambda}^{\inv} \Phi_{\S_{p+1}^1}b_{\S_{p+1}^1}}_{=:\,\beta_{c}} +\underbrace{K_{\lambda}^{\inv} \vepsilon}_{=:\,\beta_{\epsilon}} .
\end{align*}
Note the equality $\norm{b_{\S_{p+1}^1}}_2 = \norm{\partial_r f^*_{p+1}}_{L_2}$. Our goal now is to show the following estimate
\begin{align*}
     \norm{d^{-p-1}\beta^\top \Phi_{\S^1_{p+1}} - (\rho+\lambda)\inv{d^{\delta-1}}b_{\S_{p+1}^1}}_2 = o_{d,\P}(1),
\end{align*}
which combined with Eq.~\eqref{eq:s2_s2_cube_bound_1} will quickly complete the proof. To this end, using the decomposition of $\beta$, we apply triangle inequality to the norm and obtain
\begin{align}
     &~~~\norm{d^{-p-1}\beta^{\tran} \Phi_{\S^1_{p+1}} - (\rho+\lambda)\inv{d^{\delta-1}}b_{\S_{p+1}^1}}_2 \notag\\
     &\leq \norm{d^{-p-1}\beta^{\tran}_{\Q} \Phi_{\S_{p+1}^1}}_2+ \norm{d^{-p-1}\beta_c^{\tran} \Phi_{\S_{p+1}^1} - (\rho+\lambda)\inv{d^{\delta-1}}b_{\S_{p+1}^1}\tran}_2 + \norm{d^{-p-1} \beta_\varepsilon\tran \Phi_{\S_{p+1}^1}}_2.\label{eq:s2_s2_cube_bound_split}
\end{align}
Next, we show that all three terms on the right side of the above equation are of the order $o_{d,\P}(1)$.
For the first term, we apply the following proposition. The proof is deferred to Appendix~\ref{proof:feature_u_v}.
\begin{proposition}\label{prop:feature_u_v}
    The following estimate holds uniformly for all $r\in[d]$:
    \begin{align*}
         \norm{\frac{1}{d^{p+1}} \beta_{\Q}^{\top} \Phi_{\S_{p+1}^1}}_2^2 = O_d(d^{-1})\norm{b_\Q}_2^2.
    \end{align*}
\end{proposition}

\noindent Now we proceed to the second term.  The following proposition gives an estimate for the second term. We defer the proof to Appendix~\ref{proof:term_p_plus_one}.
\begin{proposition}\label{prop:term_p_plus_one}
For any $\epsilon>0$, the following estimate holds uniformly for all $r\in[d]$
\begin{align}\label{eq:term_p_plus_one}
    \norm{\frac{1}{n}\beta_{c}^{\top} \Phi_{\S_{p+1}^1} - (\rho+\lambda)\inv b_{\S_{p+1}^1}\tran}_2 =O_{d,\P}(d^{-\delta/2+\epsilon} + d^{(\delta-1)/2+\epsilon})\cdot \norm{b_{\S_{p+1}^1}}_2.
\end{align}
\end{proposition}

\noindent  Multiplying \eqref{eq:term_p_plus_one} through by $d^{\delta-1}$ we obtain
\begin{align}\label{eq:term_p_plus_one_updated}
    \norm{d^{-p-1}\beta_{c}^{\top} \Phi_{\S_{p+1}^1} - (\rho+\lambda)\inv{d^{\delta-1}}b_{\S_{p+1}^1}\tran}_2 = O_{d,\P}(d^{\delta/2-1+\epsilon} + d^{3(\delta-1)/2+\epsilon})\cdot \norm{b_{\S_{p+1}^1}}_2.
\end{align}
The following proposition gives an estimate for the third term. We defer the proof to Appendix~\ref{proof:noise_term_p_plus_one}.
\begin{proposition}\label{prop:noise_term_p_plus_one}
For any $\epsilon>0$, the following estimate holds uniformly for all $r\in[d]$:
\begin{align}\label{eq:noise_term_p_plus_one}
    \norm{d^{-p-1} \beta_\varepsilon\tran \Phi_{\S_{p+1}^1}}_2 = O_{d,\P}(d^{-1+\delta/2+\epsilon})\cdot \sigma_\varepsilon
\end{align}
\end{proposition}

\noindent Now, plugging Eq. \eqref{eq:term_p_plus_one_updated}, \eqref{eq:term_p_plus_one_updated} and \eqref{eq:noise_term_p_plus_one} into Eq.~\eqref{eq:s2_s2_cube_bound_split} yields:
\begin{align*}
    &~~~\norm{d^{-p-1}\beta^{\top} \Phi_{\S^1_{p+1}} - (\rho+\lambda)\inv{d^{\delta-1}}b_{\S_{p+1}^1}\tran}_2 \\
    &= O_{d,\P}(d^{\delta/2-1+\epsilon} + d^{3(\delta-1)/2+\epsilon})\cdot \norm{b_{\S_{p+1}^1}}_2 + O_{d,\P}(d^{-1/2})\cdot \norm{f^*}_{L_2}+O_{d,\P}(d^{-1+\delta/2+\epsilon})\cdot \sigma_\varepsilon,
\end{align*}
where we use the fact $\norm{b_\Q}_2 = \norm{f^*_\Q}_{L_2}\leq \norm{f^*}_{L_2}$.

Squaring both sides yields:
\begin{align*}
\norm{d^{-p-1}\beta^{\top}  \Phi_{\S_{p+1}^1} - (\rho+\lambda)\inv{d^{\delta-1}}b_{\S_{p+1}^1}\tran}_2^2 = O_{d,\P}(d^{\delta-2+2\epsilon} + d^{3(\delta-1)+2\epsilon})\cdot \norm{b_{\S_{p+1}^1}}_2^2 + O_{d,\P}(d^{-1})\cdot (\norm{f^*}_{L_2}^2+\sigma_\varepsilon^2).
\end{align*}
Combining this bound with Eq.\eqref{eq:s2_s2_cube_bound_1} and using the elementary inequality $|\norm{a}_2^2-\norm{b}_2^2| \leq \norm{a-b}_2^2 +2\norm{b}_2\norm{a-b}_2$ for any vectors $a,b$ gives
\begin{align*}
 &~~~\abs{\norm{d^{-p-1}\beta^{\top} \Phi_{\S_{p+1}^1}}_2^2 -(\rho+\lambda)^{-2}\norm{{d^{\delta-1}}b_{\S_{p+1}^1}}_2^2}\\
 &=  O_{d,\P}(d^{\delta-2+2\epsilon} + d^{3(\delta-1)+2\epsilon})\cdot \norm{b_{\S_{p+1}^1}}_2^2+ O_{d,\P}(d^{-1})\cdot  (\norm{f^*}_{L_2}^2+\sigma_\varepsilon^2)\\
 &~~~+ d^{\delta-1}\norm{b_{\S_{p+1}^1}}_2 \cdot \sqrt{O_{d,\P}(d^{\delta-2+2\epsilon} + d^{3(\delta-1)+2\epsilon})\cdot \norm{b_{\S_{p+1}^1}}_2^2 + O_{d,\P}(d^{-1})\cdot  (\norm{f^*}_{L_2}^2+\sigma_\varepsilon^2)}.
\end{align*}
Invoking the inequality $\sqrt{a+b}\leq \sqrt{a}+\sqrt{b}$ for any $a,b>0$ the bound reduces to 
\begin{align*}
     O_{d,\P}(d^{-1})\cdot  (\norm{f^*}_{L_2}^2+\sigma_\varepsilon^2) +  O_{d,\P}(d^{
     \frac{3}{2}\delta-2+\epsilon} + d^{\frac{5}{2}(\delta-1)
     +\epsilon})\cdot \norm{b_{\S_{p+1}^1}}_2^2 + O_{d,\P}(d^{\delta - 3/2 }) \cdot \norm{b_{\S_{p+1}^1}}_2  (\norm{f^*}_{L_2}+\sigma_\varepsilon).
\end{align*}
Combining with the fact $\norm{b_{\S_{p+1}^1}}_2 = \norm{\partial_r f^*_{p+1}}_{L_2}$ completes the proof.

\subsubsection{Proof of Proposition~\ref{lemma:beta_<_l2_cube}}\label{proof:beta_<_l2_cube}

The proof is nearly identical to that of Proposition~\ref{prop:fait}.
In the proof, we will frequently use the following estimate for any $\Delta$ with $\snorm{\Delta} = o_{d,\P}(1)$:
\begin{align}\label{eq:i_plus_delta_inv}
    \snorm{(I+\Delta)\inv - I} = O_{d,\P}(\snorm{\Delta}),
\end{align}
which can be deduced using Neumann series.
Recall that  Eq.~\eqref{eqn:K_decomp_cube} shows 
\begin{align*}
        K &= \Phi_{\leq p}D\Phi_{\leq p}\tran + \rho I_n + \Delta_1,
\end{align*}
where $\rho = g(1) - g_p(1)$ and $\Delta_1$ satisfies $\snorm{\Delta_1} = O_{d,\P}(d^{(\delta-1)/2+\epsilon})$. We now write
\begin{align}\label{eq:K_decomp}
    K_\lambda  = \Phi_{\leq p}D\Phi_{\leq p}\tran + (\rho+\lambda)H,
\end{align}
where $H = I_n + (\rho+\lambda)\inv \Delta_1$.
Invoking Sherman-Morrison-Woodbury formula and following exactly the same argument as in Proposition~\ref{sec:proofinverse_prop1} (equation \eqref{eq:s_approx_bound_R}) directly implies 
\begin{align}\label{eq:smw_<_p}
   \beta_{\gE(r)}\tran \Phi_{\leq p} R &=b_{\gE(r)}^{\top}\Phi_{\gE(r)}\tran K_{\lambda}^{-1}\Phi_{\leq p} R\notag\\
   &=\frac{1}{n}b_{\gE(r)}\tran \Phi_{\gE(r)}\tran H\inv \Phi_{\leq p}\round{\underbrace{(\rho+\lambda)(nD)\inv + \frac{\Phi_{\leq p}\tran H\inv \Phi_{\leq p}}{n}}_G}\inv\round{D\inv R}.
\end{align}
Now we simplify $G\inv$. First using triangle inequality we have
\begin{align}\label{eq:g_minus_i}
    \snorm{G - I} \leq (\rho+\lambda)\snorm{(nD)\inv} + \snorm{\frac{\Phi_{\leq p}\tran H\inv \Phi_{\leq p}}{n} -I_n}.
\end{align}
Based on the definition of $D$, we immediately have $\snorm{(nD)\inv} = O_{d}(d^{-\delta})$. Notice that $ \snorm{\frac{\Phi_{\leq p}\tran  \Phi_{\leq p}}{n} - I_n} = O_{d,\P}(d^{-\delta/2+\epsilon})$ (Theorem~\ref{lemma:phi_id_2}) and 
\begin{align}\label{eq:e_inv_minus_i}
    \snorm{H\inv - I_n} = O_d(\snorm{\Delta_1}) = O_{d,\P}(d^{(\delta-1)/2+\epsilon}),
\end{align}
which follows from Eq.~\eqref{eq:i_plus_delta_inv}.
Thus we deduce
\begin{align}\label{eq:psi_e_psi_minus_i}
 \snorm{\frac{\Phi_{\leq p}\tran H\inv \Phi_{\leq p}}{n} -I} &= \snorm{\frac{\Phi_{\leq p}\tran  \Phi_{\leq p}}{n} +\frac{\Phi_{\leq p}\tran  (H\inv - I) \Phi_{\leq p}}{n} - I}\notag\\
 &\leq \snorm{\frac{\Phi_{\leq p}\tran  \Phi_{\leq p}}{n} - I} + \snorm{\frac{\Phi_{\leq p}}{\sqrt{n}}}^2\snorm{H\inv - I}\notag\\
 &= O_{d,\P}(d^{-\delta/2+\epsilon} + d^{(\delta-1)/2+\epsilon}).
\end{align}
Plugging the above bound into Eq.~\eqref{eq:g_minus_i} we obtain
\begin{align*}
    \snorm{G-I} =  O_d(d^{-\delta}) +O_{d,\P}(d^{-\delta/2+\epsilon} + d^{(\delta-1)/2+\epsilon})= O_{d,\P}(d^{-\delta/2+\epsilon} + d^{(\delta-1)/2+\epsilon}).
\end{align*}
Invoking Eq.~\eqref{eq:i_plus_delta_inv} again we obtain
\begin{align}\label{eq:g_inv_minus_i}
     \snorm{G\inv-I} =  O_{d,\P}(d^{-\delta/2+\epsilon} + d^{(\delta-1)/2+\epsilon}).
\end{align}
The following claim gives an estimate for $D\inv R$. The proof is deferred to Appendix~\ref{proof:d_inv_r_cube}.
\begin{claim}\label{prop:d_inv_r_cube}
The following estimate holds
\begin{align}\label{eq:d_inv_r_cube}
    \snorm{D\inv R - E_r} = O_d(d^{-1}).
\end{align}   
\end{claim}
\noindent We apply this estimate for $\frac{1}{d}D\inv R$ to Eq.~\eqref{eq:smw_<_p} and obtain
\begin{align}\label{eq:b_leq_p_decomp}
    \norm{\beta_{\gE(r)}\tran \Phi_{\leq p} R -  \frac{1}{n}b_{\gE(r)}\tran \Phi_{\gE(r)}\tran H\inv \Phi_{\leq p}G\inv E_r}_2 = O_{d,\P}(d\inv)\cdot\norm{b_{\gE(r)}}_2,
\end{align}
where we use the fact that $\snorm{H\inv} = O_{d,\P}(1)$ (Eq.~\eqref{eq:e_inv_minus_i}), $\snorm{G\inv} = O_{d,\P}(1)$ (Eq.~\eqref{eq:g_inv_minus_i}), and $\snorm{\Phi_{\gE(r)}}\leq \snorm{\Phi_{\leq p}} = O_{d,\P}(\sqrt{n})$ (Lemma~\ref{lem:norm_control}).

Next, we will show
\begin{align}\label{eq:b_leq_p_decomp_2}
     \norm{\frac{1}{n}b_{\gE(r)}\tran \Phi_{\gE(r)}\tran H\inv \Phi_{\leq p}G\inv E_r-b_{\gE(r)}}_2 =o_{d,\P}(1),
\end{align}
then plugging it back into Eq.~\eqref{eq:b_leq_p_decomp} complete the proof.
To this end, applying the estimates~\eqref{eq:e_inv_minus_i} and \eqref{eq:g_inv_minus_i} along with the triangle inequality yields
\begin{align}\label{eq:b_leq_p_decomp_bound_1}
   \norm{ \frac{1}{n}b_{\gE(r)}\tran \Phi_{\gE(r)}\tran H\inv \Phi_{\leq p}G\inv E_{r} -\frac{1}{n}b_{\gE(r)}\tran \Phi_{\gE(r)}\tran  \Phi_{\leq p}E_r}_2 =  O_{d,\P}(d^{-\delta/2+\epsilon} + d^{(\delta-1)/2+\epsilon})\cdot \norm{b_{\gE(r)}}_2,
\end{align}
where we used the fact that $\|E_r\|_{\rm op}=O_d(1)$.

Next, invoking the bound $\snorm{\Phi_{\gE(r)}\tran\Phi_{\gE(r)}/n -I} = O_{d,\P}(d^{-(1+\delta)/2 +\epsilon})$ (Theorem~\ref{lemma:phi_id_2}) and the definition of $E_r$, we have
\begin{align*}
    \norm{\frac{1}{n}b_{\gE(r)}\tran \Phi_{\gE(r)}\tran\Phi_{\leq p}E_r - b_{\leq p}\tran E_r}_2 & = \norm{\frac{1}{n}b_{\gE(r)}\tran \Phi_{\gE(r)}\tran\Phi_{\gE(r)} -b_{\gE(r)}^{\top}}_2 \\
    &\leq 
    \norm{b_{\gE(r)}\tran \round{ \Phi_{\gE(r)}\tran  \Phi_{\gE(r)}/n - I}  }_2\\
    &\leq \norm{ b_{\gE(r)}}_2\cdot O_{d,\P}(d^{-(\delta+1)/2 +\epsilon}).
\end{align*}
Plugging this estimate into Eq.~\eqref{eq:b_leq_p_decomp_bound_1} yields
\begin{align*}
     \norm{\frac{1}{n}b_{\gE(r)}\tran \Phi_{\gE(r)}\tran H\inv \Phi_{\leq p}G\inv E_r -b_{\leq p}\tran E_r}_2 =  O_{d,\P}(d^{-\delta/2+\epsilon} + d^{(\delta-1)/2+\epsilon})\cdot \norm{b_{\gE(r)}}_2,
\end{align*}
which completes the proof.

\subsubsection{Proof of Claim~\ref{prop:d_inv_r_cube}}\label{proof:d_inv_r_cube}

 Recall that $D$ is a diagonal matrix with entry $D_{\I,\I} =  g^{(|\I|)}(0)d^{-|\I|} + O_d(d^{-|\I|-1})$. 
Using the expression of each block of $R$, namely Eq.~\eqref{eq:def_r_j_0} and \eqref{eq:def_r_j_1}, we can compute each diagonal entry of $D\inv R$: 
    \begin{align*}
   D_{\I,\I}\inv R_{\I,\I} = \begin{cases}
         1+O_d(d^{-1}),~~ {\rm if}~\I \in \cup_{j=1}^p\S_j^1;\\
         O_d(d^{-2}),~~ {\rm if}~\I \in \cup_{j=0}^{p-1} \S_j^0.\\
    \end{cases}
\end{align*}
Recalling the definition of $E_r$ as 
$
  (E_r)_{\I,\I} = \mathbbm{1}\{\I \in \gE(r)\},
$
the claimed bound follows immediately.

\subsubsection{Proof of Proposition~\ref{lemma:beta_>_l2_cube}}\label{proof:beta_>_l2_cube}

In the proof, we will frequently use the following estimate for any $\Delta$ with $\snorm{\Delta} = o_{d,\P}(1)$, which can be deduced from the Neumann series:
\begin{align}
    \snorm{(I+\Delta)\inv - I} = O_{d,\P}(\snorm{\Delta}),
\end{align}
therefore we have $(I+\Delta)\inv = I + \Delta'$ with some $\Delta'$ satisfying $\snorm{\Delta'} = O_{d,\P}(\snorm{\Delta})$. Now we are ready to present the proof. For simplicity of notation, we omit $(r)$ in $\gE(r)$ and $\gE^c(r)$.

First, we expand $K_\lambda\inv$ by using Eq.~\eqref{eq:kernel_to_mono_advanced} in Lemma~\ref{lemma:kernel_to_mono}.
Consequently, we have
\begin{align}\label{eq:k_fancy_decomp}
    K_{\lambda} &= K_p + \underbrace{(\rho+\lambda)I_n +g^{(p+1)}(0)d^{-p-1}\offd\round{\Phi_{\S_{p+1}}\Phi_{\S_{p+1}}\tran} + \Delta}_{:=(\rho+\lambda)H},
\end{align}
with  $\snorm{\Delta} = O_{d,\P}(d^{(\delta-2)/2 + \epsilon}) = O_{d,\P}(d\halfe)$, and we denote
\begin{align}\label{eq:E}
    H:= I_n +  (\rho+\lambda)\inv g^{(p+1)}(0)d^{-p-1}\offd\round{\Phi_{\S_k}\Phi_{\S_k}\tran} +  (\rho+\lambda)\inv \Delta.
\end{align}
Analogously to Eq.~\eqref{eq:K_decomp}, we write  $K_{\lambda} = K_{p} + (\rho+\lambda)H$ and use  the Sherman-Morrison-Woodbury formula to express $K_\lambda\inv = \round{K_{p} + (\rho+\lambda)H}\inv $ yielding:
\begin{align}\label{eq:smw_cube}
    \beta_{\gE^c}\tran \Phi_{\leq p} R =\frac{1}{n}b_{\gE^c}\tran \Phi_{\gE^c}\tran H\inv \Phi_{\leq p}\round{\underbrace{(\rho+\lambda)(nD)\inv + \frac{\Phi_{\leq p}\tran H\inv \Phi_{\leq p}}{n}}_G}\inv\round{D\inv R}.
\end{align}
We further simplify Eq.~\eqref{eq:smw_cube}. By the triangle inequality, the following holds
\begin{align}\label{eq:apply_d_inv}
   &~~~\norm{\beta_{\gE^c}\tran \Phi_{\leq p} R -\frac{1}{n}b_{\gE^c}\tran \Phi_{\gE^c}\tran H\inv \Phi_{\leq p}G\inv E_{r}}_2 \notag\\
   &=\norm{ \frac{1}{n}b_{\gE^c}\tran \Phi_{\gE^c}\tran H\inv \Phi_{\leq p}G\inv \round{D\inv R - E_{r} }}_2\notag\\
   &\leq \norm{\frac{1}{\sqrt{n}}\Phi_{\gE^c}b_{\gE^c}}_2\snorm{H\inv}\snorm{\frac{\Phi_{\leq p}}{\sqrt{n}}} \snorm{G\inv}\snorm{D\inv R - E_{r}} \notag\\  
   &=O_{d,\P}(d\inv\log(d))\cdot \norm{f_{\gE^c}^*}_{L_2},
\end{align}
where $\norm{\frac{1}{\sqrt{n}}\Phi_{\gE^c}b_{\gE^c}}_2 = O_{d,\P}(\log(d))\norm{f_{\gE^c}^*}_{L_2}$ (Markov's inequality), $\snorm{H\inv} = O_{d,\P}(1)$ (Eq.~\eqref{eq:e_inv_minus_i}), $\snorm{\frac{\Phi_{\leq p}}{\sqrt{n}}} = O_{d,\P}(1)$ (Lemma~\ref{lem:norm_control}), $\snorm{G\inv} = O_{d,\P}(1)$ (Eq.~\eqref{eq:g_inv_minus_i}) and $\snorm{D\inv R - E_{r}} = O_{d}(d\inv)$ follows from Proposition~\ref{prop:d_inv_r_cube}.
Consequently, there exists $\Delta_1$ that satisfies $\norm{\Delta_1}_2 = O_{d,\P}(d\inv \log(d))\cdot \norm{f_{\gE^c}^*}_{L_2} $  such that the following holds:
\begin{align}\label{eq:simplified_g_p}
   \beta_{\gE^c}\tran \Phi_{\leq p} R =  \frac{1}{n} b_{\gE^c}\tran \Phi_{\gE^c}\tran H\inv \Phi_{\leq p}G\inv E_{r} + \Delta_1.
\end{align}
Next, we express $H\inv$ and $G\inv$ in terms of feature matrices.
Note that for any $\Delta = o_{d,\P}(1)$, we can express $(I+\Delta)\inv$ using the Neumann series:
\begin{align*}
    (I + \Delta)\inv = I + \sum_{k=1}^\infty (-\Delta)^k.
\end{align*}
We compute the inverse of $H$ in Eq.~\eqref{eq:E} using Neumann series yielding
\begin{align}\label{eq:raw_h_inv}
    H\inv &= I_n + \sum_{q=1}^\infty \round{\underbrace{-(\rho+\lambda)\inv {g^{(p+1)}(0)}}_{:=c_{p+1}}d^{-p-1}{ \offd(\Phi_{\S_{p+1}}\Phi_{\S_{p+1}}\tran)} - (\rho+\lambda)\inv\Delta}^{q}.
\end{align}
Note that the triangle inequality on Eq.~\eqref{eq:kernel_to_mono} and \eqref{eq:kernel_to_mono_advanced} shows 
\begin{align*}
    d^{-p-1}{ \offd(\Phi_{\S_{p+1}}\Phi_{\S_{p+1}}\tran)} = O_{d,\P}(d^{(\delta-1)/2+\epsilon}),
\end{align*}
where $\offd(\cdot)$ denotes the off-diagonal component of a matrix.

The following claim now show that  we can truncate higher-order terms in $H\inv$ at constant order to achieve a desired error on the order of $O_{d,\P}(d\halfe)$.

\begin{claim}\label{claim:truncate_h}
    Setting a constant $c_H := \ceil{\frac{\delta+4\epsilon}{1-\delta-2\epsilon}}$, there exists $\Delta_2$ that satisfies $\snorm{\Delta_2} = O_{d,\P}(d\halfe)$ such that $H^{-1}$ can be written as
    \begin{align}\label{eq:e_inv_1}
  H\inv =  I_n + \sum_{q=1}^{c_H} \round{-c_{p+1} d^{-p-1}{ \offd(\Phi_{\S_{p+1}}\Phi_{\S_{p+1}}\tran)}}^{q} + \Delta_2.
\end{align}
\end{claim}
\begin{proof}
We expand $\sum_{q=1}^\infty \round{-c_{p+1}d^{-p-1}{ \offd(\Phi_{\S_{p+1}}\Phi_{\S_{p+1}}\tran)} + (\rho+\lambda)\inv\Delta}^{q}$ in Eq.~\eqref{eq:raw_h_inv}.  %
Note that by the triangle inequality and formula for the summation of a geometric series, the following holds:
\begin{align*}
    \snorm{\sum_{q=c_H+1}^\infty\round{-c_{p+1}d^{-p-1}{ \offd(\Phi_{\S_{p+1}}\Phi_{\S_{p+1}}\tran)}+(\rho+\lambda)\inv\Delta}^q} &\leq \sum_{q=c_H+1}^\infty (O_{d,\mathbb{P}}(d^{(\delta-1)/2+\epsilon}))^q\\
&= O_{d,\P}(d^{\round{(\delta-1)/2+\epsilon}(c_H+1)})\\
&=O_{d,\mathbb{P}}(d^{-1/2-\epsilon}).
\end{align*}
Next, clearly since $c_H\geq 1$ is a constant we have 
$$\sum_{q=1}^{c_H}\round{-c_{p+1}d^{-p-1}{ \offd(\Phi_{\S_{p+1}}\Phi_{\S_{p+1}}\tran)}+(\rho+\lambda)\inv\Delta}^q=\sum_{q=1}^{c_H}\round{-c_{p+1}d^{-p-1}{ \offd(\Phi_{\S_{p+1}}\Phi_{\S_{p+1}}\tran)}}^q+\Delta_2,$$
for some $\Delta_2$ satisfying
$\snorm{\Delta_2} = O_{d,\P}(d\halfe)$, which completes the proof.
\end{proof}

We now plug the  expression \eqref{eq:e_inv_1} for $H\inv$ into $G$ in Eq.~(\ref{eq:smw_cube}). We further substitute $\frac{\Phi_{\leq p}\tran \Phi_{\leq p}}{n}$ for $ \offd\round{\frac{\Phi_{\leq p}\tran \Phi_{\leq p}}{n}} + I_{L(p)}$ and we obtain
\begin{align*}
    G = (\rho+\lambda)(nD)\inv  + I + \offd\round{\frac{\Phi_{\leq p}\tran \Phi_{\leq p}}{n}}  + \frac{\Phi_{\leq p}\tran}{\sqrt{n}}\sum_{q=1}^{c_H} \round{-c_{p+1} d^{-p-1}{ \offd(\Phi_{\S_{p+1}}\Phi_{\S_{p+1}}\tran)}}^{q} \frac{\Phi_{\leq p}}{\sqrt{n}} + \Delta_3,
\end{align*}
with $\snorm{\Delta_3} \leq \snorm{\Phi_{\leq p} / \sqrt{n}}^2 \snorm{\Delta_2} = O_{d,\P}(d\halfe)$.

Computing $G\inv$ using Neumann series again gives us
\begin{align*}
    G\inv = I +\sum_{r=1}^\infty \round{ -(\rho+\lambda)(nD)\inv -\offd\round{\frac{\Phi_{\leq p}\tran \Phi_{\leq p}}{n}}  - \frac{\Phi_{\leq p}\tran}{\sqrt{n}}\sum_{q=1}^{c_H} \round{-c_{p+1} \frac{ \offd(\Phi_{\S_{p+1}}\Phi_{\S_{p+1}}\tran)}{d^{p+1}}}^{q} \frac{\Phi_{\leq p}}{\sqrt{n} } + \Delta_3}^{r},
\end{align*}
Similarly, we can truncate the infinite summation while maintaining the same error order. Note that each item in the parenthesis is of the order $o_{d,\P}(1)$, which can be deduced using the following estimates:
\begin{align*}
  \snorm{(nD)\inv} = O_d(d^{-\delta}), ~~~ \snorm{\offd\round{\frac{\Phi_{\leq p}\tran \Phi_{\leq p}}{n}}} = O_{d,\P}(d^{-\delta/2 + \epsilon}),~~~\snorm{\frac{\Phi_{\leq p}\tran}{\sqrt{n}}} = O_{d,\P}(1), \\
  \snorm{d^{-p-1}{ \offd(\Phi_{\S_{p+1}}\Phi_{\S_{p+1}}\tran)}} = O_{d,\P}(d^{(\delta-1)/2+\epsilon}),~~~\snorm{\Delta_3} = O_{d,\P}(d\halfe),
\end{align*}
where particularly the second estimate can be easily deduced from Lemma~\ref{lemma:gram_matrix} since $\diag\round{\frac{\Phi_{\leq p}\tran \Phi_{\leq p}}{n}} = I_n$.

 Therefore, a completely analogous argument to Claim~\ref{claim:truncate_h} shows that we can find a constant $c_G = \max(c_H, \ceil{\frac{1-\delta+4\epsilon}{\delta-2\epsilon}})$ and a matrix $\Delta_4$ that satisfies $\snorm{\Delta_4} = O_{d,\P}(d\halfe)$ such that the following holds
\begin{align}\label{eq:g_inv_1}
        G\inv = I +\sum_{r=1}^{c_G} \round{ -(\rho+\lambda)(nD)\inv -\offd\round{\frac{\Phi_{\leq p}\tran \Phi_{\leq p}}{n}}  - \frac{\Phi_{\leq p}\tran}{\sqrt{n}}\sum_{q=1}^{c_H} \round{-c_{p+1} \frac{ \offd(\Phi_{\S_{p+1}}\Phi_{\S_{p+1}}\tran)}{d^{p+1}}}^{q} \frac{\Phi_{\leq p}}{\sqrt{n} }}^{r} +\Delta_4.
\end{align}
To ease the analysis, we replace $\offd\round{\frac{\Phi_{\leq p}\tran \Phi_{\leq p}}{n}}$ by $  \frac{\Phi_{\leq p}\tran \Phi_{\leq p}}{n} -I_{|L(p)|} $ and replace $\frac{\offd\round{\Phi_{\S_{p+1}} \Phi_{\S_{p+1}}\tran}}{d^{p+1}}$ by  $ \frac{\Phi_{\S_{p+1}} \Phi_{\S_{p+1}}\tran}{d^{p+1}} - h_{p+1} I_n$, where $h_{p+1}={d \choose p+1} / d^{p+1} = O_d(1)$.  Then Eq.~\eqref{eq:e_inv_1} and \eqref{eq:g_inv_1} take the form 

\begin{align}\label{eq:e_inv}
  H\inv =  I_n + \sum_{q=1}^{c_H} \round{ -c_{p+1} \frac{\Phi_{\S_{p+1}} \Phi_{\S_{p+1}}\tran}{d^{p+1}} + c_{p+1}h_{p+1} I_n}^{q} + \Delta_2,
\end{align}
and
\begin{align}\label{eq:g_inv}
G^{-1} &= I_n + \sum_{r=1}^{c_G} \Biggl[ -(\rho+\lambda)(nD)^{-1} + I_{|L(p)|} - \frac{\Phi_{\leq p}\tran \Phi_{\leq p}}{n} \notag\\[1mm]
&\quad\; + \frac{\Phi_{\leq p}\tran}{\sqrt{n}} \sum_{q=1}^{c_H} \Biggl( -c_{p+1} \frac{\Phi_{\S_{p+1}} \Phi_{\S_{p+1}}\tran}{d^{p+1}} + c_{p+1}h_{p+1} I_n \Biggr)^q \frac{\Phi_{\leq p}}{\sqrt{n}} \Biggr]^{r} + \Delta_4.
\end{align}
Plugging the approximations of $H\inv$ (Eq.~(\ref{eq:e_inv})) and $G\inv$ (Eq.~(\ref{eq:g_inv})) into Eq.~(\ref{eq:simplified_g_p}) and expanding the powers we see that $ b_{\gE^c}\tran \frac{\Phi_{\gE^c}\tran }{\sqrt{n}}H\inv \frac{\Phi_{\leq p}}{\sqrt{n}}G\inv E_r$ can be written as a finite sum:
\begin{align}\label{eq:finit_summation}
    b_{\gE^c}\tran \frac{\Phi_{\gE^c}\tran }{\sqrt{n}}H\inv \frac{\Phi_{\leq p}}{\sqrt{n}}G\inv E_r = \sum_{i=1}^{c_p} b_{\gE^c}\tran  \frac{\Phi_{\gE^c}\tran}{\sqrt{n}} P_i + \Delta_5,
\end{align}
where $c_p\in \mathbb{N}$ is a constant and each matrix $P_i \in \R^{L(p)\times L(p)}$ takes the form
\begin{align*}
    P_i = C_{i}A_{i_1}A_{i_2}\ldots A_{i_m} \frac{\Phi_{\leq p} (nD)^{-t_i}E_r}{\sqrt{n}},
\end{align*}
with each $A_{i_j} \in \biground{I_n, \frac{\Phi_{\S_{p+1}}\Phi_{\S_{p+1}}\tran}{d^{p+1}}} \cup \biground{\frac{\Phi_{\leq p} (nD)^{-s} \Phi_{\leq p}\tran}{n}}_{s=0}^{c_G}$ for $j\in[m]$.
Here $C_i>0$ is a constant and $t_i\in\mathbb{N}$.
Additionally, we can deduce that $\Delta_5$ satisfies $$\norm{\Delta_5}_2 = O_{d,\P}(\log^{1/2}(d))\cdot  \norm{f_{\gE^c}^*}_{L_2} \cdot (\snorm{\Delta_2}+\snorm{\Delta_4}) =  O_{d,\P}(d\half)\cdot  \norm{f_{\gE^c}^*}_{L_2},$$
where we use the estimate of $\snorm{H\inv} = O_{d,\P}(1)$ (Eq.~\eqref{eq:e_inv_minus_i}), $\snorm{G\inv} = O_{d,\P}(1)$ (Eq.~\eqref{eq:g_inv_minus_i}), $\snorm{\Phi_{\leq p}/\sqrt{n}} = O_{d,\P}(1)$ (Lemma~\ref{lem:norm_control}) and $ \norm{ \frac{\Phi_{\gE^c}b_{\gE^c} }{\sqrt{n}}}_2 =O_{d,\P}(\log^{1/2}(d))\cdot  \norm{f_{\gE^c}^*}_{L_2}$ (Proposition~\ref{prop:bound_y}).

The following proposition shows that each summand is small. We defer the proof to Appendix~\ref{proof:feature_product}.
\begin{proposition}\label{prop:feature_product}
Let $A_{j}$ for $j\in[m]$ be defined as one of the following matrices:
\begin{align*}
I_n,~~\frac{\Phi_{\S_{p+1}}\Phi_{\S_{p+1}}\tran}{d^{p+1}},~~\biground{\frac{\Phi_{\leq p} (nD)^{-s} \Phi_{\leq p}\tran}{n}}_{s=0}^{c_G}.
\end{align*}
Then for any $t \in \mathbb{N}$ the following estimate holds 
\begin{align}\label{eq:feature_product}
\E\brac{\norm{\frac{1}{n}b_{\gE^c}\tran  {\Phi_{\gE^c}\tran}A_{1}A_{2}\ldots A_{m} {\Phi_{\leq p} (nD)^{-t}E_{r}}}_2^2} = O_{d,\P}(d^{-\delta-1 })\cdot \norm{f_{\gE^c}^*}_{L_2}^2.
\end{align}
\end{proposition}

\noindent Next, using the hypercontractivity Assumption~\ref{ass:hypercontr} we show that the estimate \eqref{eq:feature_product} yields an analogous bound uniformly for all $r\in [d]$.
To this end, Assumption~\ref{assump:finite_basis_hypercube} imposes that $f^*$ is a finite-degree polynomial. Consequently, we can see that  $\inner{b_{\gE^c}, \phi_{\gE^c}}$ is also finite-degree. Then we can deduce that each norm in Eq.~\eqref{eq:feature_product} is also finite-degree since each term is the product of finitely many finite-degree polynomials. 
Therefore, invoking Lemma~\ref{lemma:uniform_bound}, for any $q_0\geq 1$, there exists a constant $C_{q_0}>0$ such that:
\begin{align*}
    &~~~\E\brac{\max_r \norm{\frac{1}{n}b_{\gE^c}\tran  {\Phi_{\gE^c}\tran}A_{i_1}A_{i_2}\ldots A_{i_q} {\Phi_{\leq p} (nD)^{-t}E_{r}}}_2^2}\\
    &\leq  C_{q_0} d^{1/q_0}   \max_{r\in[d]}\E\brac{\norm{\frac{1}{n}b_{\gE^c}\tran  {\Phi_{\gE^c}\tran}A_{i_1}A_{i_2}\ldots A_{i_q} {\Phi_{\leq p} (nD)^{-t}E_{r}}}_2^{2}}\\
    &= C_{q_0} d^{1/q_0} \cdot O_{d,\P}(d^{-\delta-1 })\cdot \norm{f_{\gE^c}^*}_{L_2}^2,
\end{align*}
where the last equality follows from Eq.~\eqref{eq:feature_product}.

Using Markov's inequality, we have with probability as least $1- 1/\log(d)$ the following holds:
\begin{align*}
   \max_r\norm{\frac{1}{n}b_{\gE^c}\tran  {\Phi_{\gE^c}\tran}A_{i_1}A_{i_2}\ldots A_{i_q} {\Phi_{\leq p} (nD)^{-t}E_{r}}}_2^2 = O_{d,\P}(d^{-1-\delta+1/q_0}\log(d))\cdot \norm{f_{\gE^c}^*}_{L_2}^2.
\end{align*}
Plugging this bound into Eq.~\eqref{eq:finit_summation} and using Cauchy-Schwartz inequality yields
\begin{align*}
    \max_r\norm{b_{\gE^c}\tran \frac{\Phi_{\gE^c}\tran }{\sqrt{n}}H\inv \frac{\Phi_{\leq p}}{\sqrt{n}}G\inv E_{r}}_2^2 &=\round{ O_{d,\P}(d^{-1-\delta+1/q_0}\log(d)) + O_{d,\P}(d\inv)}\cdot \norm{f_{\gE^c}^*}_{L_2}^2 \\
    &= O_{d,\P}(d\inv)\cdot \norm{f_{\gE^c}^*}_{L_2}^2,
\end{align*}
where we let $q_0 > 1/\delta$.
Plugging this bound into Eq.~(\ref{eq:simplified_g_p}) and applying the AM-GM inequality yields
\begin{align*}
     \norm{\beta_{\gE^c}\tran \Phi_{\leq p} R}_2^2 \leq  2 \norm{b_{\gE^c}\tran \frac{\Phi_{\gE^c}\tran }{\sqrt{n}}H\inv \frac{\Phi_{\leq p}}{\sqrt{n}}G\inv E_{r}}_2^2 + 2\norm{\Delta_1}^2_2 = O_{d,\P}(d\inv)\cdot \norm{f_{\gE^c}^*}_{L_2}^2,
\end{align*}
which completes the proof.

\subsubsection{Proof of Proposition~\ref{prop:cube_error_1}}\label{proof:cube_error_1}
Similar to Eq.~\eqref{eq:gaussian_error_tech_bound_1}, we compute the expectation of the norm square over $\vepsilon$ and apply trace inequality then we obtain
\begin{align}\label{eq:cube_error_tech_bound_1}
    \E_{\vepsilon}\brac{\norm{\beta_{\varepsilon}\tran \Phi_{\leq p}R}_2^2}&\leq \sigma_\varepsilon^2 \cdot \snorm{K_\lambda\inv \Phi_{\leq p}D}^2 \cdot \Tr\round{D\inv R  R\tran D\inv}.
\end{align}
For the term $\snorm{K_\lambda\inv \Phi_{\leq p}D}^2$, Eq.~\eqref{eq:smw_<_p} implies 
\begin{align*}
    K_\lambda\inv \Phi_{\leq p}D = \frac{1}{n}H\inv \Phi_{\leq p}G\inv.
\end{align*}
With $\snorm{H\inv}, \snorm{G\inv} = O_{d,\P}(1)$ established in Eq.~\eqref{eq:e_inv_minus_i} and \eqref{eq:g_inv_minus_i} and $\snorm{\Phi_{\leq p}} =O_{d,\P}(\sqrt{n})$ (Lemma~\ref{lem:norm_control}), we conclude
\begin{align}\label{eq:cube_error_tech_bound_2}
    \snorm{ K_\lambda\inv \Phi_{\leq p}D }^2 = \snorm{\frac{1}{n}H\inv \Phi_{\leq p} G}^2 = O_{d,\P}(n\inv).
\end{align}
With Claim~\ref{prop:d_inv_r_cube}, following the arguments in the proof of Proposition~\ref{prop:gaussian_error_1}, we can show
\begin{align}\label{eq:cube_error_tech_bound_3}
    \Tr\round{D\inv R  R\tran D\inv}=O_d(\norm{E_r}_F^2) = O_d(d^{p-1}).
\end{align}
Plugging Eq.~\eqref{eq:cube_error_tech_bound_2} and \eqref{eq:cube_error_tech_bound_3} into Eq.~\eqref{eq:cube_error_tech_bound_1} and applying Markov's inequality completes the proof.

\subsubsection{Proof of Proposition~\ref{prop:feature_u_v} }\label{proof:feature_u_v}

The proof will mainly use the machinery developed in the proof for Proposition~\ref{lemma:beta_>_l2_cube}. 
Applying the same decomposition of $K_\lambda$ as Eq.~\eqref{eq:k_fancy_decomp}, we can write
\begin{align*}
    K_{\lambda} = K_p +(\rho+\lambda)H,
\end{align*}
where $H$ is defined in Eq.~\eqref{eq:E}.
Invoking the Sherman-Morrison-Woodbury formula for $K_\lambda\inv$, we  obtain
\begin{align}\label{eq:K_lambda_inv}
    K_\lambda\inv = (\rho+\lambda)\inv H\inv - \frac{1}{n}(\rho+\lambda)\inv H\inv \Phi_{\leq p}\tran G\inv \Phi_{\leq p} H\inv,
\end{align}
where 
$$G:= (\rho+\lambda)(nD)\inv + \frac{\Phi_{\leq p}\tran H\inv \Phi_{\leq p}}{n}.$$
Plugging the approximations of $H\inv$ (Eq.~(\ref{eq:e_inv})) and $G\inv$ (Eq.~(\ref{eq:g_inv})) into Eq.~(\ref{eq:K_lambda_inv}) and expanding the expression we can write it as the finite summation similar to Eq.~\eqref{eq:finit_summation}:
\begin{align}\label{eq:feature_product_split_p_1}
\frac{1}{n}b_{\Q}\tran \Phi_{\Q}\tran K_{\lambda}\inv \Phi_{\S_{p+1}^1} = \sum_{i=1}^{c_p} C_{i}b_{\Q}\tran  \frac{\Phi_{\Q}\tran}{\sqrt{n}} A_{i_1}A_{i_2}\ldots A_{i_m} \frac{\Phi_{\S_{p+1}^1} }{\sqrt{n}} + \Delta,
\end{align}
where  $A_{i_j} \in \biground{I_n, \frac{\Phi_{\S_{p+1}}\Phi_{\S_{p+1}}\tran}{d^{p+1}}} \cup \biground{\frac{\Phi_{\leq p} (nD)^{-s} \Phi_{\leq p}\tran}{n}}_{s\geq0}$, $C_i>0$ is a constant, $c_p,m\in \mathbb{N}$, and $\Delta$ satisfies $\snorm{\Delta} = O_{d,\P}(d\half)\cdot  \norm{b_\Q}_{2}$.

Note that the effective degree of the feature in $\Phi_{\S_{p+1}^1}$ is bounded by $p$ since $\Phi_{\S_{p+1}^1} = \Phi_{\S_p^0}\cdot X_r$ , following an almost identical argument with the proof of Proposition~\ref{prop:feature_product}, we can show each summand satisfies:
\begin{align*}
    \E\brac{\norm{\frac{1}{n}b_{\Q}\tran  {\Phi_{\Q}\tran}A_{i_1}A_{i_2}\ldots A_{i_m} {\Phi_{\S_{p+1}^1}}}_2^2} = O_{d,\P}(d^{-\delta })\cdot \norm{b_{\Q}}_{2}^2.
\end{align*}
Following a similar analysis as in the proof of Proposition~\ref{lemma:beta_>_l2_cube}, we use Lemma~\ref{lemma:uniform_bound} and Markov's inequality together to obtain the following bound, which holds uniformly for all $r\in[d]$:
\begin{align*}
    \norm{\frac{1}{n}b_{\Q}\tran  {\Phi_{\Q}\tran}A_{i_1}A_{i_2}\ldots A_{i_m} {\Phi_{\S_{p+1}^1}}}_2^2 =O_{d,\P}(d^{-\delta + 1/{q_0} })\cdot \log(d)\cdot \norm{b_{\Q}}_{2}^2,
\end{align*}
where $q_0>0$ is an arbitrarily large constant.

Combining with $\Delta$ in Eq.~\eqref{eq:feature_product_split_p_1} and applying the triangle inequality we obtain the bound
\begin{align*}
    \norm{\frac{1}{n}b_{\Q}\tran \Phi_{\Q}\tran K_{\lambda}\inv \Phi_{\S_{p+1}^1}}_2^2 = O_{d,\P}(d^{-\delta + 1/{q_0} })\cdot \log(d)\cdot \norm{b_{\Q}}_{2}^2.
\end{align*}
Multiplying both sides of the equation by $(n/d^{p+1})^2$ yields
\begin{align}
   \norm{\frac{1}{d^{p+1}} \beta_{\Q}^{\top} \Phi_{\S_{p+1}^1}}_2^2 &= O_{d,\P}(d^{\delta-2+1/q_0}) \log(d) \cdot \norm{b_{\Q}}_2^2 =  O_{d,\P}(d\inv) \cdot \norm{b_{\Q}}_2^2,\label{eq:s2_s2_cube_bound_3}
\end{align}
where the last equality holds by letting $q_0 >1/\delta$.

\subsubsection{Proof of Proposition~\ref{prop:term_p_plus_one}}\label{proof:term_p_plus_one}
Note that Eq.~(\ref{eq:K_decomp}) gives the decomposition
\begin{align}
    K_\lambda  = \Phi_{\leq p}D\Phi_{\leq p}\tran + (\rho+\lambda)H,
\end{align}
where $H = I_n + (\rho+\lambda)\inv \Delta_1$, $\rho = g(1) - g_p(1)$ and $\Delta_1$ satisfies $\snorm{\Delta_1} = O_{d,\P}(d^{(\delta-1)/2+\epsilon})$.

Invoking Sherman-Morrison-Woodbury formula to compute the inverse of $K_\lambda$ and we obtain:
\begin{align*}
    K_\lambda \inv = (\rho+\lambda)\inv H\inv -(\rho+\lambda)\inv \frac{1}{n}H\inv \Phi_{\leq p}G\inv \Phi_{\leq p}\tran H\inv.
\end{align*}
Invoking Eq.~\eqref{eq:e_inv_minus_i} and  \eqref{eq:g_inv_minus_i} that $H\inv$ and $G\inv$ are close to $I$, we can show 
\begin{align*}
   \snorm{K_\lambda\inv - (\rho+\lambda)\inv\round{I_n -\frac{\Phi_{\leq p}\Phi_{\leq p}\tran}{n}}} = O_{d,\P}(d^{-\delta/2+\epsilon} + d^{(\delta-1)/2+\epsilon}),
\end{align*}
where we use the estimate $\snorm{\Phi_{\leq p}} =O_{d,\P}(\sqrt{n})$ (Lemma~\ref{lem:norm_control}).

Plugging the above bound into left side of Eq.~\eqref{eq:term_p_plus_one} and applying triangle inequality yields
\begin{align}\label{eq:weak_feature_bound_0}
   &~~~\snorm{\frac{1}{n}b_{\S_{p+1}^1}\tran \Phi_{\S_{p+1}^1}\tran K_{\lambda}\inv \Phi_{\S_{p+1}^1} - \frac{1}{n(\rho+\lambda)}b_{\S_{p+1}^1}\tran \Phi_{\S_{p+1}^1}\tran \round{I_n -\frac{\Phi_{\leq p}\Phi_{\leq p}\tran}{n}} \Phi_{\S_{p+1}^1}}\notag\\
   &= O_{d,\P}(d^{-\delta/2+\epsilon} + d^{(\delta-1)/2+\epsilon})\cdot \norm{b_{\S_{p+1}^1}}_2.
\end{align}
Now we consider $\frac{1}{n}b_{\S_{p+1}^1}\tran \Phi_{\S_{p+1}^1}\tran \round{I_n -\frac{\Phi_{\leq p}\Phi_{\leq p}\tran}{n}} \Phi_{\S_{p+1}^1}$.  Using the triangle inequality gives
\begin{align}\label{eq:weak_feature_bound_1}
  &~~~\norm{\frac{1}{n}b_{\S_{p+1}^1}\tran \Phi_{\S_{p+1}^1}\tran \round{I_n -\frac{\Phi_{\leq p}\Phi_{\leq p}\tran}{n}} \Phi_{\S_{p+1}^1} - b_{\S_{p+1}^1}\tran}_2 \notag \\
  &\leq \norm{\frac{1}{n}b_{\S_{p+1}^1}\tran \Phi_{\S_{p+1}^1}\tran \Phi_{\S_{p+1}^1} - b_{\S_{p+1}^1}\tran} +  \norm{\frac{1}{n}b_{\S_{p+1}^1}\tran \Phi_{\S_{p+1}^1}\tran \frac{\Phi_{\leq p}\Phi_{\leq p}\tran}{n} \Phi_{\S_{p+1}^1}}_2.
\end{align}
Next we show that both terms on the right side of the above inequality are small.

For the first term, note that Theorem~\ref{lemma:gram_matrix} immediately gives 
\begin{align}\label{eq:weak_feature_bound_2}
\norm{\frac{1}{n}b_{\S_{p+1}^1}\tran \Phi_{\S_{p+1}^1}\tran \Phi_{\S_{p+1}^1} - b_{\S_{p+1}^1}\tran} \leq \norm{b_{\S_{p+1}^1}\tran}_2 \snorm{\frac{1}{n}\Phi_{\S_{p+1}^1}\tran \Phi_{\S_{p+1}^1} - I_n} =  \norm{b_{\S_{p+1}^1}\tran}_2 \cdot O_{d,\P}(d^{-\delta/2+\epsilon}).
\end{align}
We invoke Lemma~\ref{lemma:feature_product_tech} to bound the expectation of the second term, i.e.,
\begin{align*}
    \E\brac{\norm{\frac{1}{n}b_{\S_{p+1}^1}\tran \Phi_{\S_{p+1}^1}\tran \frac{\Phi_{\leq p}\Phi_{\leq p}\tran}{n} \Phi_{\S_{p+1}^1}}_2^2}.
\end{align*}
Specifically, we let non-overlapping sets $\T_1 = \{r\}$ and $\T_2 = [d]\setminus \{r\}$ with $s_1 = 0$ and $s_2 = 1$. Corresponding to the partition of the coordinate set, we can decompose the index of features such that we have  $S \in 2^{\T_1}\times 2^{\T_2}$ for any feature index $S$ that appears in the product.    The effective degree of the feature in $\Phi_{\S_{p+1}}^1$ and $\Phi_{\leq p}$ is both bounded by $s_1\cdot 1 + s_2 \cdot p  = p$. Therefore, Lemma~\ref{lemma:feature_product_tech} shows the expectation is bounded by 
\begin{align*}
     \norm{b_{\S_{p+1}^1}}_2^2 \cdot O_d(d^{p}/n) =\norm{b_{\S_{p+1}^1}}_2^2 \cdot O_d(d^{-\delta}) .
\end{align*}
Following a similar analysis as in the proof of Proposition~\ref{lemma:beta_>_l2_cube}, we use Lemma~\ref{lemma:uniform_bound} and Markov's inequality together to obtain the following bound, which holds uniformly for all $r\in[d]$:
\begin{align*}
  \norm{\frac{1}{n}b_{\S_{p+1}^1}\tran \Phi_{\S_{p+1}^1}\tran \frac{\Phi_{\leq p}\Phi_{\leq p}\tran}{n} \Phi_{\S_{p+1}^1}}_2^2 = \norm{b_{\S_{p+1}^1}}_2^2\cdot O_d(d^{-\delta+1/q_0})\log(d),
\end{align*}
where $q_0>0$ is an arbitrarily large constant.

Plugging this bound and Eq.~\eqref{eq:weak_feature_bound_1} into Eq.~\eqref{eq:weak_feature_bound_2} yields:
\begin{align*}
   \norm{ \frac{1}{n}b_{\S_{p+1}^1}\tran \Phi_{\S_{p+1}^1}\tran \round{I_n -\frac{\Phi_{\leq p}\Phi_{\leq p}\tran}{n}} \Phi_{\S_{p+1}^1} - b_{\S_{p+1}^1}\tran}_2 = \norm{b_{\S_{p+1}^1}}_2 \cdot O_{d,\P}(d^{-\delta/2+\epsilon}),
\end{align*}
where $\epsilon>0$ is an arbitrarily small constant and we absorb factor $\log(d)$ into $d^{\epsilon}$.

Finally, combining this bound with Eq.~\eqref{eq:weak_feature_bound_0} completes the proof.

\subsubsection{Proof of Proposition~\ref{prop:noise_term_p_plus_one}}\label{proof:noise_term_p_plus_one}
Taking the expectation of the norm square with respect to the noise $\vepsilon$ yields
\begin{align}\label{eq:noise_term_p_plus_one_tech_bound1}
    \E\brac{\norm{d^{-p-1} \beta_\varepsilon\tran \Phi_{\S_{p+1}^1}}_2^2} = d^{-2p-2}\Tr\round{K_\lambda\inv \Phi_{\S_{p+1}^1} \Phi_{\S_{p+1}^1}\tran K_\lambda\inv}.
\end{align}
Applying the inequality $|\Tr(AB)| \leq \Tr(A)\snorm{B}$ yields
\begin{align}\label{eq:noise_term_p_plus_one_tech_bound2}
    \Tr\round{K_\lambda\inv \Phi_{\S_{p+1}^1} \Phi_{\S_{p+1}^1}\tran K_\lambda\inv} = \Tr\round{K_\lambda^{-2} \Phi_{\S_{p+1}^1} \Phi_{\S_{p+1}^1}\tran } \leq \snorm{K_\lambda^{-2} }\Tr\round{\Phi_{\S_{p+1}^1} \Phi_{\S_{p+1}^1}\tran}.
\end{align}
The bound $ \snorm{K_\lambda^{-2}}= \snorm{K_\lambda^{-1}}^2 = O_{d,\P}(1)$ is implied by  the expression \eqref{eqn:K_decomp_cube}. Next we analyze $\Tr\round{\Phi_{\S_{p+1}^1} \Phi_{\S_{p+1}^1}\tran}$.

Since $|\S_{p+1}^1| = O_d(d^p)$, following Theorem~\ref{lemma:phi_id_2}, we have $\snorm{\Phi_{\S_{p+1}^1} \Phi_{\S_{p+1}^1}\tran/n - I_{|\S_{p+1}^1|}} = O_{d,\P}(d^{-\delta/2+\epsilon})$. Then, we can conclude 
\begin{align*}
    \Tr\round{\frac{1}{n}\Phi_{\S_{p+1}^1} \Phi_{\S_{p+1}^1}\tran}\leq O_d(d^p)(1+ O_{d,\P}\cdot (d^{-\delta/2+\epsilon})) = O_{d,\P}(d^p).
\end{align*}
Together with Eq.~\eqref{eq:noise_term_p_plus_one_tech_bound2}, the right side of Eq.~\eqref{eq:noise_term_p_plus_one_tech_bound1} is bounded by $O_{d,\P}(d^{-2+\delta})$.
Applying Markov's inequality we have $\norm{d^{-p-1} \beta_\varepsilon\tran \Phi_{\S_{p+1}^1}}_2^2 = O_{d,\P}(d^{-2+\delta})\log(d)$ with probability at least $1-1/ \log(d)$. Absorbing the logarithmic factor into $\epsilon$ completes the proof.

\subsubsection{Proof of Proposition~\ref{prop:feature_product}}\label{proof:feature_product}

Note that $(nD)^{-t}$ is diagonal matrix that satisfies $(nD)^{-t}\lesssim I$ and $E_r$ is a diagonal matrix. We can upper bound the norm by removing $(nD)^{-t}$:
\begin{align*}
    &~~~\norm{ \frac{1}{n}b_{\gE^c}\tran  {\Phi_{\gE^c}\tran}A_{1}A_{2}\ldots A_{m} \Phi_{\leq p} (nD)^{-t}E_{r}}_2^2\\
    &=\norm{ \frac{1}{n}b_{\gE^c}\tran  {\Phi_{\gE^c}\tran}A_{1}A_{2}\ldots A_{m} \Phi_{\leq p} E_{r}(nD)^{-t}}_2^2\\
    &\lesssim O_d(1)\cdot \norm{ \frac{1}{n}b_{\gE^c}\tran  {\Phi_{\gE^c}\tran}A_{1}A_{2}\ldots A_{m} \Phi_{\gE}}_2^2
\end{align*}
where we use the fact $(E_r)_{S,S} = \mathbbm{1}\{S\in \gE\}$.
Clearly, the estimate holds:
\begin{align}\label{eq:feature_produc_bound1}
\E\brac{\norm{ \frac{1}{n}b_{\gE^c}\tran  {\Phi_{\gE^c}\tran}A_{1}A_{2}\ldots A_{m} \Phi_{\gE}}_2^2} \leq \norm{b_{\gE^c}}_2^2\cdot \snorm{\E\brac{ \underbrace{\frac{1}{n^2}{\Phi_{\gE^c}\tran}A_{1}\ldots A_{m} \Phi_{\gE}\Phi_{\gE}\tran A_{m}\ldots A_1 \Phi_{\gE^c}}_{=:K}} }.    
\end{align}
We now aim to upper bound the operator norm of $\E[K]$.
To this end, we decompose the feature matrices in $K$ so as to ensure that distinct matrices in the product correspond to nonoverlapping index sets.

Following Eq.~\eqref{eq:s_j_0_1}, we have the decomposition
\begin{align}\label{eq:feature_split}
    \S_k = \S^1_k \sqcup \S^0_k.
\end{align}
Clearly, the following equalities hold:
\begin{align*}
    \gE = \bigcup_{k=1}^{p}\S^1_k,~~~\gE^c = \round{\bigcup_{k=0}^p \S_k^0} \bigcup \round{\bigcup_{k=p+1}^\ell \S_k}. 
\end{align*}
Define now the collection of matrices
\begin{align*}
    \U:= \{\Phi_{\S_k^1}\}_{k=1}^{p} \bigcup \{\Phi_{\S_k^0}\}_{k=0}^{p} \bigcup \{\Phi_{\S_{k}}\}_{k=p+1}^\ell.
\end{align*}
Note that any two distinct matrices in $\U$ correspond to distinct index sets. 
Normalizing the matrices, we further define the collection
\begin{align*}
   \widetilde{\U}&:= \{n\half (nD)^{-s/2}\Phi_{\S_k^1}\}_{k\in[p],s\geq 0} \bigcup \{n\half (nD)^{-s/2}\Phi_{\S_k^0}\}_{0\leq k\leq p, s\geq 0} \bigcup \{d^{-k/2}\Phi_{\S_{k}}\}_{p+1\leq k\leq \ell}.
\end{align*}
Let us now write the two expressions:
\begin{align*}
        \frac{\Phi_{\leq p} (nD)^{-s}\Phi_{\leq p}\tran}{n} &= \sum_{k=0}^p \frac{\Phi_{\S_k^0} (nD)^{-s} \Phi_{\S_k^0}\tran}{n} + \sum_{k=1}^p \frac{\Phi_{\S_k^1} (nD)^{-s} \Phi_{\S_k^1}\tran}{n};\\
    \frac{\Phi_{
\gE} \Phi_{\gE}\tran}{n} &= \sum_{k=1}^p \frac{\Phi_{\S_k^1}\Phi_{\S_k^1}\tran}{n}.
\end{align*}
Plugging these decompositions for the corresponding $A_i$ in the expression of $K$ Eq.~\eqref{eq:feature_produc_bound1}, we can write $K$ as the summation of finitely many terms, where each term takes the form of some positive constant multiplied by
\begin{align*}
    K_j := \frac{\max(n^{1/2},d^{k_{1}/2}) \cdot  \max(n^{1/2},d^{k_{2m+1}/2})}{n}\cdot  B_{1}^\top \Bigl(B_{2} B_{2}\tran\Bigr) \ldots \Bigl(B_{{2m}} B_{{2m}}\tran\Bigr) B_{{2m+1}},
\end{align*}
where $B_{1} \in \R^{n\times k_1},\ldots,B_{{2m+1}}\in \R^{n\times k_{2m+1}}$ are from $\widetilde{\U}$. 

Particularity note that $B_1$ and $B_{2m+1}$ corresponding to the decomposition of $\Phi_{\gE^c}$ and thus each takes the form $d^{-k/2}\Phi_{\S_{k}}$ for some $p+1\leq k\leq \ell$; $B_{m+1}$ corresponds to the decomposition of $\Phi_{\gE}$ and takes the form $n\half\Phi_{\S_k^1}$ for some $k\in[p]$.  The equation contains an additional scaling factor on the right side, which arises from the distinct scaling of $B_1$ and $B_{2m+1}$ by a factor of $1/n$ in Eq.~\eqref{eq:feature_produc_bound1}.

Next, we invoke Lemma~\ref{lemma:feature_product_tech} to bound each $\snorm{\E\brac{K_j}}$. Specifically, we let non-overlapping sets $\T_1 = \{r\}$ and $\T_2 = [d]\setminus \{r\}$ with $s_1 = 0$ and $s_2 = 1$, since $|\T_1| = O_d(d^0)$ and $|\T_2| = O_d(d^1)$. Corresponding to this partition of the coordinate set, we can decompose the index of features such that we have  $S = S_1\sqcup S_2 \in 2^{\T_1}\times 2^{\T_2}$ for any feature index $S$ that appears in the product.    The effective degree of the feature of $B_{m+1}$ is bounded by $s_1\cdot 1 + s_2 \cdot (p-1)  = p$ since $x_r$ appears in all features, namely $\Phi_{\S_p^1} = X_r\cdot \Phi_{\S_{p-1}^0}$. Consequently, we yield
\begin{align}\label{eq:each_k_j}
    \snorm{\E\brac{K_j}} &= \frac{\max(n^{1/2},d^{k_{1}/2}) \cdot  \max(n^{1/2},d^{k_{2m+1}/2})}{n}\cdot \snorm{\E\brac{B_{1}^\top \Bigl(B_{2} B_{2}\tran\Bigr) \ldots \Bigl(B_{{2m}} B_{{2m}}\tran\Bigr) B_{{2m+1}}}}\notag\\
    &=\frac{\max(n^{1/2},d^{k_{1}/2}) \cdot  \max(n^{1/2},d^{k_{2m+1}/2})}{n}\cdot O_d(d^{p-1})\cdot \min(n^{-1/2},d^{-k_{1}/2}) \cdot  \min(n^{-1/2},d^{-k_{2m+1}/2})\notag\\
    &=O_d(d^{p-1}/n)= O_d(d^{-1-\delta}).
\end{align}
Applying triangle inequality yields  $  \snorm{\E\brac{K}} = O_d(d^{-1-\delta})$.  Combining with Eq.~\eqref{eq:feature_produc_bound1} completes the proof.

%% file: sections_appendix_b/prop_gradient.tex
\subsection{Proof of Propositions for Theorem~\ref{thm:learning_hypercube} (Gradient error)}

\subsubsection{Proof of Lemma~\ref{lemma:gram_matrix_cube_m}}\label{proof:gram_matrix_cube_m}

We similarly define off-diagnoal entries of $\frac{1}{d^{(\vs+\mathbf{1}-\vlambda)\tran \vm}} \Phi_{\S_{\vm}}(Z) \round{\Phi_{\S_{\vm}}(Z)}\tran $
by \begin{align*}
 \Delta_{i,j}
 :=\begin{cases}
      &\frac{1}{d^{(\vs+\mathbf{1}-\vlambda)\tran \vm}} \inner{\phi_{\S_{\vm}}( z^{(i)}), \phi_{\S_{\vm}}( z^{(j)})},~~\text{if}~ i\neq j,\\
      &0, ~~\text{otherwise}.
 \end{cases}
\end{align*}
We omit the argument $Z$ for simplicity. Then we split $\frac{1}{d^{(\vs+\mathbf{1}-\vlambda)\tran \vm}} \Phi_{\S_{\vm}} \round{\Phi_{\S_{\vm}}\tran}$ into the diagonal and off-diagonal parts. Namely, the triangle inequality yields 
\begin{align}\label{eq:g_m_triangle_m}
 &\E\brac{\snorm{\frac{1}{d^{(\vs+\mathbf{1}-\vlambda)\tran \vm}} \Phi_{\S_{\vm}} \round{\Phi_{\S_{\vm}}}\tran - \mu_{\S_{\vm}} I_n}}\\
 &~~~\leq \E\brac{\snorm{\frac{1}{d^{(\vs+\mathbf{1}-\vlambda)\tran \vm}} \Diag(\Phi_{\S_{\vm}} \round{\Phi_{\S_{\vm}}}\tran )  
 -  \mu_{\S_{\vm}} I_n}}+ \E\brac{\snorm{\Delta}},
\end{align}
where
\begin{align*}
    \mu_{\S_{\vm}}:=\frac{1}{d^{(\vs+\mathbf{1}-\vlambda)\tran \vm}} \E\brac{\norm{\phi_{\S_{\vm}}}_2^2}.
\end{align*}
Now we show $\mu_{\S_{\vm}} = \Theta_d(1)$. 
Since $\T_1,\ldots,\T_k$ do not overlap, we can write the norm of $\phi_{\S_{\vm}}$ as the product of the norm corresponding to each $\T_k$:
\begin{align}\label{eq:tensor_product}
     \norm{\phi_{\S_{\vm}}}_2^2 = \prod_{k=1}^N \norm{\phi_{\T_k^{m_k}}}_2^2,
\end{align}
where we denote $\T_k^{m_k}:=\{T_k~\vert~T_k\in 2^{\T_k}~~{\rm and}~~|T_k| = m_k\}$.

Recall that for any $r\in \T_k$, we have  $w_r =\Theta_d(d^{\gamma_{k}})= \Theta_d(d^{1- \lambda_{k}})$. Using simple combinatorial arguments
we can infer
\begin{align*}
    \norm{\phi_{\T_k^{m_k}}}_2^2  =  \sum_{\substack{1\leq r_1 <r_2<,\ldots,< r_{m_k} \\ r_1,\ldots,r_{m_k} \in \T_k}}w_{r_1}\ldots w_{r_{m_k}} = \Theta_d(d^{(1-\lambda_k) m_{k}}) \cdot \frac{1}{m_k!}\cdot \Theta_d(d^{s_km_k}) = \frac{1}{m_k!}\cdot \Theta_d(d^{(s_k+1 -\lambda_k)m_k}).
\end{align*}
Therefore, plugging each factor into Eq.~\eqref{eq:tensor_product} yields 
\begin{align}\label{eq:norm_psi}
   \norm{\phi_{\S_{\vm}}}_2^2 =  \frac{1}{\vm!}\cdot \Theta_d(d^{(\vs+\mathbf{1}-\vlambda)\tran \vm}), 
\end{align}
 which implies $\mu_{\S_{\vm}} = \Theta_d(d^{-(\vs+\mathbf{1}-\vlambda)\tran \vm})\norm{\phi_{\S_{\vm}}}_2^2 = \Theta_d(1)$.

\paragraph{Bounding $\E\brac{\snorm{\frac{1}{d^{(\vs+\mathbf{1}-\vlambda)\tran \vm}}  \diag(\Phi_{\S_{\vm}}   \Phi_{\S_{\vm}}\tran)  
 - \mu_{\S_{\vm}} I_n}}$.} 

Note that $\norm{\phi_{\S_{\vm}}}_2^2 = \E\brac{\norm{\phi_{\S_{\vm}}}_2^2}$ holds. Therefore, we directly have $\frac{1}{d^{(\vs+\mathbf{1}-\vlambda)\tran \vm}} \diag(\Phi_{\S_{\vm}}   \Phi_{\S_{\vm}} \tran)  
 =\mu_{\S_{\vm}} I_n $ by definition of $\mu_{\S_\vm}$ hence we can conclude a stronger result
\begin{align}\label{eq:diag_psi_cube_m}
  \snorm{\frac{1}{d^{(\vs+\mathbf{1}-\vlambda)\tran \vm}}  \Diag(\Phi_{\S_{\vm}}   \Phi_{\S_{\vm}}\tran)  
 - \mu_{\S_{\vm}} I_n} = 0.
\end{align}

\paragraph{Bound $\E\brac{\snorm{\Delta}}$.} We follow the steps in the proof of Lemma~\ref{lemma:gram_matrix}. Specifically, we use the same bound as Eq.~(\ref{eq:g_m_off_diag_bound_main}).  
\begin{align}\label{eq:g_m_off_diag_bound_main_m}
    \E\brac{\snorm{\Delta}} &\leq 4 n^{1/2}\sup_{T\subseteq [n]} \round{\E_{T^c}\brac{\Sigma(T)}}^{1/2} + 4c\log^{1/2}(n)  \cdot \sup_{T\subseteq [n]} \round{\E_{T^c}\brac{\Gamma(T)}}^{1/2}.
\end{align}

\paragraph{Step 1: Bound $\E_{T^c}\brac{\Sigma(T)}$.}Note that each entry of $\Delta_{T^c,j}\Delta_{T^c,j}^\top$ takes the form $\phi_i\tran \phi_j \phi_j\tran \phi_l / d^{2(\vs+\mathbf{1}-\vlambda)\tran \vm}$ for $ i,l\in T^c, j\in T$. Since $\E_{z^{(j)}}\brac{\phi_j \phi_j\tran} = I_{|\S_{\vm}|}\cdot \Theta_d(d^{(\mathbf{1}-\vlambda)\tran \vm})$, we have each entry of $\E_{z^{(j)}}\brac{\Delta_{T^c,j}\Delta_{T^c,j}^\top}$ takes the form $\phi_i\tran \phi_l \cdot \Theta_d(d^{(\vlambda-\mathbf{1}-2\vs)\tran \vm})$. Therefore, we deduce
\begin{align*}
        \Sigma(T) &= \snorm{\E_{z^{(j)}}\brac{\Delta_{T^c,j}\Delta_{T^c,j}^\top}} \\
        &= \snorm{\round{\Phi\Phi\tran}_{T^c,T^c} }\cdot \Theta_d(d^{(\vlambda-\mathbf{1}-2\vs)\tran \vm}) \\&\leq \E_T\brac{\snorm{\Phi\Phi\tran }}\cdot\Theta_d(d^{(\vlambda-\mathbf{1}-2\vs)\tran \vm}),
\end{align*}
where in the last inequality, we use Cauchy's interlacing theorem~\cite{horn2012matrix}. Specifically, the theorem shows that adding a row and a column to a Hermitian matrix will only increase its operator norm.

With a similar analysis, corresponding to Eq.~(\ref{eq:g_m_off_diag_bound_2}), we have
\begin{align}\label{eq:g_m_off_diag_bound_2_m}
    \E_{T^c}\brac{\Sigma(T)} \leq \Theta_d(d^{-\vs\tran \vm})\cdot \E\brac{\snorm{\Delta}} +\Theta_d(d^{(\vlambda-\mathbf{1}-2\vs)\tran \vm})
    \cdot \E\brac{\max_{j\in [n]}\norm{\phi_{\S_{\vm}}( z^{(j)})}^2}.
\end{align}
With Eq.~(\ref{eq:norm_psi}), there exists a constant  $C_\Sigma>0$  such that the following holds
\begin{align}\label{eq:g_m_off_diag_bound_3_m}
    \E_{T^c}\brac{\Sigma(T)} \leq  C_\Sigma d^{-\vs\tran \vm}\cdot\E\brac{\snorm{\Delta}} + C_\Sigma d^{-\vs\tran \vm}.
\end{align}

\paragraph{Step 2: Bound $\E_{T^c}\brac{\Gamma(T)}$.}
With a similar analysis, corresponding to Eq.~(\ref{eq:g_m_off_diag_bound_3}) and (\ref{eq:g_m_off_diag_bound_3.5}), we can obtain for any $q\geq 1$, 
\begin{align*}
     \E_{T^c}\E_T \brac{\max_{i\in T}\norm{\Delta_{i,T^c}}^2} &= \E \brac{\max_{i\in T}\norm{\Delta_{i,T^c}}^2} \leq n^{1+2/q} C_{2l,2q}^2\E\brac{\Delta_{i,j}^{2q}}.
\end{align*}
Since $\Delta_{i,j}$ is polynomial with  degree at most $2\emp_\gamma$, the hypercontractivity Assumption~\ref{ass:hypercontr} implies
\begin{align}
    \E\brac{\Delta_{i,j}^{2q}}^{1/q} \leq C_{2\emp_\gamma,2q}^2\E\brac{\Delta_{i,j}^{2}}.
\end{align}
Note that $ |\phi_u|^2 \leq c_0 d^{(\mathbf{1}-\vlambda)\tran \vm}$ holds for all $u\in \S_\vm$ where $c_0>0$ is an absolute constant. Then we can show
\begin{align*}
    \E\brac{\Delta_{i,j}^{2}}  = \sum_{u\in\S_\vm}\frac{1}{d^{2(\vs+\mathbf{1}-\vlambda)\tran \vm}} |\phi_u|^4 \leq |\S_\vm|\cdot {c_0^2} d^{-2\vs\tran \vm} \leq  {c_0^2} d^{-\vs\tran \vm}
\end{align*}
where we use the inequality $|\S_\vm| \leq d^{\vs\tran \vm}$.

Therefore, we  obtain a similar bound to Eq.~(\ref{eq:g_m_off_diag_bound_5}): 
\begin{align}\label{eq:g_m_off_diag_bound_5_m}
    \E_{T^c}\brac{\Gamma(T)} \leq c_0^2 C_{2\emp_\gamma,2q}^2n^{1+2/q}d^{-\vs\tran \vm}.
\end{align}
Combining Eq.(\ref{eq:g_m_off_diag_bound_2_m}) and Eq.~(\ref{eq:g_m_off_diag_bound_5_m}) with Eq.~(\ref{eq:g_m_off_diag_bound_main_m}) we have
\begin{align}
     &~~~\E\brac{\snorm{\Delta}} \notag\\
     &\leq 4 n^{1/2}\round{\frac{C_\Sigma}{d^{\vs\tran \vm}} \E\brac{\snorm{\Delta}} + \frac{C_\Sigma}{d^{\vs\tran \vm}}}^{1/2} + 4c\log^{1/2}(n)  \round{c_0^2 C_{2\emp_\gamma,2q}^2n^{1+2/q}d^{-\vs\tran \vm}}^{1/2}\notag\\
     &\leq 4 C_{\Sigma}^{1/2}\round{\frac{n}{d^{\vs\tran \vm}}}^{1/2}\round{\E\brac{\snorm{\Delta}}}^{1/2}+4 C_{\Sigma}^{1/2}\frac{n^{1/2}}{d^{\frac{1}{2}\vs\tran \vm}} + 4cc_0 C_{2\emp_\gamma,2q}\cdot \frac{n^{\frac{1}{2}+\frac{1}{q}}\log^{\frac{1}{2}}(n)}{d^{\vs\tran \vm/2}},
\end{align}
where the last inequality is obtained from the elementary estimate $\sqrt{a+b}\leq \sqrt{a}+\sqrt{b}$ for $a,b\geq 0$.

Solving the quadratic inequality with respect to $\sqrt{\E\brac{\snorm{\Delta}}} $ we obtain:
\begin{align*}
    \sqrt{\E\brac{\snorm{\Delta}}}\leq 2  C_{\Sigma}^{1/2} \round{\frac{n}{d^{\vs\tran \vm}}}^{1/2} + 2\sqrt{\frac{n}{d^{\vs\tran \vm}}+\frac{n^{1/2}}{d^{\frac{1}{2}\vs\tran \vm}}+ cc_0 C_{2\emp_\gamma,2q}\cdot \frac{n^{\frac{1}{2}+\frac{1}{q}}\log^{\frac{1}{2}}(n)}{d^{\vs\tran \vm/2}}}
\end{align*}
Squaring both sides and using the inequality $(a+b)^2\leq 2a^2+2b^2$ and taking into account that $n = d^{q+\delta}$, we can conclude that the right side is bounded by a constant multiple of 
\begin{align}
 \max\round{d^{p+\delta- \vs\tran \vm}, \sqrt{\log(n)} \cdot d^{(p+\delta - \vs\tran \vm)/2 +\epsilon}},
\end{align}
for any $\epsilon>0$.

Shrinking $\epsilon$ to remove logarithmic factors and applying Markov's inequality completes the proof.

\subsubsection{Proof of Theorem~\ref{thm:taylor_approx_kernel_m}}\label{proof:taylor_approx_kernel_m}

The proof follows the proof of Theorem~\ref{thm:poly_approx}. Similarly, we first apply Lemma~\ref{lem:angle_app_cube_m} and we deduce that for any $s\in (0,1)$, for $i\neq j$, the estimate 
\begin{equation}\label{eqn:angle_cube_m}
   \abs{\frac{1}{d}\inner{ \sqrt{w}\odot  x^{(i)}, \sqrt{w}\odot  x^{(j)}} - \delta_{ij}} \leq s,
\end{equation}
holds with probability at least $1-2n^2\exp(-s^2d/(2\norm{w}_\infty))$. When $i = j$, we have
\begin{align}\label{eqn:norm_control_cube_m}
    \inner{\sqrt{w}\odot  x^{(i)}, \sqrt{w}\odot  x^{(i)}} = d.
\end{align}
We assume that the events \eqref{eqn:angle_cube_m} occurs for the remainder of the proof.

To simplify the notation, we denote ${\bar c} =g_{w,\emp}(1) = \sum_{k=0}^{\emp} \frac{g_{w}^{(k)}(0)}{k!}$ and we define the matrix $E_w(X):=K_w(X,X)-K_{w,\emp}(X,X)$.  Then splitting $E_M(X)$ into its diagonal and off-diagonal parts, we obtain a bound similar to Eq.~(\ref{eqn:decomp_kern}):
\begin{align}
\|E_w(X)- (  g_w(1) - {\bar c})   I_n\|_{\rm op}\leq \|E_w(X)-\Diag(E_w(X))\|_{F}+\max_{i=1,\ldots,n}|E_w(X)_{ii}-(  g(1) - {\bar c})  I_n|.
\end{align}
For off-diagonal entries, similar to Eq.~(\ref{eqn:third_eq_needed}), assuming $s<{\min(1/2,\epsilon)}$ we estimate
\begin{align}
     \|E_w(X)-\Diag(E_w(X))\|_{F}^2\leq \frac{4n^2\sup_{k\geq \emp+1}g^{(k)}(0)}{3(\emp+1)!}s^{2\emp+2}.\label{eqn:second_eq_needed_m}
\end{align}
For diagonal entries, invoking Eq.~(\ref{eqn:norm_control_cube_m}), we immediately have 
\begin{align}
    |E_w(X)_{ii}-(  g_w(1) - {\bar c}) I_n| = 0.
\end{align}
Then we conclude
\begin{align}\label{eqn:basic_plugin_cube_M}
        \|E_w(X)-(g_w(1)-\bar c)I_n\|_{\rm op} \leq &\sqrt{\frac{4\sup_{k\geq \emp+1}g^{(k)}(0)}{3(\emp+1)!}}s^{\emp+1}n.
\end{align}
Then for any $C>2$, we may set $s:=\sqrt{2C\log(n)\norm{w}_\infty/d}$ thereby ensuring that the event \eqref{eqn:angle_cube_m} occurs with probability at least $1-\frac{4}{n^{C-2}}$. Plugging these values into \eqref{eqn:basic_plugin_cube_M} yields the estimate 
$$\|E_w(x)-(g(1)-\bar c)I_n\|_{\rm op}\leq c d^{\frac{2p-\emp-1}{2}+\delta}C^{\frac{\emp+1}{2}}\norm{w}_\infty^{\frac{\emp+1}{2}}\log^{\frac{\emp+1}{2}}(n),$$
where $c:= \sqrt{\frac{4\sup_{k\geq \emp+1}g^{(k)}(0)}{3(\emp+1)!}}$ is a numerical constant. 

To meet the assumption $s<{\min(1/2,\epsilon)}$, $d$ needs to satisfy $\frac{d}{\norm{w}_\infty \log(d)}> \frac{\min(1/2,\epsilon)^2 c_1}{C(p+\delta)}$. 
The proof is complete.

\subsubsection{Proof of Lemma~\ref{lemma:kernel_to_mono_m}}\label{proof:kernel_to_mono_m}
We will analyze the component $\Phi_{\agp}(Z) D_{\agp} \Phi_{\agp}(Z)\tran$ in the composition ~\eqref{eq:k_z_decomp} therefore, in the proof, we consider $\vm$ that satisfies $\norm{\vm}_1 \leq \emp_\mgamma$.

We first present the proof for Eq.~\eqref{eq:kernel_to_mono_m}. We decompose the product as a summation based on $\vm$:
\begin{align}\label{eq:a_p_decomp}
     \Phi_{\agp} (Z) D_{\agp}  \Phi_{\agp} (Z) \tran =   \sum_{\lmgp }\Phi_{\S_{\vm}}(Z) D_{\S_{\vm}}{ \Phi_{\S_{\vm}}(Z)}\tran.
\end{align}
Since $\vs\tran \vm \leq \vlambda\tran \vm$ and $\vs\tran\vm$ can be less than $p+\delta$, we split the set $\{\vm~\vert~\vlambda\tran \vm > p+\delta\}$ into two sets:
\begin{align*}
    L_1:=\{\vm~\vert~\vs\tran \vm \geq p+\delta\},~~L_2:=\{\vm~\vert~\vs\tran \vm < p+\delta < \vlambda \tran \vm\}.
\end{align*}
For each $\vm \in L_1$, we can show
\begin{align}\label{eq:kernel_to_mono_m_tech}
    &~~~\snorm{d^{(\vlambda  - \vs)\tran \vm}\Phi_{\S_{\vm}}(Z) D_{\S_{\vm}} \Phi_{\S_{\vm}}(Z)\tran - g^{(\norm{\vm}_1)}(0)\mu_{\S_{\vm}} I} \notag\\
    &\overset{(a)} {\leq}\snorm{g^{(\norm{\vm}_1)}(0)d^{(\vlambda - \mathbf{1} - \vs)\tran \vm}\Phi_{\S_{\vm}}(Z)\Phi_{\S_{\vm}}(Z)\tran - g^{(\norm{\vm}_1)}(0)  \mu_{\S_{\vm}} I}\notag\\
    &~~~~+ O_{d}(d^{(\vlambda - \mathbf{1} - \vs)\tran \vm-1})\cdot\snorm{\Phi_{\S_{\vm}}(Z) \Phi_{\S_{\vm}}(Z)\tran}\notag\\
    &\overset{(b)} {=}O_{d,\P}(d^{(p+\delta - \vs\tran \vm)/2+\epsilon}) + O_{d,\P}(d^{-1}).
\end{align}
where $(a)$ follows from the definition of $D$ (Lemma~\ref{lemma:each_mono_cube}) and $(b)$ follows from Lemma~\ref{lemma:gram_matrix_cube_m}, which implies  $\snorm{\Phi_{\S_{\vm}}(Z) \Phi_{\S_{\vm}}(Z)\tran} = O_{d,\P}(d^{(\vs+\mathbf{1}-\vlambda)\tran \vm})$.

For each $\vm \in L_2$, by sub-multiplicativity of the spectral norm, we have
\begin{align}\label{eq:kernel_to_mono_m_tech_2}
    \snorm{\Phi_{\S_{\vm}}(Z) D_{\S_{\vm}} \Phi_{\S_{\vm}}(Z)\tran} &\leq \snorm{\Phi_{\S_{\vm}}(Z)}^2 \cdot \snorm{D_{\S_{\vm}}} \notag\\
    &= \snorm{\Phi_{\S_{\vm}}(X)}^2 \cdot \snorm{W_{\S_{\vm}}}\cdot\snorm{D_{\S_{\vm}}}\notag\\
    &\overset{(a)}{=} O_{d,\P}(d^{p+\delta})\cdot d^{-\vlambda\tran\vm} = O_{d,\P}(d^{-\zeta_2}).
\end{align}
where $(a)$ follows from Lemma~\ref{lem:norm_control}.

Consequently, multiplying both sides of Eq.~\eqref{eq:kernel_to_mono_m_tech} by $d^{(\vs -\vlambda  )\tran \vm}$ and plugging the resulted bound and Eq.~\eqref{eq:kernel_to_mono_m_tech_2} back into Eq.~(\ref{eq:a_p_decomp}) yields
\begin{align*}
    &~~~\snorm{\sum_{\vlambda\tran \vm \geq p+\delta}\Phi_{\S_{\vm}}(Z) D_{\S_{\vm}} \Phi_{\S_{\vm}}(Z)\tran  - \sum_{\vm\in L_1} d^{(\vs - \vlambda )\tran \vm} g^{(\norm{\vm}_1)}(0) \mu_{\S_{\vm}} I 
 -\sum_{\vm\in L_2} \Phi_{\S_{\vm}}(Z) D_{\S_{\vm}} \Phi_{\S_{\vm}}(Z)\tran}\\
    &\leq \sum_{\vm \in L_1 }O_{d,\P}(d^{(p+\delta - \vs\tran \vm)/2+\epsilon)} \cdot d^{(\vs-\vlambda)\tran \vm}) + \sum_{\vm \in L_2 }O_{d,\P}(d^{-\zeta_2}) \\
    &\leq \sum_{\vm \in L_1 }O_{d,\P}(d^{(p+\delta - \vs\tran \vm)/2+\epsilon)} \cdot d^{(\vs-\vlambda)\tran \vm /2}) + \sum_{\vm \in L_2 }O_{d,\P}(d^{-\zeta_2}) \\
    &= O_{d,\P}(d^{-\zeta_2/2+\epsilon}+d^{-\zeta_2}).
\end{align*}
Combining with Eq.~\eqref{eq:k_z_decomp} completes the proof for Eq.~\eqref{eq:kernel_to_mono_m}.

Now we present the proof for Eq.~\eqref{eq:kernel_to_mono_m_advanced}.
Following the definition of $D$ (Lemma~\ref{lemma:each_mono_cube}), we have for any $\vm$ such that $\vlambda\tran \vm > p+\delta$, the following holds
\begin{align}\label{eq:diag_psi_diff}
&~~~\snorm{\diag\round{\Phi_{\S_{\vm}}(Z)D_{\S_{\vm}}\Phi_{\S_{\vm}}(Z)\tran} - g^{(\norm{\vm}_1)}(0)d^{(\vs-\vlambda)\tran \vm}\mu_{\S_{\vm}}\cdot I_n}\notag\\
    &\overset{(a)}\leq O_{d}(d^{-\norm{\vm}_1-1})\cdot\snorm{\diag\round{\Phi_{\S_{\vm}}(Z)\Phi_{\S_{\vm}}(Z)\tran}}\notag\\
    &~~~+  g^{(\norm{\vm}_1)}(0)\snorm{d^{-\norm{\vm}_1}\diag\round{\Phi_{\S_{\vm}}(Z)\Phi_{\S_{\vm}}(Z)\tran} -d^{(\vs-\vlambda)\tran \vm}\mu_{\S_{\vm}}\cdot I_n} \notag  \\
    &\overset{(b)} {=}O_{d,\P}(d^{(\vs-\vlambda)\tran \vm -1})\notag\\
    &\overset{(c)}{=} O_{d,\P}(d\inv),
\end{align}
where $(a)$ follows from the definition of $D$ (Lemma~\ref{lemma:each_mono_cube}),  $(b)$ follows from Eq.~\eqref{eq:diag_psi_cube_m} and $(c)$ follows from the inequality $s^{[i]}\leq \lambda_i$.

We split those $\vm$ that satisfies $\vlambda\tran \vm >p+\delta$ into:
\begin{align}
    J_1&:=\{\vm~\vert~\vlambda\tran\vm >p+\delta~~{\rm and}~~\vs\tran \vm< p+2\},\label{eq:def_j1}\\
    J_2&:=\{\vm~\vert~\vlambda\tran\vm >p+\delta~~{\rm and}~~\vs\tran\vm \geq p+2\}.\label{eq:def_j2}
\end{align}
Then we can show the following estimate holds:
\begin{align*}
&~~~\snorm{\Phi_{\agp}(Z)  D_{\agp} \Phi_{\agp}(Z)\tran-\sum_{\vm\in J_1}\offd\round{\Phi_{\S_{\vm}}(Z)D_{\S_{\vm}}\Phi_{\S_{\vm}}(Z)\tran}  -\sum_{\lmgp} d^{(\vs-\vlambda)\tran \vm} g^{(\norm{\vm}_1)}(0) \mu_{\S_{\vm}} I }\\
&\overset{(a)}{\leq}  \snorm{\sum_{\vm\in J_2}{\Phi_{\S_{\vm}}(Z)D_{\S_{\vm}}\Phi_{\S_{\vm}}(Z)\tran} - \sum_{\vm\in J_2} d^{(\vs-\vlambda)\tran \vm} g^{(\norm{\vm}_1)}(0) \mu_{\S_{\vm}} I }\\
&~~~+\sum_{\vm\in J_1}\snorm{\diag\round{\Phi_{\S_{\vm}}(Z)D_{\S_{\vm}}\Phi_{\S_{\vm}}(z)\tran} - g^{(\norm{\vm}_1)}(0)d^{(\vs-\vlambda)\tran \vm}\mu_{\S_{\vm}}\cdot I_n}\\
    &\overset{(b)}{\leq}  \sum_{\vm\in J_2}\snorm{{\Phi_{\S_{\vm}}(Z)D_{\S_{\vm}}\Phi_{\S_{\vm}}(Z)\tran} -d^{(\vs-\vlambda)\tran \vm} g^{(\norm{\vm}_1)}(0) \mu_{\S_{\vm}} I }+O_{d,\P}(d\inv)\\
    &\overset{(c)}{\leq} \sum_{\vm\in J_2}O_{d,\P}(d^{(p+\delta - \vs\tran \vm)/2+\epsilon}) +O_{d,\P}(d\inv)\\
    &=O_{d,\P}(d^{(\delta-2)/2+\epsilon})+O_{d,\P}(d^{-1})\\
    &= O_{d,\P}(d^{(\delta-2)/2+\epsilon}),
\end{align*}
where $(a)$ is obtained by splitting $\Phi_{\S_{\vm}}D_{\S_{\vm}}\Phi_{\S_{\vm}}\tran$ into diagonal and off-diagonal components for $\vm\in J_1$, $(b)$ follows from Eq.~\eqref{eq:diag_psi_diff} and $(c)$ follows from Eq.~\eqref{eq:kernel_to_mono_m_tech}.

Combining the above bound with Eq.~\eqref{eq:k_z_decomp} completes the proof.

\subsubsection{Proof of Proposition~\ref{prop:<_l1_l1_l1_cube_m}}\label{proof:<_l1_l1_l1_cube_m}
The proof is analogous to Proposition~\ref{prop:<_l1_l1_l1_cube}, but we will provide special treatment for the case when $\vlambda\tran \vm = p+\delta$. 

Let $\gE(r)$ denote the union of all sets $S\subset [d]$ containing $r$ with $S\in \alp$:
\begin{align}\label{eq:def_dr}
    \gE(r) := \biground{S~\vert~r \in S,~~S\in \alp}.
\end{align}
Then $\gE^c(r)$ is defined as the complement of $\gE(r)$.
Particularly note that the cardinality of the set $\gE(r)$ is bounded by
\begin{align}\label{eq:d_r_size}
    O_d(d^{\vs\tran \vm - s_{\iota}}) = O_d(d^{\vlambda\tran\vm - \lambda_{\iota}}) = O_d(d^{p+\delta -\lambda_{\iota}}),
\end{align}
where we use the fact $\lambda_j\geq s_j$ for all $j\in[N]$ and $m_{\iota} \geq 1$.

We define a partial identity matrix $E_r$ of size $|\alp|\times |\alp|$ as 
\begin{align*}
    (E_r)_{S,S} = \mathbbm{1} \{S \in \gE(r)\}.
\end{align*}
We additionally define a diagonal weight matrix 
\begin{align}\label{eq:def_v}
V := \round{(\rho+\lambda)(n\hd_{\alp})\inv + I}\inv,
\end{align}
where each entry takes the form
\begin{align}
    V_{S,S} = \frac{1}{(\rho+\lambda)(nD_{S,S}W_{S,S})\inv + 1}.
\end{align}
We split the set $\alp$ into $\aslp$ and $\aep$:
\begin{align*}
   \aslp&:=   \bigcup_{\vlambda\tran \vm <p+\delta } \S_{\vm},~~~\aep:=  \bigcup_{\vlambda\tran \vm = p+\delta } \S_{\vm}.
\end{align*}
We  now write $\beta=\beta_{\gE(r)}+\beta_{\gE^c(r)} + \beta_\varepsilon$ where we define:
\begin{align*}
    \beta_{\gE(r)}:= K_{w,\lambda}\inv (\Phi_{\gE(r)}b_{\gE(r)}),~~~\beta_{\gE^c(r)}:= K_{w,\lambda}\inv (\Phi_{\gE^c(r)}b_{\gE^c(r)}),~~~\beta_{\varepsilon}= K_{w,\lambda}\inv \vepsilon.
\end{align*}
With the decomposition of $\beta$, we accordingly bound the norm by
\begin{align}\label{eq:beta_decomp_cube_m}
   &~~~\norm{\beta\tran\Phi_{\alp}R - b_{\alp} \tran V E_{r}}_2\\
   &\leq \norm{\beta_{\gE(r)}\tran  \Phi_{\alp}R - b_{\alp} \tran V E_{r} }_2+ \norm{\beta_{\gE^c(r)}\tran  \Phi_{\alp}R}_2 + \norm{\beta_{\varepsilon}\tran \Phi_{\alp} R}.
\end{align}
We will show all three terms on the right hand side of the above equation are small.

The following two Propositions control the first and the second term in Eq.~(\ref{eq:beta_decomp_cube_m}) respectively. We defer the proof to Appendix~\ref{proof:beta_<_l2_cube_m} 
 and \ref{proof:beta_>_l2_cube_m} respectively.

\begin{proposition}\label{lemma:beta_<_l2_cube_m}
For any $\epsilon>0$, the estimate holds uniformly over all coordinates $r$:
    \begin{align*}
         \norm{\beta_{\gE(r)}\tran  \Phi_{\alp}R - b_{\alp} \tran V E_{r}}_2 =  O_{d,\P}(d^{-\zeta/2+\epsilon})\cdot  \norm{\partial_r f^*_{\alp}}_{L_2}
    \end{align*}
\end{proposition}

\begin{proposition}\label{lemma:beta_>_l2_cube_m}
For any $\epsilon>0$, the estimate holds uniformly over all coordinates $r\in \T_\iota$:
\begin{align}
      \norm{\beta_{\gE^c(r)}\tran  \Phi_{\alp}R}_2
    =O_{d,\P}\round{d^{- \lambda_\iota/2}}\cdot \norm{b_{\gE^c(r)}}_2.
\end{align}
\end{proposition}

\noindent Next we show that the third term in Eq.~\eqref{eq:beta_decomp_cube_m} is small. We defer the proof to Appendix~\ref{proof:cube_error_1_m}.
\begin{proposition}\label{prop:cube_error_1_m}
For any $\epsilon>0$, the estimate holds uniformly over all coordinates $r\in \T_\iota$:
\begin{align}
      \norm{\beta_{\varepsilon}\tran  \Phi_{\alp}R}_2
    =O_{d,\P}(d^{-\lambda_\iota/2} )\cdot \sigma_\varepsilon.
\end{align}
\end{proposition}

The remainder of proof follows the proof of Proposition~\ref{prop:<_l1_l1_l1_cube}.
Combining the results of the three propositions then the right side of Eq.~\eqref{eq:beta_decomp_cube_m} becomes
\begin{align*}
       O_{d,\P}(d^{-\zeta/2+\epsilon} )\cdot  \norm{\partial_r f^*_{\alp}}_{L_2}+ O_{d,\P}\round{d^{-\lambda_\iota/2} }\cdot (\norm{b_{\gE^c(r)}}_2 + \sigma_\varepsilon).
\end{align*}
Using the elementary inequality $|\norm{a}_2^2-\norm{b}_2^2| \leq \norm{a-b}_2^2 +2\norm{b}_2\norm{a-b}_2$ for any vectors $a,b$ yields
\begin{align}\label{eq:cube_m_tech_bound_1}
   &~~~\abs{  \norm{\beta\tran\Phi_{\alp}R}_2^2 -\norm{ b_{\alp} \tran V E_{r}}_2^2}\notag\\
   &\leq  O_{d,\P}\round{d^{ - \lambda_\iota} }\cdot (\norm{b_{\gE^c(r)}}_2^2 +\sigma_\varepsilon^2) +   O_{d,\P}(d^{-\zeta/2+\epsilon} )\cdot  \norm{\partial_r f^*_{\alp}}_{L_2}^2\notag\\
   &~~~+O_{d,\P}\round{d^{ - \lambda_\iota/2}}\cdot (\norm{b_{\gE^c(r)}}_2+\sigma_\varepsilon)\norm{\partial_r f^*_{\alp}}_{L_2}.
\end{align}
Note that we have the estimate
\begin{align}\label{eq:v_minus_i}
    \snorm{V_{\aslp} - I} = O_d(d^{-\zeta_1}),
\end{align}
and for $S\in \aep$, since $n(D_{S,S}W_{S,S}) = \Theta_d(1) + O_d(d\inv)$, 
there exists a constant $c_r>0$ such that the following holds
\begin{align*}
     \abs{\sum_{S\in\aep\cap \gE(r)}V_{S,S}^2b_{S}^2 - c_r\cdot \norm{\partial_r f^*_{\aep}}_{L_2}^2} = O_{d}(d\inv)\cdot\norm{\partial_r f^*_{\aep}}_{L_2}^2.
\end{align*}
Therefore, we conclude
\begin{align}\label{eq:cube_m_tech_bound_2}
  \abs{  \norm{ b_{\alp} \tran V E_{r} }_2^2 - \norm{\partial_r f^*_{\aslp}}_{L_2}^2 - c_r \cdot \norm{\partial_r f^*_{\aep} }_{L_2}^2} = O_{d,\P}(d^{-\zeta_1})\cdot  \norm{ b_{\aslp}}_2^2.
\end{align}
Plugging the bound Eq.~\eqref{eq:cube_m_tech_bound_2} into \eqref{eq:cube_m_tech_bound_1} gives the bound for
\begin{align*}
    \abs{\norm{\beta_{\gE(r)}\tran  \Phi_{\alp}R}_2^2 - \norm{\partial_r f^*_{\aslp}}_{L_2}^2 - c_r \cdot \norm{\partial_r f^*_{\aep} }_{L_2}^2},
\end{align*}
which is
\begin{align*}
 & O_{d,\P}\round{d^{ - \lambda_\iota} }\cdot (\norm{f^*}_{L_2}^2 +\sigma_\varepsilon^2) +   O_{d,\P}(d^{-\zeta/2+\epsilon} )\cdot  \norm{\partial_r f^*_{\alp}}_{L_2}^2\notag\\
   &~~~+O_{d,\P}\round{d^{ - \lambda_\iota/2}}\cdot (\norm{f^*}_{L_2}+\sigma_\varepsilon)\norm{\partial_r f^*_{\alp}}_{L_2}.
\end{align*}
where we use the inequality $\norm{b_{\gE^c(r)}}_2^2 \leq \norm{f^*}_{L_2}^2$. Then we conclude the proof.

\subsubsection{Proof of Proposition~\ref{prop:s2_s2_cube_m}}\label{proof:s2_s2_cube_m}

Invoking Eq.~\eqref{eq:def_r_j_1_m}, we can deduce
\begin{align*}
    &~~~\norm{\beta\tran \Phi_{\S_{\vm}^1}R_{\vm}^1 - g^{(\norm{\vm}_1)}(0) \cdot \beta \tran \Phi_{\S_{\vm}^1}W_{\S_{\vm}^1}d^{-\norm{\vm}_1}}_2 \\
    &\leq \norm{\beta\tran \Phi_{\S_{\vm}^1}} \cdot O_d(d^{-\vlambda\tran \vm -1})\\
    &\leq \norm{y}_2 \norm{K_{w,\lambda}\inv} \snorm{\Phi_{\S_{\vm}^1}}\cdot  O_d(d^{-\vlambda\tran \vm -1})\\
    &\leq O_{d,\P}( \sqrt{n}\log(d))\cdot (\norm{f^*}_{L_2}+\sigma_\varepsilon)\cdot O_{d,\P}(1)\cdot O_{d,\P}({\sqrt{n}})\cdot  O_d(d^{-\vlambda\tran \vm -1 })\\
    &= O_{d,\P}(d^{-1-\zeta_1}\log(d))\cdot (\norm{f^*}_{L_2}+\sigma_\varepsilon),
\end{align*}
where we used the fact that $\snorm{\Phi_{\S_{\vm}^1}}=\snorm{X_r}\cdot\snorm{\Phi_{\S_{\vm-e_\iota}^0}} \leq \snorm{\Phi_{\alp}} = O_{d,\P}(\sqrt{n})$ from Lemma~\ref{lem:norm_control} and  the basic estimate $\norm{y}_2 = O_{d,\P}(\sqrt{n}\log^{1/2}(d))(\norm{f^*}_{L_2}+\sigma_\varepsilon)$. The bound $ \snorm{K_{w,\lambda}\inv} \leq 1/(\rho+\lambda -\snorm{\Delta_1}) = O_{d,\P}(1)$ is implied by  the expression \eqref{eqn:K_decomp_cube_m}.

Next, we analyze $d^{-\norm{\vm}_1}\beta\tran \Phi_{\S_{\vm}^1}W_{\S_{\vm}^1}$.
We split the whole index set of features $\cup_{j = 0}^\ell \S_{j}$ into $\Q\sqcup \S_{\vm}^1$.
Correspondingly, we split $\beta$ into 
\begin{align*}
    \beta_\Q := K_{w,\lambda}\inv \Phi_{\Q} b_\Q,~~\beta_{\S_{\vm}^1}:= K_{w,\lambda}\inv \Phi_{\S_{\vm}^1} b_{\S_{\vm}^1},~~{\rm and}~~~\beta_\varepsilon :=  K_{w,\lambda}\inv\vepsilon,
\end{align*}
then we obtain the bound
\begin{align}
    &~~~\norm{d^{-\norm{\vm}_1}\beta\tran \Phi_{\S_{\vm}^1}W_{\S_{\vm}^1} - (\rho+\lambda)\inv d^{p+\delta -\norm{\vm}_1 }b_{\S_{\vm}^1}\tran W_{\S_{\vm}^1}}_2\notag \\
    &\leq \norm{d^{-\norm{\vm}_1}\beta_{\Q}\tran\Phi_{\S_{\vm}^1}W_{\S_{\vm}^1}}_2+ \norm{d^{-\norm{\vm}_1}\beta_{\S_{\vm}^1}\tran \Phi_{\S_{\vm}^1}W_{\S_{\vm}^1}-(\rho+\lambda)\inv d^{p+\delta -\norm{\vm}_1 }b_{\S_{\vm}^1}\tran W_{\S_{\vm}^1}}_2 \notag \\
    &~~~+\norm{d^{-\norm{\vm}_1}\beta_{\varepsilon}\tran\Phi_{\S_{\vm}^1}W_{\S_{\vm}^1}}_2.\label{eq:s2_s2_cube_bound_1_m}
\end{align}
For the first term, we apply the following proposition. The proof is deferred to Appendix~\ref{proof:feature_u_v_m}.
\begin{proposition}\label{prop:feature_u_v_m}
For any $\epsilon>0$,  the following estimate holds uniformly over all coordinates $r\in \T_\iota$:
    \begin{align}\label{eq:s2_s2_cube_bound_3_m}
      \norm{d^{-\norm{\vm}_1}\beta_{\Q}\tran \Phi_{\S_{\vm}^1}W_{\S_\vm^1}}_2^2 =  O_{d,\P}(d^{-\zeta_2 -\lambda_\iota +\epsilon})\cdot \norm{b_{\Q}}_{2}^2.
    \end{align}
\end{proposition}

\noindent Now we proceed to analyze the second term.  The following proposition gives an estimate for the second term. We defer the proof to Appendix~\ref{proof:term_p_plus_one_m}.
\begin{proposition}\label{prop:term_p_plus_one_m}
For any $\epsilon>0$,  the following estimate holds uniformly over all coordinates $r\in[d]$
    \begin{align}\label{eq:term_p_plus_one_m}
       \norm{\frac{1}{n}\beta_{\S_{\vm}^1}\tran \Phi_{\S_{\vm}^1}-(\rho+\lambda)\inv b_{\S_{\vm}^1}\tran}_2 =O_{d,\P}(d^{-\zeta_1/2+\epsilon} + d^{-\zeta_2/2+\epsilon})\cdot   \norm{b_{\S_{\vm}^1}}_2.
    \end{align}
\end{proposition}

\noindent Similarly, after rescaling and multiplying both terms on the left side of the above equation by $W_{\S_{\vm}^1}$, we yield
\begin{align}
    &~~~\norm{d^{-\norm{\vm}_1}\beta_{\S_{\vm}^1}\tran \Phi_{\S_{\vm}^1}W_{\S_{\vm}^1} -(\rho+\lambda)\inv d^{p+\delta -\norm{\vm}_1}b_{\S_{\vm}^1}\tran W_{\S_{\vm}^1} }_2 \notag\\
    &=d^{p+\delta -\vlambda\tran \vm } \cdot O_{d,\P}(d^{-\zeta_1/2+\epsilon} + d^{-\zeta_2/2+\epsilon})\cdot   \norm{b_{\S_{\vm}^1}}_2\label{eq:s2_s2_cube_bound_4_m}.
\end{align}
Next, we show that the third term is small. We defer the proof of the following proposition to Appendix~\ref{proof:noise_term_p_plus_one_m}.
\begin{proposition}\label{prop:noise_term_p_plus_one_m}
For any $\epsilon>0$, the following estimate holds uniformly for all $r\in \T_\iota$:
\begin{align}\label{eq:noise_term_p_plus_one_m}
    \norm{d^{-\norm{\vm}_1} \beta_\varepsilon\tran \Phi_{\S_{\vm}^1}W_{\S_\vm^1}}_2^2 = O_{d,\P}(d^{-\zeta_2-\lambda_\iota+\epsilon})\cdot \sigma_\varepsilon^2
\end{align}
\end{proposition}

\noindent Lastly, plugging Eq.~\eqref{eq:s2_s2_cube_bound_3_m}, \eqref{eq:s2_s2_cube_bound_4_m} and \eqref{eq:noise_term_p_plus_one_m} into \eqref{eq:s2_s2_cube_bound_1_m} yields:
\begin{align*}
    &~~~\norm{d^{-\norm{\vm}_1}\beta\tran \Phi_{\S_{\vm}^1}W_{\S_{\vm}^1} - (\rho+\lambda)\inv d^{p+\delta -\norm{\vm}_1 }b_{\S_{\vm}^1}\tran W_{\S_{\vm}^1}}_2 \\
    &\leq d^{p+\delta - \vlambda\tran \vm }\cdot O_{d,\P}(d^{-\zeta_1/2+\epsilon} + d^{-\zeta_2/2+\epsilon})\cdot   \norm{b_{\S_{\vm}^1}}_2 +  O_{d,\P}(d^{-\zeta_2/2 -\lambda_\iota/2 +\epsilon/2})\cdot (\norm{b_{\Q}}_{2}+\sigma_\varepsilon).
\end{align*}
Following a similar argument with the proof of Proposition~\ref{prop:s2_s2_cube}, we can derive a bound for
\begin{align*}
    \abs{\norm{d^{-\norm{\vm}_1}\beta\tran \Phi_{\S_{\vm}^1}W_{\S_{\vm}^1}}_2^2 - \norm{(\rho+\lambda)\inv d^{p+\delta -\norm{\vm}_1 }b_{\S_{\vm}^1}\tran W_{\S_{\vm}^1}}_2^2},
\end{align*}
which takes the form
\begin{align*}
     &O_{d,\P}(d^{-\zeta_2 - \lambda_\iota + \epsilon})\cdot (\norm{f^*}^2_{L_2}+\sigma_\varepsilon^2) +  d^{2(p+\delta - \vlambda\tran \vm)}\cdot O_{d,\P}(d^{-\zeta_1/2+\epsilon} + d^{-\zeta_2/2+\epsilon})\cdot \norm{b_{\S_{\vm}^1}}_2^2 \\
     &~~~+ O_{d,\P}(d^{p+\delta - \vlambda\tran \vm -\zeta_2/2 - \lambda_\iota/2 + \epsilon/2}) \cdot \norm{b_{\S_{\vm}^1}}_2 (\norm{f^*}_{L_2}+\sigma_\varepsilon),
\end{align*}
where we use the fact $\norm{b_{\Q}}_{2}\leq\norm{f^*}_{L_2}$.

Combining with the fact $\norm{b_{\S_{\vm}^1}}_2 = \norm{\partial_r f^*_{\S_\vm^1}}_{L_2}$ and using the condition $\zeta_2 >\epsilon$ completes the proof.

\subsubsection{Proof of Proposition~\ref{lemma:beta_<_l2_cube_m}}\label{proof:beta_<_l2_cube_m}

In the proof, we will frequently use the following estimate for any $\Delta$ with $\snorm{\Delta} = o_{d,\P}(1)$:
\begin{align}
    \snorm{(I+\Delta)\inv - I} = O_{d,\P}(\snorm{\Delta}),
\end{align}
which can be deduced using Neumann series.

Notice that  Eq.~\eqref{eqn:K_decomp_cube_m} shows 
\begin{align*}
        K_{w} &= \Phi_{\alp}\hd_{\alp}\Phi_{\alp}\tran + \rho I_n + \Delta_1,
\end{align*}
where $\rho\geq 0$ is a constant and $\Delta_1$ satisfies $\snorm{\Delta_1} = O_{d,\P}(d^{-\zeta_2/2+\epsilon})$.

We write
\begin{align}\label{eq:K_decomp_m}
    K_{w,\lambda}  =\Phi_{\alp}\hd_{\alp}\Phi_{\alp}\tran + (\rho+\lambda)H,
\end{align}
where
\begin{align}\label{eq:def_h}
    H := I_n + (\rho+\lambda)\inv \Delta_1.
\end{align}

Invoking  Sherman-Morrison-Woodbury formula to compute $K_{w,\lambda}\inv$ in Eq.~\eqref{eq:K_decomp_m}, we can obtain 
\begin{align}\label{eq:smw_<_p_m}
     &~~~\beta_{\gE(r)}\tran \Phi_{\alp} R\notag\\
     &=\frac{1}{n}b_{\gE(r)}\tran \Phi_{\gE(r)}\tran H\inv \Phi_{\alp}\round{\underbrace{(\rho+\lambda)(n\hd_{\alp})\inv + \frac{\Phi_{\alp}\tran H\inv \Phi_{\alp}}{n}}_G}\inv\round{\hd_{\alp}\inv R}.
\end{align}
Now we analyze $G\inv$. Notice that $ \snorm{\frac{\Phi_{\alp}\tran \Phi_{\alp}}{n} -I_n} = O_{d,\P}(d^{-\zeta_1/2+\epsilon})$ (Theorem~\ref{lemma:phi_id_2}) and 
\begin{align}\label{eq:e_inv_minus_i_m}
    \snorm{H\inv - I_n} = O_d(\snorm{\Delta_1}) = O_{d,\P}(d^{-
    \zeta_2/2+\epsilon}),
\end{align}
which follows from Eq.~\eqref{eq:i_plus_delta_inv}.

Then similar to Eq.~\eqref{eq:psi_e_psi_minus_i}, we can show
\begin{align}\label{eq:psi_h_psi_minus_i}
 \snorm{\frac{\Phi_{\alp}\tran H\inv \Phi_{\alp}}{n} -I} 
 &= O_{d,\P}(d^{-\zeta_1/2+\epsilon} + d^{-\zeta_2/2+\epsilon}),
\end{align}
which immediately implies 
\begin{align}\label{eq:g_minus_nd_i}
    \snorm{G - \round{\underbrace{(\rho+\lambda)(n\hd_{\alp})\inv + I}_{=V\inv}}} = O_{d,\P}(d^{-\zeta_1/2+\epsilon} + d^{-\zeta_2/2+\epsilon}).
\end{align}
Next we bound $\snorm{G\inv}$. Since 
    $G = V\inv + G - V\inv = V\inv(I + V(G - V\inv))$, we have 
    \begin{align*}
        G\inv = (I + V(G - V\inv))\inv V,
    \end{align*}
where $ \snorm{V(G - V\inv)} \leq \snorm{G - V\inv}\snorm{V} = O_{d,\P}(d^{-\zeta_1/2+\epsilon} + d^{-\zeta_2/2+\epsilon})$.

Therefore, by inducing the Neumann series, we have
\begin{align*}
G\inv = \sum_{k=0}^\infty (-V(G - V\inv))^k V,   
\end{align*}
which implies the bound
\begin{align}\label{eq:g_inv_bound}
    \snorm{G\inv} \leq \frac{\snorm{V}}{1 - \snorm{V(G - V\inv)}\snorm{V}} = O_d(1).
\end{align}
From this, we can derive
\begin{align}\label{eq:g_inv_minus_v_inv}
    \snorm{G\inv - V} &=\snorm{G\inv (V\inv - G)V}\notag\\
    &\leq \snorm{G\inv}\snorm{G-V\inv}\snorm{V}\notag\\
    &= O_{d,\P}(d^{-\zeta_1/2+\epsilon} + d^{-\zeta_2/2+\epsilon}).
\end{align}
The following claim estimates the term in the second parentheses of Eq.~\eqref{eq:smw_<_p_m}. The proof is deferred to Appendix~\ref{proof:d_inv_r_cube_m}.
\begin{claim}\label{prop:d_inv_r_cube_m}
The following estimate holds
\begin{align}\label{eq:d_inv_r_cube_m}
    \snorm{  \hd_{\alp}\inv R -   {E}_{r}} = O_d(d^{-1} + w_r^2d^{-2}).
\end{align}   
\end{claim}
\noindent Eq.~\eqref{eq:d_inv_r_cube_m} and \eqref{eq:g_inv_minus_v_inv} together give
\begin{align*}
    \snorm{G\inv  \hd_{\alp}\inv R - V E_r}&\leq \snorm{G\inv-V}\snorm{E_r} + \snorm{G\inv}\snorm{\hd_{\alp}\inv R - E_r}\\
    &=O_{d,\P}(w_r^2d^{-2}+d^{-\zeta_1/2+\epsilon} + d^{-\zeta_2/2+\epsilon}).
\end{align*}
Applying this estimate to Eq.~\eqref{eq:smw_<_p_m} and obtain
\begin{align}\label{eq:b_leq_p_decomp_m}
    \norm{\beta_{\gE(r)}\tran \Phi_{\alp} R -  \frac{1}{n}b_{\gE(r)}\tran \Phi_{\gE(r)}\tran H\inv \Phi_{\alp}     VE_r}_2 = O_{d,\P}(d^{-\zeta_1/2+\epsilon} + d^{-\zeta_2/2+\epsilon}+w_r^2d^{-2})\cdot\norm{b_{\gE(r)}}_2,
\end{align}
where we use the fact that $\snorm{H\inv} = O_{d,\P}(1)$ (Eq.~\eqref{eq:e_inv_minus_i_m}), $\snorm{G\inv} = O_{d,\P}(1)$ (Eq.~\eqref{eq:g_inv_bound}), and $\snorm{\Phi_{\gE(r)}}\leq \snorm{\Phi_{\alp}} = O_{d,\P}(\sqrt{n})$ (Lemma~\ref{lem:norm_control}).

In the following, we show
\begin{align}\label{eq:b_leq_p_decomp_2_m}
     \norm{\frac{1}{n}b_{\gE(r)}\tran \Phi_{\gE(r)}\tran H\inv \Phi_{\alp} V E_{r}-b_{\alp}VE_r}_2 = o_{d,\P}(1),
\end{align}
then plugging it back to Eq.~\eqref{eq:b_leq_p_decomp_m} completes the proof.

First, applying the estimate Eq.~\eqref{eq:e_inv_minus_i_m}  yields
\begin{align}\label{eq:b_leq_p_decomp_bound_1_m}
   \norm{ \frac{1}{n}b_{\gE(r)}\tran \Phi_{\gE(r)}\tran H\inv \Phi_{\alp} VE_{r}-\frac{1}{n}b_{\gE(r)}\tran \Phi_{\gE(r)}\tran  \Phi_{\alp} VE_r}_2 =  O_{d,\P}(d^{-\zeta_1/2+\epsilon})\cdot \norm{b_{\gE(r)}}_2.
\end{align}
Then invoking the bound $\snorm{\Phi_{\alp}\tran\Phi_{\alp}/n -I} = O_{d,\P}(d^{-\zeta_1/2 +\epsilon})$ (Theorem~\ref{lemma:phi_id_2}), we have
\begin{align*}
    \norm{\frac{1}{n}b_{\gE(r)}\tran \Phi_{\gE(r)}\tran\Phi_{\alp}VE_r - b_{\alp}\tran VE_r}_2 &= \norm{\frac{1}{n}b_{\alp}\tran {E}_r \Phi_{\alp}\tran\Phi_{\alp}VE_r - b_{\alp}\tran VE_r}_2\\ &\leq 
    \norm{b_{\alp}\tran  {E}_r \round{ \Phi_{\alp}\tran  \Phi_{\alp}/n - I} V E_r}_2\\
    &\leq \norm{ b_{\alp}\tran E_r V E_r}_2\cdot O_{d,\P}(d^{-\zeta_1/2 +\epsilon})\\
    &= \norm{ b_{\gE(r)}}_2\cdot O_{d,\P}(d^{-\zeta_1/2 +\epsilon}).
\end{align*}
Plugging this estimate into Eq.~\eqref{eq:b_leq_p_decomp_bound_1_m} yields
\begin{align*}
     \norm{\frac{1}{n}b_{\gE(r)}\tran \Phi_{\gE(r)}\tran H\inv \Phi_{\alp} V E_{r} -b_{\alp}\tran V E_r}_2 =   O_{d,\P}(d^{-\zeta_1/2+\epsilon} + d^{-\zeta_2/2+\epsilon }+ w_r^2d^{-2})\cdot \norm{b_{\gE(r)}}_2.
\end{align*}
Combining with Eq.~\eqref{eq:b_leq_p_decomp_m} completes the proof.

\subsubsection{Proof of Claim~\ref{prop:d_inv_r_cube_m}}\label{proof:d_inv_r_cube_m}
Note that $R$ contains blocks $R_{\S_\vm^1}$ and $R_{\S_\vm^0}$. We discuss each kind of block separately.
Also  recall that $D$ is a diagonal matrix with entry $D_{\S_\vm} =  g^{(\norm{\vm}_1)}(0)d^{-\norm{\vm}_1} + O_d(d^{-\norm{\vm}_1-1})$ and each entry of $W_{\S_\vm}$ is of the order $d^{\norm{\vm}_1 - \vlambda\tran \vm}$ based on Eq.~\eqref{eq:w}. We will use these two properties in the analysis.

For $S\in \S_{\vm}^1$ or $\S_{\vm}^0$, based on the definition of $D$ Eq.~\eqref{eq:k_z_z} and $W$ Eq.~\eqref{eq:w} we have
\begin{align*}
   (\hd)_{S,S}\inv =(DW)_{S,S}\inv &= \round{g^{(\norm{\vm}_1)}(0)}\inv W_{S,S}\inv d^{\norm{\vm}_1} + O_d(d^{-\vlambda\tran \vm - 1}).
\end{align*}
Since Eq.~\eqref{eq:def_r_j_1_m} gives
\begin{align*}
    R_{S,S} &= g^{(\norm{\vm}_1)}(0)W_{S,S}d^{-\norm{\vm}_1} + O_d(d^{-\vlambda\tran \vm-1 }),
\end{align*}
by a straightforward calculation, we obtain 
\begin{align*}
  (\hd\inv R)_{S,S} = 1 + O_d(d\inv).
\end{align*}
Similarly, for $S\in \S_\vm^0$, Eq.~\eqref{eq:def_r_j_0_m} gives
\begin{align*}
    R_{S,S} &= g^{(\norm{\vm}_1+2)}(0)W_{S,S}w_r^2d^{-\norm{\vm}_1-2} + w_r^2 \cdot O_d(d^{-\vlambda\tran \vm - 3}).
\end{align*}
From this, we derive 
\begin{align*}
  (\hd\inv R)_{S,S} = w_r^2\cdot O_d(d^{-2}).
\end{align*}
Combining these two cases yields
\begin{align}
   (\hd \inv R)_{S,S} = \begin{cases}
      1 + O_d(d\inv)~~{\rm if}~S\in \cup_\vm\S_{\vm}^1;\\
      w_r^2\cdot O_d(d^{-2}),~~{\rm if}~S\in \cup_\vm\S_{\vm}^0;\\
        0~~{\rm otherwise.}
    \end{cases}
\end{align}
Combining this with the expression ${E}_r$:
\begin{align}
   ({E}_r)_{S,S} = 
     \mathbbm{1}\biground{S\in \cup_\vm\S_{\vm}^1},
\end{align}
we complete the proof.

\subsubsection{Proof of Proposition~\ref{lemma:beta_>_l2_cube_m}}\label{proof:beta_>_l2_cube_m}

The proof is analogous to that of Proposition~\ref{lemma:beta_>_l2_cube}. Similarly in the proof, we will frequently use the following estimate for any $\Delta$ with $\snorm{\Delta} = o_{d,\P}(1)$, which can be deduced from the Neumann series:
\begin{align}
    \snorm{(I+\Delta)\inv - I} = O_{d,\P}(\snorm{\Delta}),
\end{align}
therefore we have $(I+\Delta)\inv = I + \Delta'$ with some $\Delta'$ satisfying $\snorm{\Delta'} = O_{d,\P}(\snorm{\Delta})$.

Now we are ready to present the proof. For simplicity of notation, we omit $(r)$ in $\gE(r)$ and $\gE^c(r)$.
We first  expand $K_{w,\lambda}\inv$ using Eq.~\eqref{eq:kernel_to_mono_m_advanced} in Lemma~\ref{lemma:kernel_to_mono_m}. There exists a $\Delta$ that satisfies $\snorm{\Delta} = O_{d,\P}(d^{(\delta-2)/2+\epsilon}) = O_{d,\P}(d\halfe)$ when $\epsilon$ is sufficiently small such that the following holds
\begin{align}\label{eq:k_fancy_decomp_m}
    K_{w,\lambda} &= K_{w,p} + \underbrace{(\rho+\lambda)I_n +\sum_{\vm\in J_1}\offd\round{\Phi_{\S_{\vm}}(Z)D_{\S_{\vm}}\Phi_{\S_{\vm}}(Z)\tran} + \Delta}_{:=(\rho+\lambda)H},
\end{align}
where $J_1$ is defined in Eq.~\eqref{eq:def_j1}. 

We additionally denote
\begin{align}\label{eq:E_m}
    H:= I_n +  (\rho+\lambda)\inv \sum_{\vm\in J_1}\offd\round{\Phi_{\S_{\vm}}(Z)D_{\S_{\vm}}\Phi_{\S_{\vm}}(Z)\tran} +  (\rho+\lambda)\inv \Delta.
\end{align}
Then we can write  $K_{w,\lambda} = K_{w,p} + (\rho+\lambda)H$ and   use  Sherman-Morrison-Woodbury formula to compute $K_{w,\lambda}\inv $:
\begin{align}\label{eq:smw_cube_m}
     &~~\beta_{\gE^c}\tran  \Phi_{\alp}R \notag\\
     &=\frac{1}{n}b_{\gE^c}\tran \Phi_{\gE^c}\tran H\inv \Phi_{\alp}\round{\underbrace{(\rho+\lambda)(n\hd_{\alp} )\inv + \frac{\Phi_{\alp}\tran H\inv \Phi_{\alp}}{n}}_G}\inv\round{\hd_{\alp}\inv R}.
\end{align}
We first simplify Eq.~\eqref{eq:smw_cube_m}. Following Claim~\ref{prop:d_inv_r_cube_m} which shows $\snorm{  \hd_{\alp}\inv R -   {E}_{r}} = O_d(d^{-1} + w_r^2d^{-2})$ we can obtain
\begin{align}\label{eq:apply_dr_e_diff_m}
        \norm{\beta_{\gE^c}\tran \Phi_{\alp} R -  \frac{1}{n}b_{\gE^c}\tran \Phi_{\gE^c}\tran H\inv \Phi_{\alp} G\inv {E}_{r}}_2 = O_{d,\P}(d\inv+w_r^2d^{-2})\cdot \log(d)\norm{b_{\gE^c}}_2,
\end{align}
where we use the fact that $\snorm{H\inv} = O_{d,\P}(1)$ (Eq.~\eqref{eq:e_inv_minus_i_m}), $\snorm{G\inv} = O_{d,\P}(1)$ (Eq.~\eqref{eq:g_inv_bound}), and $\norm{\Phi_{\gE}b_{\gE^c}}_2 = O_{d,\P}(\sqrt{n}\log(d))\norm{b_{\gE^c}}_2$ (Proposition~\ref{prop:bound_y}).

Consequently, there exists $\Delta_1$ that satisfies $\norm{\Delta_1}_2 = O_{d,\P}((d\inv+w_r^2d^{-2}) \log(d))\cdot \norm{f_{\D^c}^*}_{L_2} $  such that the following holds:
\begin{align}\label{eq:simplified_g_p_m}
\beta_{\gE^c}\tran \Phi_{\alp} R =\frac{1}{n}   b_{\gE^c}\tran \Phi_{\gE^c}\tran H\inv \Phi_{\alp} G\inv {E}_{r} + \Delta_1.
\end{align}
Corresponding to Claim~\ref{claim:truncate_h}, we can similarly estimate $H\inv$ using Neumann series, that is, there exists a constant $c_H>0$  such that the following holds:
\begin{align}\label{eq:e_inv_m}
    H\inv &=   I_n + \sum_{q=1}^{c_H} \round{\sum_{\vm\in J_1} -c_{\vm}B_{\vm} + c_{\vm}h_{\vm}I_n}^{q} + \Delta_2,
\end{align}
where $B_{\vm} :=\Phi_{\S_{\vm}} \hd_{\S_{\vm}} \Phi_{\S_{\vm}}\tran$ and $c_{\vm}, h_{\vm}$ are constants, and $\snorm{\Delta_2} = O_{d,\P}(d\halfe)$ holds.

We plug the above expression of $H\inv$ into $G$ in Eq.~(\ref{eq:smw_cube_m}). We further substitute $\frac{\Phi_{\alp }\tran \Phi_{\alp}}{n}$ for $ \offd\round{\frac{\Phi_{\alp}\tran \Phi_{\alp}}{n}} + I_{|\alp|}$ and we obtain
\begin{align*}
    G = V\inv + \offd\round{\frac{\Phi_{\alp}\tran \Phi_{\alp}}{n}}  + \frac{\Phi_{\alp}\tran}{\sqrt{n}}\sum_{q=1}^{c_H} \round{\sum_{\vm\in J_1} -c_{\vm}B_{\vm} + c_{\vm}h_{\vm}I_n}^{q} \frac{\Phi_{\alp}}{\sqrt{n}} + \Delta_3,
\end{align*}
where $V$ is defined in Eq.~\eqref{eq:def_v}  and $\Delta_3$ satisfies $\snorm{\Delta_3} \leq \snorm{\Phi_{\alp} / \sqrt{n}}^2 \snorm{\Delta_2} = O_{d,\P}(d\halfe)$.

Note that $V\inv$ is a diagonal matrix with each entry greater than $1$.  For a small perturbation $\Delta$ with $\snorm{\Delta} = o_d(1)$, a modified Neumann series takes the form:
\begin{align*}
    (V\inv+\Delta)\inv = (I+V\Delta)\inv V = V  + \sum_{r=1}^\infty (-1)^r (V\Delta)^rV.
\end{align*}
Then similar to  Eq.~\eqref{eq:g_inv}, we express $G\inv$ using modified Neumann series with higher-order terms truncated. Consequently, we can find a consntant $c_G >0$ such that the following holds:
\begin{align}
G\inv &= V +\sum_{r=1}^{c_G} V^r\left( I -  \frac{\Phi_{\alp}\Phi_{\alp}\tran}{n}\right.\notag\\
&~~~~~~~~~~~~~~~~\left.+\frac{\Phi_{\alp}\tran}{\sqrt{n}}\sum_{q=1}^{c_H} \round{\sum_{\vm\in J_1} -c_{\vm}B_{\vm} + c_{\vm}h_{\vm}I_n}^{q}\frac{\Phi_{\alp}}{\sqrt{n} }\right)^{r} V +\Delta_4,\label{eq:g_inv_m}
\end{align}
$\Delta_4$ satisfies $\snorm{\Delta_4}= O_{d,\P}(d\halfe).$

Plugging the approximations of $H\inv$ (Eq.~(\ref{eq:e_inv_m})) and $G\inv$ (Eq.~(\ref{eq:g_inv_m})) into Eq.~(\ref{eq:simplified_g_p_m}) and expanding the expression we can write it as a finite sum:
\begin{align*}
    b_{\gE^c}\tran \frac{\Phi_{\gE^c}\tran }{\sqrt{n}}H\inv \frac{\Phi_{\alp}}{\sqrt{n}}G\inv {E}_{r} = \sum_{i=1}^{c_p}  b_{\gE^c}\tran  \frac{\Phi_{\gE^c}\tran}{\sqrt{n}} P_i + \Delta_5,
\end{align*}
where each $P_i$ takes the form
\begin{align*}
   P_i =  C_i A_{i_1}A_{i_2}\ldots A_{i_m} \frac{\Phi_{\alp}V^{t_i} {E}_{r}}{\sqrt{n}}
\end{align*}
with each $A_{i_j} \in \biground{I_n} \cup \biground{\frac{\Phi_{\alp} V^{s} \Phi_{\alp}\tran}{n}}_{s=0}^{c_G+1} \cup \biground{B_\vm}_{\vm\in J_1}$ for $j\in[m]$.

Here $t_i\in\mathbb{N}$,  $ c_p, C_i$ are constants. Additionally, $\Delta_5$ satisfies 
\begin{align*}
    \norm{\Delta_5}_2 = O_{d,\P}(\log(d))\cdot  \norm{f_{\gE^c}^*}_{L_2}  \cdot (\snorm{\Delta_2}+\snorm{\Delta_4})  =  O_{d,\P}(d\half)\cdot  \norm{f_{\gE^c}^*}_{L_2},
\end{align*}
where we use the estimate of $H\inv$ (Eq.~\eqref{eq:e_inv_minus_i_m}), $G\inv$ (Eq.~\eqref{eq:g_inv_bound}), $\snorm{\Phi_{\alp}/\sqrt{n}} = O_{d,\P}(1)$ (Lemma~\ref{lem:norm_control}) and $ \norm{ \frac{\Phi_{\gE^c}b_{\gE^c} }{\sqrt{n}}}_2 =O_{d,\P}(\log^{1/2}(d))\cdot  \norm{f_{\gE^c}^*}_{L_2}$ (Proposition~\ref{prop:bound_y}).

We can prove the following result similar to Proposition~\ref{prop:feature_product}. The proof appears in Appendix~\ref{proof:feature_product_m}.
\begin{proposition}\label{prop:feature_product_m}
Let $A_{j}$ be defined as one of the following matrices:
\begin{align*}
\biground{I_n},~~\biground{\frac{\Phi_{\alp} V^{s} \Phi_{\alp}\tran}{n}}_{s=0}^{c_G+1},~~\biground{\Phi_{\S_\vm}\hd_{\S_\vm}\Phi_{\S_\vm}\tran}_{\vm\in J_1}.
\end{align*}
Then for  any $t \in \mathbb{N}$ the following estimate holds
\begin{align}
\E\brac{\norm{\frac{1}{n}b_{\gE^c}\tran  {\Phi_{\gE^c}\tran} A_{1}A_{2}\ldots A_{m} \Phi_{\alp} V^{t}{E}_{r}}_2^2} = O_{d}\round{d^{ - \lambda_\iota - \zeta_1}}\cdot \norm{f^*_{\gE^c}}_{L_2}^2
\end{align}
\end{proposition}

The remaining steps of the proof proceed following Proposition~\ref{lemma:beta_>_l2_cube}. We can demonstrate that this estimate holds uniformly across all $r\in[d]$ with an extra multiplicative error factor, then we can conclude the bound:
\begin{align*}
    \max_r\norm{\frac{1}{n}b_{\gE^c}\tran \Phi_{\gE^c}\tran H\inv \Phi_{\alp} G\inv {E}_{r}}_2^2 &=\round{ O_{d,\P}(d^{ - \lambda_\iota-\zeta_1+1/q_0}\log(d)) + O_{d,\P}(d\inv)}\cdot \norm{f_{\gE^c}^*}_{L_2}^2,
\end{align*}
where $q_0>0$ is an arbitrarily large constant. For complete details of this argument, we direct readers to the referenced proposition.

Plugging this bound into Eq.~(\ref{eq:simplified_g_p_m}) and applying the AM-GM inequality yields
\begin{align*}
     \norm{\beta_{\gE^c}\tran \Phi_{\leq p} R}_2^2 &\leq  2 \norm{\frac{1}{n}b_{\gE^c}\tran \Phi_{\gE^c}\tran H\inv \Phi_{\alp} G\inv {E}_{r}}_2^2 + 2\norm{\Delta_1}^2_2\\
     &= \round{ O_{d,\P}(d^{ - \lambda_\iota - \zeta_1+\epsilon}) + O_{d,\P}(d\inv) + O_{d,\P}(w_r^4d^{-4+\epsilon})}\cdot \norm{f_{\gE^c}^*}_{L_2}^2,
\end{align*}
where we absorb $d^{1/{q_0}}$ and $\log(d)$ into $d^\epsilon$. With the fact $\lambda_\iota \leq 1$ and $\zeta_1 > \epsilon,$ the bound is reduced to $ O_{d,\P}(d^{ - \lambda_\iota })\cdot \norm{f_{\gE^c}^*}_{L_2}^2$ thus we conclude the proof.

\subsubsection{Proof of Proposition~\ref{prop:cube_error_1_m}}\label{proof:cube_error_1_m}
The proof follows closely the proof of Proposition~\ref{prop:cube_error_1} hence we omit intermediate derivations. We refer the readers to Appendix~\ref{proof:cube_error_1} for the complete derivation. 

We can similarly establish 
\begin{align}\label{eq:cube_m_error_tech_bound_1}
    \E_{\vepsilon}\brac{\norm{\beta_{\varepsilon}\tran \Phi_{\alp}R}_2^2}&\leq \sigma_\varepsilon^2 \cdot \snorm{K_\lambda\inv \Phi_{\alp}\hd_\alp}^2 \cdot \Tr\round{\hd_\alp\inv R  R\tran \hd_\alp\inv},
\end{align}
and successively can show
\begin{align}\label{eq:cube_m_error_tech_bound_2}
    \snorm{ K_\lambda\inv \Phi_{\alp}\hd_\alp }^2 = \snorm{\frac{1}{n}H\inv \Phi_{\alp} G}^2 = O_{d,\P}(n\inv),
\end{align}
implied by Eq.~\eqref{eq:smw_<_p_m}, and
\begin{align}\label{eq:cube_m_error_tech_bound_3}
    \Tr\round{\hd_\alp\inv R  R\tran \hd_\alp\inv}=O_d(\norm{E_r}_F^2) = O_d(d^{p+\delta - \lambda_{\iota }- \zeta_1}),
\end{align}
which follows Eq.~\eqref{eq:d_r_size}.

Plugging Eq.~\eqref{eq:cube_m_error_tech_bound_2} and \eqref{eq:cube_m_error_tech_bound_3} into Eq.~\eqref{eq:cube_m_error_tech_bound_1} and applying Markov's inequality completes the proof.

\subsubsection{Proof of Proposition~\ref{prop:feature_u_v_m}}\label{proof:feature_u_v_m}

The proof will mainly use the machinery developed in the proof for Proposition~\ref{lemma:beta_>_l2_cube_m}.

Applying the same decomposition of $K_\lambda$ as Eq.~\eqref{eq:k_fancy_decomp_m}, we can write
\begin{align*}
    K_{w,\lambda} = K_{w,p} +(\rho+\lambda)H,
\end{align*}
where $H$ is defined in Eq.~\eqref{eq:E_m}.

Invoking  Sherman-Morrison-Woodbury formula for $K_{w,\lambda}\inv$, we can obtain
\begin{align}\label{eq:K_lambda_inv_m}
    K_{w,\lambda}\inv = (\rho+\lambda)\inv H\inv - \frac{1}{n}(\rho+\lambda)\inv H\inv \Phi_{\alp}\tran G\inv \Phi_{\alp} H\inv,
\end{align}
where $$G:= (\rho+\lambda)(n\hd_{\alp})\inv + \frac{\Phi_{\alp}\tran H\inv \Phi_{\alp}}{n}.$$
Plugging the approximations of $H\inv$ (Eq.~(\ref{eq:e_inv_m})) and $G\inv$ (Eq.~(\ref{eq:g_inv_m})) into Eq.~(\ref{eq:K_lambda_inv_m}) and expanding the expression we can write it as the finite summation:
\begin{align}\label{eq:feature_product_split_m}
  \frac{1}{n}b_{\Q}\tran \Phi_{\Q}\tran K_{w,\lambda}\inv \Phi_{\S_{\vm}^1} = \sum_{i=1}^{c_p} C_{i}b_{\Q}\tran  \frac{\Phi_{\Q}\tran}{\sqrt{n}} A_{i_1}A_{i_2}\ldots A_{i_m} \frac{\Phi_{\S_{\vm}^1} }{\sqrt{n}} + \Delta,
\end{align}
where  $A_{i_j} \in \biground{I_n} \cup \biground{\frac{\Phi_{\alp} V^{s} \Phi_{\alp}\tran}{n}}_{s=0}^{c_G+1} \cup \biground{\Phi_{\S_{\vm'}} \hd_{\S_{\vm'}} \Phi_{\S_{\vm'}}\tran}_{\vm'\in J_1}$ for $j\in[m]$. Here $C_i>0$ is a constant, $c_p,m\in \mathbb{N}$, and $\Delta$ satisfies $\snorm{\Delta} = O_{d,\P}(d\half)\cdot  \norm{b_\Q}_{2}$.

Note that the adjusted effective degree of the feature in $\Phi_{\S_{\vm}^1}$is bounded by $\vs\tran\vm - s_\iota$, following an almost identical argument with the proof of Proposition~\ref{prop:feature_product_m} by using Lemma~\ref{lemma:feature_product_tech}, we can show each summand satisfies:
\begin{align*}
    \E\brac{\norm{\frac{1}{n}b_{\Q}\tran  {\Phi_{\Q}\tran}A_{i_1}A_{i_2}\ldots A_{i_m} {\Phi_{\S_{\vm}^1}}}_2^2} = O_{d,\P}(d^{\vs\tran\vm -s_\iota -p - \delta })\cdot \norm{b_{\Q}}_{2}^2.
\end{align*}
Using a similar analysis as in the proof of Proposition~\ref{lemma:beta_>_l2_cube_m}, we can show that the estimates hold uniformly for all $r\in T_\iota$. Specifically, invoking Lemma~\ref{lemma:uniform_bound} and Markov's inequality together yield the bound 
\begin{align*}
  \norm{\frac{1}{n}b_{\Q}\tran  {\Phi_{\Q}\tran}A_{i_1}A_{i_2}\ldots A_{i_q} {\Phi_{\S_{\vm}^1}}}_2^2 = O_{d,\P}(d^{\vs\tran\vm -s_\iota -p - \delta +1/q_0})\cdot \log(d)\cdot \norm{b_{\Q}}_{2}^2,
\end{align*}
where $q_0>0$ is an arbitrarily large constant.

Combining with $\Delta$ in Eq.~\eqref{eq:feature_product_split_m} and applying the triangle inequality and applying the triangle inequality we obtain the bound
\begin{align}\label{eq:p_plus_one_m}
      \norm{\frac{1}{n}b_{\Q}\tran \Phi_{\Q}\tran K_{w,\lambda}\inv \Phi_{\S_{\vm}^1}}_2^2 =  O_{d,\P}(d^{\vs\tran\vm -s_\iota -p - \delta +1/q_0})\cdot \log(d)\cdot \norm{b_{\Q}}_{2}^2.
\end{align}
Multiplying both sides of the equation by $(n/d^{\norm{\vm}_1})^2$ and taking $W_{\S_\vm^1}$ into account, which satisfies $\snorm{W_{\S_\vm^1}} = O_d(d^{\norm{\vm}_1 - \vlambda\tran\vm})$ yields:
\begin{align*}
    \norm{d^{-\norm{\vm}_1}b_{\Q}\tran \Phi_{\Q}\tran K_{w,\lambda}\inv \Phi_{\S_{\vm}^1} W_{\S_{\vm}^1}}_2^2 &=  O_{d,\P}(d^{\vs\tran\vm -s_\iota -p - \delta +1/q_0 +2(p+\delta)- 2\norm{\vm}_1 + 2\norm{\vm}_1 - 2\vlambda\tran \vm})\cdot \log(d)\cdot \norm{b_{\Q}}_{2}^2,\\
    &=O_{d,\P}(d^{\vs\tran\vm -2\vlambda\tran\vm -s_\iota +p + \delta +1/q_0})\cdot \log(d)\cdot \norm{b_{\Q}}_{2}^2\\
    &= O_{d,\P}(d^{\vs\tran\vm -\vlambda\tran\vm -s_\iota -\zeta_2 +1/q_0})\cdot \log(d)\cdot \norm{b_{\Q}}_{2}^2\\
    &= O_{d,\P}(d^{-\zeta_2 -\lambda_\iota + 1/q_0})\cdot \log(d)\cdot \norm{b_{\Q}}_{2}^2,
\end{align*}
where in the last equality we use the fact that $\vs\tran \vm - \vlambda\tran \vm \leq s_\iota - \lambda_\iota \leq 0$.

Absorbing the factor $d^{1/q_0}\log(d)$ into $d^{\epsilon}$ where $\epsilon>0$ is an arbitrarily small constant yields the desired bound.

\subsubsection{Proof of Proposition~\ref{prop:term_p_plus_one_m}}\label{proof:term_p_plus_one_m}

    Note that Eq.~(\ref{eq:K_decomp_m}) gives the decomposition
\begin{align}
    K_{w,\lambda}  = \Phi_{\alp}\hd_{\alp}\Phi_{\alp}\tran + (\rho+\lambda)H,
\end{align}
where $H = I_n + (\rho+\lambda)\inv \Delta_1$, $\rho$ is defined in Eq.~\eqref{eq:rho_cube_m} and $\Delta_1$ satisfies $\snorm{\Delta_1} = O_{d,\P}(d^{-\zeta_2/2+\epsilon})$.

Invoking Sherman-Morrison-Woodbury formula to compute the inverse of $K_{w,\lambda}$ and we obtain:
\begin{align*}
    K_{w,\lambda} \inv = (\rho+\lambda)\inv H\inv -(\rho+\lambda)\inv \frac{1}{n}H\inv \Phi_{\alp}G\inv \Phi_{\alp}\tran H\inv,
\end{align*}
where $G$ takes the same form as the one in Eq.~\eqref{eq:smw_<_p_m}.

Invoking Eq.~\eqref{eq:e_inv_minus_i_m} and  \eqref{eq:g_inv_minus_v_inv} that $H\inv$ is close to $I_n$ and $G\inv$ is close to $V$, we can show 
\begin{align*}
   \snorm{K_{w,\lambda}\inv - (\rho+\lambda)\inv\round{I_n -\frac{\Phi_{\alp}V\Phi_{\alp}\tran}{n}}} = O_{d,\P}(d^{-\zeta_1/2+\epsilon} + d^{-\zeta_2/2+\epsilon}),
\end{align*}
where we use the estimate $\snorm{\Phi_{\alp}} =O_{d,\P}(\sqrt{n})$ (Lemma~\ref{lem:norm_control}).

Plugging the above bound into the left side of Eq.~\eqref{eq:term_p_plus_one_m} and applying the triangle inequality yields
\begin{align}\label{eq:weak_feature_bound_0_m}
   &~~~\snorm{\frac{1}{n}b_{\S_{\vm}^1}\tran \Phi_{\S_{\vm}^1}\tran K_{w,\lambda}\inv \Phi_{\S_{\vm}^1} - \frac{1}{n(\rho+\lambda)}b_{\S_{\vm}^1}\tran \Phi_{\S_{\vm}^1}\tran \round{I_n -\frac{\Phi_{\alp}V\Phi_{\alp}\tran}{n}} \Phi_{\S_{\vm}^1}}\notag\\
   &= O_{d,\P}(d^{-\zeta_1/2+\epsilon} + d^{-\zeta_2/2+\epsilon})\cdot \norm{b_{\S_{\vm}^1}}_2,
\end{align}
where we use the estimate $\snorm{\Phi_{\S_{\vm}^1}}=\snorm{X_r}\cdot\snorm{\Phi_{\S_{\vm-e_\iota}^0}} \leq \snorm{\Phi_{\alp}} = O_{d,\P}(\sqrt{n})$ from Lemma~\ref{lem:norm_control}.

Next we analyze $\frac{1}{n}b_{\S_{\vm}^1}\tran \Phi_{\S_{\vm}^1}\tran \round{I_n -\frac{\Phi_{\alp}V\Phi_{\alp}\tran}{n}} \Phi_{\S_{\vm}^1}$. 
Using the triangle inequality gives
\begin{align}\label{eq:weak_feature_bound_1_m}
    &~~~\norm{\frac{1}{n}b_{\S_{\vm}^1}\tran \Phi_{\S_{\vm}^1}\tran \round{I_n -\frac{\Phi_{\alp}V\Phi_{\alp}\tran}{n}} \Phi_{\S_{\vm}^1} - b_{\S_{\vm}^1}\tran}_2 \notag \\
    &\leq \norm{\frac{1}{n}b_{\S_{\vm}^1}\tran \Phi_{\S_{\vm}^1}\tran \Phi_{\S_{\vm}^1} - b_{\S_{\vm}^1}\tran}  +  \norm{\frac{1}{n}b_{\S_{\vm}^1}\tran \Phi_{\S_{\vm}^1}\tran \frac{\Phi_{\alp}V\Phi_{\alp}\tran}{n} \Phi_{\S_{\vm}^1}}_2.
\end{align}
Next we show that both terms on the right side of the above inequality are small.

For the first term, we use Theorem~\ref{lemma:gram_matrix} which gives 
\begin{align}\label{eq:weak_feature_bound_2_m}
\norm{\frac{1}{n}b_{\S_{\vm}^1}\tran \Phi_{\S_{\vm}^1}\tran \Phi_{\S_{\vm}^1} - b_{\S_{\vm}^1}\tran} &\leq \norm{b_{\S_{\vm}^1}\tran}_2 \snorm{\frac{1}{n}\Phi_{\S_{\vm}^1}\tran \Phi_{\S_{\vm}^1} - I_n}\notag\\
&=  \norm{b_{\S_{\vm}^1}}_2 \cdot O_{d,\P}(d^{-\zeta_1/2+\epsilon}).
\end{align}
For the second term, invoking Lemma~\ref{lemma:feature_product_tech}, since the adjusted effective degree of feature  $\phi_{\alp}$ is bounded by $\vs\tran \vm'$ where $\vm'$ satisfies  $\vlambda\tran \vm' \leq p+\delta$, we can deduce
\begin{align*}
    \E\brac{\norm{\frac{1}{n}b_{\S_{\vm}^1}\tran \Phi_{\S_{\vm}^1}\tran \frac{\Phi_{\alp}V\Phi_{\alp}\tran}{n} \Phi_{\S_{\vm}^1}}_2^2} 
    &= \norm{b_{\S_{\vm}^1}}_2^2 \cdot \max_{\vlambda\tran \vm' \leq p+\delta} O_d(d^{\vs\tran\vm'-p-\delta})\\
    &= \norm{b_{\S_{\vm}^1}}_2^2 \cdot  O_d(d^{-\zeta_1}),
\end{align*}
where the last equality follows from the definition of $\zeta_1$, i.e., Eq. \eqref{eq:zeta_1}.

Using a similar analysis as in the proof of Proposition~\ref{lemma:beta_>_l2_cube_m}, we can show that the estimates hold uniformly for all $r\in[d]$. Specifically, invoking Lemma~\ref{lemma:uniform_bound} and Markov's inequality together yield the bound 
\begin{align*}
  \norm{\frac{1}{n}b_{\S_{\vm}^1}\tran \Phi_{\S_{\vm}^1}\tran \frac{\Phi_{\alp}V\Phi_{\alp}\tran}{n} \Phi_{\S_{\vm}^1}}_2^2 =\norm{b_{\S_{\vm}^1}}_2^2 \cdot  O_d(d^{-\zeta_1+1/q_0})\cdot \log(d),
\end{align*}
where $q_0>0$ is an arbitrarily large constant.

Plugging this bound and Eq.~\eqref{eq:weak_feature_bound_2_m} into Eq.~\eqref{eq:weak_feature_bound_1_m}  yields:
\begin{align*}
   \norm{ \frac{1}{n}b_{\S_{\vm}^1}\tran \Phi_{\S_{\vm}^1}\tran \round{I_n -\frac{\Phi_{\leq p}V\Phi_{\leq p}\tran}{n}} \Phi_{\S_{\vm}^1} - b_{\S_{\vm}^1}\tran}_2 = \norm{b_{\S_{\vm}^1}}_2 \cdot O_{d,\P}(d^{-\zeta_1/2+\epsilon}).
\end{align*}
Combining this bound with Eq.~\eqref{eq:weak_feature_bound_0_m} completes the proof.

\subsubsection{Proof of Proposition~\ref{prop:noise_term_p_plus_one_m}}\label{proof:noise_term_p_plus_one_m}
The proof follows closely the proof of Proposition~\ref{prop:noise_term_p_plus_one} hence we omit intermediate derivations. We refer the readers to Appendix~\ref{proof:noise_term_p_plus_one} for the complete derivation. 

In parallel to Eq.~\eqref{eq:noise_term_p_plus_one_tech_bound2}, we can obtain
\begin{align}\label{eq:noise_term_p_plus_one_m_tech_bound1}
    \E_\vepsilon\brac{\norm{d^{-\norm{\vm}_1} \beta_\varepsilon\tran \Phi_{\S_{\vm}^1}W_{\S_\vm^1}}_2^2} &\leq  \snorm{W_{\S_\vm^1}}^2\cdot d^{-2\norm{\vm}_1} \cdot \E\brac{\norm{ \beta_\varepsilon\tran \Phi_{\S_{\vm}^1}}_2^2}\notag\\
    &\leq O_d(d^{-2\vlambda\tran\vm})\cdot \snorm{K_{w,\lambda}^{-2}}\Tr\round{\Phi_{\S_{\vm}^1}\tran \Phi_{\S_{\vm}^1}}.
\end{align}
The bound 
\begin{align}\label{eq:noise_term_p_plus_one_m_tech_bound2}
    \snorm{K_{w,\lambda}^{-2}}= \snorm{K_{w,\lambda}^{-1}}^2 = O_{d,\P}(1)
\end{align}
is implied by  the expression \eqref{eqn:K_decomp_cube_m}. Next we analyze $\Tr\round{\Phi_{\S_{\vm}^1}\tran \Phi_{\S_{\vm}^1}}$.
We first control the size $|\S_{\vm}^1| = O_d(d^{\vs\tran \vm - s_{\iota}})$. The following claim implies $|\S_{\vm}^1| \lesssim n$ holds:
\begin{claim}
The constant
    \begin{align*}
        \zeta_3:= \min_{\vm: \vlambda\tran\vm \in (p+\delta,p+\delta + \lambda_\iota],~ m_\iota\geq 1} \round{p+\delta - \vs\tran\vm +s_\iota}
    \end{align*}
is bounded from $0$.
\end{claim}
\begin{proof}
    Since $\lambda_j \geq s_j$ for all $j\in[N]$, we have $\zeta_3 \geq p+\delta - (\vs\tran \vm -s_\iota) \geq p+\delta - (\vlambda\tran \vm - \lambda_\iota) \geq 0$. The equality holds when (1) $\vlambda\tran \vm -\lambda_\iota = p+\delta$ and (2) $\vlambda\tran\vm - \lambda_\iota = \vs\tran\vm -s_\iota$ hold.  For condition (2) to hold, it requires either $m_\iota = 1$ and $\lambda_j = s_j = 1$ for $j\neq \iota$; or $m_\iota\geq 2$ and $\lambda_j = s_j = 1$ for all $j\in[N]$.  Both cases contradict condition (1), hence the inequality stands strictly, which completes the proof.
\end{proof}

\noindent Consequently, following Theorem~\ref{lemma:phi_id_2}, we have $\snorm{\Phi_{\S_{\vm}^1} \Phi_{\S_{\vm}^1}\tran/n - I_{|\S_{p+1}^1|}} = O_{d,\P}(d^{-\zeta_3/2+\epsilon})$. Then, we can conclude 
\begin{align*}
    \Tr\round{\frac{1}{n}\Phi_{\S_{\vm}^1} \Phi_{\S_{\vm}^1}\tran}\leq  O_d(d^{\vs\tran \vm - s_{\iota}})(1+ O_{d,\P}(d^{-\zeta_3/2+\epsilon})) =  O_d(d^{\vs\tran \vm - s_{\iota}}).
\end{align*}
Together with Eq.~\eqref{eq:noise_term_p_plus_one_m_tech_bound2}, the right side of Eq.~\eqref{eq:noise_term_p_plus_one_m_tech_bound1} is bounded by 
\begin{align*}
    O_{d,\P}(d^{\vs\tran \vm - s_{\iota} - 2\vlambda\tran \vm+p+\delta}) =O_{d,\P}(d^{-\lambda_\iota -\zeta_2}), 
\end{align*}
where we use the fact $\vs\tran\vm - \vlambda\tran \vm \leq s_\iota - \lambda_\iota$.

Applying Markov's inequality we obtain the bound $O_{d,\P}(d^{-\lambda_\iota -\zeta_2})\log(d)$ with probability at least $1-1/ \log(d)$. Absorbing the logarithmic factor into $\epsilon$ completes the proof.

\subsubsection{Proof of Proposition~\ref{prop:feature_product_m}}\label{proof:feature_product_m}

Since $V$ is a diagonal matrix with $V\lesssim I$, then we can upper bound the norm by removing $V^{t}$ in the product:
\begin{align*}
    \norm{\frac{1}{n}b_{\gE^c}\tran  {\Phi_{\gE^c}\tran} A_{1}A_{2}\ldots A_{m} \Phi_{\alp} V^{t}{E}_{r}}_2^2 \lesssim O_{d}(1)\cdot  \norm{\frac{1}{n}b_{\gE^c}\tran  {\Phi_{\gE^c}\tran} A_{1}A_{2}\ldots A_{m} \Phi_{\gE}}_2^2,
\end{align*}
where we use the fact $(E_r)_{S,S} = \mathbbm{1}\{S\in \gE\}.$

Similar to Eq.~\eqref{eq:feature_produc_bound1}, the following estimate holds:
\begin{align}\label{eq:feature_produc_bound1_m}
    \E\brac{\norm{ \frac{1}{n}b_{\gE^c}\tran  {\Phi_{\gE^c}\tran}A_{1}A_{2}\ldots A_{m} \Phi_{\gE}}_2^2} \leq \norm{b_{\gE^c}}_2^2\cdot \snorm{\E\brac{ \underbrace{\frac{1}{n^2}{\Phi_{\gE^c}\tran}A_{1}\ldots A_{m} \Phi_{\gE}\Phi_{\gE}\tran A_{m}\ldots A_1 \Phi_{\gE^c}}_{=:K}} }.    
\end{align}
We now aim to upper bound the operator norm of $\E[K]$.
To this end, we decompose the feature matrices in $K$ so as to ensure that distinct matrices in the product correspond to nonoverlapping index sets.

Following Eq.~\eqref{eq:fine_split_cube_m} we have the decomposition 
\begin{align}\label{eq:feature_split_m}
   \alp = \sqcup_{\vm:\lmlp} \S_{\vm},~~~\S_{\vm} = \S_{\vm}^1 \sqcup \S_{\vm}^0,
\end{align}
Clearly, the following equalities hold:
\begin{align*}
    \gE = \bigcup_{\vm:\lmlp} \S_{\vm}^1,~~~\gE^c = \biground{\bigcup_{\vm:\lmlp} \S_{\vm}^0} \bigcup \biground{\bigcup_{\vm:\lmgp, \norm{\vm}_1\leq \ell} \S_{\vm}}
\end{align*}
Define now the collection of matrices
\begin{align*}
    \U:= \{\Phi_{\S_{\vm}^1}\}_{\lmlp} \bigcup \{\Phi_{\S_{\vm}^0}\}_{\lmlp}\bigcup \{\Phi_{\S_{\vm}}\}_{\lmgp}.
\end{align*}
Note that any two distinct matrices in $\U$ correspond to distinct index sets. 
Normalizing the matrices, we further define the collection
\begin{align*}
    \widetilde{\U}&:= \{n\half \Phi_{\S_{\vm}^1}\}_{\lmlp} \bigcup \{n\half \Phi_{\S_{m,j}^0}\}_{\substack{\lmlp}}\bigcup \{(\hd_{\S_{\vm}})\half \Phi_{\S_{\vm}}\}_{\lmgp},
\end{align*}
where $\snorm{\hd_{\S_{\vm}}} = \Theta_d(d^{-\vlambda\tran \vm})$.

Let us now write the two expressions:
\begin{align*}
       &\frac{\Phi_{\alp} \Phi_{\alp}\tran}{n} = \sum_{\lmlp} \round{\frac{\Phi_{\S_{\vm}^0} \Phi_{\S_{\vm}^0}\tran}{n} + \frac{\Phi_{\S_{\vm}^1} \Phi_{\S_{\vm}^1}\tran}{n}};\\
    &\frac{\Phi_{\gE} \Phi_{\gE}\tran}{n} = \sum_{\lmlp} \frac{\Phi_{\S_{\vm}^1} \Phi_{\S_{\vm}^1}\tran}{n}.
\end{align*}
Plugging these decompositions for the corresponding $A_i$ in the expression of $K$ Eq.~\eqref{eq:feature_produc_bound1_m}, we can write $K$ as the summation of finitely many terms, where each term takes the form of some positive constant multiplied by
\begin{align*}
    K_j := \frac{\max(n^{1/2},d^{k_{1}/2}) \cdot  \max(n^{1/2},d^{k_{2m+1}/2})}{n}\cdot  B_{1}^\top \Bigl(B_{2} B_{2}\tran\Bigr) \ldots \Bigl(B_{{2m}} B_{{2m}}\tran\Bigr) B_{{2m+1}},
\end{align*}
where $B_{1} \in \R^{n\times k_1},\ldots,B_{{2m+1}}\in \R^{n\times k_{2m+1}}$ are from $\widetilde{\U}$. 

Particularity note that $B_1$ and $B_{2m+1}$ corresponding to the decomposition of $\Phi_{\gE^c}$ and thus each takes the form $ (\hd_{\S_{\vm}})\half \Phi_{\S_{\vm}}$ for some $\vm$ that satisfies $\lmgp$; $B_{m+1}$ corresponds to the decomposition of $\Phi_{\gE}$ and takes the form $n\half \Phi_{\S_{\vm}^1}$ for some $\vm$ that satisfies $\lmgp$.   The equation contains an additional scaling factor on the right side, which arises from the distinct scaling of $B_1$ and $B_{2m+1}$ by a factor of $1/n$ in Eq.~\eqref{eq:feature_produc_bound1_m}.

Next, we invoke Lemma~\ref{lemma:feature_product_tech} to bound each $\snorm{\E\brac{K_j}}$. According to the definition of feature $\phi_{\S_\vm}$ Eq.~\eqref{def:psi_sm}, we naturally have non-overlapping sets $\T_1,\ldots \T_N$ with $|\T_i| = O_d(d^{s_i})$. We further split the block $\T_\iota$ where $r \in \T_\iota$  into $\T_\iota^1:=\{r\}$ and $\T_\iota^0:= \T_\iota \setminus \{r\}$. Then we have in total $N+1$ many blocks of $[d]$ and any feature index $S$ that appears in the product can be accordingly decomposed. The adjusted effective degree of the feature of $B_{m+1}$ is thus bounded by 
\begin{align*}
    \sum_{i\in[N],i\neq \iota} s_i m_i + (m_\iota-1)s_\iota^0 + 1\cdot s_\iota^1,
\end{align*}
where $|\T_\iota^1| = \Theta_d(d^{s_\iota^1}) = \Theta_d(d^0)$ and $|\T_\iota^0| = \Theta_d(d^{s_\iota^0}) =  \Theta_d(d^{s_\iota})$.
Consequently, similar to Eq.~\eqref{eq:each_k_j}  we can deduce
\begin{align*}
    \snorm{\E\brac{K_j}} = O_d(d^{\sum_{i\in[N],i\neq \iota} s_i m_i + (m_\iota-1)s_{\iota} - p -\delta}).
\end{align*}
Since $s_i\leq \lambda_i$ for all $i\in[N]$ and $m_{\iota}>0$, the exponent can be further bounded by
\begin{align*}
  \sum_{i=1}^N \lambda_i m_i - \lambda_\iota  - p -\delta \leq -  \lambda_{\iota} -\zeta_1.
\end{align*}
Consequently, we obtain the bound for each $ \snorm{\E\brac{K_j}}$. Applying triangle inequality yields  $  \snorm{\E\brac{K}} = O_d(d^{-\lambda_{\iota}-\zeta_1})$.  Combining with Eq.~\eqref{eq:feature_produc_bound1_m} completes the proof.

%% file: sections_appendix_b/prop_gradient_rkhs.tex
\subsection{Proof of Propositions for Theorem~\ref{thm:learning_hypercube} (DN estimator error)}

\subsubsection{Proof of Proposition~\ref{prop:<_l1_l1_l1_rkhs}}\label{proof:<_l1_l1_l1_rkhs}
The proof is analogous to that of Proposition~\ref{prop:<_l1_l1_l1_cube_m}, and throughout we follow the conventions and notation established there.

We decompose the diagonal matrix $(\hd_\alp \inv R)$ into two  diagonal matrices ${\widehat E}_r$ and ${\widehat E}_r^\circ$ of size $|\alp|\times |\alp|$ as 
\begin{align*}
    ({\widehat E}_r)_{S,S} &= (\hd_\alp \inv R)_{S,S} \cdot \mathbbm{1} \{S \in \gE(r)\}\\
    ({\widehat E}_r^\circ)_{S,S} &= (\hd_\alp \inv R)_{S,S} \cdot \mathbbm{1} \{S \in \gE^c(r)\},
\end{align*}
where sets $\gE(r)$ and $\gE^c(r)$ are defined in Eq.~\eqref{eq:def_dr}.

The following claim estimates the term in ${\widehat E}_r$ and ${\widehat E}_r^\circ$. The proof is deferred to Appendix~\ref{proof:d_inv_r_rkhs}.
\begin{claim}\label{prop:d_inv_r_rkhs}
The following estimate holds for any $S\in \alp$:
\begin{align}\label{eq:d_inv_r_rkhs}
     ({\widehat E}_r)_{S,S} &= \round{g^{(\norm{\vm}_1)}(0)}\half W_{S,S}\half d^{(\norm{\vm}_1-p-\delta)/2}\cdot (1+ O_d(d\half))\cdot \mathbbm{1} \{S \in \S_\vm^1\}\ ,\\
     ({\widehat E}_r^\circ )_{S,S} &= w_r \cdot O_d( d^{(\vlambda\tran\vm-p-\delta -2)/2}) \cdot \mathbbm{1} \{S \in \S_{\vm}^0\}.
\end{align}   
\end{claim}

We similarly write $\beta=\beta_{\gE(r)}+\beta_{\gE^c(r)} + \beta_\varepsilon$ where we define:
\begin{align*}
    \beta_{\gE(r)}:= K_{w,\lambda}\inv (\Phi_{\gE(r)}b_{\gE(r)}),~~~\beta_{\gE^c(r)}:= K_{w,\lambda}\inv (\Phi_{\gE^c(r)}b_{\gE^c(r)}),~~~\beta_{\varepsilon}= K_{w,\lambda}\inv \vepsilon.
\end{align*}
With the decomposition of $\beta$, we accordingly bound the norm by
\begin{align}\label{eq:beta_decomp_rkhs}
   &~~~\norm{\beta\tran\Phi_{\alp}R - b_{\alp} \tran V {\widehat  E}_{r} }_2\notag\\
   &\leq \norm{\beta_{\gE(r)}\tran  \Phi_{\alp}R - b_{\alp} \tran V {\widehat E}_{r} }_2+ \norm{\beta_{\gE^c(r)}\tran  \Phi_{\alp}R}_2 + \norm{\beta_{\varepsilon}\tran \Phi_{\alp} R},
\end{align}
where the matrix $V$ is defined in Eq.~\eqref{eq:def_v}. 

We will show all three terms on the right hand side of the above equation are small. Before doing so, we introduce some notation to simplify the proof. We define the matrix $R_\vm \in \R^{|\alp|\times |\S_\vm|}$, which is the submatrix of $R$ consisting of the columns indexed by $\S_\vm$. Similarly we let $({\widehat E}_{r})_{\S_\vm}$ denote the matrix formed by the columns of ${\widehat E}_{r}$ with indices in $\S_\vm$. Therefore, we can decompose the first term squared into
\begin{align*}
    \norm{\beta_{\gE(r)}\tran  \Phi_{\alp}R - b_{\alp} \tran V {\widehat E}_{r} }_2^2 =  \sum_{\vm: \vlambda\tran \vm \leq p+\delta}\norm{\beta_{\gE(r)}\tran  \Phi_{\alp}R_{\vm} - b_{\alp} \tran V ({\widehat E}_{r})_{\S_\vm} }_2^2.
\end{align*}

\noindent The following two propositions provide estimates for the first and the second term in Eq.~(\ref{eq:beta_decomp_rkhs}) respectively. We defer the proof to Appendix~\ref{proof:beta_<_l2_rkhs} 
 and \ref{proof:beta_>_l2_rkhs} respectively.
\begin{proposition}\label{lemma:beta_<_l2_rkhs}
For any $\epsilon>0$, the estimate holds uniformly over all coordinates $r$:
    \begin{align*}
         \norm{\beta_{\gE(r)}\tran  \Phi_{\alp}R_\vm - b_{\alp} \tran V ({\widehat E}_{r})_{\S_\vm}}_2 =  O_{d,\P}( d^{-\zeta/2 + \epsilon+ (\vlambda\tran\vm - p - \delta)/2})\cdot  \norm{\partial_r f^*_{\S_{\vm}^1}}_{L_2}
    \end{align*}
\end{proposition}

\begin{proposition}\label{lemma:beta_>_l2_rkhs}
For any $\epsilon>0$, the estimate holds uniformly over all coordinates $r\in \T_\iota$:
\begin{align}
      \norm{\beta_{\gE^c(r)}\tran  \Phi_{\alp}R}_2
    =O_{d,\P}\round{d^{- \lambda_\iota/2}}\cdot \norm{b_{\gE^c(r)}}_2.
\end{align}
\end{proposition}

\noindent Next we show that the third term in Eq.~\eqref{eq:beta_decomp_rkhs} is small. We defer the proof to Appendix~\ref{proof:cube_error_1_rkhs}.
\begin{proposition}\label{prop:cube_error_1_rkhs}
For any $\epsilon>0$, the estimate holds uniformly over all coordinates $r\in \T_\iota$:
\begin{align}
      \norm{\beta_{\varepsilon}\tran  \Phi_{\alp}R}_2
    =O_{d,\P}(d^{-\lambda_\iota/2} )\cdot \sigma_\varepsilon.
\end{align}
\end{proposition}

The remainder of proof follows the proof of Proposition~\ref{prop:<_l1_l1_l1_cube}.
Combining the results of the three propositions then the right side of Eq.~\eqref{eq:beta_decomp_rkhs} becomes
\begin{align*}
       \sum_{\vm:\vlambda\tran \vm \leq p+\delta}O_{d,\P}( d^{-\zeta/2 + \epsilon+ (\vlambda\tran\vm - p - \delta-2)/2})\cdot  \norm{\partial_r f^*_{\S_{\vm}^1}}_{L_2}+ O_{d,\P}\round{d^{- \lambda_\iota/2}}\cdot (\norm{b_{\gE^c(r)}}_2 + \sigma_\varepsilon).
\end{align*}
Using the elementary inequality $|\norm{a}_2^2-\norm{b}_2^2| \leq \norm{a-b}_2^2 +2\norm{b}_2\norm{a-b}_2$ for any vectors $a,b$ yields
\begin{align}\label{eq:rkhs_tech_bound_1}
   &~~~\abs{  \norm{\beta\tran\Phi_{\alp}R}_2^2 -\norm{ b_{\alp} \tran   V {\widehat E}_{r}}_2^2}\notag\\
   &\leq  O_{d,\P}\round{d^{ - \lambda_\iota} }\cdot (\norm{b_{\gE^c(r)}}_2^2 +\sigma_\varepsilon^2) +    \sum_{\vm:\vlambda\tran \vm \leq p+\delta}O_{d,\P}( d^{-\zeta/2 + \epsilon+ (\vlambda\tran\vm - p - \delta)})\cdot  \norm{\partial_r f^*_{\S_{\vm}^1}}_{L_2}^2\notag\\
   &~~~+O_{d,\P}\round{d^{- \lambda_\iota/2} }\cdot (\norm{b_{\gE^c(r)}}_2+\sigma_\varepsilon)\norm{\partial_r f^*_{\alp}}_{L_2}.
\end{align}
Note that we have the estimate
\begin{align*}
    \snorm{V_{\aslp} - I} = O_d(d^{-\zeta_1}),
\end{align*}
and for $S\in \aep$, since $n\cdot (D_{S,S}W_{S,S}) = \Theta_d(1) + O_d(d\inv)$, 
there exists a constant $c_r>0$ such that the following holds
\begin{align*}
     \abs{\sum_{S\in\aep\cap \gE(r)}V_{S,S}^2b_{S}^2 - c_r\cdot \norm{\partial_r f^*_{\aep}}_{L_2}^2} = O_{d}(d\inv)\cdot\norm{\partial_r f^*_{\aep}}_{L_2}^2.
\end{align*}
With the scaling of ${\widehat E}_r$~\eqref{eq:d_inv_r_rkhs}, we can deduce
\begin{align*}
    \norm{b_{\S_\vm^1}\tran ({\widehat E}_r)_{\S_\vm^1}}_2^2 = 
    (g^{(\norm{\vm}_1)}(0))\inv \norm{b_{\S_\vm^1}\tran W_{\S_{\vm}^1}\half d^{(\norm{\vm}_1 - p - \delta)/2} (1+ O_d(d\half))}_2^2 =\Theta_d(d^{\vlambda\tran \vm - p-\delta}) \norm{b_{\S_\vm^1}}_2^2.
\end{align*}
Therefore, using triangle inequality we conclude
\begin{align}\label{eq:rkhs_tech_bound_2}
 \norm{ b_{\alp} \tran V {\widehat E}_{r} }_2^2= \sum_{\vm:\vlambda\tran\vm\leq p+\delta}\Theta_d(d^{\vlambda\tran \vm - p-\delta})\cdot \norm{\partial_r f^*_{\S_\vm^1} }_{L_2}^2.
\end{align}
Plugging  Eq.~\eqref{eq:rkhs_tech_bound_2} into \eqref{eq:rkhs_tech_bound_1} gives the bound for
\begin{align*}
    \abs{\norm{\beta_{\gE(r)}\tran  \Phi_{\alp}R}_2^2 -\sum_{\vm:\vlambda\tran\vm\leq p+\delta}\Theta_d(d^{\vlambda\tran \vm - p-\delta})\cdot \norm{\partial_r f^*_{\S_\vm^1} }_{L_2}^2},
\end{align*}
which is
\begin{align*}
 &O_{d,\P}\round{d^{ - \lambda_\iota} }\cdot (\norm{f^*}_{L_2}^2+\sigma_\varepsilon^2) +    \sum_{\vm:\vlambda\tran \vm \leq p+\delta}O_{d,\P}( d^{-\zeta/2 + \epsilon+ (\vlambda\tran\vm - p - \delta)})\cdot  \norm{\partial_r f^*_{\S_{\vm}^1}}_{L_2}^2\\
   &~~~+O_{d,\P}\round{d^{ - \lambda_\iota/2}}\cdot (\norm{f^*}_{L_2}+\sigma_\varepsilon)\norm{\partial_r f^*_{\alp}}_{L_2},
\end{align*}
where we use the inequality $\norm{b_{\gE^c(r)}}_2^2 \leq \norm{f^*}_{L_2}^2$.

\subsubsection{Proof of Proposition~\ref{prop:s2_s2_rkhs}}\label{proof:s2_s2_rkhs}

Invoking Eq.~\eqref{eq:def_r_j_1_rkhs}, we can deduce
\begin{align}
    &~~~\norm{\beta\tran \Phi_{\S_{\vm}^1}R_{\vm}^1 -  \beta \tran \Phi_{\S_{\vm}^1} \round{n\inv g^{(\norm{\vm}_1)}(0)  W_{\S_{\vm}^1}d^{-\norm{\vm}_1}}^{\frac{1}{2}}}_2 \notag\\
    &\leq \norm{\beta \tran \Phi_{\S_{\vm}^1} \round{n\inv g^{(\norm{\vm}_1)}(0)  W_{\S_{\vm}^1}d^{-\norm{\vm}_1}}^{\frac{1}{2}}}_2 \cdot O_d(d^{-1/2}). \label{eq:s2_s2_cube_bound_0_rkhs}
\end{align}
For the subsequent analysis, we temporarily omit the multiplicative constant $\sqrt{g^{(\norm{\vm}_1)}(0)}$   and focus on the term $n\half d^{-\frac{\norm{\vm}_1}{2}}\beta \tran \Phi_{\S_{\vm}^1} \round{ W_{\S_{\vm}^1}}^{\frac{1}{2}}$.
We split the whole index set of features $\cup_{j = 0}^\ell \S_{j}$ into $\Q\sqcup \S_{\vm}^1$.
Correspondingly, we split $\beta$ into 
\begin{align*}
    \beta_\Q := K_{w,\lambda}\inv \Phi_{\Q} b_\Q,~~\beta_{\S_{\vm}^1}:= K_{w,\lambda}\inv \Phi_{\S_{\vm}^1} b_{\S_{\vm}^1},~~{\rm and}~~~\beta_\varepsilon :=  K_{w,\lambda}\inv\vepsilon,
\end{align*}
then we obtain the bound
\begin{align}
    &~~~\norm{n\half d^{-\frac{\norm{\vm}_1}{2}}\beta \tran \Phi_{\S_{\vm}^1} \round{ W_{\S_{\vm}^1}}^{\frac{1}{2}} - (\rho+\lambda)\inv d^{(p+\delta -\norm{\vm}_1)/2}\cdot b_{\S_{\vm}^1}\tran (W_{\S_{\vm}^1})^{\frac{1}{2}} }_2\notag \\
    &\leq \norm{n\half d^{-\frac{\norm{\vm}_1}{2}}\beta_\Q \tran \Phi_{\S_{\vm}^1} \round{ W_{\S_{\vm}^1}}^{\frac{1}{2}}}_2 \notag\\
    &~~~+ \norm{ n\half d^{-\frac{\norm{\vm}_1}{2}}\beta_{\S_\vm^1} \tran \Phi_{\S_{\vm}^1} \round{ W_{\S_{\vm}^1}}^{\frac{1}{2}} -(\rho+\lambda)\inv d^{(p+\delta -\norm{\vm}_1)/2}\cdot b_{\S_{\vm}^1}\tran (W_{\S_{\vm}^1})^{\frac{1}{2}} }_2 \notag \\
    &~~~+ \norm{n\half d^{-\frac{\norm{\vm}_1}{2}}\beta_{\S_\vm^1} \tran \Phi_{\S_{\vm}^1} \round{ W_{\S_{\vm}^1}}^{\frac{1}{2}}}_2^2.\label{eq:s2_s2_cube_bound_1_rkhs}
\end{align}
For the first term, we apply the following proposition. The proof is deferred to Appendix~\ref{proof:feature_u_v_rkhs}.
\begin{proposition}\label{prop:feature_u_v_rkhs}
For any $\epsilon>0$,  the following estimate holds uniformly over all coordinates $r\in \T_\iota$:
    \begin{align}\label{eq:s2_s2_cube_bound_3_rkhs}
      \norm{n\half d^{-\frac{\norm{\vm}_1}{2}}\beta_\Q \tran \Phi_{\S_{\vm}^1} \round{ W_{\S_{\vm}^1}}^{\frac{1}{2}}}_2^2 =  O_{d,\P}(d^{ -\lambda_{\iota} +\epsilon})\cdot \norm{b_{\Q}}_{2}^2.
    \end{align}
\end{proposition}

\noindent Now we proceed to analyze the second term.  
After rescaling and multiplying both terms on the left side of the equation~\eqref{eq:term_p_plus_one_m} by $(W_{\S_{\vm}^1})^{1/2}$, we yield
\begin{align}
    &~~~\norm{ n\half d^{-\frac{\norm{\vm}_1}{2}}\beta_{\S_\vm^1} \tran \Phi_{\S_{\vm}^1} \round{ W_{\S_{\vm}^1}}^{\frac{1}{2}} -(\rho+\lambda)\inv d^{(p+\delta -\norm{\vm}_1)/2}\cdot b_{\S_{\vm}^1}\tran (W_{\S_{\vm}^1})^{\frac{1}{2}} }_2 \notag\\
    &=d^{(p+\delta -\vlambda\tran \vm)/2 } \cdot O_{d,\P}(d^{-\zeta/2+\epsilon} )\cdot   \norm{b_{\S_{\vm}^1}}_2\label{eq:s2_s2_cube_bound_4_rkhs}.
\end{align}
Next, we show that the third term is small. We defer the proof of the following proposition to Appendix~\ref{proof:noise_term_p_plus_one_rkhs}.
\begin{proposition}\label{prop:noise_term_p_plus_one_rkhs}
For any $\epsilon>0$, the following estimate holds uniformly for all $r\in \T_\iota$:
\begin{align}\label{eq:noise_term_p_plus_one_rkhs}
    \norm{n\half d^{-\frac{\norm{\vm}_1}{2}}\beta_{\varepsilon} \tran \Phi_{\S_{\vm}^1} \round{ W_{\S_{\vm}^1}}^{\frac{1}{2}}}_2^2 = O_{d,\P}(d^{-\lambda_\iota+\epsilon})\cdot \sigma_\varepsilon^2
\end{align}
\end{proposition}

\noindent Lastly, plugging Eq.~\eqref{eq:s2_s2_cube_bound_3_rkhs}, \eqref{eq:s2_s2_cube_bound_4_rkhs} and \eqref{eq:noise_term_p_plus_one_rkhs} into \eqref{eq:s2_s2_cube_bound_1_rkhs} yields:
\begin{align*}
    &~~~\norm{n\half d^{-\frac{\norm{\vm}_1}{2}}\beta \tran \Phi_{\S_{\vm}^1} \round{ W_{\S_{\vm}^1}}^{\frac{1}{2}} - (\rho+\lambda)\inv d^{(p+\delta -\norm{\vm}_1)/2}\cdot b_{\S_{\vm}^1}\tran (W_{\S_{\vm}^1})^{\frac{1}{2}} }_2 \\
    &\leq d^{(p+\delta - \vlambda\tran \vm)/2 }\cdot O_{d,\P}(d^{-\zeta/2+\epsilon})\cdot   \norm{b_{\S_{\vm}^1}}_2 +  O_{d,\P}(d^{-\lambda_{\iota}/2 +\epsilon/2})\cdot (\norm{b_{\Q}}_{2}+\sigma_\varepsilon).
\end{align*}
Then  combining the above bound with Eq.~\eqref{eq:s2_s2_cube_bound_0_rkhs} gives
\begin{align*}
    &~~~\norm{\beta\tran \Phi_{\S_{\vm}^1}R_{\vm}^1 - (\rho+\lambda)\inv \sqrt{g^{(\norm{\vm}_1)}(0)} \cdot d^{(p+\delta -\norm{\vm}_1)/2}\cdot b_{\S_{\vm}^1}\tran (W_{\S_{\vm}^1})^{\frac{1}{2}} }_2 \\
    &\overset{(a)}{\leq}\norm{ \beta\tran \Phi_{\S_{\vm}^1}R_{\vm}^1 - \beta \tran\Phi_{\S_{\vm}^1} \round{n\inv g^{(\norm{\vm}_1)}(0)  W_{\S_{\vm}^1}d^{-\norm{\vm}_1}}^{\frac{1}{2}}}_2\\
    &~~~+\sqrt{g^{(\norm{\vm}_1)}(0)}\cdot \norm{n\half d^{-\frac{\norm{\vm}_1}{2}}\beta \tran \Phi_{\S_{\vm}^1} \round{ W_{\S_{\vm}^1}}^{\frac{1}{2}} - (\rho+\lambda)\inv d^{(p+\delta -\norm{\vm}_1)/2}\cdot b_{\S_{\vm}^1}\tran (W_{\S_{\vm}^1})^{\frac{1}{2}} }_2\\
    &\overset{(b)}{\leq} O_d(d\half) \cdot \norm{(\rho+\lambda)\inv \sqrt{g^{(\norm{\vm}_1)}(0)} \cdot d^{(p+\delta -\norm{\vm}_1)/2}\cdot b_{\S_{\vm}^1}\tran (W_{\S_{\vm}^1})^{\frac{1}{2}}}_2 + \round{1+O_d(d\half)}\\
    &~~~\cdot \round{d^{(p+\delta - \vlambda\tran \vm)/2 }\cdot O_{d,\P}(d^{-\zeta/2+\epsilon})\cdot   \norm{b_{\S_{\vm}^1}}_2 +  O_{d,\P}(d^{-\lambda_\iota/2 +\epsilon/2})\cdot (\norm{b_{\Q}}_{2}+\sigma_\varepsilon)}\\
    &= O_{d,\P}(d^{-\zeta_1/2+\epsilon})\cdot   \norm{b_{\S_{\vm}^1}}_2 +  O_{d,\P}(d^{-\lambda_\iota/2 +\epsilon/2})\cdot (\norm{b_{\Q}}_{2}+\sigma_\varepsilon),
\end{align*}
where both $(a)$ and $(b)$ follow from the triangle inequality.

Following a similar argument with the proof of Proposition~\ref{prop:s2_s2_cube_m}, we can derive a bound for
\begin{align}\label{eq:s2_s2_rkhs_complete}
    \abs{\norm{n\half d^{-\frac{\norm{\vm}_1}{2}}\beta \tran \Phi_{\S_{\vm}^1} \round{ W_{\S_{\vm}^1}}^{\frac{1}{2}}}_2^2 - \norm{(\rho+\lambda)\inv d^{(p+\delta -\norm{\vm}_1)/2}\cdot b_{\S_{\vm}^1}\tran (W_{\S_{\vm}^1})^{\frac{1}{2}}}_2^2},
\end{align}
which takes the form
\begin{align*}
     &O_{d,\P}(d^{- \lambda_\iota + \epsilon})\cdot (\norm{f^*}^2_{L_2}+\sigma_\varepsilon^2) +  d^{p+\delta - \vlambda\tran \vm}\cdot O_{d,\P}(d^{-\zeta/2+\epsilon})\cdot \norm{b_{\S_{\vm}^1}}_2^2 \\
     &~~~+ O_{d,\P}(d^{(p+\delta - \vlambda\tran \vm)/2 - \lambda_\iota/2 + \epsilon/2}) \cdot \norm{b_{\S_{\vm}^1}}_2 (\norm{f^*}_{L_2}+\sigma_\varepsilon),
\end{align*}
where we use the fact $\norm{b_{\Q}}_{2}\leq\norm{f^*}_{L_2}$.

Combining with the fact $\norm{b_{\S_{\vm}^1}}_2 = \norm{\partial_r f^*_{\S_\vm^1}}_{L_2}$ completes the proof.

\subsubsection{Proof of Proposition~\ref{lemma:beta_<_l2_rkhs}}\label{proof:beta_<_l2_rkhs}
The proof is analogous to the proof of Proposition~\ref{lemma:beta_<_l2_cube_m}. 
First, following Eq.~\eqref{eq:K_decomp_m}
we write
\begin{align}
    K_{w,\lambda}  =\Phi_{\alp}\hd_{\alp}\Phi_{\alp}\tran + (\rho+\lambda)H,
\end{align}
where $H := I_n + (\rho+\lambda)\inv \Delta_1$.

Then following Eq.~\eqref{eq:smw_<_p_m}, we have
\begin{align}
    \beta_{\gE(r)}\tran \Phi_{\alp} R_\vm = \frac{1}{n}b_{\gE(r)}\tran \Phi_{\gE(r)}\tran H\inv \Phi_{\alp}G\inv  \round{\hd_{\alp}\inv R_\vm}.
\end{align}
With the decomposition $\hd_{\alp}\inv R_\vm = ({\widehat E_r})_{\S_\vm} + ({\widehat E_r}^\circ)_{\S_\vm}$, Eq.~\eqref{eq:d_inv_r_rkhs} and \eqref{eq:g_inv_minus_v_inv} together give
\begin{align*}
    \snorm{G\inv  \hd_{\alp}\inv R_\vm - V ({\widehat E}_r)_{\S_\vm}}&\leq \snorm{G\inv}\snorm{({\widehat E}_r^\circ)_{\S_\vm}} +  \snorm{G\inv-V}\snorm{({\widehat E}_r)_{\S_\vm}} \\
    &=O_{d,\P}(w_r\cdot d^{(\vlambda\tran\vm - p - \delta-2)/2} + d^{-\zeta/2+\epsilon+ (\vlambda\tran\vm - p - \delta)/2})\\
    &= O_{d,\P}(d^{-\zeta/2+\epsilon+ (\vlambda\tran\vm - p - \delta)/2}),
\end{align*}
where we use the fact $\zeta\leq \lambda_\iota$.

Corresponding to Eq.~\eqref{eq:b_leq_p_decomp_m}, we obtain
\begin{align*}
        \norm{\beta_{\gE(r)}\tran \Phi_{\alp} R_\vm -  \frac{1}{n}b_{\gE(r)}\tran \Phi_{\gE(r)}\tran H\inv \Phi_{\alp}     V({\widehat E}_r)_{\S_\vm}}_2 = O_{d,\P}(d^{-\zeta/2+\epsilon+ (\vlambda\tran\vm - p - \delta)/2})\cdot\norm{b_{\gE(r)}}_2.
\end{align*}
Note that Eq.~\eqref{eq:b_leq_p_decomp_bound_1_m} correspondingly becomes
\begin{align*}
       &~~~\norm{ \frac{1}{n}b_{\gE(r)}\tran \Phi_{\gE(r)}\tran H\inv \Phi_{\alp} V({\widehat E}_{r})_{\S_\vm}-\frac{1}{n}b_{\gE(r)}\tran \Phi_{\gE(r)}\tran  \Phi_{\alp} V({\widehat E}_r)_{\S_\vm}}_2 \\
       &=O_{d,\P}\round{\norm{b_{\gE(r)}\tran V ({\widehat E}_r)_{\S_\vm}}_2} = O_{d,\P}( d^{-\zeta_1/2 + \epsilon+ (\vlambda\tran\vm - p - \delta-2)/2})\cdot \norm{b_{\gE(r)}}_2.
\end{align*}
Then invoking the bound $\snorm{\Phi_{\alp}\tran\Phi_{\alp}/n -I} = O_{d,\P}(d^{-\zeta_1/2 +\epsilon})$ (Theorem~\ref{lemma:phi_id_2}), we have
\begin{align*}
    \norm{\frac{1}{n}b_{\gE(r)}\tran \Phi_{\gE(r)}\tran\Phi_{\alp}V({\widehat E}_r)_{\S_\vm} - b_{\alp}\tran V({\widehat E}_r)_{\S_\vm}}_2 &= \norm{\frac{1}{n}b_{\alp}\tran {E}_r \Phi_{\alp}\tran\Phi_{\alp}V({\widehat E}_r)_{\S_\vm} - b_{\alp}\tran V({\widehat E}_r)_{\S_\vm}}_2\\ &\leq 
    \norm{b_{\alp}\tran  {E}_r \round{ \Phi_{\alp}\tran  \Phi_{\alp}/n - I} V ({\widehat E}_r)_{\S_\vm} }_2\\
    &\leq \norm{ b_{\alp} \tran E_r V ({\widehat E}_r)_{\S_\vm}}_2\cdot O_{d,\P}(d^{-\zeta_1/2 +\epsilon})\\
    &= \norm{ b_{\gE(r)}}_2\cdot O_{d,\P}(d^{-\zeta_1/2 +\epsilon + (\vlambda\tran\vm - p - \delta)/2}).
\end{align*}
Then the remainder of the analysis is identical to the proof of Proposition~\ref{lemma:beta_<_l2_cube_m}.

\subsubsection{Proof of Claim~\ref{prop:d_inv_r_rkhs}}\label{proof:d_inv_r_rkhs}
Note that $R$ contains blocks $R_{\S_\vm^1}$ and $R_{\S_\vm^0}$. We discuss each kind of block separately.

Also  recall that $D$ is a diagonal matrix with entry $D_{\S_\vm} =  g^{(\norm{\vm}_1)}(0)d^{-\norm{\vm}_1} + O_d(d^{-\norm{\vm}_1-1})$ and each entry of $W_{\S_\vm}$ is of the order $d^{\norm{\vm}_1 - \vlambda\tran \vm}$ based on Eq.~\eqref{eq:w}. We will use these two properties in the analysis.

For $S\in \S_{\vm}^1$ or $\S_{\vm}^0$, based on the definition of $D$ Eq.~\eqref{eq:k_z_z} and $W$ Eq.~\eqref{eq:w} we have
\begin{align*}
   (\hd)_{S,S}\inv =(DW)_{S,S}\inv &= \round{g^{(\norm{\vm}_1)}(0)}\inv W_{S,S}\inv d^{\norm{\vm}_1} + O_d(d^{-\vlambda\tran \vm - 1}).
\end{align*}
Since Eq.~\eqref{eq:def_r_j_1_m} gives
\begin{align*}
    R_{S,S} &= \sqrt{g^{(\norm{\vm}_1)}(0)}\cdot  W_{S,S}^{\frac{1}{2}}d^{-(\norm{\vm}_1+p+\delta)/2}+ O_d(d^{-(\vlambda\tran \vm+p+\delta+1)/2 }),
\end{align*}
by a straightforward calculation, we obtain 
\begin{align*}
  (\hd\inv R)_{S,S} =\round{g^{(\norm{\vm}_1)}(0)}\half W_{S,S}\half d^{(\norm{\vm}_1-p-\delta)/2}\cdot (1+ O_d(d\half)).
\end{align*}
Similarly, for $S\in \S_\vm^0$, Eq.~\eqref{eq:def_r_j_0_m} gives
\begin{align*}
    R_{S,S} &= \sqrt{g^{(\norm{\vm}_1+2)}(0)}\cdot w_r\cdot W_{S,S}^{\frac{1}{2}} d^{-(\norm{\vm}_1+p+\delta+2)/2} + w_r \cdot O_d(d^{-(\vlambda\tran \vm +p+\delta+3)/2}).
\end{align*}
From this, we derive 
\begin{align*}
  (\hd\inv R)_{S,S} = w_r \cdot O_d( d^{(\vlambda\tran\vm-p-\delta -2)/2}).
\end{align*}
Combining these two cases yields
\begin{align}
   (\hd \inv R)_{S,S} = \begin{cases}
      \round{g^{(\norm{\vm}_1)}(0)}\half W_{S,S}\half d^{(\norm{\vm}_1-p-\delta)/2}\cdot (1+ O_d(d\half))~~{\rm if}~S\in \cup_\vm\S_{\vm}^1;\\
     w_r \cdot O_d( d^{(\vlambda\tran\vm-p-\delta -2)/2}),~~{\rm if}~S\in \cup_\vm\S_{\vm}^0;\\
        0~~{\rm otherwise.}
    \end{cases}
\end{align}
Combining this with the expression ${E}_r$:
\begin{align}
   ({E}_r)_{S,S} = 
     \mathbbm{1}\biground{S\in \cup_\vm\S_{\vm}^1},
\end{align}
we complete the proof.

\subsubsection{Proof of Proposition~\ref{lemma:beta_>_l2_rkhs}}\label{proof:beta_>_l2_rkhs}

The proof is analogous to that of Proposition~\ref{lemma:beta_>_l2_cube_m}. First, following Eq.~\eqref{eq:smw_cube_m} we have
\begin{align}
     \beta_{\gE^c}\tran  \Phi_{\alp}R =\frac{1}{n}b_{\gE^c}\tran \Phi_{\gE^c}\tran H\inv \Phi_{\alp}G\inv\round{\hd_{\alp}\inv R}.
\end{align}
Following Claim~\ref{prop:d_inv_r_rkhs} which shows $\snorm{  \hd_{\alp}\inv R -   {\widehat E}_{r}} = \snorm{{\widehat E}_r^\circ}= O_d( w_r\cdot d\inv )$ we can obtain
\begin{align}
        \norm{\beta_{\gE^c}\tran \Phi_{\alp} R -  \frac{1}{n}b_{\gE^c}\tran \Phi_{\gE^c}\tran H\inv \Phi_{\alp} G\inv \widehat{E}_{r}}_2 = O_{d,\P}(w_r\cdot d\inv )\cdot \log(d)\norm{b_{\gE^c}}_2,
\end{align}
where we use the fact that $\snorm{H\inv} = O_{d,\P}(1)$ (Eq.~\eqref{eq:e_inv_minus_i_m}), $\snorm{G\inv} = O_{d,\P}(1)$ (Eq.~\eqref{eq:g_inv_bound}), and $\norm{\Phi_{\gE}b_{\gE^c}}_2 = O_{d,\P}(\sqrt{n}\log(d))\norm{b_{\gE^c}}_2$ (Proposition~\ref{prop:bound_y}).

Consequently, there exists $\Delta_1$ that satisfies $\norm{\Delta_1}_2 = O_{d,\P}(w_r\cdot d\inv \log(d))\cdot \norm{f_{\gE^c}^*}_{L_2} $  such that the following holds:
\begin{align}
\beta_{\gE^c}\tran \Phi_{\alp} R =\frac{1}{n}   b_{\gE^c}\tran \Phi_{\gE^c}\tran H\inv \Phi_{\alp} G\inv {\widehat E}_{r} + \Delta_1.
\end{align}
Comparing ${\widehat E}_r$ with $E_r$ appearing in Eq.~\eqref{eq:simplified_g_p_m}, they have the same support. By Claim~\ref{prop:d_inv_r_rkhs}, we have each non-zero entry of ${\widehat E}_r$ is of the order $O_d(1)$. Therefore, following the identical analysis in the proof of Proposition~\ref{lemma:beta_>_l2_cube_m}, we can yield the  bound
\begin{align*}
     \norm{\beta_{\gE^c}\tran \Phi_{\alp} R}_2^2 &\leq  2 \norm{\frac{1}{n}b_{\gE^c}\tran \Phi_{\gE^c}\tran H\inv \Phi_{\alp} G\inv \widehat{E}_{r}}_2^2 + 2\norm{\Delta_1}^2_2\\
     &= \round{ O_{d,\P}(d^{ - \lambda_\iota-\zeta_1+\epsilon}) + O_{d,\P}(d\inv) + O_{d,\P}(w_r^2d^{-2+\epsilon})}\cdot \norm{f_{\gE^c}^*}_{L_2}^2\\
     &=  O_{d,\P}(d^{ - \lambda_\iota})\cdot \norm{f_{\gE^c}^*}_{L_2}^2
\end{align*}
Thus we conclude the proof.

\subsubsection{Proof of Proposition~\ref{prop:cube_error_1_rkhs}}\label{proof:cube_error_1_rkhs}
The proof is almost identical to that of Proposition~\ref{prop:cube_error_1_m}.  
The only difference appears in the bound Eq.~\eqref{eq:cube_m_error_tech_bound_3} which becomes
\begin{align}
    \Tr\round{\hd_\alp\inv R  R\tran \hd_\alp\inv}=O_d\round{\norm{{\widehat E}_r}_F^2 + \norm{{\widehat E}^\circ_r}_F^2}  = O_d(d^{p+\delta - \lambda_{\iota}-\zeta_1}) + O_d(d^{p+\delta-1}).
\end{align}
Then, applying a similar analysis to the proof of Proposition~\ref{prop:cube_error_1_m}, we have
\begin{align*}
    \E_{\vepsilon}\brac{\norm{\beta_{\varepsilon}\tran \Phi_{\alp}R}_2^2} = O_d(d^{-\lambda_{\iota}-\zeta_1} + d\inv).
\end{align*}
Since we have $\zeta_1 >0$, applying Markov's inequality completes the proof.

\subsubsection{Proof of Proposition~\ref{prop:feature_u_v_rkhs}}\label{proof:feature_u_v_rkhs}

The proof is analogous to that of Proposition~\ref{prop:feature_u_v_m}, and throughout we follow the conventions and notation established there.

Multiplying both sides of the equation \eqref{eq:p_plus_one_m} 
 by $\frac{n}{d^{\norm{\vm}_1}}$ and taking $W_{\S_\vm^1}$ into account, which satisfies $\snorm{W_{\S_\vm^1}} = O_d(d^{\norm{\vm}_1 - \vlambda\tran\vm})$ yields:
\begin{align*}
    \norm{\beta_\Q \tran \Phi_{\S_{\vm}^1} \round{n\inv  W_{\S_{\vm}^1}d^{-\norm{\vm}_1}}^{\frac{1}{2}}}_2^2 &=  O_{d,\P}(d^{\vs\tran\vm -s_\iota -p - \delta +1/q_0 +(p+\delta)- \norm{\vm}_1 + \norm{\vm}_1 - \vlambda\tran \vm})\cdot \log(d)\cdot \norm{b_{\Q}}_{2}^2,\\
    &=O_{d,\P}(d^{\vs\tran\vm -\vlambda\tran\vm -s_\iota  +1/q_0})\cdot \log(d)\cdot \norm{b_{\Q}}_{2}^2\\
    &= O_{d,\P}(d^{-\lambda_\iota + 1/q_0})\cdot \log(d)\cdot \norm{b_{\Q}}_{2}^2,
\end{align*}
where in the last equality we replace $s_\iota$ by $\lambda_\iota$ since $m_\iota\geq 1$ and we use the fact $\vs\tran \vm - \vlambda\tran \vm \leq s_\iota - \lambda_\iota \leq 0$.

Absorbing the factor $d^{1/q_0}\log(d)$ into $d^{\epsilon}$ where $\epsilon>0$ is an arbitrarily small constant yields the desired bound.

\subsubsection{Proof of Proposition~\ref{prop:noise_term_p_plus_one_rkhs}}\label{proof:noise_term_p_plus_one_rkhs}
The proof is almost identical to that of  Proposition~\ref{prop:noise_term_p_plus_one_m}.

In parallel to Eq.~\eqref{eq:noise_term_p_plus_one_m_tech_bound1}, we can obtain
\begin{align}
    \E_\vepsilon\brac{\norm{n\half d^{-\frac{\norm{\vm}_1}{2}}\beta_{\S_\vm^1} \tran \Phi_{\S_{\vm}^1} \round{ W_{\S_{\vm}^1}}^{\frac{1}{2}}}_2^2} &\leq  \snorm{W_{\S_\vm^1}}\cdot d^{-\norm{\vm}_1 - p - \delta} \cdot \E\brac{\norm{ \beta_\varepsilon\tran \Phi_{\S_{\vm}^1}}_2^2}\notag\\
    &\leq O_d(d^{-\vlambda\tran\vm - p -\delta})\cdot \snorm{K_{w,\lambda}^{-2}}\Tr\round{\Phi_{\S_{\vm}^1}\tran \Phi_{\S_{\vm}^1}}.
\end{align}
Then the remainder of the proof follows exactly with the proof of Proposition~\ref{prop:noise_term_p_plus_one_m} and we can get the bound 
\begin{align*}
    O_{d,\P}(d^{\vs\tran \vm - s_\iota - \vlambda\tran\vm - p - \delta +p+\delta}) =O_{d,\P}(d^{-\lambda_\iota}).
\end{align*}
Applying Markov's inequality we obtain the bound $O_{d,\P}(d^{-\lambda_\iota})\log(d)$ with probability at least $1-1/ \log(d)$. Absorbing the logarithmic factor into $\epsilon$ completes the proof.

%% file: sections_appendix_b/prop_generalization.tex
\subsection{Proof of Propositions for Theorem~\ref{thm:learning_hypercube} (Generalization bound)}
\subsubsection{Proof of Lemma~\ref{prop:mercer}}\label{proof:mercer}
For kernel $K_w(x,y) = g(\frac{x\tran \Diag(w) y}{d})$, we can associate a self adjoint operator $\mathcal{H}_d:L_2(\H^d,\tau_d) \rightarrow L_2(\H^d,\tau_d)$:
\begin{align}
    (\mathcal{H}_d f)(x) = \E_y\brac{K_w(x,y)f(y)}.
\end{align}
We denote the Fourier basis by $\phi_\I(x) := x^{\I}$ for any multi-index $\I\in[d]$ and denote $z_i = x_iy_i$ for $i = 1,\ldots,d$. Then we have
\begin{align*}
  (\mathcal{H}_d \phi_\I)(x) &= \E_y\brac{K(x,y)\phi_\I(y)} \\
  &\overset{(a)}{=}\frac{1}{2^d} \sum_{z\in\H^d}\round{g\round{\frac{w\tran z}{d}}\phi_\I(x)\phi_\I(z)}\\
&= \underbrace{\frac{1}{2^d} \sum_{z\in\H^d}\round{g\round{\frac{w\tran z}{d}}\phi_\I(z)}}_{\lambda_\I} \phi_\I(x),
\end{align*}
where $(a)$ follows from the fact $y^\I = x^\I z^\I$.

Now we analyze the eigenvalue $\lambda_\I$. 
We take the Taylor expansion of $g\round{\frac{w\tran z}{d}}$ around $0$ and we obtain
\begin{align}\label{eq:lambda_I}
    \lambda_\I = \sum_{m=0}^\infty\frac{g^{(m)}(0)}{m!} \underbrace{\frac{1}{2^d} \sum_{z\in \H^d}  \round{\frac{w\tran z}{d}}^m \phi_{\I}(z)}_{E_{m,\I}}.
\end{align}
We now consider
\begin{align*}
 E_{m,\I} =  \E_z \brac{\round{\frac{w\tran z}{d}}^m \phi_{\I}(z)} =   \frac{1}{d^m}\E_z \brac{\round{\sum_{{j_1},\ldots,j_m \in [d]^{m}}\round{w_{j_1}z_{j_1}}\ldots \round{w_{j_m}z_{j_m}}}\phi_{\I}(z)}.
\end{align*}
Denoting $r = |\I|$, with simple calculation, we observe that $E_{m,\I } = 0$ unless $m - r\geq 0$ and $(m-r)$ is even. Since $\norm{w}_\infty = O_d(d^\gamma)$ with $\gamma<1$, we have
\begin{align*}
    |E_{m,\I}| = O_d(d^{-m}\norm{w}_\infty^{m-r} |w^{\I}|) = o_d(d^{-m + (m-r)\gamma})|w^{\I}|.
\end{align*}
Therefore, the series is dominated by $m=r$ and we can conclude
\begin{align}\label{eq:e_m_i}
     E_{m,\I} = \frac{m!}{d^m} w^\I + O_d(d^{-m-2})\cdot w^\I\cdot d^{2\gamma},
\end{align}
where the main term comes from matching $r$ indices of $\I$ and the error is driven by the first non-vanishing $m=r+2$ contribution.

Plugging \eqref{eq:e_m_i} back into Eq.~\eqref{eq:lambda_I} yields
\begin{align*}
      \lambda_\I = g^{(|\I|)}(0)d^{-|\I|}w^\I + O_d(d^{-|\I|-2+2\gamma})\cdot w^\I.
\end{align*}
Then writing $K_w(x,y)$ as the summation of $\lambda_\I \phi_\I(x)\phi_\I(y)$ over all $\I$ completes the proof.

\subsubsection{Proof of Proposition~\ref{prop:approx_H}}\label{proof:approx_H}
Invoking Lemma~\ref{prop:mercer}  we can derive the decomposition
\begin{align}\label{eq:H_decomp}
     \E_x[K(Z,\sqrt{w}\odot x)K(\sqrt{w}\odot x, Z)] = \sum_{k=0}^\infty  \sum_{\vm:\norm{\vm}_1 = k}\Phi_{\S_\vm}(X) \Lambda^2_{\S_\vm} \Phi_{\S_\vm}(X)\tran,
\end{align}
where $ \Lambda_{\S_\vm}$ is a diagonal matrix with entries $\Lambda_{S,S} =g^{-\norm{\vm}_1}(0) d^{-\norm{\vm}_1}w^{S} + O_d(d^{-\norm{\vm}_1-2+2\gamma})\cdot w^S$ for $S \in \S_\vm$.

We split Eq.~\eqref{eq:H_decomp} into two parts based on $\vm$: $\vlambda\tran\vm \leq p+\delta$ and $\vlambda\tran \vm > p+\delta$:
\begin{align}\label{eq:decom_psi_lambda_psi}
    &~~~\sum_{k=0}^\infty  \sum_{\vm:|\vm| = k}\Phi_{\S_\vm}(X) \Lambda_{\S_\vm}^2 \Phi_{\S_\vm}(X)\tran \notag\\
    &= \sum_{\lmlp}\Phi_{\S_\vm}(X) \Lambda_{\S_\vm}^2 \Phi_{\S_\vm}(X)\tran + \sum_{\lmgp}\Phi_{\S_\vm}(X) \Lambda_{\S_\vm}^2 \Phi_{\S_\vm}(X)\tran.
\end{align}
Next, we show that the second term is small. Similar to Theorem~\ref{thm:taylor_approx_kernel_m}, we split the set:
\begin{align*}
    \{\vm~\vert~\vlambda\tran \vm > p+\delta\} = \biground{\vm~\vert~\vlambda\tran \vm \geq 2p+3\delta+1} \bigsqcup \biground{\vm~\vert~p+\delta< \vlambda\tran \vm < 2p+3\delta+1}
\end{align*}
and we apply different techniques for each set. 

First, note the Frobeniusm norm of $\Phi_{\S_\vm}(X)$ can be bound as 
\begin{align*}
    \norm{\Phi_{\S_\vm}(X)}_F^2 = n\cdot d^{\vs\tran \vm}
\end{align*}
Then for $\vm \in  \biground{\vm~\vert~\vlambda\tran \vm \geq 2p+3\delta+1}$ we have the bound
\begin{align}
    \snorm{\Phi_{\S_\vm}(X) \Lambda_{\S_\vm}^2\Phi_{\S_\vm}(X)\tran} &\leq \norm{\Phi_{\S_{\vm}}(X)}_F^2\cdot \snorm{ \Lambda_{\S_\vm}}^2 \notag \\
    &=  O_{d,\P}(d^{\vs\tran \vm} \cdot d^{p+\delta}) O_d(d^{-2\norm{\vm}_1})\cdot O_d(d^{2(\mathbf{1} - \vlambda)\tran \vm})\notag\\
    &= O_{d,\P}(d^{p+\delta + (\vs-2\vlambda)\tran\vm}) = O_{d,\P}(d^{p+\delta - \vlambda\tran \vm}),
\end{align}
where the last inequality follows from the fact that $s_k\leq \lambda_{k}$ for all $k\in[N]$.

We proceed to use triangle inequality to bound the summation
\begin{align}\label{eq:decom_psi_lambda_psi_extra}
        \snorm{\sum_{\vlambda\tran\vm \geq 2p+3\delta+1}\Phi_{\S_\vm}(Z) \Lambda_{\S_\vm}^2 \Phi_{\S_\vm}(Z)\tran} &\leq          \sum_{\vlambda\tran\vm \geq 2p+3\delta+1}\snorm{\Phi_{\S_\vm}(Z) \Lambda_{\S_\vm}^2 \Phi_{\S_\vm}(Z)\tran}\notag \\
        &{\leq} \sum_{k = 2p}^\infty \sum_{\vm:\vlambda\tran\vm \in[k+3\delta+1,k+3\delta+2)}\snorm{\Phi_{\S_\vm}(Z) \Lambda_{\S_\vm}^2 \Phi_{\S_\vm}(Z)\tran}\notag \\
        &\overset{(a)}{\leq}  \sum_{k = 2p}^\infty  O_d((k+5)^N) O_{d,\P}(d^{p-k-2\delta - 1})\notag\\
        &=   \sum_{k = 0}^\infty O_{d,\P}(d^{-p- 2\delta-1 -k + N\log_d(2p+k+5)}),
\end{align}
where $(a)$ follows from the fact that all $\vm\in \R^N$ that satisfies $\vlambda\tran \vm \in [k+3\delta+1,k+3\delta+2)$ is contained in the set $\{\vm: \norm{\vm}_1 \leq (k+5)/ \norm{\vlambda}_\infty \}$, whose cardinality is bounded by    $O_d((k+5)^N)$. 

Now we consider the sequence 
$$\{d^{- p - 2\delta -1 - k +N\log_d(2p+k+5)}\}_{k=0}^\infty.$$ Observe that  the ratio $r$ of consecutive terms satisfies 
\begin{align*}
    r \leq d^{-1 + N\log_d((2p+6)/(2p+5))}.
\end{align*}
Therefore, the summation of the sequence is bounded by the summation a geometric series with first term $d^{-p - 2\delta -1 + N\log_d(2p+5)}$ and a common ratio $d^{-1 + N\log_d((2p+6)/(2p+5))}$.

Using the formula to compute the summation of the geometric series we can show
\begin{align*}
     \sum_{k = 0}^\infty O_{d,\P}(d^{- p - 2\delta -1 - k +N\log_d(2p+k+5)}) = O_{d,\P}(d^{- p - 2\delta -1 + N\log_d(2p+5)}).
\end{align*}
Since $\delta \gtrsim N\log_d(2p+5)$, plugging the bound back into  Eq.~\eqref{eq:decom_psi_lambda_psi_extra} yields
\begin{align}\label{eq:psi_lambda_psi_part1}
      \snorm{\sum_{\vlambda\tran\vm \geq 2p+3\delta+1/2}\Phi_{\S_\vm}(Z) \Lambda_{\S_\vm}^2 \Phi_{\S_\vm}(Z)\tran} = O_{d,\P}(d^{-p-\delta-1}).
\end{align}
Next we consider the set $\biground{\vm~\vert~p+\delta< \vlambda\tran \vm < 2p+3\delta+1}$. Lemma~\ref{lem:norm_control} shows the following hold for any $\vm \in \biground{\vm~\vert~p+\delta< \vlambda\tran \vm < 2p+3\delta+1}$  
\begin{align}\label{eq:norm_psi_sm}
 \snorm{\Phi_{\S_{\vm}}(X) }^2=  O_{d,\P}(d^{\vs\tran \vm} + d^{p+\delta}).
\end{align}
Then for each $\vm$ we have the bound
\begin{align}
    \snorm{\Phi_{\S_\vm}(X) \Lambda_{\S_\vm}^2\Phi_{\S_\vm}(X)\tran} &\leq \snorm{\Phi_{\S_{\vm}}(X)}^2\cdot \snorm{ \Lambda_{\S_\vm}}^2 \notag \\
    &=  O_{d,\P}(d^{\vs\tran \vm} + d^{p+\delta}) O_d(d^{-2\norm{\vm}_1})\cdot O_d(d^{2(\mathbf{1} - \vlambda)\tran \vm})\notag\\
    &= O_{d,\P}(d^{(\vs-2\vlambda)\tran \vm} + d^{p+\delta - 2\vlambda\tran \vm}) = O_{d,\P}(d^{-\vlambda\tran \vm}).\label{eq:psi_lambda_square_psi}
\end{align}
Invoking the triangle inequality yields:
\begin{align}\label{eq:psi_lambda_psi_part2}
        \snorm{\sum_{p+\delta< \vlambda\tran \vm < 2p+3\delta+1}\Phi_{\S_\vm}(Z) \Lambda_{\S_\vm}^2 \Phi_{\S_\vm}(Z)\tran} &\leq \sum_{p+\delta< \vlambda\tran \vm < 2p+3\delta+1}O_{d,\P}(d^{-\vlambda\tran \vm}) \\
        &= O_{d,\P}(d^{-p-\delta-\zeta_2}),
\end{align}
where the second equality follows from the fact that the cardinality of the set $\{\vm~\vert~p+\delta< \vlambda\tran \vm < 2p+3\delta+1\}$ is finite.

Combining Eq.~\eqref{eq:psi_lambda_psi_part1} and \eqref{eq:psi_lambda_psi_part2} with \eqref{eq:decom_psi_lambda_psi} yields:
 \begin{align*}
    \E_x[K(Z,\sqrt{w}\odot x)K(\sqrt{w}\odot x,Z)] = \Phi_{\alp}(X) \Lambda_{\alp}^2 \Phi_{\alp}(X)\tran + \Delta,
\end{align*}
with  $
    \snorm{\Delta} =O_{d,\P}(d^{- p - \delta - \zeta_2})$, which completes the proof.

\end{proof}

\subsubsection{Proof of Proposition~\ref{prop:t11_est}}\label{proof:t11_est}
We define $\beta:= K_{w,\lambda}\inv y$ and decompose $\beta$ into 
\begin{align}
    \beta  = \underbrace{K_{w,\lambda}\inv\Phi_{\alp}b_{\alp}}_{:=\beta_{\alp}}  +\underbrace{K_{w,\lambda}\inv\Phi_{\agp}b_{\agp}}_{:=\beta_{\agp}} + \underbrace{K_{w,\lambda}\inv\vepsilon}_{:=\beta_\varepsilon}.
\end{align}
and we have the decomposition of the absolute difference
\begin{align}
    &~~~\abs{T_{11} - b_{\alp}\tran V^2 b_{\alp}  } \notag\\
    &\leq \abs{\beta_{\alp}\tran M_{\alp} \beta_{\alp} -  b_{\alp}\tran V^2 b_{\alp}}+ \abs{\beta_{\agp}\tran  M_{\alp}\beta_{\agp}} + \abs{\beta_{\varepsilon}\tran  M_{\alp}\beta_{\varepsilon} }\notag\\
    &~~~+ 2 \abs{\beta_{\alp}\tran  M_{\alp} \beta_{\agp}} +  2\abs{(\beta_{\alp}+\beta_{\agp})\tran  M_{\alp} \beta_{\varepsilon}}.\label{eq:t_11_decomp}
\end{align}
For the bound above, we will show that the first three terms are small hence the last two terms are also small followed by Cauchy-Schwartz inequality.

We first establish the following key estimates. The proof appears in Appendix~\ref{proof:test_loss_tech_1}.
\begin{proposition}\label{prop:test_loss_tech_1}
For any $\epsilon >0$, the following estimates hold:
    \begin{align}
    \snorm{n K_{\lm,\lambda}\inv M_{\alp} K_{\lm,\lambda}\inv  - \frac{1}{n}\Phi_\alp V^2\Phi_\alp\tran } &=  O_{d,\P}(d^{-\zeta_1/2+\epsilon} + d^{-\zeta_2/2+\epsilon} + d^{-2+2\mgamma})\label{eq:test_loss_tech_bound_1}\\
        \snorm{ \Phi_{\alp}\tran K_{\lm,\lambda}\inv M_{\alp} K_{\lm,\lambda}\inv\Phi_{\alp}  -  V^2 } &=  O_{d,\P}(d^{-\zeta_1/2+\epsilon} + d^{-\zeta_2/2+\epsilon} + d^{-2+2\mgamma})\label{eq:test_loss_tech_bound_2},
    \end{align}
where $\mgamma:= \max_{k\in[N]}\gamma_k$.
\end{proposition}

\noindent Following Eq.~\eqref{eq:zeta}, we denote $\zeta:=\min(\zeta_1,\zeta_2)$ to simplify the notation. Then we can immediately obtain that the first term in Eq.~(\ref{eq:t_11_decomp}) is bounded by
\begin{align}\label{eq:test_loss_t11_term_1}
   &~~~\norm{b_{\alp}}_2^2\cdot \snorm{V^2 - \Phi_{\alp}\tran K_{\lm,\lambda}\inv M_{\alp} K_{\lm,\lambda}\inv\Phi_{\alp} }\notag\\
   &= O_{d,\P}(d^{-\zeta/2+\epsilon} + d^{-2+2\mgamma}) \cdot \norm{b_{\alp}}_2^2
\end{align}
For the second term in Eq.~(\ref{eq:t_11_decomp}), invoking Proposition~\ref{prop:test_loss_tech_1} we can obtain the estimate
\begin{align}\label{eq:t_11_term2}
    &~~~\abs{y_{\agp}\tran K_{\lm,\lambda}\inv M_{\alp} K_{\lm,\lambda}\inv y_{\agp}- \norm{\frac{1}{n}\Phi_{\alp}y_{\agp}}_2^2}\notag\\
    &\leq O_{d,\P}(d^{-\zeta/2+\epsilon} + d^{-2+2\mgamma})\cdot \frac{1}{n}\norm{y_{\agp}}_2^2 \notag \\
    &=  O_{d,\P}(d^{-\zeta/2+\epsilon} + d^{-2+2\mgamma})\cdot \log(d) \cdot \norm{f^*_{\agp}}_{L_2}^2,
\end{align}
where the last equality follows from Proposition~\ref{prop:bound_y}.

Similar to the analysis to Eq.~\eqref{eqn:blup}, we can see that 
\begin{align*}
    \E\brac{\norm{\frac{1}{n}\Phi_{\alp}\Phi_{\agp}b_{\agp}}_2^2} &= \frac{1}{n^2}\E\brac{\inner{\Phi_{\alp}\Phi_{\alp}\tran ,\Phi_{\agp}b_{\agp }b_{\agp}\tran  \Phi_{\agp}\tran}}\\
    &=\frac{|\alp|}{n}\norm{b_{\agp}}_2^2= O_{d}(d^{-\zeta_1})\cdot \norm{b_{\agp}}_2^2.
\end{align*}
Applying Markov's inequality we obtain 
$$\norm{\frac{1}{n}\Phi_{\alp}\Phi_{\agp}b_{\agp}}_2^2 = O_{d,\P}(d^{-\zeta_1}\log(d))\cdot \norm{b_{\agp}}_2^2.$$
Plugging it back into Eq.~(\ref{eq:t_11_term2}) yields
\begin{align}\label{eq:test_loss_t11_term_2}
    \abs{y_{\agp}\tran K_{\lm,\lambda}\inv M_{\alp} K_{\lm,\lambda}\inv y_{\alp}} = O_{d,\P}(d^{-\zeta/2+\epsilon} + d^{-2+2\mgamma})\cdot \log(d) \cdot \norm{f^*_{\agp}}_{L_2}^2.
\end{align}
For the third term in Eq.~(\ref{eq:t_11_decomp}), taking the expectation over $\vepsilon$ yields
\begin{align*}
    \E_\vepsilon\brac{\beta_\varepsilon\tran M_{\alp} \beta_\varepsilon} &= \sigma_\varepsilon^2\cdot \Tr\round{K_{\lm,\lambda}\inv M_{\alp} K_{\lm,\lambda}\inv}\\
    &\overset{(a)}{=} \frac{\sigma_\varepsilon^2}{n}\round{\Tr\round{\frac{1}{n}\Phi_\alp V^2 \Phi_\alp \tran} + n\cdot O_{d,\P}(d^{-\zeta/2+\epsilon} + d^{-2+2\mgamma})}\\
    &\overset{(b)}{\leq}  \frac{\sigma_\varepsilon^2}{n}\Tr\round{\frac{1}{n}\Phi_\alp \tran \Phi_\alp} \cdot \snorm{V}^2 + \sigma_\varepsilon^2 \cdot O_{d,\P}(d^{-\zeta/2+\epsilon} + d^{-2+2\mgamma})\\
    &\overset{(c)}{\leq} O_d(d^{-\zeta_1})\cdot \sigma_\varepsilon^2 +  \sigma_\varepsilon^2 \cdot O_{d,\P}(d^{-\zeta/2+\epsilon} + d^{-2+2\mgamma})\\
    &= \sigma_\varepsilon^2 \cdot O_{d,\P}(d^{-\zeta/2+\epsilon} + d^{-2+2\mgamma}),
\end{align*}
where $(a)$ follows from  Proposition~\ref{prop:test_loss_tech_1},  $(b)$ follow from the inequality  $|\Tr(AB)| \leq \Tr(A)\snorm{B}$ and $(c)$ follows from the bound $\snorm{V} = O_{d}(1)$ (Eq.~\eqref{eq:def_v}) and the equality $\Tr\round{\frac{1}{n}\Phi_\alp \tran \Phi_\alp} =\frac{1}{n}\norm{\Phi_\alp}_F^2 = |\alp|$.

Using Markov's inequality, we conclude the bound
\begin{align}\label{eq:test_loss_t11_term_3}
    \abs{\beta_\varepsilon\tran M_{\alp} \beta_\varepsilon} = \sigma_\varepsilon^2\cdot  O_{d,\P}(d^{-\zeta/2+\epsilon} + d^{-2+2\mgamma}).
\end{align}
With Eq.~\eqref{eq:test_loss_t11_term_1} and \eqref{eq:test_loss_t11_term_2}, we apply Cauchy-Schwarz inequality for the fourth term in Eq.~(\ref{eq:t_11_decomp})
\begin{align}\label{eq:test_loss_t11_term_4}
    \abs{\beta_{\alp}\tran  M_{\alp} \beta_{\agp}} &\leq \sqrt{ \abs{\beta_{\alp}\tran M_{\alp} \beta_{\alp}}} \sqrt{ \abs{\beta_{\agp}\tran M_{\alp} \beta_{\agp}}} \notag\\
    &\leq 2 \round{ \abs{\beta_{\alp}\tran M_{\alp} \beta_{\alp}} + \abs{\beta_{\agp}\tran M_{\alp} \beta_{\agp}}} \notag\\
    &\leq O_{d,\P}(d^{-\zeta/2+\epsilon} + d^{-2+2\mgamma})\cdot \log(d) \cdot \norm{f^*}_{L_2}^2,
\end{align}
where the second inequality follows from the AM-GM inequality. 

Similarly, we bound the last term in Eq.~(\ref{eq:t_11_decomp}) using Cauchy-Schwarz and AM-GM inequality:
\begin{align}\label{eq:test_loss_t11_term_5}
   &~~~\abs{(\beta_{\alp}+\beta_{\agp})\tran  M_{\alp} \beta_{\varepsilon}}\notag\\
   &\leq  \abs{\beta_{\alp}\tran  M_{\alp} \beta_{\varepsilon}}  + \abs{\beta_{\agp} M_{\alp} \beta_{\varepsilon}}\notag\\
   &\leq 2 \round{ \abs{\beta_{\alp}\tran M_{\alp} \beta_{\alp}} + \abs{\beta_{\agp}\tran M_{\alp} \beta_{\agp}}} + 4 \abs{\beta_\varepsilon\tran M_{\alp}\beta_\varepsilon}\notag\\
   &=  O_{d,\P}(d^{-\zeta/2+\epsilon} + d^{-2+2\mgamma})\cdot\log(d)\cdot (\norm{f^*}_{L_2}^2+ \sigma_\varepsilon^2).
\end{align}
Following a similar analysis with the proof of Proposition~\ref{prop:<_l1_l1_l1_cube_m}, in parallel to Eq.~\eqref{eq:cube_m_tech_bound_2}, we can show  that with constant $c>0$ \eqref{eq:def_c_generalization} the following holds:
\begin{align}\label{eq:bvb}
  \abs{  \norm{ b_{\alp}  V}_2^2 - \norm{ f^*_{\aslp}}_{L_2}^2 - c \cdot \norm{ f^*_{\aep} }_{L_2}^2} = O_{d,\P}(d^{-\zeta_1})\cdot  \norm{ b_{\aslp}}_2^2.
\end{align}
Plugging the above bound together with with Eq.~\eqref{eq:test_loss_t11_term_1}, \eqref{eq:test_loss_t11_term_2}, \eqref{eq:test_loss_t11_term_3}, \eqref{eq:test_loss_t11_term_4} and \eqref{eq:test_loss_t11_term_5} into Eq.~\eqref{eq:t_11_decomp} and absorbing the logarithmic factor into $\epsilon$ completes the proof.

\subsubsection{Proof of Proposition~\ref{prop:t2_est}}\label{proof:t2_est}
Note that with Lemma~\ref{prop:mercer} we can deduce
\begin{align}
 \E_x[(f^*_{\aslp}(x)+\sqrt{c}\cdot f^*_{\aep}(x)) K(\sqrt{w}\odot x,Z)] = b_{\alp}\tran \widetilde{V}\Lambda_{\alp}\Phi_{\alp}(X)\tran,
\end{align}
where $\widetilde{V}_{\aslp} = I$ and $\widetilde{V}_{\aep} = V_{\aep}$.
With the decomposition 
\begin{align}
    \beta  = \underbrace{K_{w,\lambda}\inv\Phi_{\alp}b_{\alp}}_{:=\beta_{\alp}}  +\underbrace{K_{w,\lambda}\inv\Phi_{\agp}b_{\agp}}_{:=\beta_{\agp}} + \underbrace{K_{w,\lambda}\inv\vepsilon}_{:=\beta_\varepsilon},
\end{align}
we decompose the absolute difference:
\begin{align}
    &~~~\abs{T_2 - b_{\alp}\tran V^2 b_{\alp}} \notag\\
    &\leq \abs{b_{\alp}\tran (V^2 -\widetilde{V}^2) b_{\alp}}+ \abs{ b_{\alp}\tran  V\Lambda_{\alp} \Phi_{\alp}\tran \beta_{\alp}- b_{\alp}\tran V^2 b_{\alp}}\notag\\ 
    &~~~+ \abs{b_{\alp}\tran V \Lambda_{\alp} \Phi_{\alp}\tran \beta_{\agp}}+\abs{b_{\alp}\tran V \Lambda_{\alp} \Phi_{\alp}\tran \beta_{\varepsilon}},\label{eq:t_22_decomp}
\end{align}
where the diagonal matrix $V$ is defined in Eq.~\eqref{eq:def_v_generalization}.

It is straightforward to bound the first term using Eq.~\eqref{eq:v_minus_i}:
\begin{align*}
    \abs{b_{\alp}\tran (V^2 -\widetilde{V}^2) b_{\alp}} \leq \norm{b_{\alp}}_2^2\cdot \snorm{V^2 -\widetilde{V}^2} = O_d(d^{-\zeta_1})\norm{b_{\alp}}_2^2.
\end{align*}
In the following, we will show that the remaining three terms on the right side of the above inequality are small.
We first establish the following key estimate. We defer the proof to Appendix~\ref{proof:test_loss_tech_2}.
\begin{proposition}\label{prop:test_loss_tech_2}
For any $\epsilon>0$, the following estimate holds:
    \begin{align}
        \snorm{ \sqrt{n}\Lambda_{\alp} \Phi_{\alp}\tran K_{\lm,\lambda}\inv  - \frac{1}{\sqrt{n}} V \Phi_{\alp}\tran } =O_{d,\P}(d^{-\zeta_1/2+\epsilon} + d^{-\zeta_2/2+\epsilon}).
    \end{align}
\end{proposition}

\noindent Following Eq.~\eqref{eq:zeta}, we denote $\zeta:=\min(\zeta_1,\zeta_2)$ to simplify the notation.
For the second term in Eq.~\eqref{eq:t_22_decomp}, we can obtain 
\begin{align}\label{eq:test_loss_t2_bound1}
    &~~~\abs{ b_{\alp}\tran V \Lambda_{\alp} \Phi_{\alp}\tran K_{\lm,\lambda}\inv \Phi_{\alp}b _{\alp}- b_{\alp}\tran V^2 b_{\alp}} \notag\\
    &\overset{(a)}{=} \abs{\frac{1}{n} b_{\alp}\tran  V ^2\Phi_{\alp}\tran \Phi_{\alp}b_{\alp}-  b_{\alp}\tran V^2 b_{\alp} } + O_{d,\P}(d^{-\zeta/2+\epsilon})\cdot \snorm{\frac{1}{\sqrt{n}}\Phi_{\alp}} \cdot \norm{V b_{\alp} }_2^2 \notag\\
    &\overset{(b)}{=}  \abs{ b_{\alp}\tran V^2  \round{ \frac{1}{n}\Phi_{\alp}\tran \Phi_{\alp} - I}b_{\alp} } + O_{d,\P}(d^{-\zeta/2+\epsilon} )\cdot\norm{V b_{\alp}}_2^2 \notag\\
    &\overset{(c)}{=} O_{d,\P}(d^{-\zeta_1/2 + \epsilon})\cdot\norm{b_{\alp}}_2^2\snorm{V}^2 + O_{d,\P}(d^{-\zeta/2+\epsilon})\cdot\norm{b_{\alp}}_2^2\snorm{V}^2 \notag\\
    &=O_{d,\P}(d^{-\zeta/2+\epsilon})\cdot\norm{b_{\alp}}_2^2,
\end{align}
where $(a)$ follows from Proposition~\ref{prop:test_loss_tech_2}, $(b)$ follows from Lemma~\ref{lem:norm_control} and $(c)$ follows from Theorem~\ref{lemma:phi_id_2}.

For the third term of \eqref{eq:t_22_decomp}, we similar invoke Proposition~\ref{prop:test_loss_tech_2} and we can obtain
\begin{align}\label{eq:test_loss_t2_bound2}
   &~~~\abs{b_{\alp}\tran V \Lambda_{\alp} \Phi_{\alp}\tran K_{\lm,\lambda}\inv \Phi_{\agp} b_{\agp}} \notag\\
   &= \abs{\frac{1}{n}b_{\alp}\tran V^2 \Phi_{\alp}\tran \Phi_{\agp}b_{\agp}} +O_{d,\P}(d^{-\zeta/2+\epsilon})\cdot \snorm{\frac{1}{\sqrt{n}}\Phi_{\alp}}\norm{\frac{1}{\sqrt{n}}\Phi_{\agp} b_{\agp}}_2 \norm{b_{\alp}}_2\snorm{V}\notag \\
   &= \abs{\frac{1}{n}b_{\alp}\tran V^2 \Phi_{\alp}\tran \Phi_{\agp}b_{\agp}} + O_{d,\P}(d^{-\zeta/2+\epsilon})\log^{1/2}(d)\norm{b_{\agp}}_2\norm{b_{\alp}}_2,
\end{align}
where the last equality follows from Lemma~\ref{lem:norm_control} and Proposition~\ref{prop:bound_y}.

Since $V_{j,j}\leq  1$ for all $j\in \alp$, a simple calculation shows
\begin{align*}
  &~~~\E\brac{\abs{\frac{1}{n}b_{\alp}\tran V^2 \Phi_{\alp}\tran \Phi_{\agp}b_{\agp}}^2}\\
  &= \frac{1}{n^2} \E\brac{\round{\sum_{i\in[n]}\sum_{j\in \alp}\sum_{k\in \agp}b_j V_{j,j}^2 \phi_j(x^{(i)}) \phi_{k}(x^{(i)})b_k}^2} \\
  &= \frac{1}{n^2} \E\brac{\sum_{i\in[n]}\sum_{j\in \alp}\sum_{k\in\agp}b_j^2 V_{j,j}^{4} b_k ^2 \phi_j(x^{(i)})^2 \phi_{k}(x^{(i)})^2} \\
  &\leq \frac{1}{n^2} \E\brac{\sum_{i\in[n]}\sum_{j\in \alp}\sum_{k\in\agp}b_j^2 b_k ^2 \phi_j(x^{(i)})^2 \phi_{k}(x^{(i)})^2} \\
  &= \frac{\norm{b_{\alp}}_2^2\norm{b_{\agp}}_2^2 }{n}.
\end{align*}
Using Markov's inequality, we obtain
\begin{align*}
    \abs{\frac{1}{n}b_{\alp}\tran V^2\Phi_{\alp}\tran \Phi_{\agp}b_{\agp}} = O_{d,\P}(n\half \log^{\frac{1}{2}}(d))\cdot \norm{b_{\alp}}_2\norm{b_{\agp}}_2.
\end{align*}
Combining this bound with Eq.~\eqref{eq:test_loss_t2_bound2} completes the bound for the second term in Eq.~\eqref{eq:t_22_decomp}:
\begin{align}\label{eq:test_loss_t2_bound5}
    \abs{b_{\alp}\tran V\Lambda_{\alp} \Phi_{\alp}\tran \beta_{\agp}} &= O_{d,\P}(d^{-\zeta/2+\epsilon})\log^{1/2}(d)\cdot \norm{b_{\alp}}_2\norm{b_{\agp}}_2\notag\\
    &\leq  O_{d,\P}(d^{-\zeta/2+\epsilon})\log^{1/2}(d)\cdot \norm{f^*}_{L_2}^2.
\end{align}
For the last term in Eq.~\eqref{eq:t_22_decomp}, taking the expectation of the square over $\vepsilon$ yields
\begin{align}\label{eq:test_loss_t2_bound3}
    \E_{\vepsilon}\brac{\abs{b_{\alp}\tran V \Lambda_{\alp} \Phi_{\alp}\tran \beta_{\varepsilon}}^2} = \sigma_\varepsilon^2\cdot \norm{K_{w,\lambda}\inv \Phi_{\alp} \Lambda_{\alp} V b_{\alp}}_2^2.
\end{align}
Invoking Proposition~\ref{prop:test_loss_tech_2} gives 
\begin{align*}
    \norm{K_{w,\lambda}\inv \Phi_{\alp} \Lambda_{\alp} V b_{\alp}}_2 &= \frac{1}{n}\norm{\Phi_{\alp} V^2b_{\alp}}_2 + \frac{1}{\sqrt{n}}\cdot  O_{d,\P}(d^{-\zeta/2+\epsilon})\cdot\norm{V b_{\alp}}_2\\
    &= O_{d,\P}(1)\cdot \frac{1}{\sqrt{n}} \norm{b_{\alp}}_2,
\end{align*}
where the second equality follows the bound  $\snorm{\Phi_{\alp}} = O_{d,\P}(\sqrt{n})$ (Lemma~\ref{lem:norm_control}) and $\snorm{V} = O_{d}(1)$ (Eq.~\eqref{eq:def_v}).

Plugging the above equation into Eq.~\eqref{eq:test_loss_t2_bound3} and invoking Markov's inequality yields 
\begin{align}\label{eq:test_loss_t2_bound4}
    \abs{b_{\alp}\tran V \Lambda_{\alp} \Phi_{\alp}\tran \beta_{\varepsilon}} = \frac{1}{\sqrt{n}}\cdot O_{d,\P}(1)\cdot \round{\sigma_\varepsilon^2 +\norm{b_{\alp}}_2^2}.
\end{align}
Lastly,  combining Eq.~\eqref{eq:bvb} with the bound for Eq.~\eqref{eq:t_22_decomp} using Eq.~\eqref{eq:test_loss_t2_bound1}, \eqref{eq:test_loss_t2_bound5} and \eqref{eq:test_loss_t2_bound4}, and absorbing the logarithmic factors into $\epsilon$  completes the proof.

\subsubsection{Proof of Proposition~\ref{prop:test_loss_tech_1}}\label{proof:test_loss_tech_1}
By Lemma~\ref{lemma:kernel_to_mono_m}, we have the estimate for $K_{\lm}$:
\begin{align*}
    K_{\lm} = \Phi_{\alp}(Z)D_{\alp} \Phi_{\alp}(Z) + \rho(I+\Delta), 
\end{align*}
with $\snorm{\Delta} = O_{d,\P}(d^{-\zeta_2/2+\epsilon})$.

Using the representation $\Phi_{\alp}(Z) = \Phi_{\alp}(X) W_{\alp}^{1/2}$ and adding the ridge parameter, we can write
\begin{align*}
     K_{\lm,\lambda} = \Phi_{\alp}(X)\hd_{\alp} \Phi_{\alp}(X) + (\rho+\lambda)H,
\end{align*}
where we denote $\hd_{\alp} := D_{\alp}W_{\alp}$ and $H :=I+(\rho+\lambda)\inv \Delta$

For simplicity of notation, we omit the subscript and the argument $X$ when it is clear from the context.
Invoking  Sherman–Morrison–Woodbury formula, we can obtain
\begin{align*}
    K\inv = (\rho+\lambda)\inv H\inv - (\rho+\lambda)\inv H \inv  \Phi((\rho+\lambda) \hd\inv + \Phi\tran H\inv \Phi)\inv \Phi\tran H\inv 
\end{align*}
Next we show $  K\inv\Phi \Lambda$ and  $\Lambda \Phi\tran K\inv$ can be further simplified.  With the formula $I - (A+B)\inv B = (A+B)\inv A$ and $I - B(A+B)\inv = A(A+B)\inv$ we can get the following respectively:
\begin{align}
     K\inv\Phi \Lambda &=  H\inv \Phi ((\rho+\lambda) \hd\inv + \Phi\tran H\inv \Phi)\inv (\hd\inv \Lambda),\label{eq:k_inv_dw}\\
    \Lambda \Phi\tran K\inv &=  (\hd\inv \Lambda) ((\rho+\lambda) \hd\inv + \Phi\tran H\inv \Phi)\inv \Phi\tran  H\inv.\label{eq:dw_k_inv}
\end{align}
Combining these two equations, we obtain the following equalities corresponding to terms in Eq.~\eqref{eq:test_loss_tech_bound_1} and \eqref{eq:test_loss_tech_bound_2}:
\begin{align}
  n K\inv \Phi \Lambda^2 \Phi\tran K\inv &=  \frac{1}{n}H\inv\Phi G\inv \round{\hd\inv \Lambda}^2  G\inv\Phi\tran H\inv, \\
  \Phi\tran K\inv \Phi \Lambda^2 \Phi\tran K\inv\Phi &=   \round{\frac{1}{n}\Phi\tran H\inv\Phi} G\inv \round{\hd\inv \Lambda}^2  G\inv\round{\frac{1}{n}\Phi\tran H\inv\Phi}, 
\end{align}
where $G:= \round{(\rho+\lambda) (n\hd)\inv + \Phi\tran H\inv \Phi/n}$.

Note that any entry of $\hd\inv \Lambda$ takes the form
\begin{align*}
    &~~~\round{\round{g^{|S|}(0)d^{-|S|}+O_d(d^{-|S|-1})}\cdot w^S}\inv \round{g^{(|S|)}(0)d^{-|S|}w^S + O_d(d^{-|S|-2+2\mgamma}w^S)}\\
    &= \round{d^{-|S|}w^S}^{-1} \round{g^{|S|}(0) + O_d(d\inv)}\inv\cdot \round{g^{(|S|)}(0)d^{-|S|}w^S + O_d(d^{-|S|-2+2\mgamma}w^S)}\\
    &= \round{g^{|S|}(0) + O_d(d\inv)}\inv \round{g^{|S|}(0) + O_d(d^{-2+2\mgamma})}\\
    &= 1 + O_{d}(d^{-2+2\mgamma}).
\end{align*}
Then we immediately get
\begin{align}
    \snorm{\hd\inv\Lambda - I} = O_{d}(d^{-2+2\mgamma}).
\end{align}
Together with Eq.~\eqref{eq:e_inv_minus_i_m} and \eqref{eq:g_inv_minus_v_inv} and i.e., $
   \snorm{H\inv - I} = O_{d,\P}(d^{-
    \zeta_2/2+\epsilon})$ and $\snorm{G\inv - V} =O_{d,\P}(d^{-\zeta_1/2+\epsilon} + d^{-\zeta_2/2+\epsilon})$, using triangle inequalities yields
\begin{align*}
    \snorm{\frac{1}{n} H\inv\Phi G\inv\round{\hd\inv \Lambda}^2  G\inv \Phi\tran H\inv  - \frac{1}{n}\Phi V^2\Phi\tran} =  O_{d,\P}(d^{-\zeta_1/2+\epsilon} + d^{-\zeta_2/2+\epsilon} + d^{-2+2\mgamma}),
\end{align*}
where we use the fact that $\snorm{\Phi / \sqrt{n}} = O_{d,\P}(1)$ holds (Lemma~\ref{lem:norm_control}).

Similarly, together with Eq.~\eqref{eq:psi_h_psi_minus_i} and i.e., $
   \snorm{\frac{1}{n}\Phi\tran H\inv \Phi - I} = O_{d,\P}(d^{-\zeta_1/2-
    \zeta_2/2+\epsilon})$ and \eqref{eq:g_inv_minus_v_inv} , using triangle inequalities yields
\begin{align*}
    \snorm{\frac{1}{n^2}\Phi\tran H\inv\Phi G\inv\round{\hd\inv \Lambda}^2  G\inv \Phi\tran H\inv\Phi  - V^2} =  O_{d,\P}(d^{-\zeta_1/2+\epsilon} + d^{-\zeta_2/2+\epsilon} + d^{-2+2\mgamma}),
\end{align*}
which completes the proof.

\subsubsection{Proof of Proposition~\ref{prop:test_loss_tech_2}}\label{proof:test_loss_tech_2}
Similar to the proof of Proposition~\ref{prop:test_loss_tech_1}, we use the estimate for $K_{\lm,\lambda}$ and use Sherman-Morrison-Woodbury formula to compute the inverse. We can obtain the following equality similar to Eq.~(\ref{eq:dw_k_inv}):
\begin{align}\label{eq:dw_k_inv_extra}
  \hd_{\alp} \Phi_{\alp}\tran K_{\lm,\lambda}\inv =  ((\rho+\lambda) \hd\inv + \Phi\tran H\inv \Phi)\inv \Phi\tran  H\inv 
\end{align}
where $H :=I+(\rho+\lambda)\inv \Delta$ and $\Delta$ satisfies $\snorm{\Delta} = O_{d,\P}(d^{-\zeta_2/2+\epsilon})$.

We similarly denote $G:= \round{(\rho+\lambda) (n\hd)\inv + \Phi\tran H\inv \Phi/n}$. Eq.~\eqref{eq:e_inv_minus_i_m} and \eqref{eq:g_inv_minus_v_inv} give $
   \snorm{H\inv - I} = O_{d,\P}(d^{-
    \zeta_2/2+\epsilon})$ and $\snorm{G\inv - V} =O_{d,\P}(d^{-\zeta_1/2+\epsilon} + d^{-\zeta_2/2+\epsilon})$. Therefore, using the triangle inequality we have
\begin{align*}
    \snorm{G\inv \Phi_{\alp}\tran H\inv /\sqrt{n} -V\Phi_{\alp}\tran /\sqrt{n} }=O_{d,\P}(d^{-\zeta_1/2+\epsilon} + d^{-\zeta_2/2+\epsilon}).
\end{align*}
Combining this bound with Eq.~\eqref{eq:dw_k_inv_extra} completes the proof.

%% file: sections/appendix_C.tex
\section{Technical lemmas}

\begin{lemma}\label{lem:angle_app}
Consider independent, mean zero, isotropic random vectors $ x^{(1)},\ldots, x^{(n)}$ in $\R^d$ and suppose that the coordinates of each vector $x^{(i)}$ are independent and $\sigma$-subGaussian.
Fix a constant $s\in (0,1)$. Then with probability at least  $1-2n^2\exp(-cs^2d)$ the estimate 
$$\left|\tfrac{1}{d}\langle  x^{(i)}, x^{(j)}\rangle-\delta_{ij}\right|\leq s\sigma^2,$$
holds simultaneously for all $i,j=1,\ldots n$.
\end{lemma}
\begin{proof}
For any fixed $i$ and $j$ and $s$, Bernstein's inequality \cite[Corollary 2.8.3]{vershynin2018high} gives 
$$\mathbb{P}(|\tfrac{1}{d}\langle  x^{(i)}, x^{(j)}\rangle-\delta_{ij}|\geq s\sigma^2)\leq 2\exp(-cs^2d).$$
Taking a union bound over $i,j=1,\ldots, n$ completes the proof.
\end{proof}

\begin{lemma}\label{lem:angle_app_cube_m}
Consider random vectors $ x^{(1)},\ldots, x^{(n)}$ i.i.d. drawn from the uniform distribution over $\H_d = \{-1,+1\}^d$. Fix a constant $s\in(0,1)$. Given a vector $w\in \R^d$ with $\norm{w}_{1} = d$ and $w_r\geq 0$ for $r = 1,2,\ldots,d$, then with probability at least $1-2n^2\exp(-s^2d/(2\norm{w}_\infty))$ the estimate 
\begin{align}
   \abs{\frac{1}{d}\inner{ x^{(i)},w\odot  x^{(j)}} - \delta_{ij}} \leq s,
\end{align}
holds simultaneously for all $i,j = 1,2,\ldots,n$. 
\end{lemma}
\begin{proof}
For any fixed $i,j$, if $i=j$, then we have
\begin{align*}
    \frac{1}{d}\inner{ x^{(i)},w\odot  x^{(i)}} = \norm{w}_1/d = 1.
\end{align*}
For $i\neq j$ and any fixed $s$, Hoeffding's inequality~\cite[Theorem 2.2.6]{vershynin2018high} gives
\begin{align*}
    \P\round{\abs{\frac{1}{d}\inner{ x^{(i)},w\odot  x^{(j)}} } \geq s}\leq 2\exp\round{-\frac{s^2 d}{2\norm{w}_\infty}},
\end{align*}
where we use the fact that each coordinate of $ x$ satisfies $w_k x_k^{(i)}x_k^{(j)} \in[-w_k,w_k]$ hence $\sum_{k=1}^d w_k^2 \leq \norm{w}_1\norm{w}_\infty = d\norm{w}_\infty$.

Taking a union bound over $i,j=1,\ldots, n$ completes the proof.
\end{proof}

\begin{lemma}\label{lem:norms_app}
Consider independent, mean zero, isotropic random vector $ x \in \R^d$ and suppose that the coordinates of $ x$ are independent and $\sigma$-sub-Gaussian.
Fix a constant $s\in (0,\sigma^{-2})$. Then with probability at least $1-2\exp(-cd s^2)$ the following estimate holds: 
$$\left|\tfrac{1}{d^q}\| x\|^{2q}_2-1 \right|\leq q 2^{2q}s\sigma^2\qquad \forall q>0.$$
\end{lemma}
\begin{proof}
Concentration of the norm \cite[Theorem 3.1.1]{vershynin2018high} yields for any $s>0$ the estimate 
$$\mathbb{P}(|\tfrac{1}{\sqrt{d}}\| x\|_2-1 |\geq s\sigma^2)\leq 2\exp(-cd s^2).$$
Next, we will use the elementary observation: for any $z>0$, $t\in (0,1)$, and $q\geq 1$ the implication holds:
$$|z^q-1|\geq q\cdot 2^{q-1}t \qquad \Longrightarrow\qquad |z-1|\geq t.$$
Therefore, setting $z=\tfrac{1}{\sqrt{d}}\| x\|_2$ and $t=s\sigma^2$ we deduce
\begin{align*}
\mathbb{P}(|z^q-1 |\geq q\cdot 2^{q-1}t)&\leq \mathbb{P}(|z-1 |\geq t)\leq  2\exp(-cd s^2).
\end{align*}
Replacing $q$ by $2q$ completes the proof.
\end{proof}

\begin{lemma}[Matrix decoupling (Lemma 4 from~\cite{mei2022generalization}]\label{lemma:m_decouple}
Let $A \in \R^{n\times n}$ be a real symmetric random matrix. For $T_1,T_2 \subseteq [n]$, we denote $A_{T_1,T_2} = (A_{i,j})_{i\in T_1, j\in T_2}$. Then we have
\begin{align}
    \E\brac{\snorm{A - \Diag(A)}}\leq 4 \max_{T\subseteq[n]} \E\brac{\snorm{A_{T,T^c}}}.
\end{align}  
\end{lemma}

\begin{lemma}[Theorem 5.48 from \cite{vershynin2010introduction}]\label{lemma:m_norm} Let $A \in \R^{N\times n}$ with $A\tran = \round{\rva_1,\cdots,\rva_N}$ where $\rva_i$ are independent random vectors in $\R^n$ with the common second moment matrix $\Sigma = \E\brac{\rva \rva_i\tran}$. Let $m := \E\brac{\max_{i\in[N] }\norm{\rva_i}^2}$. Then there exists an absolute constant $C>0$ such that the following holds
\begin{align}
    \E\brac{\snorm{A}^2}^{1/2} \leq \snorm{\Sigma}^{1/2} n^{1/2} + C m^{1/2}\log^{1/2}(\min(N,n)).
\end{align}
    
\end{lemma}

\begin{lemma}[Uniform bound using hypercontractivity]\label{lemma:uniform_bound}
Let $\phi_1,\ldots,\phi_n$ be finite-degree polynomials with maximum degree $l$ that satisfy the hypercontractivity Assumption~\ref{ass:hypercontr} with respect to measure $\mu_d$. Then the following holds for any $q\geq 1$:
\begin{align*}
    \E\brac{\max_{i\in[n]}|\phi_i|^2} \leq n^{1/q}C_{l,2q}^{1/2} \max_{i\in[n]}\E\brac{|\phi_i|^2},
\end{align*}
where $C_{l,2q}>0$ is some constant.
\end{lemma}
\begin{proof}
Using Lyapunov’s inequality for any $q\geq 1$, we deduce
\begin{align}
    \E\brac{\max_{i\in[n]}|\phi_i|^2} 
    &\leq  \E\brac{\max_{i\in[n]}|\phi_i|^{2q}}^{1/q}\notag\\
    &\overset{(a)}{\leq}   \E\brac{\sum_{i\in[n]}|\phi_i|^{2q}}^{1/q}\notag\\
    &\leq n^{1/q}\max_{i\in[n]}\E\brac{|\phi_i|^{2q}}^{1/q},\label{eq:uniform_bound_eq1}
\end{align}
where $(a)$ follows from bounding the maximum by a sum. The hypercontractivity assumption implies there exists a constant $C_{l,2q}>0$ such that the following holds
\begin{align}\label{eq:uniform_bound_eq2}
    \E\brac{|\phi_i|^{2q}}^{1/(2q)} \leq C_{l,2q} \E\brac{|\phi_i|^{2}}^{1/2}, 
\end{align}
where $l$ is the maximum degree of $\phi_1,\ldots,\phi_n$.
Plugging Eq.~\eqref{eq:uniform_bound_eq2} into Eq.~\eqref{eq:uniform_bound_eq1} yields
\begin{align*}
    \E\brac{\max_{i\in[n]}|\phi_i|^2} \leq n^{1/q}C_{l,2q}^{1/2}\max_{i\in[n]}\E\brac{|\phi_i|^2},
\end{align*}
which completes the proof.
\end{proof}

\begin{lemma}[Expectation of  Fourier feature matrices product (Theorem  2 from~\cite{zhu2025spectral})]\label{lemma:feature_product_tech}
Suppose there exist non-overlapping sets $\T_1,\ldots, \T_\ell \subset [d]$ that satisfy $|\T_k| = \Theta(d^{s_k})$ for all $k\in[\ell]$ where $s_i\in[0,1]$. Fix the set families $ \S_1^{[k]}, \ldots  \S_{m+1}^{[k]}\in 2^{\T_k}$ for $k\in[\ell]$.
Consider collections of sets $\S_1,\ldots,\S_{m+1} \in 2^{[d]}$ and weights 
$w^{(i)}\in \R^{\S_i}_{+}$, where $S_i = \sqcup_{k=1}^\ell S_{i}^{[k]}$ for all $S_i\in \S_i$.  Define the matrix product
$$M= A_1^\top \Bigl(A_2 A_2^\top \Bigr) \cdots \Bigl(A_{m} A_{m}^\top \Bigr) A_{m+1},$$
where $A_i=X_{\mathcal{S}_i} \Diag(w^{(i)})$ are the scaled Fourier-Walsh matrices. Define the adjusted effective degree for $A_i$: $$p_i^{\rm aeff} := \sum_{k=1}^\ell s_k \cdot p_i^{[k]}.$$ 
Assume  that the following regularity conditions hold for all indices $i,j\in [m+1]$:
\begin{enumerate}
    \item\label{it:deg_bound_new} {\bf (degree bound)} The inequalities $| S^{[k]}_i|\leq  p_i^{[k]}$ hold for all $ S_i^{[k]}\in  \S_i^{[k]}$,
    \item\label{it:triv_intersec_new} {\bf (trivial intersection)} Whenever $\S_i$ intersects $\S_j$, the equality $\S_i=\S_j$ holds,
    \item\label{it:small_weights_new} {\bf (small weights)} The weights satisfy 
    $w^{(i)}= O_d(n^{-1/2}\wedge d^{-p_i^{\rm aeff}/2})$ for all $i\in [m+1]$.
\end{enumerate}
    Then the estimate
    \begin{equation}\label{eqn:labeled_hard_mat_new}
    \left\|\E M \right\|_{\rm op}
    = O_d \left(d^{p_j^{\rm aeff}}\right)\|w^{(1)}\|_\infty \|w^{(m+1)}\|_\infty,
    \end{equation}
    holds for any index $j \in [m+1]$ such that $\S_{j}$ is distinct from $\S_{1}$ and $\S_{{m+1}}$.

\end{lemma}